\documentclass[11pt]{article}
\usepackage[utf8]{inputenc}
\usepackage{float}
\usepackage{stfloats}
\usepackage{titling}
\usepackage[flushleft]{threeparttable}
\usepackage{amssymb,amsthm,amsmath,amsfonts,romannum, dsfont}
\usepackage{setspace}
\usepackage[ruled,vlined]{algorithm2e}
\usepackage{mathrsfs,graphicx}

\usepackage{xcolor}
\usepackage{standalone}
\usepackage{tikz}
\usetikzlibrary{fit,positioning,shadows.blur}
\usepackage{enumerate,enumitem,booktabs,multirow}  
\usepackage{easyReview}
\usepackage{subcaption}
\usepackage{authblk}
\usepackage[left=1.25in, right=1.25in, top=1.25in, bottom=1.25in]{geometry}
\usepackage[authoryear,sort]{natbib}
\usepackage{tcolorbox}
\usepackage{longtable}
\usepackage{appendix}
\allowdisplaybreaks

\usepackage{mathtools}
\usepackage{hyperref}
\hypersetup{
	colorlinks=true,
	linkcolor=black,
	urlcolor=blue,
	citecolor=blue
}

\newcommand{\RR}{\mathbb{R}}
\newcommand{\EE}{\mathbb{E}}
\newcommand{\PP}{\mathbb{P}}
\newcommand{\one}{\mathds{1}}

\newcommand{\cX}{\mathcal{X}}
\newcommand{\cA}{\mathcal{A}}
\newcommand{\cB}{\mathcal{B}}
\newcommand{\cC}{\mathcal{C}}
\newcommand{\cD}{\mathcal{D}}
\newcommand{\cE}{\mathcal{E}}
\newcommand{\cF}{\mathcal{F}}
\newcommand{\cG}{\mathcal{G}}
\newcommand{\cH}{\mathcal{H}}
\newcommand{\cL}{\mathcal{L}}
\newcommand{\cM}{\mathcal{M}}
\newcommand{\cS}{\mathcal{S}}
\newcommand{\cR}{\mathcal{R}}
\newcommand{\cT}{\mathcal{T}}
\newcommand{\cV}{\mathcal{V}}

\newcommand{\fA}{\mathfrak{A}}
\newcommand{\fL}{\mathfrak{L}}

\newcommand{\simplex}[1]{\Delta^{#1-1}}

\newcommand{\MR}{M_R}
\newcommand{\MS}{M_S}
\newcommand{\Mzero}{M_0}

\newcommand{\soft}{\mathrm{soft}}
\newcommand{\lin}{\mathrm{lin}}
\newcommand{\qdr}{\mathrm{quad}}
\newcommand{\loc}{\mathrm{local}}
\newcommand{\glob}{\mathrm{global}}
\newcommand{\TopK}{\mathrm{Top}\text{-}K}
\newcommand{\sparse}{\mathrm{sparse}}
\newcommand{\MSE}{\mathrm{MSE}}

\DeclareMathOperator*{\argmin}{arg\,min}

\DeclareMathOperator{\Unif}{Unif}

\numberwithin{equation}{section}

\theoremstyle{plain}

\newtheorem{theorem}{Theorem}[section]
\newtheorem{proposition}[theorem]{Proposition}

\newtheorem{lemma}[theorem]{Lemma}
\newtheorem{corollary}[theorem]{Corollary}

\theoremstyle{definition}

\newtheorem{definition}[theorem]{Definition}
\newtheorem{assumption}[theorem]{Assumption}
\renewcommand\theassumption{\arabic{assumption}}
\newtheorem{example}[theorem]{Example}

\theoremstyle{remark}

\newtheorem{remark}[theorem]{Remark}

\begin{document}
	\makeatletter
	\let\arxiv@originallabel\label
	\let\arxiv@originalmathlabel\ltx@label
	\renewcommand{\label}[1]{\arxiv@originallabel{#1}\arxiv@originallabel{main-#1}\def\arxiv@currentlabelname{#1}\def\arxiv@selectionlabelname{cor:selection(regionwise)}\ifx\arxiv@currentlabelname\arxiv@selectionlabelname\arxiv@originallabel{main-cor:selection(region-wise)}\fi}
	\renewcommand{\ltx@label}[1]{\arxiv@originalmathlabel{#1}\arxiv@originalmathlabel{main-#1}}
	\let\arxiv@originaladdtocontents\addtocontents
	\renewcommand{\addtocontents}[2]{\def\arxiv@targetfile{#1}\def\arxiv@tocfile{toc}\ifx\arxiv@targetfile\arxiv@tocfile\else\arxiv@originaladdtocontents{#1}{#2}\fi}
	\makeatother
	\pagenumbering{arabic}
	
	\def\spacingset#1{\renewcommand{\baselinestretch}{#1}\small\normalsize} \spacingset{1}

	{
		\title{\bf Towards a Statistical Understanding of Mixture-of-Experts}
		
		\author[a]{Siyuan He}
		\author[a]{Bokai Yang\thanks{The first two authors contributed equally to this work.}}
		\author[b]{Jie Hu}
		\author[c]{Ziwen Gao}
		\author[b,d]{Yuhong Yang\thanks{Corresponding author: \href{mailto:yyangsc@mail.tsinghua.edu.cn}{yyangsc@mail.tsinghua.edu.cn}}}
		\affil[a]{Qiuzhen College, Tsinghua University, Beijing, China.}
		\affil[b]{Yau Mathematical Sciences Center, Tsinghua University, Beijing, China.}
		\affil[c]{KLATASDS-MOE, School of Statistics, East China Normal University, Shanghai, China.}
		\affil[d]{Beijing Institute of Mathematical Sciences and Applications (BIMSA), Beijing, China.}
		\maketitle
	}
	
	\bigskip
	\begin{abstract}
		Mixture-of-experts (MoE) architectures increase model capacity by combining a collection of expert predictors through input-dependent routing, while often activating only a small subset of experts for each input. Despite their growing importance in modern large-scale models, the statistical roles of their design choices, especially routing, sparse activation, and shared experts, remain only partially understood, as existing theory has largely focused on parametric or correctly specified MoE models. In this paper, we view MoE as a form of localized aggregation and show how this localization reshapes the approximation-estimation-computation tradeoff. We derive oracle risk bounds for learning dense and sparse routing with evolving experts, separating approximation, expert-learning, and router-estimation errors, and characterize how sparse Top-$K$ routing can retain the benefits of localized aggregation while controlling per-input computation. We also interpret gating through the geometry of input space, relating routing performance to regions of local expert advantage, and show how shared experts, as adopted in architectures such as DeepSeekMoE, can extract common predictive structure so that routed experts focus on residual local variation. Together, these results provide a unified statistical framework for understanding MoE through input-dependent expert aggregation, in which expert specialization and computational tradeoffs are governed by local predictive structure. 
	\end{abstract}
	
	\vfill
	\newpage
	\spacingset{1.45}	
	\section{Introduction}
	Modern artificial intelligence systems increasingly rely on architectures whose total capacity greatly exceeds the computation activated for any individual input \citep{vaswani2017attention,openai2024gpt4,brown2020fewshot,rombach2022image}. A prominent example is the mixture-of-experts (MoE) architecture, in which a router assigns each input to a subset of expert networks and combines their outputs in an input-dependent manner. In transformer-based systems, MoE layers are often used as replacements for feed-forward blocks, as in GShard, Switch Transformer, GLaM, and Mixtral \citep{shazeer2017outrageously,lepikhin2020gshard,fedus2022switch,du2022glam,jiang2024mixtral}. These models use sparse expert activation to increase model capacity while keeping per-token computation manageable. More recent designs, such as DeepSeekMoE \citep{dai2024deepseekmoe}, also include shared experts that are always active and are intended to capture common knowledge across inputs. These developments suggest that the statistical roles of routing, sparsity, and shared components are central to understanding modern MoE systems. 
    
    This paper studies MoE architectures from the viewpoint of covariate-dependent aggregation. In this view, the expert networks provide a collection of candidate predictors, while the router determines how these predictors are locally combined. Dense softmax gating corresponds to smooth input-dependent averaging over all experts, while sparse Top-$K$ routing performs localized selection and encourages expert specialization. Shared experts add a global component that is reused across the input space. Together, these mechanisms provide different ways to balance global structure, local specialization, and computational constraints. 
    
    The statistical study of MoE models has a long history. Classical works introduced MoE and hierarchical MoE architectures together with likelihood-based and EM-type estimation procedures \citep{jacobs1991adaptive,jordan1994hierarchical}. Subsequent theory studied identifiability, convergence and asymptotic inference, often under parametric or correctly specified formulations \citep{jordan1995convergence,jiang1999identifiability,olteanu2011asymptotic}. More recent work has analyzed estimation rates for modern variants of MoE models, including dense softmax gating, sparse routing and architectures with shared experts \citep{nguyen2023statistical,nguyen2024least,nguyen2025deepseekmoe}. On the approximation side, classical and recent results show that softmax-gated MoE classes can have strong universal approximation properties \citep{pmlr-vR2-jiang99a,nguyen2016universal}.
	
	Despite these advances, there remains a gap between existing theory and the structural questions raised by modern MoE architectures. Much of the available theory treats the target function as belonging to a finite-dimensional MoE class and focuses on estimation of all model parameters. This perspective is less directly suited to large-scale systems, where the experts and router are highly overparameterized, the model is typically misspecified, and end-to-end optimization is non-convex. Moreover, existing results do not fully separate the statistical roles of dense averaging, sparse routing, adaptive partitioning and shared experts. 
	In particular, it remains unclear how these architectural choices shape approximation, estimation, and computation outside a fully specified joint parametric MoE model.
	
	Recent empirical MoE studies make these structural questions particularly salient. Large-scale sparse MoE models increase the total number of experts while activating only a small subset for each input \citep{shazeer2017outrageously,lepikhin2020gshard,fedus2022switch,du2022glam,jiang2024mixtral}. More recent architectures further refine this idea by using fine-grained routed experts, multiple active experts, or shared experts \citep{dai2024deepseekmoe,ludziejewski2024scaling,yun2024toward,he2024mixture}. These design choices are often motivated by empirical efficiency and scalability considerations. They also raise an important statistical question: how do the size and composition of the expert pool, the sparsity and flexibility of the routing mechanism, and the inclusion of shared components jointly shape prediction accuracy, statistical complexity, and computational cost?
	
	Addressing this question is challenging because joint optimization of experts and routers in large-scale MoE models is highly coupled and non-convex, making it difficult to disentangle the statistical contributions of different architectural components. We therefore base our analysis on a predictable expert-training trajectory, which provides a tractable way to study router learning without directly characterizing the full end-to-end training dynamics. Conditional on the past data, the current experts enter the next-step prediction problem as plug-in components, while their quality is allowed to evolve over time. Under this formulation, we develop a unified statistical framework that treats MoE models as flexible aggregation rules built from a collection of expert predictors and a gating mechanism. This framework then provides a common ground for comparing dense softmax routing, sparse Top-$K$ routing, adaptive gating classes and shared-expert architectures in terms of their approximation power, statistical complexity and computational tradeoffs.

    Our first contribution is to clarify the approximation role of routing. We show that input-dependent routing turns a collection of experts into a localized aggregation class, which can be substantially more expressive than global averaging. Dense and sparse routing differ not only computationally but also statistically: dense softmax gating provides smooth averaging, while sparse Top-$K$ gating induces localized expert selection. We characterize how these mechanisms affect approximation ability and complexity, and show when localization, through either dense or Top-$K$ aggregation, can improve upon global averaging.
    
    Our second contribution is to develop oracle risk bounds for estimating the gating mechanism along an evolving expert trajectory. We use a discretization-based scheme \citep{yang2001adaptive,yang2004Combining} that avoids assuming global optimization of a highly non-convex MoE objective. The resulting bounds separate the oracle approximation risk, the cumulative learning error of the evolving experts, and the statistical cost of learning the router. The framework applies to both dense and sparse routing and highlights the estimation price of routing flexibility. 
    
    Our third contribution is to interpret gating as a problem of input-space partitioning. Different routing functions induce different soft or hard partitions of the input space, and their performance depends on whether these partitions align with the regions on which different experts perform well. Linear, quadratic and kernel-based gating functions therefore represent different levels of geometric flexibility. This perspective explains why a single routing class may be effective for some target structures but inadequate for others, and motivates data-adaptive selection among gating mechanisms.
    
    Our final contribution is to study the statistical role of shared experts. We show that shared experts are not merely additional predictors in the ensemble. Their main role is to extract common structure that would otherwise have to be repeatedly learned by routed experts. Once this common component is removed, the routed experts only need to learn residual local structure, which can reduce the effective complexity of the routed learning problem. This provides a statistical interpretation of recent MoE designs that combine shared and routed components. 
    
    The rest of the paper is organized as follows. Section \ref{subsec:formulation} introduces the mathematical formulation of dense routing, sparse routing and shared experts. Section \ref{sec:local_averaging} studies the approximation properties of MoE predictors as localized aggregation rules. Sections \ref{sec:oracle_softmax} and \ref{sec:oracle_topK} develop risk bounds for dense and Top-$K$ gating estimation. Section \ref{sec:adaptation} studies gating as input-space partitioning and discusses adaptive selection of routing classes. Section \ref{sec:shared} analyzes the role of shared experts. Technical proofs and additional discussions are provided in the supplementary material. 
    
    \textbf{Notation}. Throughout the paper, for two nonnegative sequences $a_n$ and $b_n$, we write $a_n \lesssim b_n$ if there exists a constant $C > 0$, independent of $n$, such that $a_n \le C b_n$, and write $a_n \asymp b_n$ if both $a_n \lesssim b_n$ and $b_n \lesssim a_n$ hold. For two real numbers $a$ and $b$, we write $a \vee b = \max\{a, b\}$ and $a \wedge b = \min\{a, b\}$. For a positive integer $M$, $[M] = \{1, \ldots, M\}$. For a finite set $\cS$, $|\cS|$ denotes its cardinality. The set $\simplex{M} = \{w \in \RR^M : w_j \ge 0,\ \sum_{j=1}^M w_j = 1\}$ denotes the probability simplex in $\RR^M$. For a set $\cX$, $\one_{\cX}(x)$ denotes the indicator function taking the value one when $x \in \cX$ and zero otherwise. For two vectors $a$ and $b$, $\langle a, b\rangle = a^\top b$ denotes their inner product. Unless otherwise specified, $\|\cdot\|$ denotes the Euclidean norm for vectors and the spectral norm for matrices. For a differentiable function $f$, $\nabla f$ denotes its gradient. For a measurable function $f=(f_1,\ldots,f_M)^\top:\cX\to\RR^M$, with the scalar case corresponding to $M=1$, define $\|f\|_{L_2(P_X)} = \left(\int_{\cX}\sum_{m=1}^M f_m^2(x)\,dP_X(x)\right)^{1/2}$. For a scalar-valued $f$, we also write $\|f\|_{\infty} = \sup_{x \in \cX} |f(x)|$. Finally, for $x \in \RR$, $\lfloor x \rfloor$ denotes the largest integer not exceeding $x$, and $\lceil x \rceil$ denotes the smallest integer not smaller than $x$.

    \subsection{Mathematical formulation of MoE architecture}\label{subsec:formulation}
    
    We begin with a routed-only MoE layer, which serves as the baseline formulation before shared experts are introduced. Let $\cX \subset \RR^d$ be a compact input space and let $x \in \cX$ denote an input. The layer contains $\MR$ routed experts. For $m \in [\MR]$, we write $f_{R,m}(x;W_{R,m})$ for the output of the $m$-th routed expert, where $W_{R,m}$ denotes its network parameter. In transformer implementations, such experts are typically feed-forward networks, although the statistical analysis below treats them only as candidate predictors.
    
    A router maps the input $x$ to a vector of gating weights
    \begin{align*}
    	g(x;\theta) = (g_1(x;\theta), \ldots, g_{\MR}(x;\theta))^\top \in \simplex{\MR},
    \end{align*}
    where $\theta$ is the gating parameter. The output of the conventional MoE \citep{jordan1994hierarchical} is then
    \begin{align}
    	F_{\mathrm{MoE}}(x;W_R,\theta) = \sum_{m=1}^{\MR} g_m(x;\theta) f_{R,m}(x;W_{R,m}), \label{eq:formulation_MoE}
    \end{align}
    where $W_R = (W_{R,1}, \ldots, W_{R,\MR})$. Throughout the paper, we use \textbf{router} for the architectural module that assigns inputs to experts, and \textbf{gating function} for its mathematical representation $x \mapsto g(x;\theta)$. Experts weighted or selected by the router are called \textbf{routed experts}. Experts that are always active, and hence do not depend on the router, are called \textbf{shared experts}.
    
    Many MoE gating functions are constructed by first computing unnormalized routing scores and then normalizing them. Let $s_m(x;\theta_m)$ be the score assigned to the $m$-th routed expert, where $\theta_m$ is the expert-specific score parameter. We also write $\theta = (\theta_1^\top, \ldots, \theta_{\MR}^\top)^\top$. We use two score classes as running examples. The first is the \textbf{linear score},
    \begin{align*}
    	s_m^{\lin}(x;\theta_m) = \beta_m^\top x + \alpha_m, \qquad \theta_m = (\beta_m^\top, \alpha_m)^\top \in \RR^{d+1},
    \end{align*}
    which induces routing regions separated by affine decision boundaries. The second is the \textbf{quadratic score},
    \begin{align*}
    	s_m^{\qdr}(x;\theta_m) = x^\top B_m x + \beta_m^\top x + \alpha_m, \qquad B_m = B_m^\top,
    \end{align*}
    where $\theta_m = (\mathrm{vech}(B_m)^\top, \beta_m^\top, \alpha_m)^\top$, and $\mathrm{vech}(B_m)$ vectorizes the upper triangular part of the symmetric matrix $B_m$. The quadratic score allows curved decision boundaries. These two score classes will later be used to illustrate how gating geometry affects the input-space partition induced by the router.
    
    Given scores $s_1, \ldots, s_{\MR}$, dense softmax gating assigns positive weight to every routed expert:
    \begin{align}
    	g_m^{\soft}(x;\theta) = \frac{\exp\{s_m(x;\theta_m)\}}{\sum_{j=1}^{\MR} \exp\{s_j(x;\theta_j)\}}, \qquad m = 1, \ldots, \MR. \label{eq:formulation_gating_dense}
    \end{align}
    Sparse Top-$K$ gating instead activates only the $K$ experts with the largest scores, where $1 \le K \le \MR$. Let $\cT_K(x;\theta)$ denote the resulting set of exactly $K$ selected indices. The Top-$K$ softmax gate is
    \begin{align}
    	g_m^{(K)}(x;\theta) = \frac{\exp\{s_m(x;\theta_m)\} \one\{m \in \cT_K(x;\theta)\}}{\sum_{j \in \cT_K(x;\theta)} \exp\{s_j(x;\theta_j)\}}, \qquad m = 1, \ldots, \MR. \label{eq:formulation_gating_sparse}
    \end{align}
    Thus dense softmax gating and Top-$K$ gating represent two different forms of covariate-dependent aggregation. Dense gating gives smooth averaging over the full expert collection, whereas Top-$K$ gating gives a piecewise-defined rule whose active expert set changes across the input space. This partition viewpoint will be central to the approximation and estimation analyses below.
    
    We next incorporate shared experts. Motivated by recent MoE architectures such as DeepSeekMoE \citep{dai2024deepseekmoe}, we consider a layer with $\MS$ shared experts and $\MR$ routed experts. The shared experts are always active, while the gating function is applied only to the routed experts. For input $x$, let $\{f_{S,\ell}(x;W_{S,\ell})\}_{\ell=1}^{\MS}$ denote the shared expert outputs, and let $\{f_{R,m}(x;W_{R,m})\}_{m=1}^{\MR}$ denote the routed expert outputs. The shared-routed MoE layer is
    \begin{align}
    	F_{\mathrm{shared}}(x;W_S,W_R,\theta) = \sum_{\ell=1}^{\MS} f_{S,\ell}(x;W_{S,\ell}) + \sum_{m=1}^{\MR} g_m(x;\theta) f_{R,m}(x;W_{R,m}), \label{eq:formulation_MoE_shared}
    \end{align}
    where $W_S = (W_{S,1}, \ldots, W_{S,\MS})$. Combining this formulation with dense softmax or Top-$K$ gating gives dense or sparse routing with shared experts. This decomposition separates an always-active shared component from an input-dependent routed component. Issues of identifiability, estimation and the statistical benefits of this separation are discussed in later sections.
    
	\section{Localized averaging in MoE}\label{sec:local_averaging}
	Building on the formulation in Section \ref{subsec:formulation}, we now examine the statistical role of input-dependent aggregation in MoE models. The key distinction is between global averaging, where the aggregation weights are fixed across the input space, and localized averaging, where the weights are generated by a gating function and therefore vary with the covariate value $x$. In modern MoE terminology, this localization is implemented by the router. Thus, the router determines not only which experts are used, but also how the prediction rule varies across the input space. 
	
	Existing approximation results have established that input-dependent weighting generates rich approximation classes even when the experts are constants \citep{pmlr-vR2-jiang99a, nguyen2016universal,zeevi1998error,gao2026combining}. We take this constant expert setting as a starting point and first ask whether sparse Top-$K$ routing retains universal approximation when only $K$ experts are activated for each input. We then consider a fixed number of experts to characterize more explicitly the approximation ability and complexity induced by localized aggregation, before moving beyond constant experts to examine whether its expressive advantage persists for richer expert classes. Throughout this section, we consider routed experts without shared experts, write $M=\MR$ and take $\mathcal{X} = [0,1]^d$.

    \subsection{Constant experts: global versus localized averaging}
    Suppose that every expert belongs to the class of constant functions $\cF_{\mathrm c} := \{f:\cX \to \RR:f(x)\equiv c,\ c\in\RR\}$. The corresponding global averaging class is
    \begin{align*}
    	\cA_{\glob} := \left\{\sum_{m=1}^M w_m f_m \;\middle|\; M\in\mathbb{N}_+,\ f_m\in\cF_{\mathrm c},\ w_m\geq 0,\ \sum_{m=1}^M w_m=1\right\}.
    \end{align*}
    Since convex combinations of constant functions remain constant, $\cA_{\glob}=\cF_{\mathrm c}$. In contrast, for softmax routing \eqref{eq:formulation_gating_dense} with linear scores, define
    \begin{align*}
    	\cA_{\loc}^{\soft} := \left\{\sum_{m=1}^M g_m^{\soft}(x;\theta)f_m \;\middle|\; M\in\mathbb{N}_+,\ f_m\in\cF_{\mathrm c},\ \theta_m=(\beta_m^\top,\alpha_m)^\top\in\RR^{d+1},\ \theta=(\theta_1^\top,\dots,\theta_M^\top)^\top\right\}.
    \end{align*}
    Although the experts are constant, input-dependent softmax weighting induces a function class substantially richer than that generated by global averaging. The following existing result shows that this class can approximate any continuous function arbitrarily well on a compact input space. 
    
    \begin{lemma}[Theorem~7 in \citep{nguyen2016universal}]\label{lem:approximation of softmax}
		For every compact $\cX$, $\cA_{\loc}^{\soft}$ is dense in $C(\cX)$ with respect to the uniform norm, where $C(\cX)$ denotes the class of all continuous real-valued functions on $\cX$.
    \end{lemma}
    
    In many modern large-scale MoE architectures, however, only the Top-$K$ experts are activated at each input in order to control per-input computation. This widespread architectural choice motivates a corresponding approximation question: does activating only $K$ of the $M$ routed experts at each input fundamentally reduce the expressive power relative to dense softmax routing?
    
    To study this question, we consider bounded constant experts and restrict the linear score parameters to a fixed box. Specifically, let
    \begin{align*}
    	\Theta_{\lin}(C_1) := \left\{\vartheta=(\beta^\top,\alpha)^\top\in\RR^{d+1}:\|\vartheta\|_\infty\leq C_1\right\}, \qquad C_1>0.
    \end{align*}
   The bounded constant-expert class is defined by
    \begin{align*}
    	\cF_{\mathrm c}^R := \{f:\cX\to\RR:f(x)\equiv c,\ |c|\leq R\}.
    \end{align*}
    For $1\leq K\leq M$, define the Top-$K$ localized averaging class with $M$ routed experts by
    \begin{align*}
    	\cA_{\loc}^{\TopK}(M) := \left\{\sum_{m=1}^M g_m^{(K)}(x;\theta)f_m(x) \;\middle|\; f_m\in\cF_{\mathrm c}^R,\ \theta_m\in\Theta_{\lin}(C_1),\ \theta=(\theta_1^\top,\ldots,\theta_M^\top)^\top\right\},
    \end{align*}
    where $g_m^{(K)}$ is the Top-$K$ softmax gate generated by the linear score $s_m^{\lin}$. The full Top-$K$ localized averaging class is
    \begin{align*}
    	\cA_{\loc}^{\TopK} := \bigcup_{M\geq K}\cA_{\loc}^{\TopK}(M).
    \end{align*}
    For comparison, the corresponding bounded global averaging class is
    \begin{align*}
    	\cA_{\glob}^R := \left\{\sum_{m=1}^M w_m f_m \;\middle|\; M\in\mathbb{N}_+,\ f_m\in\cF_{\mathrm c}^R,\ w_m\geq 0,\ \sum_{m=1}^M w_m=1\right\} = \cF_{\mathrm c}^R.
    \end{align*}
    
    The next proposition shows that fixing the per-input activation level $K$ does not destroy universal approximation as the total number of routed experts increases: for every fixed $K$, the union over $M \ge K$ remains sufficiently rich to approximate any bounded continuous target function. 
    
    \begin{proposition}\label{prop:topk_density}
		Let $\cX = [0,1]^d$ and fix $K\in\mathbb{N}_+$. Then $\cA_{\loc}^{\TopK}$ is dense in
		\begin{align*}
			C_R(\cX) := \{h\in C(\cX):\|h\|_\infty\leq R\}
		\end{align*}
		with respect to the uniform norm, where $C(\cX)$ denotes the class of all continuous real-valued functions on $\cX$.
    \end{proposition}
    
    Proposition~\ref{prop:topk_density} shows that the expressive advantage of input-dependent routing over global averaging persists under Top-$K$ activation: Top-$K$ localized averaging of bounded constant experts remains dense in $C_R(\cX)$. For every fixed $K$, increasing $M$ allows the router to generate increasingly fine local regions even though only $K$ experts are active at each input and the score parameters remain in a fixed box. Thus, Top-$K$ sparsity controls per-input computation without sacrificing universal approximation, provided that sufficiently many routed experts are available.
	
	\subsection{Approximation and metric entropy with a fixed number of experts}
	We next consider the regime in which the number of routed experts $M$ is fixed, complementing the preceding analysis that allows $M$ to grow. While increasing $M$ allows sparse Top-$K$ routing to retain universal approximation, fixing $M$ limits the number of effective local regions that the router can generate, leading to a more explicit approximation-complexity tradeoff.
	
	Throughout this subsection, $C_1$ and $R$ are fixed as above. We compare the Top-$K$ localized averaging class $\cA_{\loc}^{\TopK}(M)$ with the bounded global averaging class $\cA_{\glob}^R$ over the Lipschitz ball
	\begin{align*}
		\cF_b^{R,L} := \left\{h:\cX \to \RR \;\middle|\; \|h\|_\infty \le R,\ |h(x)-h(y)| \le L\|x-y\|_2,\ \forall x,y \in \cX\right\}.
	\end{align*}
	The following proposition quantifies the approximation gain from localized routing with a fixed number of experts.
	
	\begin{proposition}\label{prop:approximation_ability}
		Let $d \ge 1$ and fix $M,K \in \mathbb{N}_+$ with $1 \le K \le M$. Then
		\begin{align*}
			\sup_{h \in \cF_b^{R,L}} \inf_{f \in \cA_{\glob}^R} \|f-h\|_\infty = \min\left\{R,\frac{L\sqrt{d}}{2}\right\},
		\end{align*}
		and
		\begin{align*}
			\sup_{h \in \cF_b^{R,L}} \inf_{f \in \cA_{\loc}^{\TopK}(M)} \|f-h\|_\infty \le \frac{L\sqrt{d}}{2\lfloor(M/K)^{1/d}\rfloor}.
		\end{align*}
	\end{proposition}
	Proposition~\ref{prop:approximation_ability}  shows that Top-$K$ localized averaging can improve worst-case approximation over global averaging even with a fixed number of experts. Global averaging cannot adapt to local variation and therefore incurs a non-vanishing approximation error for general Lipschitz functions. By contrast, for fixed $M$, the Top-$K$ approximation bound improves as $K$ decreases, showing that greater sparsity yields a sharper worst-case approximation guarantee.
	
	The preceding result concerns approximation accuracy. To understand the statistical price of this increased flexibility, we next examine metric entropy under the uniform norm. Similar to $\cA_{\loc}^{\TopK}(M)$, we define the localized averaging class with a fixed number of experts for dense softmax gating by
	\begin{align*}
		\cA_{\loc}^{\soft}(M) := \left\{\sum_{m=1}^M g_m^{\soft}(x;\theta) f_m(x) \;\middle|\; f_m \in \cF_{\mathrm c}^R,\ \theta_m \in \Theta_{\lin}(C_1),\ \theta = (\theta_1^\top,\ldots,\theta_M^\top)^\top\right\},
	\end{align*}
	where $g_m^{\soft}$ is generated by the linear score $s_m^{\lin}$. The bounded parameter box $\Theta_{\lin}(C_1)$ in this definition is introduced only for the metric entropy comparison below: it makes the covering numbers finite and directly comparable across global, Top-$K$, and dense softmax classes. For the other analyses of dense softmax gating in this section, this bounded-parameter restriction is not needed.
	
	For a function class $\cF$, let $N(\epsilon,\cF,\|\cdot\|_\infty)$ denote its $\epsilon$-covering number with respect to the uniform norm. The following theorem compares three levels of flexibility: the global constant class, the dense softmax class, and the sparse Top-$K$ class.
	
	\begin{theorem}\label{thm:metric_entropy}
		Let $\cX = [0,1]^d$ and fix $M \in \mathbb{N}_+$ and $R,C_1 > 0$. Then the following hold.
		\setlength{\abovedisplayskip}{4pt}
		\setlength{\belowdisplayskip}{4pt}
		\setlength{\abovedisplayshortskip}{4pt}
		\setlength{\belowdisplayshortskip}{4pt}
		\begin{enumerate}[itemsep=0pt, parsep=0pt, topsep=0.25em, partopsep=0pt]
			\item[\textnormal{(i)}] \textbf{Global constant class.}
			As $\epsilon\downarrow0$,
			\begin{align*}
				\log N\left(\epsilon,\cA_{\glob}^R,\|\cdot\|_\infty\right) = \log(1/\epsilon)+O(1).
			\end{align*}
			\item[\textnormal{(ii)}] \textbf{Dense softmax class.}
			If $M=1$, the result in part~\textnormal{(i)} applies. If $M\ge2$, then as $\epsilon\downarrow0$,
			\begin{align*}
				(d+1)\log(1/\epsilon) + O(1) \le \log N\left(\epsilon,\cA_{\loc}^{\soft}(M),\|\cdot\|_\infty\right) \le [M+(M-1)(d+1)]\log(1/\epsilon) + O(1).
			\end{align*}
			
			\item[\textnormal{(iii)}] \textbf{Sparse Top-$K$ class.}
			If $1 \le K < M$, then for any $0 < \epsilon \le R/K$,
			\begin{align*}
				\log N\left(\epsilon,\cA_{\loc}^{\TopK}(M),\|\cdot\|_\infty\right) = \infty.
			\end{align*}
		\end{enumerate}
	\end{theorem}
	
	Theorem~\ref{thm:metric_entropy} shows that, for fixed $M$, input-dependent routing can substantially enlarge the complexity of the resulting prediction class, with sparse Top-$K$ routing permitting particularly rich local variation through changes in the active expert set. For dense softmax routing, the larger logarithmic entropy coefficient, relative to global averaging, reflects the additional finite-dimensional complexity introduced by input-dependent weighting. By contrast, Top-$K$ routing can change the selected expert set discontinuously across routing boundaries, which leads to infinite covering numbers at sufficiently small scales. This infinite entropy is specific to the uniform norm. Under $L_2(P_X)$, boundary effects can be controlled by regularity or margin conditions on $P_X$, a distinction that becomes important in the estimation analysis below. 
	
	\subsection{Beyond constant experts}
	The preceding subsections used constant experts to isolate the expressive contribution of the router. This simplification shows that input dependence in the gating weights alone can generate nontrivial approximation power. Modern MoE architectures, however, use highly expressive expert networks. It is therefore natural to ask whether the advantage of localized routing persists beyond constant experts, or whether it can be reproduced by sufficiently rich input-independent aggregation.
	
	To address this question, let $\cE$ be a generic expert class. Define the global aggregation class
	\begin{align*}
		\cA_{\glob}(\cE;M) := \left\{\sum_{m=1}^M w_m e_m \;\middle|\; e_m \in \cE,\ w_m \ge 0,\ \sum_{m=1}^M w_m = 1\right\},
	\end{align*}
	and the localized softmax aggregation class
	\begin{align*}
		\cA_{\loc}^{\soft}(\cE;M) := \left\{\sum_{m=1}^M g_m^{\soft}(x;\theta) e_m(x) \;\middle|\; e_m \in \cE,\ \theta_m = (\beta_m^\top,\alpha_m)^\top \in \RR^{d+1},\ \theta = (\theta_1^\top,\ldots,\theta_M^\top)^\top\right\},
	\end{align*}
	where $g_m^{\soft}$ is generated by the linear score $s_m^{\lin}$. The next example shows that localization can create functions outside the finite global aggregation class even when the experts are already nonlinear.
	
	\begin{example}\label{exa:complex_experts}
		Let $\cX = [0,1]^d$ and let $\sigma(t) := \frac{1}{1+\exp(-t)}$. Consider the expert class
		\begin{align*}
			\cE_{\sigma} := \{x \mapsto c_0 + c_1\sigma(a x_1 + b):c_0,c_1,a,b \in \RR\}.
		\end{align*}
		Then the function $h(x) = \sigma(x_1)^2$ belongs to $\cA_{\loc}^{\soft}(\cE_{\sigma};2)$, but
		\begin{align*}
			h \notin \bigcup_{M \ge 1} \cA_{\glob}(\cE_{\sigma};M).
		\end{align*}
	\end{example}
	
	Example~\ref{exa:complex_experts} illustrates a structural distinction between global and localized aggregation. A global aggregator forms an input-independent convex combination of experts from $\cE_{\sigma}$, and hence remains within the finite linear span generated by such expert functions. In contrast, localized aggregation multiplies expert outputs by input-dependent gating weights. This interaction between the router and the experts creates multiplicative structure, and can generate functions, such as $\sigma(x_1)^2$, that are not representable by any finite input-independent aggregation of the same expert class.
	
	\subsection{Two specialized experts: the advantage of localized routing}\label{subsec:averaging_advantage}
	
	This subsection provides a risk-oriented illustration of how gating can amplify the strengths of specialized experts. 
	For convenience, let $\cX=[0,1]$, $\cX_1=[0,1/2]$ and $\cX_2=(1/2,1]$. Let $f:\cX\to\RR$ be the target regression function, with marginal law $P_X$. For any measurable $u:\cX\to\RR$, write
	\begin{align*}
		\|u\|_{\cX_j} := \left\{\int_{\cX_j} u(x)^2\,dP_X(x)\right\}^{1/2}, \qquad j=1,2.
	\end{align*}
	Assume that $\hat f_1$ is specialized to $\cX_1$ and $\hat f_2$ to $\cX_2$ in the sense that
	\begin{align*}
		\|\hat f_1-f\|_{\cX_1}=\epsilon_1, \qquad \|\hat f_2-f\|_{\cX_2}=\epsilon_2, \qquad \|\hat f_2-f\|_{\cX_1}=\delta_1, \qquad \|\hat f_1-f\|_{\cX_2}=\delta_2,
	\end{align*}
	where $\epsilon_j$ is the good-region error, assumed to be relatively small, and $\delta_j$ is the error of the other expert on $\cX_j$, typically relatively large.
	
	For the fixed pair $(\hat f_1,\hat f_2)$, define
	\begin{align*}
		\cA_{\glob}(\hat f_1,\hat f_2) := \left\{w\hat f_1+(1-w)\hat f_2 \;\middle|\; 0\le w\le 1\right\}.
	\end{align*}
	This is the restriction of $\cA_{\glob}(\cE;2)$ to the two fitted experts. By contrast, the oracle Top-$1$ routed predictor that uses the same two experts and splits the input space at $1/2$ is
	\begin{align*}
		\hat f_{\loc}^{\mathrm{Top\mbox{-}1}}(x) :=
		\begin{cases}
			\hat f_1(x), & x\in\cX_1,\\
			\hat f_2(x), & x\in\cX_2.
		\end{cases}
	\end{align*}
	Here the routing rule assigns inputs in $\cX_1$ to expert $1$ and inputs in $\cX_2$ to expert $2$.
	
	\begin{proposition}\label{prop:two_specialized_experts}
		Assume $\delta_1+\delta_2+\epsilon_1+\epsilon_2>0$. Let $\gamma := \left[\frac{\delta_1\delta_2-\epsilon_1\epsilon_2}{\delta_1+\delta_2+\epsilon_1+\epsilon_2}\right]_+$, where $[a]_+:=\max\{a,0\}$, and write $R(h) := \|h-f\|_{L_2(P_X)}^2 = \int_{\cX}\{h(x)-f(x)\}^2\,dP_X(x)$. Then
		\begin{align*}
			\inf_{h\in\cA_{\glob}(\hat f_1,\hat f_2)} R(h) \ge \gamma^2, \qquad R(\hat f_{\loc}^{\mathrm{Top\mbox{-}1}})=\epsilon_1^2+\epsilon_2^2.
		\end{align*}
		Consequently, if $\epsilon_1^2+\epsilon_2^2<\gamma^2$, then the localized Top-$1$ rule using the same two experts has strictly smaller population $L_2(P_X)$ risk than every global convex combination of them.
	\end{proposition}
	
	\begin{remark}
		Proposition~\ref{prop:two_specialized_experts} formalizes the compromise faced by global aggregation. Accuracy on $\cX_1$ pushes $w$ toward $\hat f_1$, while accuracy on $\cX_2$ pushes it toward $\hat f_2$; a single input-independent weight cannot generally satisfy both. The quantity $\gamma$ measures this tension. In the symmetric case $\delta_1=\delta_2=\delta$ and $\epsilon_1=\epsilon_2=\epsilon$, we have $\gamma=[(\delta-\epsilon)/2]_+$. When $\delta$ is much larger than $\epsilon$, each expert is genuinely specialized: it performs well on its own region but poorly on the other. In this regime, local routing has a larger advantage because it preserves the good-region performance of both experts instead of forcing a global compromise.
	\end{remark}
	
	The results in this section give a simple but useful message. The expressive power of MoE models does not come only from the complexity of the individual experts. It also comes from the router, which turns a fixed collection of expert outputs into a locally adaptive prediction rule. Dense routing, sparse routing and global averaging therefore define genuinely different approximation classes.

	\section{Risk bounds for dense MoE}\label{sec:oracle_softmax}
	The previous section studied MoE architectures from an approximation perspective and showed that input-dependent routing can substantially enlarge the class of attainable predictors. We now turn to the corresponding statistical question: Given such a localized aggregation class, what is the cost of learning the router from finite data? This section focuses on dense softmax gating and derives oracle risk bounds for router learning along an online training trajectory. These bounds decompose the risk into approximation error, cumulative expert-learning error, and the statistical cost of learning the routing rule.
	
	\subsection{Basic setting and training procedure}\label{subsec:BasicSetting}
	We observe independent data $Z_i=(X_i,Y_i)$ generated from $Y_i = f_0(X_i) + \varepsilon_i,$
	where $f_0(x) = \EE(Y_i \mid X_i = x)$ is an unknown regression function and $X_i \sim P_X$ on a compact input space $\cX \subset \mathbb{R}^d$ with $\|x\|_2 \le C_X\sqrt{d}$ for every $x \in \cX$ and some constant $C_X$. Throughout the remainder of the paper, the error $\varepsilon$ is assumed to be independent of $X$ and to have a scale-family density $h_{\varepsilon}(z) = \sigma^{-1} h_{\varepsilon,0}(z/\sigma)$, where $h_{\varepsilon,0}$ is known, centered at zero, and normalized to have unit variance.
	
	We adopt the shared-routed MoE formulation introduced in Section \ref{subsec:formulation}, specialize to dense softmax gating, and use an online decoupled formulation to accommodate sequential expert updating. Let $\mathcal U$ collect all auxiliary randomization used by the initialization and the blackbox expert training procedure, and assume that $\mathcal U$ is independent of the data. For each $t = 1, \dots, n-1$, the shared and routed expert functions available after processing $Z_1, \dots, Z_t$ are measurable with respect to the $\sigma$-field generated by $Z_1, \dots, Z_t$ and $\mathcal{U}$, namely, $\sigma(Z_1,\ldots,Z_t,\mathcal U)$. After observing $X_{t+1}$, their evaluations and the candidate MoE predictions of $Y_{t+1}$ are measurable with respect to the $\sigma$-field $\sigma(Z_1,\ldots,Z_t,\mathcal U,X_{t+1})$. Once $Y_{t+1}$ is observed, the resulting prediction loss is used to update the aggregation weights, and $Z_{t+1}$ may also be used to update the experts for subsequent predictions. 
	
	Under this formulation, for each gating parameter $\theta$ and each time $t$, the candidate predictor takes the form
	\begin{align}
		\check f^{t}(x;\theta) = \sum_{\ell=1}^{\MS}\hat f_{S,\ell}^{t}(x) + \sum_{m=1}^{\MR} g_m^{\soft}(x;\theta)\hat f_{R,m}^{t}(x), \label{eqn:predictor}
	\end{align}
	where $\hat f_{S,\ell}^{t}$ and $\hat f_{R,m}^{t}$ denote the shared and routed experts available after processing the first $t$ observations, respectively. Thus, conditional on the past data, the experts enter the next-step prediction problem as fixed plug-in components, while their quality is allowed to improve along the training trajectory.
	
	This online formulation also simplifies identifiability issues. In a fully joint MoE model, the parameters are generally non-identifiable because experts can be permuted without changing the predictor. Here, for each time $t$, the experts $\hat f_{S,\ell}^{t}$ and $\hat f_{R,m}^{t}$ are viewed as fixed expert slots produced by the past training process, so this permutation ambiguity is irrelevant. The only remaining non-identifiability is the usual shift invariance of softmax scores: adding a common function to all scores does not change the normalized gating weights. Since our results concern the induced gating functions and predictors, this invariance has no effect on the risk bounds. When a parameter level normalization is convenient for softmax gating, we use a standard reference constraint, for example $s_{\MR}(x;\theta_{\MR}) = 0$.
	
	Even along this online training trajectory, router learning remains statistically and computationally nontrivial. Conditional on the experts available at time $t$, estimating the dense softmax gate from the next observation or from a block of subsequent observations still involves a non-convex parametrization. We therefore use a discretization-based aggregation strategy, inspired by adaptive regression via mixing \citep{yang2001adaptive} and aggregated forecasting through exponential reweighting \citep{yang2004Combining}, to obtain risk guarantees without assuming global optimization of a non-convex empirical objective. The procedure constructs a finite net of candidate gates, aggregates the corresponding MoE predictors using likelihood-based weights, and, when a single MoE parameter is required, projects the aggregate back onto the discretized class. This construction is intended as a theoretical benchmark for router learning along the training trajectory, not as a practical substitute for end-to-end MoE training. The estimation procedure is summarized in Algorithm~\ref{alg:general} for a general gating class. In this algorithm, $i$ indexes the i.i.d. observations, while $t$ is an algorithmic time index that records the sequential use of observations in the training procedure.
		
	\begin{algorithm}[htbp]
		\caption{Discretized Aggregation over a General Router Family}\label{alg:general}
		\textbf{Input:} A burn-in size $n_1$; a variance calibration size $n_2$; a router family $g(\cdot;\theta)$ with $\theta \in \Theta_{\MR}$; a blackbox expert update rule; data $\{(X_i,Y_i)\}_{i=1}^n$.\\
		\textbf{Output:} An online aggregated sequence $\{\tilde{f}^t(x)\}_{t=n_1+n_2}^{n-1}$; estimated gating parameter $\hat{\theta}$ and final estimator $\hat{f}_0(x)$.
		{\small
			\begin{enumerate}[label=(\alph*)]
				\item Choose an $\eta$-net of the normalized bounded gating parameter space $\Theta_{\MR}$, denoted by $\Theta(\eta)$. Write $\Theta(\eta) = \{\theta^{(1)},\ldots,\theta^{(S_\eta)}\}$ where $S_\eta = |\Theta(\eta)|$. The mesh size $\eta$ is chosen according to the continuity and covering properties of the router family. These parameters induce the discretized routers $g(\cdot;\theta^{(1)}),\ldots,g(\cdot;\theta^{(S_\eta)})$.
				
				\item Randomly permute the observations once, relabel them as $Z_1,\dots,Z_n$, and split the relabeled sample into three consecutive blocks, $Z^{(1)} = (X_i,Y_i)_{i=1}^{n_1}$, $Z^{(2)} = (X_i,Y_i)_{i=n_1+1}^{n_1+n_2}$ and $Z^{(3)} = (X_i,Y_i)_{i=n_1+n_2+1}^{n}$, and let $n_3 = n-n_1-n_2$. 
				
				\item \textbf{Burn-in:} Process $Z^{(1)}$ by the blackbox update rule to obtain $\{\hat{f}^{n_1}_{S,\ell}\}_{\ell=1}^{\MS}$ and $\{\hat{f}^{n_1}_{R,m}\}_{m=1}^{\MR}$.
				
				\item \textbf{Calibration:} For $t = n_1,\ldots,n_1+n_2-1$:
				\begin{enumerate}[label=(\roman*)]
					\item For each $\theta^{(s)} \in \Theta(\eta)$, construct the candidate MoE predictor $\check{f}^{t}(x;\theta^{(s)})$ as in \eqref{eqn:predictor} by replacing softmax gating $g^{\soft}$ by the specified gating $g$ and using experts trained at time $t$. After observing $Z_{t+1}$, compute the residual $e_{s,t+1} = Y_{t+1}-\check{f}^{t}(X_{t+1};\theta^{(s)})$.
					
					\item Update the experts by the blackbox rule to obtain $\{\hat{f}^{t+1}_{S,\ell}\}_{\ell=1}^{\MS}$ and $\{\hat{f}^{t+1}_{R,m}\}_{m=1}^{\MR}$.
				\end{enumerate}
				Estimate the corresponding scale parameter by $\tilde{\sigma}_s^2 = 1/n_2\sum_{t=n_1+1}^{n_1+n_2} e_{s,t}^2$, and let $\hat{\sigma}_s^2 = \Pi_{[\underline{\sigma}^2,\bar{\sigma}^2]}(\tilde{\sigma}_s^2)$ where $\Pi_{[a,b]}(u) = \min\{b,\max\{a,u\}\}$.
				
				\item \textbf{Aggregation:} Set $\hat{w}_{s,n_1+n_2} = 1/S_{\eta}$. For $t = n_1+n_2,\ldots,n-1$:
				\begin{enumerate}[label=(\roman*)]
					\item Obtain the candidate predictors $\check{f}^{t}(x;\theta^{(s)})$ similarly as before, and output $\tilde{f}^t(x) = \sum_{s=1}^{S_{\eta}}\hat{w}_{s,t}\check{f}^t(x;\theta^{(s)})$.
					
					\item After observing $Z_{t+1}$, update the weights $\hat{w}_{s,t+1}$ by
					\begin{align}
						\hat w_{s,t+1} = \frac{\hat w_{s,t}\hat\sigma_s^{-1} h_{\varepsilon,0}\left((Y_{t+1}-\check{f}^t(X_{t+1};\theta^{(s)}))/\hat\sigma_s\right)}{\sum_{r=1}^{S_\eta}\hat w_{r,t}\hat\sigma_r^{-1}h_{\varepsilon,0}\left((Y_{t+1}-\check{f}^t(X_{t+1};\theta^{(r)}))/\hat\sigma_r\right)}. \label{eq:cal-after-update}
					\end{align}
					
					\item Update the experts by the blackbox rule to obtain $\{\hat{f}^{t+1}_{S,\ell}\}_{\ell=1}^{\MS}$ and $\{\hat{f}^{t+1}_{R,m}\}_{m=1}^{\MR}$.
				\end{enumerate}
				
				\item Averaging the online outputs gives the predictor $\tilde{f}_0$, while, for each $s \in [S_\eta]$, we define the corresponding time-averaged candidate predictor $f^{n_3}(\cdot;\theta^{(s)})$ by
				\begin{align*}
					\tilde{f}_0(x) = \frac{1}{n_3}\sum_{t=n_1+n_2}^{n-1}\tilde{f}^t(x), \qquad \bar f^{n_3}(x;\theta^{(s)}) = \frac{1}{n_3}\sum_{t=n_1+n_2}^{n-1}\check f^{t}(x;\theta^{(s)}).
				\end{align*}
				Project the aggregated predictor back onto the discretized MoE class by choosing the candidate predictor closest to $\tilde{f}_0$ in $L_2(P_X)$:
				\begin{align*}
					s^* = \argmin_{1 \le s \le S_\eta} \Vert \tilde{f}_0-\bar f^{n_3}(\cdot;\theta^{(s)}) \Vert_{L_2(P_X)}^2, \qquad \hat{\theta} = \theta^{(s^*)}.
				\end{align*}
				The final estimator is $\hat{f}_0(x) = \bar{f}^{n_3}(x;\hat{\theta})$.
		\end{enumerate}}
	\end{algorithm}
	
	\begin{remark}
		The shared and routed experts play different roles in the scaling considered below. Shared experts are intended to capture structure reused across the input space, so $\MS$ is treated as fixed. Routed experts, by contrast, are used to represent local heterogeneity induced by the gate, and increasing their number allows a finer collection of local components. Accordingly, we consider a triangular array setting in which $\MR = M_{R,n}$ is deterministic and may increase with $n$. This convention is consistent with recent shared-routed MoE designs, such as DeepSeekMoE, where a small number of shared experts is combined with a larger set of routed experts. The formal triangular array notation is introduced before the assumptions used in the oracle analysis. 
	\end{remark}
	
	\begin{remark} \label{remark3.2}
		The projection in Step (f) is defined at the population level. In the online setting, the aggregated predictor is obtained by averaging over both time and discretized gates. The projection is therefore taken onto the time-averaged MoE class whose elements are
		\begin{align*}
			\bar f^{n_3}(\cdot;\theta^{(s)}) = \sum_{\ell=1}^{\MS}\bar f^{n_3}_{S,\ell} + \sum_{m=1}^{\MR} g_m(\cdot;\theta^{(s)})\bar f^{n_3}_{R,m},
		\end{align*}
		where $\bar f^{n_3}_{S,\ell} = \sum_{t=n_1+n_2}^{n-1}\hat f^t_{S,\ell}/n_3$ and $\bar f^{n_3}_{R,m} =\sum_{t=n_1+n_2}^{n-1}\hat f^t_{R,m}/n_3$
		are the corresponding time-averaged experts. The population criterion involves $P_X$, which is unknown in practice. When $X$ is supported on a compact set and $P_X$ has a density bounded above and away from zero, the $L_2(P_X)$ and Lebesgue $L_2$ norms are equivalent, so a Lebesgue-integrated surrogate yields the same approximation order. For fixed $\MS$ and $\MR$, Supplement Section \ref{subsec:supp-cross-entropy-fixed-mr} gives a fully data-driven version of this projection step based on time-expanded cross-entropy minimization; see also Remark \ref{remark:fixedM}.
	\end{remark}
	
	\begin{remark}
		If the noise scale is known, the calibration block can be omitted. If a rolling or external scale estimator is used instead, the same finite-grid aggregation argument yields an additional scale-estimation contribution. The formulation in Algorithm \ref{alg:general} keeps the main statement self-contained by absorbing scale estimation into the displayed learning cost below.
	\end{remark}

	\subsection{Oracle inequalities for dense softmax gating} \label{subsec:softmax}
	
	We now analyze the estimator produced by Algorithm \ref{alg:general} when the gating class is the dense softmax class with linear scores. Recall that the $m$-th gating weight is given by
	\begin{align*}
		g_m^{\soft}(x;\theta) = \frac{\exp\{s_m^{\lin}(x;\theta_m)\}}{\sum_{j=1}^{\MR} \exp\{s_j^{\lin}(x;\theta_j)\}}, \qquad m = 1,\ldots,\MR,
	\end{align*}
	where $s_m^{\lin}(x;\theta_m) = x^\top \beta_m + \alpha_m$. Let $\hat{\theta}$ denote the parameter selected by Algorithm~\ref{alg:general} under this dense softmax specification. Equivalently, the projected estimator in Step (f) can be written as
	\begin{align*}
		\hat{f}_0(x) := \sum_{\ell=1}^{\MS} \bar{f}^{n_3}_{S,\ell}(x) + \sum_{m=1}^{\MR} g_m^{\soft}(x;\hat{\theta}) \bar{f}^{n_3}_{R,m}(x),
	\end{align*}
	where $\bar f^{n_3}_{S,\ell}$ and $\bar f^{n_3}_{R,m}$ are the time-averaged experts introduced in Remark \ref{remark3.2}.

	Since the number of routed experts can increase with the sample size, we make the triangular array dependence explicit and write the evolving expert predictors available at time $t$ as
	$\{\hat f_{S,\ell}^{t,n}\}_{\ell\in[\MS]}$ and $\{\hat f_{R,m}^{t,n}\}_{m\in[M_{R,n}]}.$
	An admissible pair $(t,n)$ refers to a sample size $n$ and an algorithmic time index $t$ at which these experts are used in Algorithm~\ref{alg:general}.
	We impose the following assumptions.
	
	\begin{assumption}\label{ass:compact-theta}
		For each sample size $n$, the dense softmax gating parameter belongs to a compact parameter space $\Theta_{\MR}$. In addition, there exists a constant $C_1 > 1/2$, independent of $n$ and hence of $\MR$, such that $\Theta_{\MR}$ is contained in a box of radius $C_1\log \MR$ in $\RR^{\MR(d+1)}$. In particular, under softmax gating we impose the normalization $\theta_{\MR} = 0$, so that the effective parameter space has dimension $(\MR-1)(d+1)$, although the ambient space is written as $\RR^{\MR(d+1)}$.
	\end{assumption}
	
	The logarithmic growth of the parameter radius in Assumption \ref{ass:compact-theta} allows the dense softmax gate to remain sufficiently selective as the number of routed experts increases. If the parameter box were fixed uniformly in $\MR$, then the score differences would remain uniformly bounded, and the largest attainable softmax weight assigned to any single expert would decrease as $\MR$ grows. This limits the router's ability to approximate hard specialization among a growing number of experts. The $\log \MR$ scaling avoids this degeneracy while preserving compactness of the parameter space for each fixed $\MR$.
	
	\begin{assumption}\label{ass:evolving-experts}
		For each sample size $n$, the evolving experts admit deterministic population benchmarks $\{f_{S,\ell,n}^*\}_{\ell\in[\MS]}$ and $\{f_{R,m,n}^*\}_{m\in[M_{R,n}]}$. 
		There exist a constant $C_2<\infty$ and a nonnegative deterministic triangular array $\{a_{t,n}\}$ such that, for every admissible pair $(t,n)$,
		\begin{align*}
			\max_{\ell\le\MS}
			\EE\left\|\hat f_{S,\ell}^{t,n}-f_{S,\ell,n}^*\right\|_{L_2(P_X)}^2
			\vee
			\max_{m\le M_{R,n}}
			\EE\left\|\hat f_{R,m}^{t,n}-f_{R,m,n}^*\right\|_{L_2(P_X)}^2
			\le C_2a_{t,n}^2.
		\end{align*}
	\end{assumption}
	Assumption~\ref{ass:evolving-experts} requires the evolving experts to be controlled relative to deterministic population benchmarks. The array $a_{t,n}$ quantifies the corresponding $L_2(P_X)$ discrepancy at time $t$ and sample size $n$. Pointwise convergence of $a_{t,n}$ along every admissible sequence is not required; the oracle inequality instead depends on the block averages $\bar a_{n_2}^2$ and $\bar a_{n_3}^2$ defined in Theorem~\ref{thm:softmax}.
	
	\begin{assumption}\label{ass:bounded-experts}
		The regression function, the population benchmarks of the experts, and the evolving expert estimators
		are uniformly bounded. Specifically,
		there exists a constant $A < \infty$ such that the following bound holds almost surely for every admissible pair $(t,n)$,
		\begin{align*}
			\|f_0\|_{\infty} \vee \max_{1 \le \ell \le \MS}\left\{\|f_{S,\ell,n}^*\|_{\infty} \vee \|\hat{f}_{S,\ell}^{t,n}\|_{\infty}\right\} \vee \max_{1 \le m \le M_{R,n}}\left\{\|f_{R,m,n}^*\|_{\infty} \vee \|\hat{f}_{R,m}^{t,n}\|_{\infty}\right\} \le A.
		\end{align*}
	\end{assumption}
	
	Assumption~\ref{ass:bounded-experts} requires $f_0$, the evolving experts and their population benchmarks to be uniformly bounded over all admissible $(t,n)$. Whenever $n$ is fixed, we suppress its index and write $\MR=M_{R,n}$, $\hat f_{S,\ell}^t=\hat f_{S,\ell}^{t,n}$, $\hat f_{R,m}^t=\hat f_{R,m}^{t,n}$, $f_{S,\ell}^*=f_{S,\ell,n}^*$, $f_{R,m}^*=f_{R,m,n}^*$, and $a_t=a_{t,n}$.
	
	\begin{assumption}\label{ass:ARM-regularity}
		The error term $\varepsilon$ is independent of $X$ and has density $h_{\varepsilon}(z) = \sigma^{-1} h_{\varepsilon,0}(z/\sigma)$, where the known standardized density $h_{\varepsilon,0}$ satisfies $\int z h_{\varepsilon,0}(z) dz = 0, \int z^2 h_{\varepsilon,0}(z) dz = 1$ and $\int z^4 h_{\varepsilon,0}(z)dz < \infty$. For any $0 < s_0 < 1$ and $T_0 > 0$, there exists a constant $C_3 = C_3(s_0,T_0)$ such that
		\begin{align*}
			\int h_{\varepsilon,0}(z) \log \frac{h_{\varepsilon,0}(z)}{s^{-1} h_{\varepsilon,0}((z-u)/s)} dz \le C_3\{(1-s)^2+u^2\},
		\end{align*}
		for all $s_0 \le s \le s_0^{-1}$ and $|u| \le T_0$. Moreover, the scale parameter satisfies $0 < \underline{\sigma} \le \sigma \le \bar{\sigma} < \infty$ where $\underline{\sigma}$ and $\bar{\sigma}$ are known constants.
	\end{assumption}
	
	Assumption~\ref{ass:ARM-regularity} standardizes $h_{\varepsilon,0}$ to have mean zero and unit variance, so that $f_0$ is the conditional mean function and $\sigma^2$ is the error variance. The remaining condition is a local Kullback--Leibler regularity requirement for location-scale perturbations.
	It is satisfied by Gaussian, double-exponential, and many other commonly used distributions.

	Under the above assumptions, Algorithm~\ref{alg:general} yields an oracle inequality for dense softmax gating. The bound separates two sources of error. The first is the \textbf{oracle approximation risk}, which measures how well the population shared and routed experts can approximate $f_0$ under the best dense softmax gate in the prescribed class. The second is the \textbf{statistical learning cost}, which consists of the cumulative error of the evolving experts and the finite-sample cost of estimating the gate. The following theorem gives the explicit form of these two terms.
	
	\begin{theorem}\label{thm:softmax}
		Assume Assumptions \ref{ass:compact-theta}--\ref{ass:ARM-regularity} hold. Run Algorithm \ref{alg:general} with dense softmax gating, $n_1\asymp n_2\asymp n_3$ and mesh size $\eta \asymp (\MR\sqrt{n})^{-1}$. Suppose that $\MS$ is fixed. Then the online sequence $\{\tilde{f}^t\}_{t=n_1+n_2}^{n-1}$ satisfies
		\begin{align}
			\frac{1}{n_3}\sum_{t=n_1+n_2}^{n-1}\EE\|\tilde{f}^t-f_0\|_{L_2(P_X)}^2 \le C\left\{\fA_{\soft}(\MS,\MR)+\fL_{\soft}(\MS,\MR)\right\}, \label{eqn:softmax_oracle}
		\end{align}
		where the softmax oracle approximation risk is
		\begin{align*}
			\fA_{\soft}(\MS,\MR) := \inf_{\theta \in \Theta_{\MR}} \left\|f_0-\sum_{\ell=1}^{\MS} f_{S,\ell}^*-\sum_{m=1}^{\MR} g_m^{\soft}(\cdot;\theta) f_{R,m}^*\right\|_{L_2(P_X)}^2,
		\end{align*}
		and the softmax statistical learning cost is
		\begin{align*}
			\fL_{\soft}(\MS,\MR) := \MR(\bar{a}_{n_2}^2+\bar{a}_{n_3}^2) + \frac{\MR d}{n}\log(\MR d n),
		\end{align*}
		where $\bar{a}_{n_2}^2 = \sum_{t=n_1}^{n_1+n_2-1}a_t^2/n_2$ and $\bar{a}_{n_3}^2 = \sum_{t=n_1+n_2}^{n-1}a_t^2/n_3$.
		Here $C$ is a constant independent of $n$ and $\MR$. The same bound holds for the risk $\EE\|\hat{f}_0-f_0\|_{L_2(P_X)}^2$ of the final projected estimator.
	\end{theorem}
	
	Theorem~\ref{thm:softmax} gives a non-asymptotic oracle inequality in which the number of routed experts $\MR$ is allowed to grow with the sample size. The approximation term $\fA_{\soft}(\MS,\MR)$ depends on the quality of the population experts and on the expressive power of the dense softmax gate. The learning term $\fL_{\soft}(\MS,\MR)$ has two components. The term $\MR(\bar{a}_{n_2}^2+\bar{a}_{n_3}^2)$ measures the average quality of the expert trajectory during the calibration and aggregation stages, while the term $\frac{\MR d}{n}\log(\MR d n)$ is the cost of estimating the gate over a growing softmax class. Because $\MS$ is fixed, the statistical contribution of the shared experts is absorbed into the constant.
	
	Thus $\MR$ plays a dual role. Increasing $\MR$ may improve localization and reduce the oracle approximation risk, but it also increases the statistical cost of learning the router. This is the statistical analogue of a familiar practical tradeoff in MoE systems: adding routed experts increases capacity, but may also introduce load-balancing difficulties, communication overhead, and low expert utilization. Hence, a larger number of routed experts is not automatically preferable. The choice of $\MR$ should balance the approximation gain from finer localization against the estimation cost of a more complex gating class. 
	The next example illustrates how a shared expert can reduce the oracle approximation term by extracting a common component that would otherwise have to be approximated by the routed experts. This same shared-residual mechanism is studied more broadly in Section \ref{sec:shared}, with emphasis on how it changes the estimation problem faced by local learning rules.
	
	\begin{example}\label{exa:role-shared-routed-expert}
		Let $X \sim \Unif[0,1]$ and consider the target function $f_0(x) = A\sin(2\pi Lx) + x^2$,
		where $L \gg 1$ controls the complexity of the oscillatory component. Suppose that $\MS = 1$ and that the shared expert captures the oscillatory component, namely $f_{S,1}^*(x) = A\sin(2\pi Lx)$. The routed experts are then used to approximate the residual function $x^2$ through local aggregation. Specifically, suppose the population routed experts belong to the affine class $\cF_{R,\lin} = \{f(x) = ax+b:a,b \in \RR\}$. Then there exist routed experts $f_{R,1}^*, \ldots, f_{R,\MR}^* \in \cF_{R,\lin}$ such that the softmax oracle approximation risk satisfies
		\begin{align*}
			\fA_{\soft}(1,\MR) &= \inf_{\theta \in \Theta_{\MR}} \left\|f_0-f_{S,1}^*-\sum_{m=1}^{\MR} g_m^{\soft}(\cdot;\theta) f_{R,m}^*\right\|_{L_2(P_X)}^2 = O\left(\frac{1}{\log^2\MR}\right).
		\end{align*}
		Without the shared expert, the routed experts must approximate the full function $f_0$. In that case, the corresponding oracle approximation risk is of order $O(L^4/\log^2\MR)$.
	\end{example}
	
	With the shared expert, the routed experts only need to approximate the smoother residual $x^2$, and the approximation error no longer scales with the oscillation parameter $L$. Without the shared expert, the routed experts must approximate both the oscillatory component and the residual structure, which inflates the oracle approximation risk by a factor of order $L^4$. Thus, in this example, the shared expert improves oracle approximation by removing a component that is difficult for the routed affine class to approximate.
	
	The example also highlights the role of $\MR$. If $\MR$ is fixed as $n \to \infty$, the oracle approximation risk generally does not vanish unless the fixed expert collection already yields an exact representation. Consistency therefore requires either a sufficiently rich fixed expert collection or a growing number of routed experts.
	
	\begin{remark}\label{remark:ERM}
		A direct empirical risk minimization (ERM) analysis of the non-convex softmax class typically proceeds through uniform concentration bounds for the empirical squared loss. Even if the empirical objective is assumed to attain a global minimum, this route gives an excess risk bound of order $\sqrt{\MR d \log(\MR d n)/n}$; see Supplement Section \ref{subsec:supp-erm-subgaussian-noise}. By contrast, the excess risk bound obtained from the aggregation argument in Theorem \ref{thm:softmax} includes a term of order $\MR d \log(\MR d n)/n$ for estimating the gating function, in addition to the error contributed by the evolving experts. Thus, when the routed-expert complexity is moderate relative to the sample size, the aggregation approach provides a sharper oracle benchmark than the direct ERM route. 
	\end{remark}
	
	\begin{remark}\label{remark:fixedM}
		When $\MR$ is fixed, the population projection step (f) in Algorithm~\ref{alg:general} can be replaced by a fully data-driven surrogate. Recall that the aggregation step produces time-varying aggregated gating weights of the form $\tilde{g}_{m,t}(x) = \sum_{s=1}^{S_\eta}\hat w_{s,t}g_m^{\soft}(x;\theta^{(s)})$. We estimate $\theta$ by projecting these aggregated weights back onto the dense softmax class through the time-expanded cross-entropy criterion
		\begin{align*}
			\cL_{\mathrm{ce}}(\theta;n_3,n_4) = -\frac{1}{n_3n_4}\sum_{i=1}^{n_4}\sum_{t=n_1+n_2}^{n-1}\sum_{m=1}^{\MR}\tilde g_{m,t}(\widetilde{X}_i)\log g_m^{\soft}(\widetilde{X}_i;\theta),
		\end{align*}
		where $\{\widetilde{X}_i\}_{i=1}^{n_4}$ is an independent validation sample from $P_X$ with $n_4 \asymp n$, and set $\hat\theta_{\mathrm{ce}} = \argmin_{\theta \in \Theta_{\MR}}\cL_{\mathrm{ce}}(\theta;n_3,n_4).$
		Under the linear-score softmax parametrization, this is a convex multinomial logistic regression problem with soft labels $\tilde g_t(\widetilde{X}_i)$. Hence, in the fixed-$\MR$ regime, the projection step admits a computationally tractable implementation without additional non-convex optimization. Under mild regularity conditions, the resulting estimator satisfies the oracle inequality
		\begin{align}
			\EE\|\hat f_{0,\mathrm{ce}}-f_0\|_{L_2(P_X)}^2 \le C_{\MR}\left\{\fA_{\soft}(\MS,\MR) + \bar a_{n_2}^2+\bar a_{n_3}^2 + \frac{\log n}{n}\right\}, \label{eqn:softmax_practical_oracle}
		\end{align}
		where $C_{\MR}$ is a constant depending on $\MR$, and $\fA_{\soft}(\MS,\MR)$ is the softmax oracle approximation risk defined in Theorem \ref{thm:softmax}. The proof is deferred to Supplement Section \ref{subsec:supp-cross-entropy-fixed-mr}.
	\end{remark}
	
	This fixed-$\MR$ implementation also clarifies the connection with adaptive regression by mixing. The classical adaptive regression by mixing (ARM) procedure \citep{yang2001adaptive} aggregates a finite set of candidate estimators using input-independent weights. Here, the aggregation is carried out over candidate gating functions, so the resulting aggregate induces covariate-dependent expert weights. This extension is essential for MoE models, because the relative performance of different experts may vary across the input space.
	
	\section{Risk bounds for sparse MoE}\label{sec:oracle_topK}
	Dense softmax gating assigns positive weights to all routed experts for each input. This smooth averaging can be statistically effective, but it becomes computationally expensive when the number of routed experts $\MR$ is large. Sparse Top-$K$ gating addresses this issue by activating only a small subset of routed experts for each input. Since the seminal work of \citep{shazeer2017outrageously}, this idea has become a central component of large-scale MoE architectures, including GShard, Switch Transformer, GLaM and Mixtral \citep{lepikhin2020gshard,fedus2022switch,du2022glam,jiang2024mixtral}. Importantly, MoE systems are not restricted to Top-$1$ routing: many architectures activate multiple experts for each input, with $K = 2$ or larger values commonly adopted \citep{du2022glam,jiang2024mixtral,dai2024deepseekmoe,team2026kimi}. This motivates a statistical comparison between hard expert selection and local multi-expert aggregation.
	
	This section studies the statistical effect of Top-$K$ sparsification under the same online training formulation as Section~\ref{sec:oracle_softmax}. 
	We ask whether sparse Top-$K$ routing can retain the oracle-risk guarantees of dense softmax gating while reducing per-input computation.
	
	\subsection{Oracle inequalities for Top-$K$ gating}\label{subsec:Top-K}
	We use the linear routing scores $s_m^{\lin}(x;\theta_m) = x^\top \beta_m + \alpha_m$, $m \in [\MR]$, as in Section~\ref{subsec:formulation} and write $s_m(x;\theta_m) = s_m^{\lin}(x;\theta_m)$ for simplicity in this section. As in the dense softmax case, we impose the reference normalization $\theta_{\MR} = 0$ to remove the shift invariance of the scores.
	
	Fix $K \in [\MR]$. For a given input $x$ and parameter $\theta$, let $s_{(K)}(x;\theta)$ denote the $K$-th largest value among the scores $s_1(x;\theta_1), \ldots, s_{\MR}(x;\theta_{\MR})$.	To resolve possible ties at the selection threshold, we adopt a deterministic tie-breaking rule, with one possible choice being to favor experts with smaller indices. The Top-$K$ active set is then defined as
	\begin{align*}
		\cT_K(x;\theta) = \{m \in [\MR]: s_m(x;\theta_m) \ge s_{(K)}(x;\theta)\},
	\end{align*}
	so that $|\cT_K(x;\theta)| = K$.
	Under mild regularity conditions, for example, when $X$ has a continuous distribution and the pairwise score differences are nondegenerate, ties occur with probability zero for each fixed $\theta$, and the particular tie-breaking rule is immaterial.
	
	We allow the selected scores to be transformed before normalization. Let $\phi:\RR \to (0,\infty)$ be a positive score transformation. The Top-$K$ normalized gate associated with $\phi$ is
	\begin{align*}
		g_m^{(K,\phi)}(x;\theta) = \frac{\phi\{s_m(x;\theta_m)\}\one\{m \in \cT_K(x;\theta)\}}{\sum_{j \in \cT_K(x;\theta)} \phi\{s_j(x;\theta_j)\}}, \qquad m \in [\MR].
	\end{align*}
	Thus unselected experts receive weight zero, while selected experts are normalized over the active set $\cT_K(x;\theta)$. The choice $\phi(t) = \exp(t)$ gives the Top-$K$ softmax gate introduced earlier; in this case, we write $g^{(K)}$ for short. More generally, the analysis accommodates any regular score transformation satisfying the regularity conditions in Definition~\ref{ass:regular gating} below.
	
	Given the evolving experts available at time $t$, the candidate predictor indexed by $\theta$ is
	\begin{align*}
		\check f^{K,\phi,t}(x;\theta) = \sum_{\ell=1}^{\MS}\hat f_{S,\ell}^{t}(x) + \sum_{m=1}^{\MR} g_m^{(K,\phi)}(x;\theta)\hat f_{R,m}^{t}(x).
	\end{align*}
	Running Algorithm~\ref{alg:general} with the Top-$K$ gate defined above gives scale estimates $\hat\sigma_s$ and aggregation weights $\hat w_{s,t}$ over the finite grid $\Theta(\eta) = \{\theta^{(1)},\ldots,\theta^{(S_\eta)}\}$. During the aggregation block, the online aggregated predictor is $\tilde{f}^{K,\phi,t}(x) = \sum_{s=1}^{S_\eta}\hat w_{s,t}\check f^{K,\phi,t}(x;\theta^{(s)})$. Equivalently, this predictor can be written as a shared-expert MoE with effective gating weights $\tilde{g}_{m,t}^{K,\phi}(x) = \sum_{s=1}^{S_\eta}\hat w_{s,t}g_m^{(K,\phi)}(x;\theta^{(s)})$. This resulting gate is a convex combination of Top-$K$ gates and need not itself be Top-$K$ for a single parameter. This distinction is harmless for the online oracle inequality, which evaluates the actual aggregate predictor. When a single sparse MoE predictor is required, Step~(f) of Algorithm~\ref{alg:general} projects the online-to-batch average back onto the discretized Top-$K$ class.
	
	To control the complexity of the resulting gating family, we impose a regularity condition on the normalized transformation over any fixed active set. This condition requires that, once the selected experts are fixed, the normalized weights vary smoothly with the scores.
	
	\begin{definition}\label{ass:regular gating}
		Let $\MR \ge 2$. A function $\phi:\RR \to (0,\infty)$ is called a \textbf{regular score transformation} if it is continuous, strictly increasing, and satisfies the following Lipschitz condition after normalization. For any subset $J \subset [\MR]$ with $|J| = K$, define
		\begin{align*}
			G_{J,m}^{\phi}(z) := \frac{\phi(z_m)\one\{m \in J\}}{\sum_{j \in J} \phi(z_j)}, \qquad m \in [\MR],
		\end{align*}
		and write $G_J^{\phi}(z) = (G_{J,1}^{\phi}(z),\ldots,G_{J,\MR}^{\phi}(z))^\top$. There exist constants $L_{\phi} > 0$ and $l_{\phi} \ge 0$, independent of $\MR$ and $J$, such that, for all $z,z' \in \RR^{\MR}$,
		\begin{align}
			\|G_J^{\phi}(z)-G_J^{\phi}(z')\|_2 \le L_{\phi} \MR^{l_{\phi}} \|z-z'\|_2. \label{eq:g-lipschitz}
		\end{align}
	\end{definition}
	
	Typical examples covered by this condition include the exponential transformation $\phi(t) = \exp(\alpha t)$ with $\alpha > 0$, which corresponds to softmax-type routing, and the sigmoid transformation $\phi(t) = 1/(1+\exp(-t))$, which is used in normalized-sigmoid MoE routing, such as DeepSeek-V3 \citep{liu2024deepseek}, and has been studied theoretically in related work \citep{nguyen2025deepseekmoe}. More generally, the condition holds whenever $\phi$ is continuously differentiable, strictly increasing, and has uniformly bounded relative derivative, namely $\sup_{t \in \RR} \phi'(t)/\phi(t) < \infty$. This relative-derivative condition ensures that, after normalization over any fixed active set, the gating weights are Lipschitz functions of the score vector.
	
	To analyze the statistical properties of Top-$K$ gating, we restrict attention to a subset of the parameter space where pairwise score differences are nondegenerate. Since Top-$K$ selection depends on the relative ordering of the scores, the following separation condition keeps the corresponding score boundaries well conditioned. 
	
	\begin{definition}\label{def:parameter}
		For $r_{\MR} \in (0,1)$, let $\widetilde{\Theta}_{\MR} := \widetilde{\Theta}_{\MR}(r_{\MR})$ be the subset of $\Theta_{\MR}$ consisting of parameters satisfying $\min_{i \neq j} \|\beta_i-\beta_j\|_2 \ge r_{\MR}$,
		where $\beta_m$ is the slope component of the linear score $s_m(x;\theta) = x^\top \beta_m + \alpha_m$, with $\beta_{\MR} = 0$ under the reference normalization $\theta_{\MR} = 0$.
	\end{definition}
	
	\begin{remark}
		The separation condition rules out nearly parallel or degenerate score differences without preventing changes in the Top-$K$ active set. Combined with the nondegeneracy condition on $P_X$ below, it allows the probability of active-set changes under small parameter perturbations to be controlled. Thus $r_{\MR}$ serves as a boundary-stability parameter rather than merely an identifiability condition.
	\end{remark}
	
	The associated Top-$K$ gating class is $\cG_{K,\phi} := \{g^{(K,\phi)}(\cdot;\theta): \theta \in \widetilde{\Theta}_{\MR}\}$. We equip this class with the vector-valued $L_2(P_X)$ metric $d_{K,\phi}(\theta,\theta') := \|g^{(K,\phi)}(\cdot;\theta)-g^{(K,\phi)}(\cdot;\theta')\|_{L_2(P_X)}$. To use the separation condition in Definition \ref{def:parameter} to control active-set changes, define the ``good'' set $G_{\theta,\theta'} := \{x \in \cX:\cT_K(x;\theta) = \cT_K(x;\theta')\}$ and the ``bad'' set $B_{\theta,\theta'} := \cX \setminus G_{\theta,\theta'}$. On $G_{\theta,\theta'}$, the Top-$K$ active sets coincide, so the normalized gating weights vary smoothly with the scores under the regularity condition on $\phi$. On $B_{\theta,\theta'}$, the selected expert set changes, and the gating weights may jump. To control the probability of this bad set, we impose the following nondegeneracy condition on the distribution of $X$.
	
	\begin{assumption}\label{ass:distribution}
		There exists a constant $C_4 > 0$ such that, for every $a \in \RR^d$ with $\|a\|_2 = 1$, every $b \in \RR$, and every $t > 0$, $\PP(|a^\top X+b| \le t) \le C_4t$. 
	\end{assumption}
	
	This assumption rules out excessive probability mass near affine decision boundaries. It ensures that inputs lying close to these decision boundaries have probability proportional to the width of the boundary band. Without such a condition, small perturbations of $\theta$ could change the Top-$K$ active set on a set with non-negligible probability, and the gating map need not be continuous in $L_2(P_X)$.
	
	Under Assumption~\ref{ass:distribution}, the Top-$K$ gating map satisfies the following $L_2(P_X)$ continuity property.
	
	\begin{proposition}\label{pro:continuity}
		Suppose Assumption~\ref{ass:distribution} holds. Then, for any regular score transformation $\phi$, there exists a constant $C_K > 0$ depending on $K$, $L_{\phi}$ and $C_4$, such that for all $\theta \in \widetilde{\Theta}_{\MR}$ and $\theta' \in \Theta_{\MR}$ satisfying $\|\theta-\theta'\|_2 \le 1$, we have
		\begin{align*}
			\|g^{(K,\phi)}(\cdot;\theta)-g^{(K,\phi)}(\cdot;\theta')\|_{L_2(P_X)} \le C_K d^{1/4}\MR^{1+l_{\phi}}\left(\frac{\|\theta-\theta'\|_2}{r_{\MR}}\right)^{1/2}.
		\end{align*}
	\end{proposition}
	Thus, uniformly over separated $\theta$ and nearby $\theta'$, the map $\theta \mapsto g^{(K,\phi)}(\cdot;\theta)$ admits a one-sided local H\"older modulus of order $1/2$ from the Euclidean parameter metric $\|\cdot\|_2$ to the distributional $L_2(P_X)$ metric on the Top-$K$ gating class $\cG_{K,\phi}$.
	
	\begin{remark}\label{remark:topk-l2-covering}
		The one-sided nature of the proposition is important. The target parameter $\theta$ is required to lie in the separated set $\widetilde{\Theta}_{\MR}$, but the approximation parameter $\theta'$ may lie in the ambient parameter box $\Theta_{\MR}$. This is sufficient because, for nearby $\theta'$, active-set switches are confined to narrow bands around the decision boundaries induced by $\theta$. It is also convenient for covering arguments, since the covering centers can be chosen from the ambient parameter space. As a consequence, the covering number of the Top-$K$ gating class under $L_2(P_X)$ can be controlled by the covering number of the normalized parameter space. Let $N(\eta,\cG_{K,\phi},L_2(P_X))$ denote the $\eta$-covering number of $\cG_{K,\phi}$. Then, under Assumption~\ref{ass:compact-theta},
		\begin{align*}
			N(\eta,\cG_{K,\phi},L_2(P_X)) \le \mathcal{O}\left[\left(\frac{\MR^{5/2+2l_{\phi}}d\log \MR}{\eta^2 r_{\MR}}\right)^{(\MR-1)(d+1)}\right].
		\end{align*}
	\end{remark}
	
	As in the dense softmax analysis, the risk bound separates an oracle approximation term from a statistical learning term. The following theorem gives the corresponding bound for Top-$K$ routing.
	
	\begin{theorem}\label{thm:top-k}
		Assume Assumptions~\ref{ass:compact-theta}--\ref{ass:ARM-regularity} hold. In addition, suppose Assumption~\ref{ass:distribution} holds and $\phi$ is regular. Run Algorithm~\ref{alg:general} with Top-$K$ gates $g^{(K,\phi)}$ with $n_1\asymp n_2\asymp n_3$, and mesh size $\eta \asymp \frac{\sqrt{d} r_{\MR}}{n_3\MR^{2+2l_{\phi}}}$.
		Suppose that $\MS$ is fixed. Then the online sequence $\{\tilde f^{K,\phi,t}\}_{t=n_1+n_2}^{n-1}$ satisfies
		\begin{align}
			\frac{1}{n_3}\sum_{t=n_1+n_2}^{n-1}\EE\|\tilde{f}^{K,\phi,t}-f_0\|_{L_2(P_X)}^2 \le C\left\{\fA_{K,\phi}(\MS,\MR)+\fL_{\sparse}(\MR,n,r_{\MR})\right\}, \label{eqn:Topk_oracle}
		\end{align}
		where the Top-$K$ oracle approximation risk is
		\begin{align*}
			\fA_{K,\phi}(\MS,\MR) := \inf_{\theta \in \widetilde\Theta_{\MR}(r_{\MR})}\left\|f_0-\sum_{\ell=1}^{\MS} f_{S,\ell}^*-\sum_{m=1}^{\MR}g_m^{(K,\phi)}(\cdot;\theta)f_{R,m}^*\right\|_{L_2(P_X)}^2,
		\end{align*}
		and the sparse-routing statistical learning cost is
		\begin{align}
			\fL_{\sparse}(\MR,n,r_{\MR}) := \MR(\bar a_{n_2}^2+\bar a_{n_3}^2) + \frac{\MR d}{n}\left\{\log(\MR d n)+\log\left(\frac{1}{r_{\MR}}\right)\right\}, \label{eq:topk-learning-cost}
		\end{align}
		where $\bar a_{n_2}^2 = \frac{1}{n_2}\sum_{t=n_1}^{n_1+n_2-1}a_t^2$ and $\bar a_{n_3}^2 = \frac{1}{n_3}\sum_{t=n_1+n_2}^{n-1}a_t^2$. 		Here $C$ is independent of $n$, $\MR$ and $r_{\MR}$. The same bound holds for the risk $\EE\|\hat f_0-f_0\|_{L_2(P_X)}^2$ of the final projected estimator.
	\end{theorem}
	
	Theorem~\ref{thm:top-k} parallels the dense softmax oracle inequality in Theorem~\ref{thm:softmax}. As before, $\MR(\bar a_{n_2}^2+\bar a_{n_3}^2)$ is the cost inherited from learning the routed experts, while $(\MR d/n)\log(\MR d n)$ is the cost of estimating the gate over a growing parameter space. The additional contribution $(\MR d/n)\log(1/r_{\MR})$ is specific to Top-$K$ routing and reflects the cost of controlling boundary-switch events. Equivalently, $r_{\MR}$ acts as a boundary-stability parameter: smaller $r_{\MR}$ enlarges the admissible class and may reduce the approximation term, but it also makes the active-set map less stable and increases the statistical learning cost.
	
	When $\phi(t)=\exp(t)$, the result specializes to Top-$K$ softmax gating. In this case, the statistical learning cost matches that of dense softmax gating up to the additional logarithmic term $\log(1/r_{\MR})$. Thus, under the stated regularity conditions, sparsification does not substantially increase the router-learning rate. At the same time, after computing the router scores, only $K$ routed experts are activated rather than all $\MR$ experts, reducing the per-input expert-evaluation cost.

	\subsection{Regionwise performance: hard selection versus soft aggregation} \label{subsec:regionwise}
	The oracle inequalities above compare dense softmax gating and sparse Top-$K$ gating in terms of approximation and estimation error. They do not, however, by themselves explain when hard selection should be preferred to soft aggregation. We now give a regionwise interpretation of this distinction. The main message is that Top-1 routing is statistically attractive in a specialist regime, where each region has a dominant expert, whereas soft or multi-expert aggregation becomes necessary when the best local predictor is a mixture of several experts.
	
	For clarity, we omit shared experts in this subsection and focus on the routed component. Let $f_m^*$ denote the population benchmark of the $m$-th routed expert. The same reasoning applies in the presence of shared experts after subtracting the population shared-expert contribution from $f_0$. We first work in a stylized oracle setting in which the relevant performance regions are represented by a measurable partition of the input space.
	
	Let $\{\cX_r\}_{r=1}^{\Mzero}$ be a measurable partition of $\cX \subset \RR^d$, with the regions disjoint up to sets of $P_X$-measure zero. To relate this partition to Top-1 linear routing, we assume that the regions are linearly separable.
	
	\begin{assumption}\label{ass:linearly-separable}
		Let $\bar x=(x^\top,1)^\top$. There exist vectors $v_r=(\zeta_r^\top,\psi_r)^\top\in\RR^{d+1}$, $r\in[\Mzero]$, such that, for each $r$, $\cX_r = \{x\in\cX:\bar x^\top v_r>\bar x^\top v_j,\ \forall j\ne r\}$, up to a boundary set of $P_X$-measure zero.
	\end{assumption}
	Under Assumption~\ref{ass:linearly-separable}, Top-1 linear routing can realize the oracle region assignment. We assume $\MR\ge\Mzero$ so that there are enough routed experts to represent the regions. We first consider a specialist regime in which each region admits a single strong expert, and then turn to regimes where local mixtures of experts can achieve strictly better performance.
	
	\subsubsection{Regionwise existence of a strong expert}
	We first consider the specialist regime. Fix a separating representation $v_r=(\zeta_r^\top,\psi_r)^\top$, $r\in[\Mzero]$, satisfying Assumption~\ref{ass:linearly-separable}, normalized so that $\max_{r\le\Mzero}\|v_r\|_2=1$. Assume that the slope components are separated: $r_0 := \min_{1\le r<s\le \Mzero}\|\zeta_r-\zeta_s\|_2>0$. In each region $\cX_r$, suppose there is a routed expert that provides a good local approximation to $f_0$. After relabeling the routed experts if necessary, write this region-specific specialist as $f_r^*$, $r\in[\Mzero]$. As in Section \ref{subsec:averaging_advantage}, we use $\|f_0-f_r^*\|_{\cX_r}^2$ to denote its approximation error on $\cX_r$. The following result specializes the Top-$K$ oracle inequality to the case $K=1$. It shows that, when the regions are linearly separable and each region has a strong specialist, Top-1 routing can achieve the corresponding regionwise oracle risk.
	
	\begin{corollary}\label{cor:Top1-single}
		Assume the conditions of Theorem~\ref{thm:top-k} hold with $K=1$. To rule out degenerate parameter spaces, assume in addition to Assumption~\ref{ass:compact-theta} that, after the normalization $\theta_{\MR}=0$, there exists a constant $c_0>0$ such that the centered box $[-c_0,c_0]^{(\MR-1)(d+1)}$ is contained in $\Theta_{\MR}$. Assume further that the partition satisfies Assumption~\ref{ass:linearly-separable} and that $\MR\ge\Mzero$. Let $\hat f_0$ denote the estimator produced by Algorithm~\ref{alg:general} with Top-1 routing. Then there exists a sufficiently small constant $c>0$, depending only on $\cX$ and $c_0$, such that for any $0 < r_{\MR}\le c\{r_0\wedge \MR^{-1/d}\}$, we have
		\begin{align*}
			\EE\|\hat f_0-f_0\|_{L_2(P_X)}^2 \le C\left\{\sum_{r=1}^{\Mzero}\|f_0-f_r^*\|_{\cX_r}^2+\fL_{\sparse}(\MR,n,r_{\MR})\right\},
		\end{align*}
		where $\fL_{\sparse}$ is defined in Theorem~\ref{thm:top-k}. The same bound holds for the online average of the Top-1 aggregate, and $C$ is independent of $n$, $\Mzero$, $\MR$, and $r_{\MR}$.
	\end{corollary}
	
	The bound in Corollary~\ref{cor:Top1-single} highlights the advantage of Top-1 routing in a pure specialist regime. If each region admits a strong local expert, then the regionwise approximation term $\sum_{r=1}^{\Mzero}\|f_0-f_r^*\|_{\cX_r}^2$ is small, and the remaining error is the statistical price of learning the routing rule and using evolving experts. The choice $r_{\MR}\asymp r_0\wedge\MR^{-1/d}$ uses the largest feasible separation order for which the regionwise assignment can be realized by a separated system of linear scores. The factor $\MR^{-1/d}$ arises from separating the dummy experts. For larger separation levels, such a representation is not guaranteed, so the oracle approximation risk need not be bounded by the regionwise approximation term above. Since the dependence of $\fL_{\sparse}$ on $r_{\MR}$ is through $\MR d\log(1/r_{\MR})/n$, this choice contributes $O(\MR d\log\MR/n)$ when $r_0$ is fixed.
	
	Dense softmax gating behaves differently in this setting. Since every routed expert receives positive weight when the scores are finite, dense softmax cannot represent the hard regionwise assignment exactly. Under the additional conditions in Supplement Section~\ref{subsec:supp-dense-softmax-upper-bounds-section-4-2}, the resulting approximation error is bounded by $O(\Mzero/\log\MR)$. Thus, when each region is dominated by a single specialist, hard Top-1 selection is more naturally aligned with the approximation structure.
	
	\subsubsection{Improved performance via multi-expert aggregation}
	We now consider the opposite regime, in which selecting a single expert is not sufficient. In such settings, the best local predictor may be a nontrivial combination of several experts. For clarity, we focus on mixtures of two experts. For $r\in[\Mzero]$, we say that distinct indices $i_{r1},i_{r2}\in[\MR]$ and weights $w_r=(w_{r1},w_{r2})^\top\in\Delta^1$ solve the best two-expert approximation problem on $\cX_r$ if
	\begin{align} \label{eq:bestexpert}
		\left\|f_0-w_{r1}f_{i_{r1}}^*-w_{r2}f_{i_{r2}}^*\right\|_{\cX_r}^2=\min_{\substack{j_1,j_2\in[\MR],j_1 \neq j_2\\w=(w_1,w_2)^\top\in\Delta^1}}\left\|f_0-w_1f_{j_1}^*-w_2f_{j_2}^*\right\|_{\cX_r}^2. 
	\end{align}
	The following assumption formalizes a regionwise advantage of two-expert mixtures.
	
	\begin{assumption}\label{ass:two-expert-combination}
		For each region $\cX_r$, there exist two distinct routed expert indices $i_{r1},i_{r2}\in[\MR]$ and weights $(w_{r1},w_{r2})^\top\in\Delta^1$ that solve \eqref{eq:bestexpert} on $\cX_r$. Write $f_{r,1}^*=f_{i_{r1}}^*$ and $f_{r,2}^*=f_{i_{r2}}^*$. For some $\delta_r>0$,
		\begin{align*}
			\int_{\cX_r}\{f_0(x)-w_{r1}f_{r,1}^*(x)-w_{r2}f_{r,2}^*(x)\}^2\,dP_X(x) \le \int_{\cX_r}\min_{k\in[\MR]}\{f_0(x)-f_k^*(x)\}^2\,dP_X(x)-\delta_r.
		\end{align*}
		The optimal expert pairs are disjoint across regions: $\{i_{r1},i_{r2}\}\cap\{i_{r'1},i_{r'2}\}=\emptyset$ for $r\ne r'$.
	\end{assumption}
	
	Assumption~\ref{ass:two-expert-combination} requires the best two-expert approximation on each region to improve on the pointwise best single-expert oracle by at least $\delta_r$. The disjointness condition further describes a heterogeneous setting in which different regions rely on different pairs of experts. In such regimes, Top-1 gating is intrinsically limited. It assigns one routed expert to each input and therefore cannot form a nontrivial convex combination of experts at the same point. We formalize this limitation through the population Top-1 oracle approximation risk
	\begin{align*}
		\fA_1(\MR) := \inf_{\theta\in\widetilde\Theta_{\MR}(r_{\MR})}\left\|f_0-\sum_{m=1}^{\MR}g_m^{(1,\phi)}(\cdot;\theta)f_m^*\right\|_{L_2(P_X)}^2.
	\end{align*}
	
	\begin{proposition}\label{prop: Top-1 mixing}
		Suppose Assumption~\ref{ass:two-expert-combination} holds. Then
		\begin{align*}
			\fA_1(\MR) \ge \sum_{r=1}^{\Mzero}\|f_0-w_{r1}f_{r,1}^*-w_{r2}f_{r,2}^*\|_{\cX_r}^2+\sum_{r=1}^{\Mzero}\delta_r.
		\end{align*}
		Moreover, if Assumption~\ref{ass:evolving-experts} holds, let $\hat f_0$ denote the final projected estimator produced by Algorithm~\ref{alg:general} with Top-1 routing. Then
		\begin{align*}
			\EE\|\hat f_0-f_0\|_{L_2(P_X)}^2 \ge \left[\sqrt{\fA_1(\MR)}-C\sqrt{\MR}\,\bar a_{n_3}\right]_+^2,
		\end{align*}
		where $[u]_+=\max\{u,0\}$ and $C$ is independent of $n$, $\Mzero$, and $\MR$. In particular, when $\MR\bar a_{n_3}^2$ is negligible relative to the approximation gap $\sum_{r=1}^{\Mzero}\delta_r$, the estimator lower bound reduces to the lower bound for $\fA_1(\MR)$.
	\end{proposition}
	
	Thus, relative to the best regionwise two-expert mixture, the Top-1 function class incurs an irreducible approximation error of at least $\sum_{r=1}^{\Mzero}\delta_r$. This approximation barrier is intrinsic to the Top-1 function class and arises in addition to any error due to estimation.
	
	Soft aggregation behaves differently. Dense softmax routing, and sparse Top-$K$ routing with $K\ge2$, can assign nontrivial weights to multiple experts and can therefore approximate regionwise mixtures. Under suitable conditions, dense softmax gating yields an upper bound of the form
	\begin{align*}
		\EE\|\hat f_0-f_0\|_{L_2(P_X)}^2
		\lesssim \sum_{r=1}^{\Mzero}\|f_0-w_{r1}f_{r,1}^*-w_{r2}f_{r,2}^*\|_{\cX_r}^2+\fL_{\soft}(0,\MR)+\frac{\Mzero}{\log\MR},
	\end{align*}
	where $\fL_{\soft}$ is defined in Theorem \ref{thm:softmax}; see Section~\ref{subsec:supp-dense-softmax-upper-bounds-section-4-2} of the Supplement. The first term is the oracle error of the best regionwise two-expert mixture, the second is the statistical learning cost, and the final term bounds the error from approximating the regionwise selection weights with dense softmax gating. This contrasts with Top-1 routing, whose excess term $\sum_r\delta_r$ remains at the oracle approximation level whenever mixtures are strictly superior.
	
	These results clarify the role of $K$ in sparse MoE. When each region has a dominant specialist, hard Top-1 selection can be sufficient. When prediction is locally compositional, however, activating multiple experts has a statistical role: it allows the model to represent local mixtures that cannot be captured by single-expert assignments. This interpretation is consistent with practical MoE design, where Top-1 routing is used in models such as Switch Transformer \citep{fedus2022switch}, while architectures such as Mixtral \citep{jiang2024mixtral} and DeepSeekMoE \citep{dai2024deepseekmoe} activate multiple experts per token.

	\section{Adaptation of the gating function}\label{sec:adaptation}
	\subsection{Performance-induced regions and gating approximation losses}
	The previous section compared hard selection and soft aggregation under stylized oracle partitions. In practice, however, the relevant regions are not known in advance and may have a geometry that is far more complex than a linearly separable partition. This raises a more basic question: how should the gating class be chosen so that it adapts to the regions induced by expert performance?
	
	We first introduce an oracle device that describes where each routed expert is locally most accurate. For clarity, we omit shared experts in this section as before. In contrast to the prespecified linearly separable partition in Assumption~\ref{ass:linearly-separable}, the definition below constructs an oracle partition directly from pointwise expert performance. The resulting regions provide a benchmark for assessing how well a gating class aligns with expert specialization.
	
	\begin{definition}\label{def:region}
		For each $x \in \cX$, define
		\begin{align*}
			r^*(x) := \min\left\{r \in [\MR]:r \in \argmin_{j \in [\MR]} |f_j^*(x)-f_0(x)|\right\},
		\end{align*}
		where ties are resolved in favor of the smallest expert index. For each routed expert $r \in [\MR]$, define $\cX_r := \{x \in \cX:r^*(x)=r\}$.
	\end{definition}
	
	Let $\cG = \{g_{\theta}(\cdot) = (g_{\theta,1}(\cdot),\ldots,g_{\theta,\MR}(\cdot))^\top:\theta \in \Theta_{\MR}\}$ be a gating class, such as dense softmax gating, Top-$K$ gating, or
	quadratic-softmax gating. For a generic $g \in \cG$, write $g_m$ for its $m$-th component. After omitting shared experts, the oracle approximation term in Theorems~\ref{thm:softmax} and \ref{thm:top-k} has the form
	\begin{align*}
		\inf_{g \in \cG} \left\|f_0-\sum_{m=1}^{\MR} g_m(\cdot) f_m^*\right\|_{L_2(P_X)}^2.
	\end{align*}
	Define the oracle assignment map $g^*(x) := \sum_{r=1}^{\MR} e_r \one_{\cX_r}(x)$, where $e_r$ is the $r$-th canonical basis vector in $\RR^{\MR}$. This map assigns each input to the expert that is pointwise closest to $f_0$. When the routed experts are uniformly bounded, this approximation term can be related to the geometry of the performance-induced partition through the decomposition
	\begin{align*}
		\inf_{g \in \cG} \left\|f_0-\sum_{m=1}^{\MR} g_m(\cdot) f_m^*\right\|_{L_2(P_X)}^2 \le C\left\{\sum_{r=1}^{\MR} \|f_0-f_r^*\|_{\cX_r}^2+\inf_{g \in \cG} \|g-g^*\|_{L_2(P_X)}^2\right\}.
	\end{align*}
	Here $C$ is independent of $n$ and $\MR$. The first term is the oracle expert approximation error since $f_r^*$ is the pointwise best routed expert on $\cX_r$. Once this expert approximation error is fixed, the remaining term $\inf_{g \in \cG} \|g-g^*\|_{L_2(P_X)}^2$ measures how well the gating class can represent the corresponding oracle regionwise assignment.
	
	This motivates the following global region-assignment loss:
	\begin{align} \label{eq:L1loss}
		\cL^{(1)}(\cG) := \inf_{g \in \cG} \|g-g^*\|_{L_2(P_X)}^2 = \inf_{g \in \cG} \sum_{r=1}^{\MR} \|(e_r-g)\one_{\cX_r}\|_{L_2(P_X)}^2.
	\end{align}
	Equivalently, this loss sums the squared discrepancy between $g$ and the corresponding canonical assignment vector over all oracle regions. The quantity $\cL^{(1)}(\cG)$ measures how well the gating class can approximate the oracle assignment on the whole input space. 
	
	Near region boundaries, however, the identity of the best expert may be intrinsically unstable: two experts may have nearly identical approximation error, and small perturbations of either the experts or the input can change the oracle label. We therefore also consider an interior version of the loss. Fix $\rho > 0$, define the shrunken region
	\begin{align}\label{eq:shrunken-region}
		\cX_r^{-\rho} := \{x \in \cX_r:\operatorname{dist}(x,\partial \cX_r) \ge \rho\}, \qquad r \in [\MR],
	\end{align}
	where $\partial \cX_r$ denotes the boundary of $\cX_r$ and $\operatorname{dist}(x,A)$ is the Euclidean distance from $x$ to a set $A$. The interior region-assignment loss is
	\begin{align} \label{eq:L2loss}
		\cL^{(2)}(\cG;\rho) := \inf_{g \in \cG} \sum_{r=1}^{\MR} \|(e_r-g)\one_{\cX_r^{-\rho}}\|_{L_2(P_X)}^2.
	\end{align}
	Thus, $\cL^{(2)}(\cG;\rho)$ measures how well the gating class approximates the oracle assignment on region interiors, while deliberately ignoring a $\rho$-neighborhood of the boundaries. The two losses capture complementary aspects of gating performance: $\cL^{(1)}(\cG)$ evaluates global alignment including boundary behavior, whereas $\cL^{(2)}(\cG;\rho)$ focuses on whether the gate makes the correct assignments away from ambiguous boundary regions.
	
	\subsection{Linear and quadratic gating: examples of geometric alignment}
	We use two simple examples to illustrate the roles of $\cL^{(1)}$ and $\cL^{(2)}$. For concreteness, in addition to Assumption~\ref{ass:compact-theta}, we assume that each free gating parameter in these examples is restricted to $[-C\log\MR,C\log\MR]$ for some constant $C>1/2$. The first example isolates boundary smoothing from interior assignment, while the second shows how the interior loss reflects geometric alignment between the gating class and the oracle partition. Details are given in Supplement Sections \ref{subsec:supp-proof-ex-5-2} and \ref{subsec:supp-proof-ex-5-3}.
	
	\begin{example}[Global versus interior region-assignment loss]\label{ex:L1-L2}
		Consider $\cX=[-1,1]$ with $X\sim\Unif[-1,1]$. The two oracle regions are $[-1,0]$ and $(0,1]$, with oracle assignments $e_1$ and $e_2$, respectively. Top-1 linear gating can implement the threshold at zero exactly, and hence $\cL^{(1)}(\cG_{\mathrm{Top\mbox{-}1}})=0$ and $\cL^{(2)}(\cG_{\mathrm{Top\mbox{-}1}};\rho)=0$. Dense linear-softmax gating with bounded scores cannot represent this discontinuous assignment exactly. Under the logarithmically growing parameter box in Assumption \ref{ass:compact-theta}, its global boundary-smoothing error satisfies $\cL^{(1)}(\cG_{\soft})\asymp \frac{1}{\log \MR}$. In contrast, once a fixed $\rho$-neighborhood of the boundary is removed, dense softmax can create a large score margin on the remaining interior regions, so that $\cL^{(2)}(\cG_{\soft};\rho)\lesssim \MR^{-c\rho}$ for some constant $c>0$. Thus $\cL^{(1)}$ captures the cost of smoothing sharp boundaries, whereas $\cL^{(2)}$ measures assignment accuracy away from the boundary.
	\end{example}
	
	\begin{example}[Linear-softmax versus quadratic-softmax]\label{exa:linear-quadratic-l1-l2}
		The interior loss also reflects whether the gating geometry matches the oracle partition. Linear score differences generate affine decision boundaries, whereas quadratic score differences can generate curved boundaries. For a partition of a disk into an inner ball and an outer annulus, the oracle boundary is circular. Hence linear-softmax remains geometrically misspecified: for any fixed $\rho > 0$, $\cL^{(2)}(\cG_{\lin};\rho)\ge C_L$ for some constant $C_L > 0$, independent of $\MR$. By contrast, quadratic-softmax can generate radial score differences, and under the same logarithmic parameter scaling, satisfies $\cL^{(2)}(\cG_{\qdr};\rho)\le C_R\MR^{-\alpha}$ for some constants $C_R, \alpha > 0$, independent of $\MR$. Hence, $\cL^{(2)}$ detects geometric mismatch: linear gating is adequate for affine partitions, while quadratic gating can substantially improve approximation for curved partitions.
	\end{example}
	
	Taken together, these examples show that $\cL^{(1)}$ and $\cL^{(2)}$ emphasize different aspects of gating approximation. The global loss penalizes boundary smoothing, while the interior loss focuses on assignment away from boundary layers and reveals geometric mismatch between the gating class and the expert-induced partition. 
	
	\subsection{Gaussian-kernel-based gating function}
	The preceding examples show that the effectiveness of a gating class depends on its geometric alignment with the expert-induced regions. When the geometry of these regions cannot be captured by a given parametric score class, the resulting gating approximation error may remain nonnegligible. This motivates a more flexible construction whose approximation ability is not tied to a prespecified boundary geometry. We now introduce the Gaussian-kernel-based gating class that can approximate regionwise hard assignments and soft mixtures over partitions with controlled boundary complexity.
	
	Following Definition~\ref{def:region}, let $\cX = [0,1]^d$ be partitioned into $\MR$ regions $\{\cX_r\}_{r=1}^{\MR}$. To quantify the contribution of boundary layers to the kernel approximation error, we impose the following regularity condition on the partition.
	
	\begin{definition}\label{def:boundary-regularity-main}
		Let $\mathcal P_{\MR}=\{\cX_i\}_{i=1}^{\MR}$ be a measurable partition of $[0,1]^d$. Recall the shrunken regions $\cX_i^{-\rho}$ defined in \eqref{eq:shrunken-region}, and for $\rho>0$ define $B_i(\rho):=\cX_i\setminus\cX_i^{-\rho}$. We say that $\mathcal P_{\MR}$ has boundary-layer regularity if there exist $\rho_0>0$ and constants $\gamma_1, \dots, \gamma_{\MR}<\infty$ such that $\mu\bigl(B_i(\rho)\bigr)\le\gamma_i\rho$ for all $i\in[\MR]$ and $0<\rho\le\rho_0$, where $\mu$ denotes the Lebesgue measure on $\cX$. We write $\Gamma(\mathcal P_{\MR}):=\sum_{i=1}^{\MR}\gamma_i$ for the total boundary-layer complexity of the partition, and use $\Gamma_{\MR}$ as shorthand when the partition is clear.
	\end{definition}
	
	This condition rules out pathological partitions whose boundary layers have disproportionately large volume. The quantity $\Gamma_{\MR}$ summarizes the cumulative boundary complexity of the partition. It may depend on $\MR$, so the approximation bound below explicitly records the cost of smoothing increasingly refined partitions.
	
	Unlike the canonical assignment $x \mapsto e_r$, we allow each region to be associated with a general target weight vector $w_r^* = (w_{r,1}^*,\ldots,w_{r,\MR}^*)^\top \in \simplex{\MR}, r \in [\MR]$. 
	Thus, on region $\cX_r$, the ideal gating target need not select a single routed expert; it may assign an arbitrary convex combination of the $\MR$ routed experts.
	
	We construct the gating function using Gaussian kernels centered on a fixed, data-independent grid. Let $h>0$ denote the mesh width of the grid, and define
	\begin{align*}
		\cC = \{(k_1h,\ldots,k_dh):k_\ell \in \{0,1,\ldots,\lfloor h^{-1} \rfloor\}\} \cap [0,1]^d.
	\end{align*}
	The number of kernel centers is of order $h^{-d}$. Thus $h$ controls the approximation power of the kernel grid, while through the number of centers it also determines the statistical complexity of the gating class.
	
	For each region $\cX_r$, let $C_r = \cC \cap \cX_r$ denote the set of grid centers lying in that region. Define the Gaussian kernel
	\begin{align*}
		\varphi(x,c;\sigma) = \exp\left\{-\frac{\|x-c\|_2^2}{2\sigma^2}\right\}, \qquad \sigma = \kappa h, \qquad \kappa \in \left[\frac{1}{2},1\right].
	\end{align*}
	The parameter space is
	\begin{align*}
		\Theta_{\mathrm{ker}} = \left\{w = (w_{j,c})_{1 \le j \le \MR,\ c \in \cC}:w_{j,c} \ge 0,\ \sum_{j=1}^{\MR} w_{j,c} = 1,\ \forall c \in \cC\right\}.
	\end{align*}
	The induced Gaussian-kernel gating class is
	\begin{align*}
		\cG_{\mathrm{ker}} = \{g(\cdot;w) = (g_1(\cdot;w),\ldots,g_{\MR}(\cdot;w))^\top:w \in \Theta_{\mathrm{ker}}\},
	\end{align*}
	where
	\begin{align*}
		g_j(x;w) = \frac{\sum_{c \in \cC} w_{j,c}\varphi(x,c;\sigma)}{\sum_{\ell=1}^{\MR}\sum_{c \in \cC} w_{\ell,c}\varphi(x,c;\sigma)}, \qquad j \in [\MR].
	\end{align*}
	Equivalently, each grid center $c$ carries a local weight vector in $\simplex{\MR}$. The Gaussian kernels interpolate these local weights across the input space, and the normalization in $g_j(x;w)$ converts the kernel-weighted averages into a valid gating vector. The key point is that this construction can approximate not only hard region assignments, but also arbitrary regionwise target weight vectors.
	
	To see the approximation mechanism, consider the following oracle choice
	\begin{align*}
		w_{j,c} = \sum_{r=1}^{\MR} w_{r,j}^* \one_{\{c \in C_r\}}, \qquad j \in [\MR], \qquad c \in \cC.
	\end{align*}
	Under this assignment, each center $c \in C_r$ carries the target weight vector $w_r^*$. Since $w_r^* \in \simplex{\MR}$, this particular construction satisfies $\sum_{j=1}^{\MR} w_{j,c} = 1, c \in \cC$. Thus, this oracle choice belongs to $\Theta_{\mathrm{ker}}$. For this choice of $w$, the resulting gating function simplifies to
	\begin{align*}
		g_j(x;w) = \frac{\sum_{r=1}^{\MR} w_{r,j}^* \sum_{c \in C_r} \varphi(x,c;\sigma)}{\sum_{c \in \cC} \varphi(x,c;\sigma)}.
	\end{align*}
	This formula has a simple geometric interpretation. If $x$ lies in the interior of $\cX_r$, then the Gaussian kernels concentrate their weight on nearby centers in $C_r$, while contributions from centers in other regions become negligible. Consequently, $g(x;w) \approx w_r^*$ away from the region boundaries. Accordingly, we define $\cL^{(1)}(\cG_{\mathrm{ker}})$ and $\cL^{(2)}(\cG_{\mathrm{ker}};\rho)$ as in \eqref{eq:L1loss} and \eqref{eq:L2loss}, respectively, with the canonical vectors $e_r$ replaced by the general target weights $w_r^*$.
	
	\begin{theorem}\label{thm:kernel-global-main}
		Assume that $\cC$ is a uniform grid with mesh width $h$, and that the density $p_X$ of $P_X$ is continuous on $[0,1]^d$. Then, for any measurable partition $\mathcal P_{\MR}=\{\cX_i\}_{i=1}^{\MR}$ of $[0,1]^d$ satisfying the boundary-layer regularity condition in Definition~\ref{def:boundary-regularity-main}, any target weights $w_1^*,\ldots,w_{\MR}^*\in\simplex{\MR}$, and any $h$ satisfying
		\begin{align*}
			0<h<\exp\left\{-\frac{1}{d/2+1}\max\left\{1,\frac{d}{16\kappa^2}\right\}\right\},
		\end{align*}
		the Gaussian-kernel-based gating class satisfies
		\begin{align*}
			\cL^{(1)}(\cG_{\mathrm{ker}}) \le C\left\{\MR h^{d+2}+\Gamma_{\MR}h\sqrt{\log(1/h)}\right\},
		\end{align*}
		where $C$ is a constant independent of $h$, $\Gamma_{\MR}$ and $\MR$.
	\end{theorem}
	
	Theorem~\ref{thm:kernel-global-main} shows that $\cG_{\mathrm{ker}}$ can approximate arbitrary regionwise target weights over partitions with controlled boundary layers. The first term reflects the interior approximation error from the Gaussian-kernel grid, while the second term reflects the boundary smoothing error. In particular, the boundary contribution is governed by $\Gamma_{\MR}$, which may grow with $\MR$. The canonical hard-assignment case is recovered by taking $w_r^*=e_r$, so the result extends regionwise expert selection to regionwise expert mixing.
	
	\begin{remark}[Resolution and learning cost]
		If $\max_{1\le i\le \MR}\gamma_i=O(1)$, then $\Gamma_{\MR}=O(\MR)$. Taking, for example, $h\asymp \MR^{-2}$ yields
		\begin{align*}
			\MR h^{d+2}+\Gamma_{\MR}h\sqrt{\log(1/h)} \lesssim \frac{\sqrt{\log \MR}}{\MR}.
		\end{align*}
		Thus, a sufficiently fine kernel grid yields a vanishing oracle approximation error. At the same time, a smaller $h$ requires more kernel centers and hence more parameters, which increases the statistical learning cost. Therefore, in estimation problems, the choice of $h$ should balance the oracle approximation error against the statistical cost of learning the enlarged gating class.
	\end{remark}
	
	The preceding theorem evaluates the global loss and therefore includes a boundary-layer contribution. Restricting attention to points that stay a fixed distance away from the region boundaries removes this contribution and yields the sharper interior bound below. 

	\begin{corollary}\label{cor:kernel-interior-main}
		Under the same conditions as in Theorem~\ref{thm:kernel-global-main}, fix a constant $\rho>0$. Then, for any $h$ satisfying $0<h<\min\left\{\frac{2\rho}{\sqrt d},\frac{\rho}{2\kappa}\right\}$, the Gaussian-kernel-based gating class satisfies
		\begin{align*}
			\cL^{(2)}(\cG_{\mathrm{ker}};\rho) \le C \MR \exp\left\{-\frac{\rho^2}{2\kappa^2h^2}\right\},
		\end{align*}
		where $C$ is a constant independent of $h$ and $\MR$.
	\end{corollary}
	
	\begin{remark}
		Compared with the global loss $\cL^{(1)}$, the interior bound eliminates the boundary-layer term $\Gamma_{\MR}h\sqrt{\log(1/h)}$. For fixed $\rho > 0$, the remaining approximation error decays exponentially in $h^{-2}$, up to the multiplicative factor $\MR$.
	\end{remark}
	
	The preceding analysis shows that the approximation ability of a gating class depends on how well it captures the geometry of expert-induced regions. To illustrate this point in finite samples, Table \ref{tab:mc-summary} compares linear, quadratic and Gaussian-kernel-based gating under partitions with different boundary geometries. In each design, the lowest gating weight errors are attained by the gating class most closely aligned with the geometry of the underlying partition. This pattern illustrates the role of geometric alignment in the choice of gating class; see Supplement Section \ref{subsec:supp-comparison-different-gating-class} for simulation details and visualizations of the region partitions. 
	
	\begin{table}[ht]
		\centering
		\caption{Monte Carlo performance of gating classes under three DGPs}
		\label{tab:mc-summary}
		\scalebox{0.7}{
			\begin{threeparttable}
				\begin{tabular}{crrrrrrrrr}
					\toprule
					\multirow{2}[4]{*}{Gating class} & \multicolumn{3}{c}{Linear DGP} & \multicolumn{3}{c}{Quadratic DGP} & \multicolumn{3}{c}{Nonlinear DGP} \\
					\cmidrule(lr){2-4} \cmidrule(lr){5-7} \cmidrule(lr){8-10}
					& $n$ & $\ell_1$ error (SD) & $\ell_2^2$ error (SD) & $n$ & $\ell_1$ error (SD) & $\ell_2^2$ error (SD) & $n$ & $\ell_1$ error (SD) & $\ell_2^2$ error (SD) \\
					\midrule
					\multirow[t]{2}[1]{*}{$\mathcal{G}_{\lin}$} & 200 & \textbf{0.100} (0.022) & \textbf{0.006} (0.003) & 200 & 0.230 (0.033) & 0.030 (0.012) & 500 & 0.745 (0.062) & 0.339 (0.051) \\
					& 400 & \textbf{0.081} (0.023) & \textbf{0.004} (0.002) & 400 & 0.220 (0.019) & 0.027 (0.006) & 1000 & 0.713 (0.033) & 0.322 (0.025) \\
					\multirow[t]{2}[0]{*}{$\mathcal{G}_{\qdr}$} & 200 & 0.178 (0.045) & 0.018 (0.011) & 200 & \textbf{0.141} (0.062) & \textbf{0.014} (0.016) & 500 & 0.840 (0.041) & 0.352 (0.035) \\
					& 400 & 0.151 (0.040) & 0.013 (0.007) & 400 & \textbf{0.133} (0.061) & \textbf{0.013} (0.013) & 1000 & 0.821 (0.030) & 0.339 (0.025) \\
					\multirow[t]{2}[1]{*}{$\mathcal{G}_{\mathrm{ker}}$} & 200 & 0.211 (0.052) & 0.026 (0.013) & 200 & 0.204 (0.042) & 0.023 (0.011) & 500 & \textbf{0.711} (0.028) & \textbf{0.267} (0.021) \\
					& 400 & 0.188 (0.043) & 0.020 (0.009) & 400 & 0.181 (0.027) & 0.018 (0.006) & 1000 & \textbf{0.713} (0.019) & \textbf{0.266} (0.014) \\
					\bottomrule
				\end{tabular}
				\vspace{1ex}
				{\raggedright
					Note: Linear, quadratic and nonlinear DGPs refer to region partitions with linear, quadratic and nonlinear boundaries, respectively. The gating classes $\mathcal{G}_{\lin}, \mathcal{G}_{\qdr}$ and $\mathcal{G}_{\mathrm{ker}}$ denote dense softmax gating with linear scores, dense softmax gating with quadratic scores and Gaussian-kernel-based gating, respectively. The $\ell_1$ error is the sample average of the $\ell_1$ distance between the estimated and true gating-weight vectors, while the $\ell_2^2$ error is the sample average of their squared Euclidean distance. Reported values are Monte Carlo means, with standard deviations in parentheses. For each DGP and sample size, the smallest value of each error measure across gating classes is shown in bold. The results show that gating performance is closely related to the geometric alignment between the gating class and the underlying region partition: linear, quadratic and Gaussian-kernel-based gating perform best under linear, quadratic and nonlinear designs, respectively. 
					\par}
			\end{threeparttable}
		}
	\end{table}

	\subsection{Selection procedure for different classes of gating functions}
	The preceding subsections show that the performance of a gating rule depends on its alignment with the geometry of the expert-induced regions. In practice, neither the relative performance of the experts nor the induced region partition is known a priori. Although the Gaussian-kernel construction provides a flexible approximation template, it may not always be the most efficient choice: a simpler linear or quadratic gate can be preferable when the region geometry is correspondingly simple. This motivates a data-adaptive procedure that selects among candidate gating classes according to their validation performance.
	
	Suppose that we are given $L \ge 2$ candidate gating classes $\{\cG_\ell\}_{\ell=1}^{L}$. For each class $\cG_\ell$, define the corresponding population MoE oracle class
	\begin{align*}
		\cF_\ell^* = \left\{x \mapsto \sum_{m=1}^{\MR} g_m(x) f_m^*(x):g = (g_1,\ldots,g_{\MR})^\top \in \cG_\ell\right\}.
	\end{align*}
	Write $\fA_\ell := \inf_{f \in \cF_\ell^*}\|f_0-f\|_{L_2(P_X)}^2$ for the oracle approximation risk of the $\ell$-th class. Using Algorithm~\ref{alg:general}, or another class-specific estimation method with the same type of oracle guarantee, we construct a final projected estimator $\check f_\ell$ for each $\ell \in [L]$. We assume a common envelope bound $\max_{\ell\in[L]}\|\check f_\ell\|_\infty\le A_0$, where $A_0$ is independent of $n$ and $L$. We also assume that these candidate estimators satisfy the uniform class-specific bound
	\begin{align}
		\EE\|f_0-\check f_\ell\|_{L_2(P_X)}^2 \le C_\ell\left\{\fA_\ell+\fL_\ell\right\}, \qquad \ell \in [L], \label{eq:2}
	\end{align}
	where $\fL_\ell$ denotes the statistical learning cost associated with the gating class $\cG_\ell$, and the constants $C_\ell$ are uniformly bounded by a constant independent of $n$ and $L$. This condition summarizes the class-specific results established in the previous sections. For example, a dense softmax candidate has $\fA_\ell=\fA_{\soft}(\MS,\MR)$ and $\fL_\ell=\fL_{\soft}(\MS,\MR)$, whereas a Top-$K$ candidate has $\fA_\ell=\fA_{K,\phi}(\MS,\MR)$ and $\fL_\ell=\fL_{\sparse}(\MR,n,r_{\MR})$, with the sample sizes in these learning costs replaced by the corresponding block sizes inside the estimation sample.

	To select the best-performing estimator among the candidates, we retain the aggregation-and-projection principle of Algorithm \ref{alg:general}, but adapt the procedure to aggregate the candidate estimators using an independent sample and project the aggregate back onto the candidate list, as presented in Algorithm \ref{alg:selection}.
	
	\begin{algorithm}[!htbp]
		\caption{Aggregation-projection over gating-induced MoE classes}\label{alg:selection}
		{\small
			\textbf{Input:} A data set $\mathcal D_{\mathrm{sel}}=\{(X_i,Y_i)\}_{i=1}^n$ reserved for candidate construction and class selection; candidate classes $\{\cG_\ell\}_{\ell=1}^L$.\\
			\textbf{Output:} Selected gating class $\cG_{\ell^*}$ and final estimator $\check{f}_{\ell^*}$.
			\begin{enumerate}[label=(\alph*)]
				\item \textbf{Sample splitting.} Randomly permute and relabel the observations, and split the resulting sample into three independent blocks:
				\begin{align*}
					Z^{(1)}=(X_i,Y_i)_{i=1}^{n_1}, \qquad Z^{(2)}=(X_i,Y_i)_{i=n_1+1}^{n_1+n_2}, \qquad Z^{(3)}=(X_i,Y_i)_{i=n_1+n_2+1}^{n},
				\end{align*}
				and define $n_3=n-n_1-n_2$.
				
				\item \textbf{Candidate construction.} For each candidate class $\cG_\ell$, use $Z^{(1)}$ to construct a final projected estimator $\check f_\ell$ satisfying the class-specific oracle bound
				\begin{align*}
					\EE\|f_0-\check f_\ell\|_{L_2(P_X)}^2 \le C_\ell\left\{\fA_\ell+\fL_\ell\right\}, \qquad \ell \in [L].
				\end{align*}
				
				\item \textbf{Calibration and aggregation.} Conditional on the fitted list $\{\check f_\ell\}_{\ell=1}^L$, compute the calibration residual variance
				\begin{align*}
					\tilde\sigma_\ell^2 = \frac{1}{n_2}\sum_{(X_i,Y_i)\in Z^{(2)}}\bigl(Y_i-\check f_\ell(X_i)\bigr)^2, \qquad \hat\sigma_\ell^2 = \min\left\{\bar\sigma^2,\max\left\{\underline\sigma^2,\tilde\sigma_\ell^2\right\}\right\}, \qquad \ell \in [L].
				\end{align*}
				Initialize $\hat\omega_{\ell,n_1+n_2}=1/L$ for $\ell\in[L]$. For $t=n_1+n_2,\ldots,n-1$, first form the one-step-ahead weighted predictor
				\begin{align*}
					\tilde f^{t}(x) = \sum_{\ell=1}^{L}\hat\omega_{\ell,t}\check f_\ell(x),
				\end{align*}
				and then, after observing $(X_{t+1},Y_{t+1})$, update the aggregation weights by
				\begin{align*}
					\hat\omega_{\ell,t+1} = \frac{\hat\omega_{\ell,t}\hat\sigma_\ell^{-1} h_{\varepsilon,0}\{(Y_{t+1}-\check f_\ell(X_{t+1}))/\hat\sigma_\ell\}}{\sum_{r=1}^{L}\hat\omega_{r,t}\hat\sigma_r^{-1} h_{\varepsilon,0}\{(Y_{t+1}-\check f_r(X_{t+1}))/\hat\sigma_r\}}.
				\end{align*}
				Define the aggregated predictor by the online-to-batch average
				\begin{align*}
					\tilde f_n(x) = \frac{1}{n_3}\sum_{t=n_1+n_2}^{n-1}\tilde f^{t}(x) = \sum_{\ell=1}^{L}\bar\omega_\ell\check f_\ell(x), \qquad \bar\omega_\ell = \frac{1}{n_3}\sum_{t=n_1+n_2}^{n-1}\hat\omega_{\ell,t}.
				\end{align*}
				
				\item \textbf{Projection and output.} Project the aggregate $\tilde{f}_n$ back to the candidate list by the population $L_2(P_X)$ criterion,
				\begin{align*}
					\ell^* = \argmin_{\ell \in [L]}\|\check{f}_\ell-\tilde{f}_n\|_{L_2(P_X)}^2.
				\end{align*}
				Output the selected gating class $\cG_{\ell^*}$ and the final estimator $\check{f}_{\ell^*}$.
			\end{enumerate}
		}
	\end{algorithm}

	We now analyze the selected estimator and show that the procedure adapts to the unknown gating structure by nearly matching the best candidate class, up to a logarithmic selection cost. We first present the result in a global $L_2(P_X)$ risk framework and then derive a regionwise counterpart that explicitly separates the expert approximation error from the gating approximation error. We begin with the global case. 
	
	\begin{theorem}\label{thm:selection(global)}
		Suppose Assumptions~\ref{ass:bounded-experts} and \ref{ass:ARM-regularity} hold. Assume that the candidate estimators constructed in Algorithm~\ref{alg:selection} satisfy \eqref{eq:2}. If $n_1\asymp n_2\asymp n_3\asymp n$, then the selected estimator $\check{f}_{\ell^*}$ satisfies
		\begin{align*}
			\EE\|f_0-\check{f}_{\ell^*}\|_{L_2(P_X)}^2 \le C\left[\inf_{\ell \in [L]}\left\{\fA_\ell+\fL_\ell\right\}+\frac{\log L}{n}\right],
		\end{align*}
		where $C$ is a constant independent of $n$ and $L$.
	\end{theorem}
	
	The leading term $\inf_{\ell \in [L]}\left\{\fA_\ell+\fL_{\ell}\right\}$ is the best achievable risk among the candidate gating-induced MoE classes. It balances the approximation power of each gating class with its statistical learning cost. The additive term $(\log L)/n$ is the price of adapting over $L$ candidate classes.
	
	The same result can be expressed in regionwise form.
	
	\begin{corollary}\label{cor:selection(regionwise)}
		Under the conditions of Theorem~\ref{thm:selection(global)}, and with the expert-wise optimal regions defined as in Definition \ref{def:region}, the selected estimator $\check{f}_{\ell^*}$ satisfies
		\begin{align*}
			\EE\|f_0-\check{f}_{\ell^*}\|_{L_2(P_X)}^2 \le C\left[\sum_{r=1}^{\MR}\|f_0-f_r^*\|_{\cX_r}^2+\inf_{\ell \in [L]}\left\{\cL^{(1)}(\cG_\ell)+\fL_{\ell}\right\}+\frac{\log L}{n}\right],
		\end{align*}
		where $\cL^{(1)}(\cG_\ell)$ is defined in \eqref{eq:L1loss} and $C$ is a constant independent of $n$ and $L$.
	\end{corollary}
	
	This regionwise inequality separates three components: the approximation error of the population experts on their locally optimal regions, the best achievable gating approximation error among the candidate classes, and the logarithmic cost of selecting the class from data. Thus, the procedure adapts to the unknown geometry of expert performance by choosing a gating class whose approximation complexity and statistical learning cost are well balanced.

	\section{The role of shared experts}\label{sec:shared}
    The previous sections allowed for shared experts as given components, without isolating their statistical contribution. We now examine what is gained by combining a shared expert with routed experts. In contrast to routed experts, shared experts remain active for all inputs and are designed to capture information that is useful across different regions of the input space. This architectural idea is particularly prominent in recent large-scale models such as DeepSeekMoE \citep{dai2024deepseekmoe}, where shared experts are introduced to reduce redundancy among routed experts and to encourage local specialization. Recent statistical work has also begun to analyze the sample-efficiency benefit of shared experts in DeepSeek-type architectures \citep{nguyen2025deepseekmoe}. Our analysis is complementary: instead of studying full parameter convergence in a specified MoE model, we focus on how the presence of a shared component changes the statistical learning problem faced by the routed experts. 

    In this section, we study the statistical role of shared experts from a simplified but revealing perspective. Rather than analyzing the full joint training dynamics of experts and the router, which is highly non-convex, we fix the routing partition and the regional learning rules and compare the estimation problems induced by shared-routed and pure-routed representations. This comparison isolates the statistical effect of introducing an explicit shared component. 
    
    The resulting comparison is learner-specific. The benefit of a shared expert does not require the shared component to be simpler in an absolute sense. Rather, the benefit appears when the common component is better matched to the shared learning rule than to the routed learning rules. In that case, separating common from region-specific structure can reduce the effective complexity of the targets assigned to the routed experts and improve estimation efficiency. 
	
	\subsection{Model and learning rules}
	Let $\cX = \cX_1\cup\cX_2$ be a fixed Top-1 routing partition with $\cX_1 \cap \cX_2 = \emptyset$. Suppose that $Y_i = f_0(X_i) + \varepsilon_i$, where $\EE[\varepsilon_i \mid X_i] = 0$ and $\varepsilon_i$ is conditionally sub-Gaussian with variance proxy $\sigma^2$. We assume that the target regression function admits the common-local decomposition
	\begin{align}
		f_0(x) = f_S(x) + \one_{\cX_1}(x) f_{R,1}(x) + \one_{\cX_2}(x) f_{R,2}(x), \qquad x \in \cX, \label{eq:model_shared}
	\end{align}
	where $f_S \in \cF_S$ is a common component and $f_{R,j} \in \cF_j$, $j = 1,2$, are region-specific residual components. The corresponding shared-representation class is
	\begin{align*}
		\cA_{\mathrm{sh}} := \{f_S + \one_{\cX_1} f_{R,1} + \one_{\cX_2} f_{R,2}:f_S \in \cF_S,\ f_{R,j} \in \cF_j,\ j = 1,2\}.
	\end{align*}
	To understand the role of the shared expert, we compare this representation with a pure-routed representation using the same routing partition. Without an explicit shared component, the branch function on region $\cX_j$ is $	H_j(x) := f_S(x) + f_{R,j}(x)$ for $x \in \cX_j$.
	Thus the pure-routed architecture represents the same target as $f_0(x) = \one_{\cX_1}(x) H_1(x) + \one_{\cX_2}(x) H_2(x),$
	provided that the regional routed class is enlarged to $\cH_j := \cF_S|_{\cX_j} + \cF_j$, where $\cF_S|_{\cX_j}$ denotes the restriction of the shared class to region $\cX_j$. Thus, at the level of approximation, the pure-routed representation can match the shared representation if its regional learners are rich enough to contain the full branch functions. The distinction is statistical rather than purely representational: with a shared expert, the routed learner on region $\cX_j$ estimates only the residual component $f_{R,j}$; without a shared expert, the same routed learner must estimate the full branch function $H_j = f_S + f_{R,j}$.
	
	Let $\cM_S$ denote the learning rule designed for the shared class $\cF_S$, and let $\cM_j$ denote the routed learning rule used on region $\cX_j$. Under the shared representation, write $\cM_{\mathrm{sh}}$ for the combined shared-residual learning rule that uses $\cM_S$, $\cM_1$, and $\cM_2$. Given the full training sample $\cD_n = \{(X_i,Y_i)\}_{i=1}^n$, this combined rule produces $(\hat{f}_S,\hat{f}_{R,1},\hat{f}_{R,2}) = \cM_{\mathrm{sh}}(\cD_n).$
	The combined rule may be implemented, for example, by joint least squares or penalized least squares over the additive shared-residual class. It may also be implemented by a two-stage procedure that first estimates the shared component and then applies the regional learners to residuals, provided that the resulting residualization step is well posed. The resulting shared estimator is
	\begin{align*}
		\hat{f}_{0,\mathrm{sh}}(x) := \hat{f}_S(x) + \one_{\cX_1}(x) \hat{f}_{R,1}(x) + \one_{\cX_2}(x) \hat{f}_{R,2}(x).
	\end{align*}
	
	Under the pure-routed representation, there is no explicit shared expert. The same regional learning rule $\cM_j$ is applied directly to the full branch target $H_j(x)$. Equivalently, the routed learner is now required to estimate an element of the larger branch class $\cH_j = \cF_S|_{\cX_j} + \cF_j$, rather than only the residual class $\cF_j$. Let $\cD_{n,j} = \{(X_i,Y_i):X_i \in \cX_j\}$ denote the regional training sample. The branch estimators are $\hat{H}_j = \cM_j(\cD_{n,j})$, $j\in\{1,2\}$, and the corresponding pure-routed estimator is
	\begin{align*}
		\hat{f}_{0,\mathrm{rt}}(x) := \one_{\cX_1}(x) \hat H_1(x) + \one_{\cX_2}(x) \hat H_2(x).
	\end{align*}
	Thus the comparison keeps the routing partition and the regional learning rules fixed, but changes the target faced by the routed learner: the shared representation replaces the full branch-learning problem by a residual-learning problem on each region. 
	
	This formulation reflects the practical motivation for shared experts. As argued in \citep{dai2024deepseekmoe}, routed experts may otherwise redundantly acquire common knowledge in their parameters; a shared expert can absorb this common information and leave the routed experts to specialize locally. 
	
	\subsection{A learner-specific comparison}
	We now formalize the preceding idea through the risk profiles of the learning rules. For a function $g$ on $\cX_j$, write $\|g\|_j^2 := \EE\{g(X)^2 \mid X \in \cX_j\},$
	so that $\|\cdot\|_j$ denotes the $L_2$ norm under the conditional distribution of $X$ given $X\in\cX_j$. Let $p_j := P_X(\cX_j) > 0$, $j = 1,2$. For a target $h$, let $\cD_{m,j}^{h}$ and $\cD_m^{h}$ denote training samples of size $m$ generated from $Z = h(X)+\xi$, where $\xi$ has conditional mean zero and is conditionally sub-Gaussian with variance proxy $\sigma^2$. Here, $X \sim P_X(\cdot \mid X \in \cX_j)$ for $\cD_{m,j}^{h}$ and $X \sim P_X$ for $\cD_m^{h}$. 
	
	For a local target $h$ on $\cX_j$, define the learner-specific risk of the regional rule $\cM_j$ by
	\begin{align*}
		\cR_j^{\cM}(h;m,\sigma) := \EE_{h,\sigma}\left\|\cM_j(\cD_{m,j}^{h})-h\right\|_j^2.
	\end{align*}
	This is the risk of a fixed learning rule at a specified target, not a minimax risk over a function class. Similarly, for a global target $h$ on $\cX$, define the shared learner risk
	\begin{align*}
		\cR_S^{\cM}(h;m,\sigma) := \EE_{h,\sigma}\left\|\cM_S(\cD_m^h)-h\right\|_{L_2(P_X)}^2.
	\end{align*}
	Recall that the full branch class in the pure-routed representation is $\cH_j = \cF_S|_{\cX_j} + \cF_j$ for $j\in\{1,2\}$.
	
	\begin{assumption}\label{ass:learner-separation}
		There exist positive rate profiles $\varrho_S(m)$, $\varrho_j(m)$ and $\tilde{\varrho}_j(m)$, $j\in\{1,2\}$, where $m$ denotes the number of observations available to the corresponding learning rule, and constants $0 < c < C < \infty$ such that the following conditions hold:
		\begin{enumerate}[label=(\alph*),leftmargin=*,nosep]
			\item $\sup_{h\in\cF_S}\cR_S^{\cM}(h;m,\sigma)\le C\varrho_S(m)$;
			
			\item $\sup_{h\in\cF_j}\cR_j^{\cM}(h;m,\sigma)\le C\varrho_j(m)$ for $j\in\{1,2\}$;
			
			\item for each $j\in\{1,2\}$, there exists a nonempty subclass $\cB_j\subseteq\cH_j$, independent of $m$, such that $\inf_{h\in\cB_j}\cR_j^{\cM}(h;m,\sigma)\ge c\tilde\varrho_j(m)$.
		\end{enumerate} 
		In addition, the rate profiles are nonincreasing and regular under fixed rescaling: for every fixed $0 < a < \infty$, $\varrho(am) \asymp \varrho(m)$
		for each $\varrho \in \{\varrho_S,\varrho_1,\varrho_2,\tilde\varrho_1,\tilde\varrho_2\}$. Finally, writing $p_j=P_X(\cX_j)$, the shared representation estimator satisfies
		\begin{align}
			\sup_{f_0 \in \cA_{\mathrm{sh}}} \EE_{f_0,\sigma}\left\|\hat{f}_{0,\mathrm{sh}}-f_0\right\|^2_{L_2(P_X)} \le C\{\varrho_S(n)+p_1\varrho_1(np_1)+p_2\varrho_2(np_2)\}, \label{eq:shared_upper}
		\end{align}
	\end{assumption}

 	Assumption~\ref{ass:learner-separation} compares the performance of fixed learning rules across different targets. In particular, the same regional learner $\cM_j$ may perform well on residual targets in $\cF_j$, but poorly on full branch targets in $\cH_j$, either because the latter are harder to estimate or because they are less compatible with the learner's approximation class.
    	
    The bound in \eqref{eq:shared_upper} requires a joint estimator of the shared and regional components to attain the displayed rate. A naive sequential procedure that first regresses $Y$ on $\cF_S$ need not recover $f_S$, since the regional residuals may have nonzero projection onto $\cF_S$. Such a sequential interpretation therefore requires additional normalization conditions, such as regional orthogonality or residualization. Alternatively, the shared upper bound \eqref{eq:shared_upper} can be justified directly by a joint procedure, for example penalized least squares over the additive class $\left\{s + \one_{\cX_1} r_1 + \one_{\cX_2} r_2:s \in \cF_S,\ r_j \in \cF_j\right\}.$
    Under standard oracle inequalities for such sieve or penalized estimators, the risk decomposes at the scale $\varrho_S(n)+p_1\varrho_1(np_1)+p_2\varrho_2(np_2)$; see Supplement Section \ref{subsec:supp-sieve-justification-shared-upper-bound} for details. We use this high-level formulation to keep the comparison independent of a particular implementation, such as sieve estimation or a specific orthogonalization scheme.
    
    The preceding assumption leads to the following generic rate comparison.
    
    \begin{proposition}\label{prop:procedure-shared}
    	Assume the routing partition is known and Assumption \ref{ass:learner-separation} holds. Suppose also that, for each $j = 1,2$, the hard subclass $\cB_j \subseteq \cH_j$ in Assumption \ref{ass:learner-separation} is compatible with the shared representation, in the sense that its elements can arise as branch functions $H_j = f_S|_{\cX_j}+f_{R,j}$ for some $f_0 \in \cA_{\mathrm{sh}}$. Then, for all sufficiently large $n$, there exist constants $0 < c < C < \infty$ such that the shared representation satisfies
    	\begin{align*}
    		\sup_{f_0 \in \cA_{\mathrm{sh}}} \EE_{f_0,\sigma}\left\|\hat{f}_{0,\mathrm{sh}}-f_0\right\|_{L_2(P_X)}^2 \le C\{\varrho_S(n)+p_1\varrho_1(np_1)+p_2\varrho_2(np_2)\},
    	\end{align*}
    	whereas the pure-routed architecture trained with the same regional learners satisfies
    	\begin{align*}
    		\sup_{f_0 \in \cA_{\mathrm{sh}}} \EE_{f_0,\sigma}\left\|\hat{f}_{0,\mathrm{rt}}-f_0\right\|_{L_2(P_X)}^2 \ge c\{p_1\tilde\varrho_1(np_1) \vee p_2\tilde\varrho_2(np_2)\}.
    	\end{align*}
    	Consequently, if
    	\begin{align*}
    		\varrho_S(n)+p_1\varrho_1(np_1)+p_2\varrho_2(np_2) = o\{p_1\tilde\varrho_1(np_1) \vee p_2\tilde\varrho_2(np_2)\},
    	\end{align*}
    	then the shared architecture has asymptotically smaller worst-case risk for this fixed routed training procedure.
    \end{proposition}
    
    \begin{remark}
    	Proposition~\ref{prop:procedure-shared} compares the shared-routed and pure-routed representations under the same regional learning rules, with residual targets in the shared-routed case and full branch targets in the pure-routed case. The latter may be statistically harder, or may simply be less well approximated by the chosen regional learner. This comparison is motivated by practical MoE training, where routed experts are typically trained under the same optimization pipeline rather than by oracle procedures tailored to each induced branch class. In special cases, for instance when $\cH_j$ is finite dimensional and well matched to the routed learner, the pure-routed minimax rate over $\cH_j$ may match the shared minimax rate up to constants.
    \end{remark}
    
    \subsection{An illustrative example}
   	We now consider a setting in which the regional learning rules estimate the residual components more efficiently than the full branch functions. The shared component is a linear trend, whereas the routed residuals are finite periodic Fourier signals. A routed learner based only on periodic Fourier sieves can estimate the residuals efficiently after the shared trend has been extracted, but it is poorly matched to the full branch functions because the linear trend is nonperiodic in the local coordinates. For simplicity, both regions use the same Fourier basis. The same learner-specific mechanism can also arise with region-specific bases when each regional learning rule estimates its residual target more efficiently than the corresponding full branch target.
    
    \begin{example}[A Fourier-sieve illustration]\label{ex:fourier-x-sin-cos}
    	Let $\cX = [0,1]$, $X\sim\Unif[0,1]$, $\cX_1 = [0,1/2]$, $\cX_2 = (1/2,1]$, and $p_1 = p_2 = 1/2$. For $x \in \cX_j$, define the local coordinate as $t_1(x) = 2x$, $t_2(x) = 2x-1$,
    	so that $t_j(X) \sim \Unif[0,1]$ conditional on $X \in \cX_j$. Consider the shared-plus-routed target
    	\begin{align*}
    		f_0(x) = f_S(x) + \one_{\cX_1}(x) f_{R,1}(x) + \one_{\cX_2}(x) f_{R,2}(x),
    	\end{align*}
    	where $f_S(x) = \gamma x$ with $\gamma_0 \le |\gamma| \le B$,
    	for fixed constants $0 < \gamma_0 \le B < \infty$. The routed residuals are finite periodic Fourier signals: for a fixed integer $K_0 > 0$,
    	\begin{align*}
    		f_{R,j}(x) = \sum_{k=1}^{K_0}\{a_{j,k}\sin(2k\pi t_j(x)) + b_{j,k}\cos(2k\pi t_j(x))\}, \qquad \max_{1 \le k \le K_0}\{|a_{j,k}|,|b_{j,k}|\} \le B.
    	\end{align*}
    	Assume that the noise $\varepsilon \sim N(0,\sigma^2)$, independent of $X$.
    	
    	Let the routed learner be a periodic Fourier-sieve learner on $t \in [0,1]$, based on $\cS_K = \operatorname{span}\{1,\sin(2\pi kt),\cos(2\pi kt):1 \le k \le K\}.$
    	Under the shared representation, use a joint sieve least-squares estimator over the additive class
    	\begin{align*}
    		\cA_{\mathrm{sh},K_0} = \left\{\alpha x + \one_{\cX_1}(x) r_1(t_1(x)) + \one_{\cX_2}(x) r_2(t_2(x)):\alpha \in \RR,\ r_j \in \cS_{K_0}\right\}.
    	\end{align*}
    	The true function $f_0$ belongs to this finite-dimensional class. Since $K_0$ is fixed, standard least-squares theory gives $\EE\|\hat{f}_{0,\mathrm{sh}}-f_0\|_{L_2(P_X)}^2 \lesssim \sigma^2/n.$
    	Thus, in the notation of Assumption~\ref{ass:learner-separation}, for a sample of size $m$, the corresponding rate profiles can be taken as $\varrho_S(m) \asymp\sigma^2/m$ and $\varrho_j(m) \asymp \sigma^2/m$ for $j\in\{1,2\}$.
    	Consequently,
    	\begin{align*}
    		\varrho_S(n) + p_1\varrho_1(np_1) + p_2\varrho_2(np_2) \asymp \frac{\sigma^2}{n},
    	\end{align*}
    	which verifies the shared upper-rate condition for the joint Fourier-sieve estimator.
    	
    	Now consider the pure-routed representation trained with the same regional Fourier-sieve learner. On region $\cX_j$, the routed learner must estimate the full branch function $H_j = f_S(x)+f_{R,j}(x)$. In the local coordinate, the shared component becomes a nonperiodic linear ramp. Specifically, $f_S(x) = \gamma t_1(x)/2$ on $\cX_1$ and $f_S(x) = \gamma\{t_2(x)+1\}/2$ on $\cX_2$. The constant part is contained in $\cS_K$, but the ramp $t \mapsto t$ is not periodic. Its squared approximation tail under the periodic Fourier basis is of order $K^{-1}$. Therefore, even if the sieve dimension $K$ is chosen optimally, the regional Fourier learner faces the bias-variance tradeoff
    	\begin{align*}
    		\cR_j^{\cM}(H_j;m,\sigma) \gtrsim \inf_{K \ge 1}\left\{\gamma_0^2 K^{-1} + \frac{\sigma^2 K}{m}\right\}.
    	\end{align*}
    	The first term is the Fourier approximation tail of the nonperiodic linear ramp, and the second term is the variance cost of estimating $K$ Fourier modes. Optimizing over $K$ yields $\cR_j^{\cM}(H_j;m,\sigma) \gtrsim \gamma_0\sigma/\sqrt{m}$. Therefore, for balanced regions, $\EE\|\hat{f}_{0,\mathrm{rt}}-f_0\|_{L_2(P_X)}^2 \gtrsim \gamma_0\sigma/\sqrt{n}.$
    \end{example}

    This example illustrates a learner-specific complexity separation. After the linear common trend is extracted by the shared expert, the routed experts only need to estimate finite Fourier residuals and achieve the parametric rate $n^{-1}$. In contrast, a pure-routed Fourier learner must estimate the full branch function $H_j = f_S+f_{R,j}$, whose approximation by periodic Fourier sieves is limited by the nonperiodic linear component. This yields the slower learner-specific rate $n^{-1/2}$.
   
    The rate gap in this example arises from the periodic Fourier basis used by the routed learner. A pure-routed learner whose dictionary contains the linear function $x$, or whose basis is adapted to $\cH_j$, can avoid the corresponding approximation error.
    
    \begin{remark}
    	A complementary form of separation arises when the shared component is complex but globally smooth, while the routed residuals are simple discontinuities. Suppose $f_0(x) = h(x) + \sum_{r=1}^{\Mzero} a_r \one_{\cX_r}(x)$, where $h$ is a smooth nonparametric function and the routed residual is piecewise constant over the routing regions. If the shared learner is well suited to estimating $h$, while the routed learners are simple piecewise-constant or low-order local learners, then the shared representation leaves only the discontinuous regional offsets to the routed learners. In contrast, a pure-routed learner must approximate $h+a_r$ separately on each region. When $h$ is not well represented by the routed learner, the pure-routed estimator may suffer a large approximation bias, even if the residual offsets themselves are easy to learn. Figure \ref{fig:fourier-simulation} illustrates the two complementary settings considered in the preceding example and this remark.  
    \end{remark}
    
    \begin{figure}[t]
    	\centering
    	
    	\begin{subfigure}[b]{0.48\textwidth}
    		\centering
    		\includegraphics[width=\textwidth]{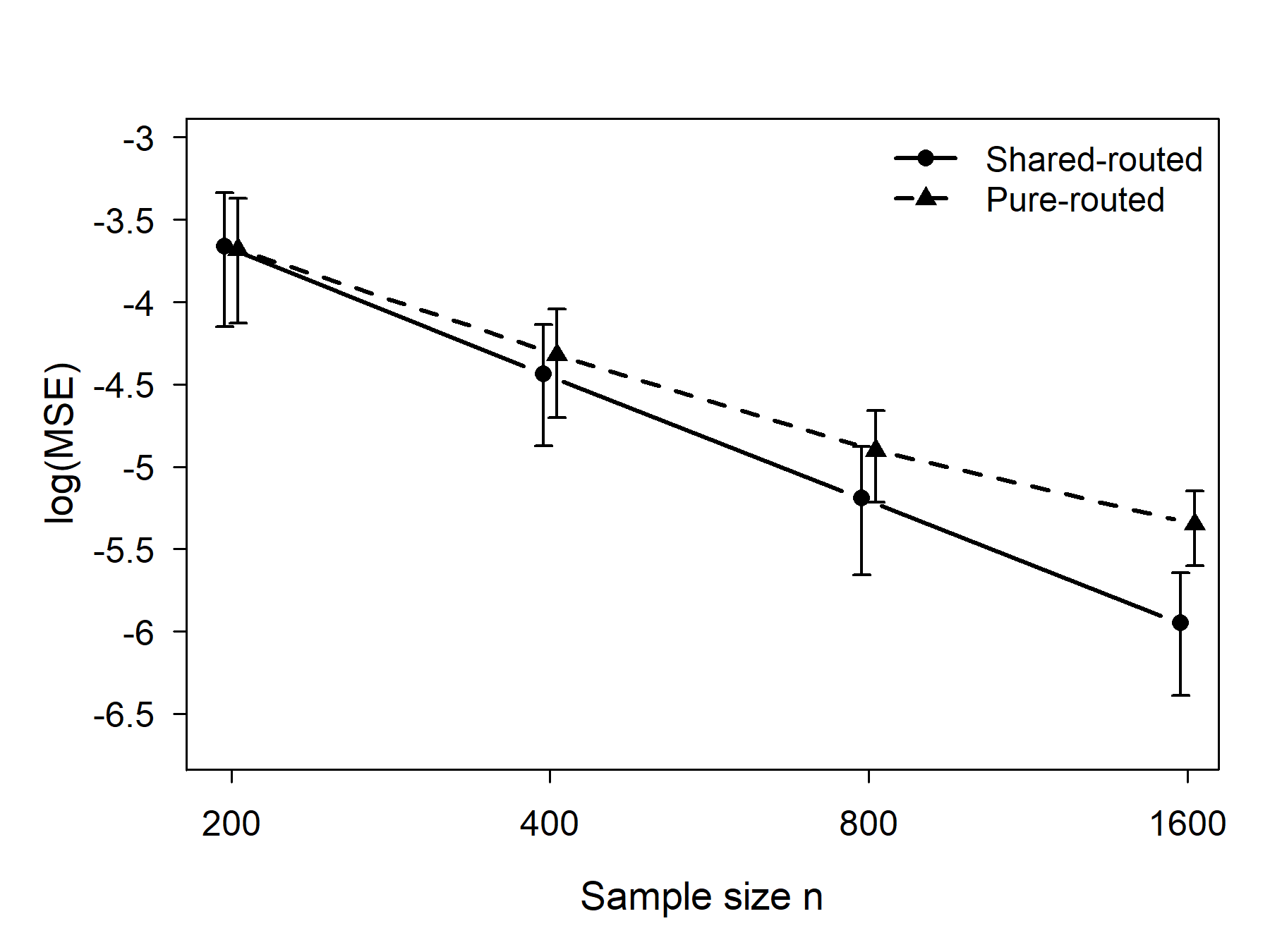}
    		\caption{Simple shared component, complex routed components}
    		\label{fig:fourier-mse}
    	\end{subfigure}
    	\hfill
    	\begin{subfigure}[b]{0.48\textwidth}
    		\centering
    		\includegraphics[width=\textwidth]{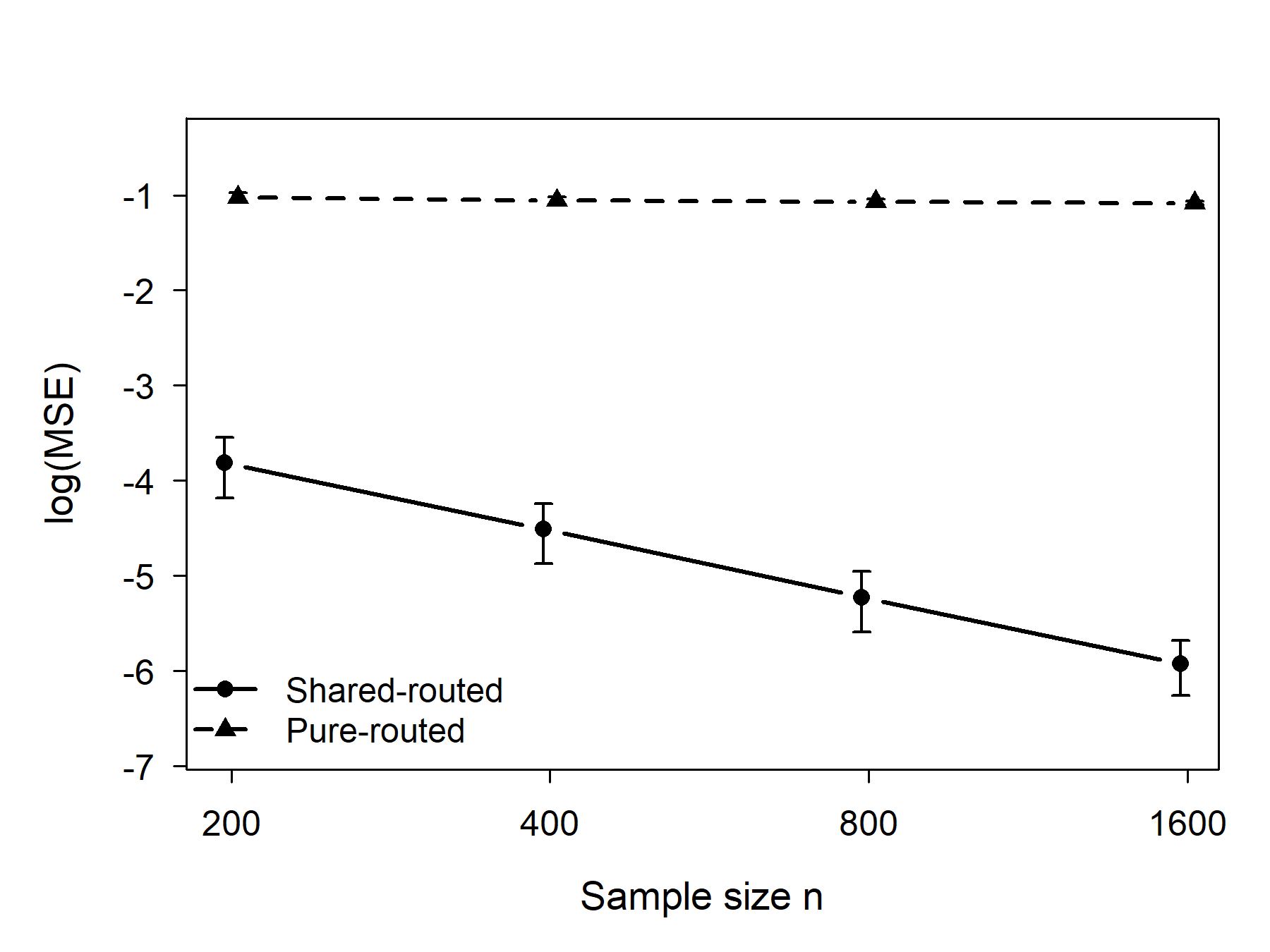}
    		\caption{Complex shared component, simple routed components}
    		\label{fig:fourier-cerr}
    	\end{subfigure}
    	
    	\caption{
    		Prediction MSE for the shared-routed and pure-routed estimators on log-log axes. Panel (a) considers a simple shared component and complex routed components, while Panel (b) considers a complex shared component and simple routed components. Error bars represent one standard deviation across Monte Carlo replications. The shared-routed estimator has lower prediction MSE across the displayed sample sizes in both settings. See Supplement Section \ref{subsec:supp-mse-comparison-shared-expert} for implementation details.
    	}
    	\label{fig:fourier-simulation}
    \end{figure}

    Together, the two settings show that the gain from a shared expert can arise through either a slower estimation rate or approximation bias for the pure-routed learner, depending on how the induced targets interact with the local learning rules. 
    
	\section{Conclusions and outlook}
	\subsection{Main conclusions}
	
	This paper develops a statistical perspective on MoE architectures. The central message is that the statistical role of MoE stems from how the gating function makes the combination of experts depend on the input, thereby changing the approximation-estimation-computation tradeoff. From this viewpoint, dense routing, sparse routing, adaptive gating classes, and shared experts are different mechanisms for allocating modeling capacity locally across the input space.
	
	The analysis highlights three complementary mechanisms. First, sparse Top-$K$ routing admits an oracle inequality with a router-learning cost of the same order as dense softmax gating, up to a boundary-stability logarithmic term, while reducing per-input expert-evaluation cost. Second, the choice of gating class is a geometric approximation problem: linear, quadratic, and kernel-based gates encode different assumptions about the geometry of the expert-induced partition of the input space. Third, shared experts can change the targets faced by routed learners, replacing full branch-learning problems by residual-learning problems when the common component is better matched to the shared learning rule. Taken together, these results suggest that the statistical value of MoE architectures lies not only in using more experts, but in matching routing, sparsity, and shared structure to the local complexity of the target.
	
	These conclusions also provide a statistical interpretation of several heuristics in large-scale MoE practice. Increasing the total number of routed experts is useful only insofar as it creates a richer library of locally accurate predictors; otherwise the router pays a larger estimation and search cost without a corresponding approximation gain. Activating more than one routed expert per input is useful when the locally best predictor is not a single expert but a nondegenerate mixture of several experts. Shared experts are useful when a common component is repeatedly needed across many regions, because they prevent routed experts from redundantly learning the same global structure. Thus, the empirical success of fine-grained experts, Top-$K$ routing with $K>1$, and shared-routed designs can be viewed as different ways of improving the local approximation-estimation-computation tradeoff.

	\subsection{Limitations and open questions}
	
	The theory developed here is intentionally modular. It isolates several statistical mechanisms: localized aggregation, sparse activation, geometry-adaptive gating, and shared-residual decomposition. This modular view clarifies which architectural component is responsible for which statistical effect and the conditions under which each effect should be expected: Top-$K$ routing requires control of boundary-switching events; flexible gating classes require enough data to offset their larger estimation cost; and shared experts improve efficiency only when the induced residual learning problem is better matched to the routed learners.
	
	Several limitations remain. First, although the dense and sparse oracle inequalities allow the experts to evolve predictably along the training trajectory, our analysis relies on high-level expert-error control and boundedness conditions and does not characterize the underlying joint gradient-based expert-router dynamics. Other approximation and shared-expert results condition on population experts, oracle partitions, or fixed learning rules. Second, some arguments use oracle or performance-induced partitions to describe the geometry of expert specialization. These partitions are useful for interpretation and approximation analysis, but they are not directly observed in practice. Third, the framework abstracts away from engineering features that are central in large-scale MoE systems, including load balancing, token capacity constraints, optimization dynamics, and interactions between routing and representation learning.
	
	These limitations point to several open questions. A first question is how to empirically validate the local-performance premise underlying the analysis. The oracle or performance-induced partitions used in the theory are not meant to be directly observed objects, but they suggest a testable diagnostic question: do trained experts, or trained mixtures of experts, exhibit systematic predictive advantages on different parts of the input space? Answering this question would require data-driven measures of local expert performance and methods for checking whether the learned router aligns with regions where the selected experts are genuinely stronger. This would help distinguish meaningful expert specialization from purely architectural sparsity.
	
	A second question concerns end-to-end expert-router co-adaptation. The modular perspective adopted here isolates the statistical role of routing and localized aggregation, but it leaves open how these mechanisms interact with gradient-based training, representation learning, router regularization, load balancing, and token capacity constraints in large-scale MoE systems. A more complete theory should explain not only how a given collection of experts should be combined, but also when training dynamics produce experts that are diverse, stable, and locally useful.
	
	A third direction is to develop data-driven procedures for choosing the composition of the expert pool under a given computational or model-capacity budget. Our analysis shows that increasing the number of routed experts can improve local approximation, but in practice this expansion cannot be separated from the resulting costs in expert capacity, router learning and system coordination. This naturally leads to the question of expert granularity: how finely should a fixed modeling budget be divided across experts in order to achieve the best predictive performance? The answer may depend jointly on the number and capacity of routed experts, the allocation to shared experts, and the complexity of the routing mechanism. Developing principled procedures for selecting such budget-constrained expert granularity remains an important open problem.

	Finally, extending the approximation and risk analyses to deep MoE architectures with compositional routing across layers and high-dimensional input spaces, establishing minimax lower bounds, and incorporating computational constraints into the statistical analysis remain important steps toward a more complete theory of modern MoE architectures. 
	
	\subsection{A statistical outlook for future MoE architectures}
	
	The preceding analysis suggests several architecture-level directions that may be worth exploring in future MoE systems.
	
	\paragraph{Specialization-aware routing under system constraints.}
	A central message of our analysis is that routing should preserve and exploit local predictive advantages among experts, thereby maintaining statistically meaningful specialization. In large-scale systems, this objective must coexist with load, capacity, and communication constraints. Recently, DeepSeek-V3 \citep{liu2024deepseek} and Kimi K3 \citep{team2026kimi} highlight the tension between expert specialization and load control in large-scale MoE systems. This motivates designing routing mechanisms that maintain acceptable system balance while disturbing predictive specialization as little as possible. More broadly, predictive specialization, routing stability, and system feasibility may be better treated as jointly designed objectives rather than separate considerations.
	
	\paragraph{Input-adaptive expert activation.}
	Conditional on a given expert pool, Section~\ref{sec:oracle_topK} shows that the statistically appropriate level of expert activation depends on the local prediction regime: a dominant specialist may be sufficient for some inputs, whereas others benefit from combining several experts. This suggests going beyond the choice of a single common value of $K$. Rather than fixing the same Top-$K$ rule for all inputs, the router could adapt the number of active experts to local predictive structure, activating fewer experts where one or a few specialists suffice and more experts where prediction is intrinsically compositional. Related token- and sequence-adaptive routing mechanisms have already begun to explore variable expert activation \citep{zeng2024adamoe,wen2025route}. Developing statistically principled procedures for such input-dependent activation, including how the activation level should be selected from data, is therefore a natural direction. Such input-dependent expert allocation may be viewed as a first step toward a broader framework of adaptive expert specialization.
	
	\paragraph{Multiscale shared--routed architectures.}
	Section~\ref{sec:shared} suggests that shared experts are useful when they absorb predictive structure that would otherwise be repeatedly learned by routed experts. The same principle may extend beyond the binary distinction between globally shared and fully routed components. Classical hierarchical MoE architectures \citep{jordan1994hierarchical} illustrate how experts can be organized at intermediate levels, while more recent constructions such as CartesianMoE \citep{su2025cartesianmoe} introduce structured forms of sharing across expert components. We therefore envision shared--routed architectures with multiple scales of reuse, where some components are shared globally, others within groups of related inputs, and the remaining components are routed locally. The scope of each shared component could therefore adapt to the set of inputs over which it captures common predictive structure.
	
	\paragraph{Boundary-aware counterfactual expert discovery.}
	Section~\ref{sec:adaptation} defines meaningful specialization through local predictive advantage, but standard sparse training observes only the executed route and may therefore reinforce early routing decisions without sufficient information about alternatives. Counterfactual routing and exploration have recently been investigated as ways to provide such information \citep{fedus2022switch,kim2025exploration,yoon2026experts}. Our analysis of routing geometry and active-set stability motivates investigating whether this additional information may be particularly valuable near routing boundaries. Rather than exploring alternative experts uniformly, a limited exploration budget could be concentrated on inputs with unstable expert assignments, for example near routing boundaries where small perturbations may change the active set. Such boundary-aware exploration may reduce specialization driven by early routing noise without increasing the number of experts activated at inference time.
	
	Taken together, these directions suggest that future MoE design may benefit from making specialization, expert granularity, expert activation, sharing, and exploration increasingly adaptive to the local prediction structure, while respecting the computational constraints that make sparse MoE architectures attractive in the first place.
	
	\renewcommand\refname{References}
	\bibliographystyle{apalike}
	\bibliography{RefList}

	\clearpage

	\makeatletter
	\let\addtocontents\arxiv@originaladdtocontents
	\let\label\arxiv@originallabel
	\let\ltx@label\arxiv@originalmathlabel
	\let\arxiv@supplementlabel\label
	\let\arxiv@supplementmathlabel\ltx@label
	\renewcommand{\label}[1]{\def\arxiv@currentlabelname{#1}\def\arxiv@duplicatelabelname{eqn:softmax_practical_oracle}\ifx\arxiv@currentlabelname\arxiv@duplicatelabelname\arxiv@supplementlabel{supp-#1}\else\arxiv@supplementlabel{#1}\fi}
	\renewcommand{\ltx@label}[1]{\def\arxiv@currentlabelname{#1}\def\arxiv@duplicatelabelname{eqn:softmax_practical_oracle}\ifx\arxiv@currentlabelname\arxiv@duplicatelabelname\arxiv@supplementmathlabel{supp-#1}\else\arxiv@supplementmathlabel{#1}\fi}
	\makeatother
	\renewcommand{\baselinestretch}{}
	\setcounter{algocf}{0}
	\renewcommand{\thealgocf}{\thesection.\arabic{algocf}}
	\renewcommand\theassumption{\thetheorem}

	\def\spacingset#1{\renewcommand{\baselinestretch}{#1}\small\normalsize} \spacingset{1}

	\bigskip
	\bigskip
	\bigskip
	\begin{center}
		{\LARGE\bf Supplementary Materials for ``Towards a Statistical Understanding of Mixture-of-Experts''}
	\end{center}
	\medskip
	
	\bigskip
	\spacingset{1.45}
	This supplementary file contains the proofs of lemmas, propositions and theorems in the main article, as well as simulation details. This file is organized as follows. Appendix~\ref{sec:supp-proof-section-2} provides the proofs for Section~\ref{main-sec:local_averaging}, including the approximation properties, metric entropy and the comparison between global and localized routing. Appendix~\ref{sec:supp-proof-section-3} proves the oracle risk bounds for dense softmax gating in Section~\ref{main-sec:oracle_softmax}. Appendix~\ref{sec:supp-proof-section-4} establishes the oracle results for sparse Top-$K$ gating in Section~\ref{main-sec:oracle_topK}, together with the corresponding Top-1 results. Appendix~\ref{sec:supp-proof-section-5} contains the gating-approximation decomposition and the gating-class selection results in Section~\ref{main-sec:adaptation}. Appendix~\ref{sec:supp-proof-section-6} proves the shared-expert results in Section~\ref{main-sec:shared}. Appendix~\ref{sec:supp-additional-materials} collects additional technical results on empirical risk minimization, cross-entropy-based parameter estimation, and dense-softmax upper bounds. Finally, Appendix~\ref{sec:supp-simulation-details} provides the implementation details for the simulation studies.

	\textbf{Additional Notation}. We will use the following notations in this file additional to the main article. $\mathbf{1}$ denotes the all-ones vector, with its dimension understood from context. For a vector $v=(v_1,\ldots,v_m)^\top$ and $1\le p<\infty$, we denote its $p$-norm by $\|v\|_p=\left(\sum_{j=1}^m|v_j|^p\right)^{1/p}$, while $\|v\|_\infty=\max_{1\le j\le m}|v_j|$. For a matrix $A$, $\|A\|_{\mathrm{op}}=\sup_{\|v\|_2=1}\|Av\|_2$ denotes its operator norm induced by the Euclidean norm. For $x\in\RR^d$ and $r>0$, $B_d(x,r)=\{y\in\RR^d:\|y-x\|_2\le r\}$ denotes the closed Euclidean ball centered at $x$ with radius $r$. For a finite set $\mathcal S$, $\#\mathcal S$ denotes its cardinality. Finally, for a Lebesgue-measurable set $\cA\subseteq\RR^d$, $\mathrm{Vol}(\cA)$ denotes its volume with respect to the Lebesgue measure.

	\tableofcontents
	\clearpage
	
	\appendix
	\section{Proof of Section \ref{main-sec:local_averaging}}
	\label{sec:supp-proof-section-2}
	
	\subsection{Proof of Proposition \ref{main-prop:topk_density}}
	\label{subsec:supp-proof-prop-2-2}
	
	\begin{proof}[Proof of Proposition \ref{main-prop:topk_density}]
		Fix $\varepsilon > 0$ and $h \in C_R(\cX)$. Since $\cX=[0,1]^d$ is compact, $h$ is uniformly continuous. Hence there exists $q \in \mathbb{N}_+$ such that
		\begin{align*}
			|h(x)-h(y)| < \varepsilon \qquad \text{whenever} \qquad \|x-y\|_2 \le \frac{\sqrt{d}}{q}.
		\end{align*}
		Partition $[0,1]^d$ into the $q^d$ cubes
		\begin{align*}
			Q_\nu = \left[\frac{\nu_1-1}{q},\frac{\nu_1}{q}\right] \times \left[\frac{\nu_2-1}{q},\frac{\nu_2}{q}\right] \times \cdots \times \left[\frac{\nu_d-1}{q},\frac{\nu_d}{q}\right], \qquad \nu=(\nu_1,\dots,\nu_d)\in\{1,\dots,q\}^{d},
		\end{align*}
		and let
		\begin{align*}
			m_\nu = \left(\frac{2\nu_1-1}{2q},\dots,\frac{2\nu_d-1}{2q}\right)
		\end{align*}
		be the center of $Q_\nu$. Set
		\begin{align*}
			n=q^d, \qquad M=Kn, \qquad \lambda=\frac{C_1}{\max\{2,d\}}.
		\end{align*}
		For each $\nu \in \{1,\dots,q\}^{d}$ and each $r \in [K]$, define the constant expert
		\begin{align*}
			f_{\nu,r}(x) \equiv h(m_\nu)
		\end{align*}
		and the linear score
		\begin{align*}
			u_{\nu,r}(x) = 2\lambda m_\nu^\top x-\lambda\|m_\nu\|_2^2.
		\end{align*}
		These scores belong to the parameter box $\Theta_{\lin}(C_1)$. Indeed, for every coordinate,
		\begin{align*}
			|2\lambda m_{\nu,j}| \le 2\lambda \le C_1,
		\end{align*}
		and the intercept is bounded by
		\begin{align*}
			\lambda\|m_\nu\|_2^2 \le \lambda d \le C_1.
		\end{align*}
		
		For fixed $x \in [0,1]^d$, maximizing $u_{\nu,r}(x)$ over $\nu$ is equivalent to maximizing $2 m_\nu^\top x-\|m_\nu\|_2^2$, or equivalently, to minimizing $\|x - m_\nu \|_2^2$. Thus the largest scores are attained by the grid center or centers closest to $x$. Since each center has $K$ identical copies, the Top-$K$ rule selects copies attached to nearest centers. If the nearest center is unique, it selects exactly the $K$ copies attached to that center. The resulting MoE output is therefore
		\begin{align*}
			\ell(x) = \sum_{\nu,r} g_{\nu,r}^{(K)}(x) f_{\nu,r}(x), \qquad g_{\nu,r}^{(K)}(x) = \frac{\exp\{u_{\nu,r}(x)\}\one_{\{(\nu,r)\in\cT_K(x)\}}}{\sum_{(\nu',r')\in\cT_K(x)}\exp\{u_{\nu',r'}(x)\}},
		\end{align*}
		where $\ell(x)$ is a convex combination of values $h(m_\nu)$ at nearest grid centers. Each such center belongs to a cube whose closure contains $x$. Consequently,
		\begin{align*}
			\|x-m_\nu\|_2 \le \frac{\sqrt{d}}{2q}<\frac{\sqrt{d}}{q},
		\end{align*}
		and therefore
		\begin{align*}
			|h(x)-h(m_\nu)|<\varepsilon.
		\end{align*}
		Since $\ell(x)$ is a convex combination of such values $h(m_\nu)$, it follows that $|\ell(x)-h(x)|<\varepsilon$ for all $x\in\cX$. Thus
		\begin{align*}
			\|\ell-h\|_\infty<\varepsilon.
		\end{align*}
		Because $\varepsilon>0$ was arbitrary, $\cA_{\loc}^{\TopK}$ is dense in $C_R(\cX)$.
	\end{proof}

	\subsection{Proof of Proposition \ref{main-prop:approximation_ability}}
	\label{subsec:supp-proof-prop-2-3}
	
	\begin{proof}[Proof of Proposition \ref{main-prop:approximation_ability}]
		We first analyze the global class. Since $\cA_{\glob}^R=\cF_{\mathrm c}^R$, every $f \in \cA_{\glob}^R$ is a constant $f(x)\equiv c$ with $|c| \le R$. For a fixed $h \in \cF_b^{R,L}$, the midpoint of the range of $h$ belongs to $[-R,R]$, and hence the best constant approximation error equals half of the oscillation:
		\begin{align*}
			\inf_{f \in \cA_{\glob}^R}\|f-h\|_\infty = \frac{1}{2}\left(\sup_{x \in \cX}h(x)-\inf_{x \in \cX}h(x)\right).
		\end{align*}
		Because $h$ is $L$-Lipschitz and $\cX=[0,1]^d$ has diameter $\sqrt{d}$, the oscillation is at most $L\sqrt{d}$. Also $|h|\le R$, so the oscillation is at most $2R$. Therefore
		\begin{align*}
			\inf_{f \in \cA_{\glob}^R}\|f-h\|_\infty \le \min\left\{R,\frac{L\sqrt{d}}{2}\right\}.
		\end{align*}
		To show that this bound is sharp, let
		\begin{align*}
			a := \min\left\{R,\frac{L\sqrt{d}}{2}\right\}, \qquad h_a(x) = \frac{a}{d}\sum_{j=1}^d (2x_j-1).
		\end{align*}
		Then $\|h_a\|_\infty\le a\le R$, and
		\begin{align*}
			\|\nabla h_a\|_2 = \frac{2a}{d}\sqrt{d} = \frac{2a}{\sqrt{d}} \le L.
		\end{align*}
		Thus $h_a \in \cF_b^{R,L}$. Its range is exactly $[-a,a]$, and therefore
		\begin{align*}
			\inf_{f \in \cA_{\glob}^R}\|f-h_a\|_\infty = a.
		\end{align*}
		This proves
		\begin{align*}
			\sup_{h \in \cF_b^{R,L}} \inf_{f \in \cA_{\glob}^R} \|f-h\|_\infty = \min\left\{R,\frac{L\sqrt{d}}{2}\right\}.
		\end{align*}
		
		We next prove the local upper bound. Set
		\begin{align*}
			q := \left\lfloor(M/K)^{1/d}\right\rfloor.
		\end{align*}
		Then $q \ge 1$ and $Kq^d \le M$. Partition $[0,1]^d$ into the $q^d$ cubes $Q_\nu$ defined in the proof of Proposition~\ref{main-prop:topk_density}, and let $m_\nu$ be their centers. For any fixed $h \in \cF_b^{R,L}$, we use the same nearest-center construction a before. Assign $K$ copies of an expert to each center, with constant value $f_{\nu,r}(x)\equiv h(m_\nu)$, and take the logits
		\begin{align*}
			u_{\nu,r}(x) = 2\lambda m_\nu^\top x-\lambda\|m_\nu\|_2^2, \qquad \lambda = \frac{C_1}{2\max\{2,d\}}.
		\end{align*}
		If $M>Kq^d$, assign the remaining experts the constant value $0$ and the logit $-C_1$. Since $u_{\nu,r}(x) \ge -\lambda d \ge -C_1/2$ for every constructed copy and every $x \in \cX$, the remaining experts have strictly smaller scores than all constructed copies and are never selected by the Top-$K$ rule.
		
		For any $x \in \cX$, maximizing $u_{\nu,r}(x)$ over $\nu$ is equivalent to minimizing $\|x-m_\nu\|_2^2$. Hence the selected constructed experts are attached to nearest grid centers. Therefore the output $\ell(x)$ is a convex combination of the values $h(m_\nu)$ associated with centers $m_\nu$, whose cubes contain $x$ in their closure. For each such center,
		\begin{align*}
			\|x-m_\nu\|_2 \le \frac{\sqrt{d}}{2q},
		\end{align*}
		so the Lipschitz condition gives
		\begin{align*}
			|h(x)-h(m_\nu)| \le \frac{L\sqrt{d}}{2q}.
		\end{align*}
		Taking convex combinations yields
		\begin{align*}
			|h(x)-\ell(x)| \le \frac{L\sqrt{d}}{2q}, \qquad x \in \cX.
		\end{align*}
		Hence
		\begin{align*}
			\|h-\ell\|_\infty \le \frac{L\sqrt{d}}{2\left\lfloor(M/K)^{1/d}\right\rfloor}.
		\end{align*}
		Taking the supremum over $h \in \cF_b^{R,L}$ proves the claim.
	\end{proof}
	
	\subsection{Proof of Theorem \ref{main-thm:metric_entropy}}
	\label{subsec:supp-proof-thm-2-4}
	
	\begin{proof}[Proof of Theorem \ref{main-thm:metric_entropy}]
		\textbf{For part (i),}
		since $\cA_{\glob}^R=\cF_{\mathrm c}^R$ is isometric to the interval $[-R,R]$ under the metric $\|f_c-f_{c'}\|_\infty=|c-c'|$, its covering number is the covering number of $[-R,R]$. For $0<\epsilon\le 2R$, this is
		\begin{align*}
			N\left(\epsilon,\cA_{\glob}^R,\|\cdot\|_\infty\right) = \left\lceil\frac{R}{\epsilon}\right\rceil.
		\end{align*}
		Taking logarithm yields the desired order.
		
		\textbf{For part (ii),}
		if $M=1$, then $\cA_{\loc}^{\soft}(1)=\cA_{\glob}^R$, so part (i) applies. Assume $M \ge 2$. For the proof, define
			$T_M:=2(d+1)C_1+\log(M-1)$ and $m_M:=\exp\{-T_M\}/(1+\exp\{-T_M\})^2$. For $\beta \in [0,2C_1]^d$ and $\alpha \in [0,2C_1]$, define
		\begin{align*}
			\sigma(x)=\frac{1}{1+\exp\{-x\}}, \qquad f_{\beta,\alpha}(x) = R\sigma\left(\beta^\top x+\alpha-\log(M-1)\right).
		\end{align*}
		We first check $f_{\beta,\alpha} \in \cA_{\loc}^{\soft}(M)$. Take $c_1=R, c_2=\cdots=c_M=0$, and choose the scores
		\begin{align*}
			u_1(x)=\frac{1}{2}(\beta^\top x+\alpha), \qquad u_2(x)=\cdots=u_M(x)=-\frac{1}{2}(\beta^\top x+\alpha).
		\end{align*}
		All coefficients lie in $[-C_1,C_1]$, since each component of $\beta/2$ and the intercept $\alpha/2$ belong to $[0,C_1]$. The softmax weight of expert $1$ is
		\begin{align*}
			\frac{\exp\{u_1(x)\}}{\exp\{u_1(x)\}+\sum_{j=2}^M \exp\{u_j(x)\}} = \frac{1}{1+(M-1)\exp\{-(\beta^\top x+\alpha)\}} = \sigma\left(\beta^\top x+\alpha-\log(M-1)\right).
		\end{align*}
		Thus the resulting output is exactly $f_{\beta,\alpha}$.
		
		For every $x\in[0,1]^d$, $\beta\in[0,2C_1]^d$, and $\alpha\in[0,2C_1]$, the argument of $\sigma$ satisfies
		\begin{align*}
			-\log(M-1) \le \beta^\top x+\alpha-\log(M-1) \le 2(d+1)C_1-\log(M-1).
		\end{align*}
		Since $T_M = 2(d+1)C_1 + \log(M-1)$, this argument lies in $[-T_M,T_M]$. Hence
		\begin{align*}
			\sigma'(t)\ge m_M \qquad \text{for all } |t| \le T_M.
		\end{align*}
		
		For the lower bound on the covering number, take $\delta:=\frac{3\epsilon}{Rm_M}$. Let $\Gamma_\beta$ be a grid of $[0,2C_1]^d$ with mesh size $\delta$ in each coordinate, and let $\Gamma_\alpha$ be a grid of $[0,2C_1]$ with mesh size $\delta$. Then
		\begin{align*}
			|\Gamma_\beta| = \left(\left\lfloor\frac{2Rm_MC_1}{3\epsilon}\right\rfloor+1\right)^d, \qquad |\Gamma_\alpha| = \left\lfloor\frac{2Rm_MC_1}{3\epsilon}\right\rfloor+1.
		\end{align*}
		Consider $\cH = \{f_{\beta,\alpha}:\beta\in\Gamma_\beta,\alpha\in\Gamma_\alpha\}$. Distinct elements of $\cH$ are $3\epsilon$-separated under $\|\cdot\|_\infty$. Indeed, take $(\beta,\alpha)\neq(\beta',\alpha')$. If $\alpha\neq\alpha'$, then evaluate at $x=0$ gives
		\begin{align*}
			\|f_{\beta,\alpha}-f_{\beta',\alpha'}\|_\infty \ge |f_{\beta,\alpha}(0)-f_{\beta',\alpha'}(0)| \ge Rm_M|\alpha-\alpha'| \ge 3\epsilon.
		\end{align*}
		If $\alpha=\alpha'$, then some coordinate $\beta_j\neq\beta_j'$. Evaluating at $x=e_j$ gives
		\begin{align*}
			\|f_{\beta,\alpha}-f_{\beta',\alpha}\|_\infty \ge |f_{\beta,\alpha}(e_j)-f_{\beta',\alpha}(e_j)| \ge Rm_M|\beta_j-\beta_j'| \ge 3\epsilon.
		\end{align*}
		Therefore $\cH$ is $3\epsilon$-separated. Since every ball of radius $\epsilon$ has diameter at most $2\epsilon$, each such ball contains at most one element of $\cH$. It follows that
		\begin{align*}
			N\left(\epsilon,\cA_{\loc}^{\soft}(M),\|\cdot\|_\infty\right) \ge |\cH| = \left(\left\lfloor\frac{2Rm_MC_1}{3\epsilon}\right\rfloor+1\right)^{d+1}.
		\end{align*}
		
		We next prove the upper bound of the covering number. Fix $h \in \cA_{\loc}^{\soft}(M)$. Then
		\begin{align*}
			h(x)=\sum_{i=1}^M s_i(x)c_i, \qquad |c_i|\le R,
		\end{align*}
		where
		\begin{align*}
			s_i(x)=g_i^{\soft}(x;\theta)=\frac{\exp\{\beta_i^\top x+\alpha_i\}}{\sum_{j=1}^M \exp\{\beta_j^\top x+\alpha_j\}}, \qquad (\beta_i^\top,\alpha_i)^\top\in\Theta_{\lin}(C_1).
		\end{align*}
		Using expert $M$ as a reference, define
		\begin{align*}
			v_0:=c_M, \qquad v_i:=c_i-c_M\in[-2R,2R], \qquad a_i:=\beta_i-\beta_M\in[-2C_1,2C_1]^d, \qquad b_i:=\alpha_i-\alpha_M\in[-2C_1,2C_1]
		\end{align*}
		for $i=1,\dots,M-1$. Then $h(x)=v_0+\sum_{i=1}^{M-1} v_i q_i(x)$, where
		\begin{align*}
			q_i(x) = \frac{\exp\{a_i^\top x+b_i\}}{1+\sum_{j=1}^{M-1} \exp\{a_j^\top x+b_j\}}.
		\end{align*}
		
		Let $h$ and $h'$ correspond to parameter tuples $(v_0,v,a,b)$ and $(v_0',v',a',b')$. Then
		\begin{align*}
			|h(x)-h'(x)| \le |v_0-v_0'| + \sum_{i=1}^{M-1}|v_i-v_i'|q_i(x) + \sum_{i=1}^{M-1}|v_i'||q_i(x)-q_i'(x)|.
		\end{align*}
		Since $\sum_{i=1}^{M-1}q_i(x)\le 1$,
		\begin{align*}
			\sum_{i=1}^{M-1}|v_i-v_i'|q_i(x)\le \|v-v'\|_\infty.
		\end{align*}
		Also $|v_i'|\le 2R$.
		
		It remains to control the difference between $q$ and $q'$. Let
		\begin{align*}
			S(z_1,\dots,z_M) = \left(\frac{\exp\{z_1\}}{\sum_{j=1}^M \exp\{z_j\}},\dots,\frac{\exp\{z_M\}}{\sum_{j=1}^M \exp\{z_j\}}\right)
		\end{align*}
		be the $M$-dimensional softmax map. Its Jacobian satisfies
		\begin{align*}
			DS(z)v = p\odot\left(v-\langle p,v\rangle \mathbf{1}\right), \qquad p=S(z),
		\end{align*}
		where $\odot$ denotes the Hadamard product. Therefore,
		\begin{align*}
			\|DS(z)v\|_1 = \sum_{i=1}^M p_i |v_i-\langle p,v\rangle| \le 2\|v\|_\infty.
		\end{align*}
		By the mean value theorem, $\|S(z)-S(z')\|_1\le 2\|z-z'\|_\infty$. Applying this to
		\begin{align*}
			z(x)=(a_1^\top x+b_1,\dots,a_{M-1}^\top x+b_{M-1},0), \qquad z'(x)=({a'_1}^\top x+b_1',\dots,{a'_{M-1}}^\top x+b_{M-1}',0),
		\end{align*}
		we obtain
		\begin{align*}
			\sum_{i=1}^{M-1}|q_i(x)-q_i'(x)| \le 2\|z(x)-z'(x)\|_\infty.
		\end{align*}
		Since $x\in[0,1]^d$,
		\begin{align*}
			\|z(x)-z'(x)\|_\infty \le d\max_{1\le i\le M-1}\|a_i-a_i'\|_\infty + \max_{1\le i\le M-1}|b_i-b_i'|.
		\end{align*}
		Combining the estimates gives
		\begin{align*}
			\|h-h'\|_\infty \le |v_0-v_0'| + \|v-v'\|_\infty + 4R\left(d\max_{1\le i\le M-1}\|a_i-a_i'\|_\infty + \max_{1\le i\le M-1}|b_i-b_i'|\right).
		\end{align*}
		
		Now choose parameter nets satisfying
		\begin{align*}
			|v_0-v_0'|\le \frac{\epsilon}{3}, \qquad \|v-v'\|_\infty\le \frac{\epsilon}{3}, \qquad \max_i\|a_i-a_i'\|_\infty\le \frac{\epsilon}{24Rd}, \qquad \max_i|b_i-b_i'|\le \frac{\epsilon}{24R}.
		\end{align*}
		Then $\|h-h'\|_\infty\le \epsilon$. The interval $[-R,R]$ can be covered at accuracy $\epsilon/3$ by $\left\lceil\frac{3R}{\epsilon}\right\rceil$ points. Each interval $[-2R,2R]$ can be covered at the same accuracy by $\left\lceil\frac{6R}{\epsilon}\right\rceil$ points. Each interval $[-2C_1,2C_1]$ can be covered at accuracy $\epsilon/(24Rd)$ by $\left\lceil\frac{48RdC_1}{\epsilon}\right\rceil$ points, and at accuracy $\epsilon/(24R)$ by $\left\lceil\frac{48RC_1}{\epsilon}\right\rceil$ points. Since there are $M-1$ scalar parameters $v_i$, $(M-1)d$ coordinates in the vectors $a_i$, and $M-1$ intercept differences $b_i$, the stated upper bound follows.
		
		Finally, taking logarithms yields
		\begin{align*}
			(d+1)\log(1/\epsilon)+O(1) \le \log N\left(\epsilon,\cA_{\loc}^{\soft}(M),\|\cdot\|_\infty\right) \le \left[M+(M-1)(d+1)\right]\log(1/\epsilon)+O(1)
		\end{align*}
		for fixed $M\ge 2$.
		
		\textbf{For part (iii),}
		fix $1 \le K < M$. For each $t\in(0,1)$, define experts
		\begin{align*}
			f_1\equiv R,\qquad f_2\equiv -R,\qquad f_3=\cdots=f_M\equiv 0,
		\end{align*}
		and scores
		\begin{align*}
			u_1^{(t)}(x)=C_1(x_1-t),\qquad u_2^{(t)}(x)=-C_1(x_1-t),\qquad u_3^{(t)}(x)=\cdots=u_M^{(t)}(x)=0.
		\end{align*}
		All these parameters lie in $\Theta_{\lin}(C_1)$. Let $h_t\in\cA_{\loc}^{\TopK}(M)$ be the resulting function.
		
		Take $s<t$ and set $x=((s+t)/2,0,\dots,0)^\top$. Then $x_1=(s+t)/2$ lies strictly between $s$ and $t$. Hence
		\begin{align*}
			u_1^{(s)}(x)>0,\qquad u_2^{(s)}(x)<0,
		\end{align*}
		Thus the Top-$K$ rule for $h_s$ selects expert $1$ together with $K-1$ zero experts. Therefore
		\begin{align*}
			h_s(x) = R\cdot \frac{\exp\{u_1^{(s)}(x)\}}{\exp\{u_1^{(s)}(x)\}+(K-1)} > \frac{R}{K}.
		\end{align*}
		Similarly,
		\begin{align*}
			u_1^{(t)}(x)<0,\qquad u_2^{(t)}(x)>0.
		\end{align*}
		Thus the Top-$K$ rule for $h_t$ selects expert $2$ together with $K-1$ zero experts, and
		\begin{align*}
			h_t(x) = -R\cdot \frac{\exp\{u_2^{(t)}(x)\}}{\exp\{u_2^{(t)}(x)\}+(K-1)} < -\frac{R}{K}.
		\end{align*}
		Consequently,
		\begin{align*}
			\|h_s-h_t\|_\infty\ge |h_s(x)-h_t(x)|> \frac{2R}{K}.
		\end{align*}
		Thus $\{h_t:t\in(0,1)\}$ is an uncountable $2R/K$-separated family. Therefore, for every $0<\epsilon\le R/K$, every ball of radius $\epsilon$ has diameter at most $2\epsilon\le 2R/K$ and hence contains at most one element of this family. Consequently,
		\begin{align*}
			N\left(\epsilon,\cA_{\loc}^{\TopK}(M),\|\cdot\|_\infty\right)=\infty.
		\end{align*}
	\end{proof}
	
	\subsection{Proof of Example \ref{main-exa:complex_experts}}
	\label{subsec:supp-proof-ex-1}
	
	\begin{proof}[Proof of Example \ref{main-exa:complex_experts}]
		For the local representation, take two experts
		\begin{align*}
			e_1(x)=\sigma(x_1), \qquad e_2(x)\equiv 0.
		\end{align*}
		Both experts belong to $\cE_\sigma$. Choose the scores
		\begin{align*}
			u_1(x)=x_1, \qquad u_2(x)=0.
		\end{align*}
		Then the softmax weight of expert $1$ is
		\begin{align*}
			\frac{\exp(x_1)}{\exp(x_1)+1} = \sigma(x_1),
		\end{align*}
		and the resulting local average is
		\begin{align*}
			\frac{\exp(x_1)}{\exp(x_1)+1}e_1(x) + \frac{1}{\exp(x_1)+1}e_2(x) = \sigma(x_1)^2 = h(x).
		\end{align*}
		Thus $h\in\cA_{\loc}^{\soft}(\cE_{\sigma};2)$.
		
		We next show that $h\notin \bigcup_{M\ge 1}\cA_{\glob}(\cE_{\sigma};M)$. Suppose otherwise. Then for some $M$,
		\begin{align*}
			h(x)=\sum_{i=1}^M w_i e_i(x), \qquad w_i\ge 0, \qquad \sum_{i=1}^M w_i=1,
		\end{align*}
		with $e_i(x)=c_{0,i}+c_{1,i}\sigma(a_i x_1+b_i)$. Collecting the constants, we obtain an identity on $[0,1]$:
		\begin{align*}
			\sigma(t)^2 = \beta_0+\sum_{i=1}^M \beta_i\sigma(a_i t+b_i), \qquad t \in [0,1],
		\end{align*}
		for suitable real coefficients $\beta_0,\dots,\beta_M$. Although the aggregation weights $w_i$ are nonnegative and sum to one, the coefficients $\beta_i$ need not be nonnegative, since the expert coefficients $c_{1,i}$ are arbitrary real numbers.
		
		Now view both sides as meromorphic functions of a complex variable $z$. The logistic function $\sigma(z)=\frac{1}{1+\exp(-z)}$ has simple poles at $z=(2\ell+1)\pi i$ for $\ell\in\mathbb{Z}$. Therefore each nonconstant function $\sigma(a_i z+b_i)$ has only simple poles, while if $a_i = 0$, the function is constant and has no pole. Hence any finite linear combination
		\begin{align*}
			\beta_0+\sum_{i=1}^M \beta_i\sigma(a_i z+b_i)
		\end{align*}
		has at most simple poles at every point of $\mathbb{C}$. On the other hand, $\sigma(z)^2$ has a second-order pole at $z=\pi i$.
		
		The two meromorphic functions agree on the real interval $[0,1]$, which has an accumulation point in a domain where both sides are analytic. By the identity theorem for meromorphic functions, they must agree as meromorphic functions. This is impossible, because the left-hand side has a second-order pole at $z=\pi i$, while the right-hand side can have at most a simple pole there. Hence $h\notin \bigcup_{M\ge 1}\cA_{\glob}(\cE_{\sigma};M)$.
	\end{proof}
	
	\subsection{Proof of Proposition \ref{main-prop:two_specialized_experts}}
	
	\begin{proof}[Proof of Proposition \ref{main-prop:two_specialized_experts}]
		Fix any $w \in [0,1]$ and write $\hat f_w^{\glob}:=w\hat f_1+(1-w)\hat f_2$. On $\cX_1$, the reverse triangle inequality gives
		\begin{align*}
			\|\hat f_w^{\glob}-f\|_{\cX_1}
			&= \|w(\hat f_1-f)+(1-w)(\hat f_2-f)\|_{\cX_1} \\
			&\ge (1-w)\|\hat f_2-f\|_{\cX_1}-w\|\hat f_1-f\|_{\cX_1} \\
			&= (1-w)\delta_1-w\epsilon_1.
		\end{align*}
		Similarly, on $\cX_2$,
		\begin{align*}
			\|\hat f_w^{\glob}-f\|_{\cX_2}
			&= \|w(\hat f_1-f)+(1-w)(\hat f_2-f)\|_{\cX_2} \\
			&\ge w\|\hat f_1-f\|_{\cX_2}-(1-w)\|\hat f_2-f\|_{\cX_2} \\
			&= w\delta_2-(1-w)\epsilon_2.
		\end{align*}
		Hence
		\begin{align*}
			\max_{j=1,2}\|\hat f_w^{\glob}-f\|_{\cX_j} \ge \max\left\{(1-w)\delta_1-w\epsilon_1,\,w\delta_2-(1-w)\epsilon_2\right\}_+.
		\end{align*}
		The first affine term is decreasing in $w$, while the second is increasing in $w$. Balancing them gives
		\begin{align*}
			w_\star = \frac{\delta_1+\epsilon_2}{\delta_1+\delta_2+\epsilon_1+\epsilon_2},
		\end{align*}
		and the common value is $(\delta_1\delta_2-\epsilon_1\epsilon_2)/(\delta_1+\delta_2+\epsilon_1+\epsilon_2)$. Taking the positive part yields the lower bound $\gamma$.
		
		Since
		\begin{align*}
			R(\hat f_w^{\glob}) = \|\hat f_w^{\glob}-f\|_{\cX_1}^2+\|\hat f_w^{\glob}-f\|_{\cX_2}^2,
		\end{align*}
		at least one regional norm being no smaller than $\gamma$ implies $R(\hat f_w^{\glob})\ge \gamma^2$. Taking the infimum over $w \in [0,1]$ gives the desired lower bound for $\cA_{\glob}(\hat f_1,\hat f_2)$. Finally, by the definition of $\hat f_{\loc}^{\mathrm{Top\mbox{-}1}}$,
		\begin{align*}
			R(\hat f_{\loc}^{\mathrm{Top\mbox{-}1}}) = \|\hat f_1-f\|_{\cX_1}^2+\|\hat f_2-f\|_{\cX_2}^2 = \epsilon_1^2+\epsilon_2^2.
		\end{align*}
		This proves the proposition.
	\end{proof}
	
	\section{Proof of Section \ref{main-sec:oracle_softmax}}
	\label{sec:supp-proof-section-3}
	
	\subsection{Proof of Theorem \ref{main-thm:softmax}}
	\label{subsec:supp-proof-thm-3-1}
	
	We consider the same regression model as in Section \ref{main-subsec:BasicSetting} of the main article. We first introduce the following generalized version of AFTER algorithm \citep{yang2004Combining} in Algorithm \ref{alg:scale_online}, which allows the aggregation of covariate-dependent forecasting procedures. Throughout this subsection, all candidate forecasts and scale estimators are understood to satisfy the non-anticipation convention stated in Section \ref{main-subsec:BasicSetting}. In particular, the quantities used to predict $Y_{t+1}$ may depend on the observations available up to time $t$, the auxiliary algorithmic randomization, and the newly observed covariate $X_{t+1}$, but not on $Y_{t+1}$ or future observations.
	
	\begin{algorithm}[!h]
		\caption{General AFTER Algorithm}\label{alg:scale_online}
		\textbf{Input:} A burn-in size $n_1$; a class of $S$ forecasts $\{\mu_{s,t}(x)\}_{s=1}^S$; a black-box scale estimator; data $\{(X_i,Y_i)\}_{i=1}^n$; prior weights $\{\pi_s\}_{s=1}^S$.\\
		\textbf{Output:} An online aggregated sequence $\{\tilde{\mu}_t(x)\}_{t=n_1}^{n-1}$.
		{\small
			\begin{enumerate}[label=(\alph*)]
				\item Randomly partition the data into two subsamples $Z^{(1)} = (X_i,Y_i)_{i=1}^{n_1}$ and $Z^{(2)} = (X_i,Y_i)_{i=n_1+1}^{n}$, and let $n_2 = n-n_1$. The first subsample is used for burn-in, and the second is for aggregation.
				
				\item \textbf{Burn-in:} Process $Z^{(1)}$ to obtain the initial forecasts $\mu_{s,n_1}(x)$ and initial scale estimators $\sigma_{s,n_1}$.
				
				\item \textbf{Aggregation:} Initialize $w_{s,n_1}=\pi_s$. For $t = n_1, \ldots, n-1$:
				\begin{enumerate}[label=(\roman*)]
					\item Construct the aggregated forecast estimator
					\begin{align*}
						\tilde{\mu}_t(x)=\sum_{s=1}^S w_{s,t}\mu_{s,t}(x).
					\end{align*}
					
					\item Denote the one-step-ahead joint density of $Z_{t+1}=(X_{t+1},Y_{t+1})$ under candidate $s$ at time $t$, with respect to the product measure $P_X(dx)dy$, by
					\begin{align*}
						q_{s,t}(x,y)=\frac{1}{\hat{\sigma}_{s,t}}h_{\varepsilon,0}\left(\frac{y-\mu_{s,t}(x)}{\hat{\sigma}_{s,t}}\right).
					\end{align*}
					
					\item After observing $Z_{t+1}$, update the weights $w_{s,t+1}$ by
					\begin{align}
						w_{s,t+1} = \frac{w_{s,t}q_{s,t}(X_{t+1},Y_{t+1})}{\sum_{r=1}^S w_{r,t}q_{r,t}(X_{t+1},Y_{t+1})}. \label{eqn:gen-after-update}
					\end{align}
					
					\item Using $Z_{t+1}$ and previous data, update the scale estimators by its black-box rule to obtain $\{\hat{\sigma}_{s,t+1}\}_{s=1}^{S}$, then evolve forecasts to obtain $\{\mu_{s,t+1}(x)\}_{s=1}^S$.
				\end{enumerate}
		\end{enumerate}}
	\end{algorithm}
	
	We first introduce the assumptions needed to establish the theory for the introduced algorithm.
	
	\begin{assumption}[Uniform boundedness]\label{ass:bound_forecast}
		There exists a constant $A<\infty$ such that, almost surely for $t\ge n_1$,
		\begin{align*}
			\|f_0\|_{\infty}\vee\max_{s\in[S]}\|\mu_{s,t}\|_{\infty}\le A.
		\end{align*}
	\end{assumption}
	
	\begin{assumption}\label{ass:bound_scale_est}
		Assume that the produced scale estimators $\{\hat{\sigma}_{s,t}\}_{s=1}^S$ satisfies $0<\underline{\sigma}\le\hat{\sigma}_{s,t}\le\bar{\sigma}<\infty$, where $\underline{\sigma}$ and $\bar{\sigma}$ are constants in Assumption \ref{main-ass:ARM-regularity}.
	\end{assumption}
	
	Here, Assumption \ref{ass:bound_scale_est} can be reached through simple truncation. Then, we have the following result.
	
	\begin{proposition}[Risk bound for general AFTER algoritm]\label{prop:risk_AFTER}
		Let $\{\tilde{\mu}_t\}_{t=n_1}^{n-1}$ be the online sequence produced by Algorithm \ref{alg:scale_online}. Assume the data generation mechanism satisfies Assumption \ref{main-ass:ARM-regularity}, and Assumptions \ref{ass:bound_forecast}, \ref{ass:bound_scale_est} hold. Then, there exists a constant $C$ depending only on the boundedness and error-distribution constants, such that
		\begin{align}
			\frac{1}{n_2}\sum_{t=n_1}^{n-1}\EE\left\|\tilde{\mu}_t-f_0\right\|_{L_2(P_X)}^2 \le C\inf_{s \le S}\left(\frac{1}{n_2}\sum_{t=n_1}^{n-1}\EE\left\|\mu_{s,t}-f_0\right\|_{L_2(P_X)}^2+\frac{\log(1/\pi_s)}{n_2}+\cV_s\right), \label{eqn:riskbound_after}
		\end{align}
		where
		\begin{align*}
			\cV_s=\frac{1}{n_2}\sum_{t=n_1}^{n-1}\EE\left(\hat{\sigma}_{s,t}-\sigma\right)^2.
		\end{align*}
	\end{proposition}
	
	\begin{proof}[Proof of Proposition \ref{prop:risk_AFTER}]
		Let
		\begin{align*}
			p_t(x,y)=\frac{1}{\sigma}h_{\varepsilon,0}\left(\frac{y-f_0(x)}{\sigma}\right)
		\end{align*}
		be the true joint density of $Z_{t+1}$. Define the mixture density
		\begin{align*}
			q_t(x,y)=\sum_{s=1}^S w_{s,t}q_{s,t}(x,y).
		\end{align*}
		The AFTER recursion \eqref{eqn:gen-after-update} yields
		\begin{align*}
			w_{s,t}=\frac{\pi_s\prod_{\ell=n_1}^{t-1}q_{s,\ell}(X_{\ell+1},Y_{\ell+1})}{\sum_{r=1}^S\pi_r\prod_{\ell=n_1}^{t-1}q_{r,\ell}(X_{\ell+1},Y_{\ell+1})},
		\end{align*}
		which results in
		\begin{align*}
			q_t(X_{t+1},Y_{t+1})=\frac{\sum_{r=1}^S\pi_r\prod_{\ell=n_1}^{t}q_{r,\ell}(X_{\ell+1},Y_{\ell+1})}{\sum_{r=1}^S\pi_r\prod_{\ell=n_1}^{t-1}q_{r,\ell}(X_{\ell+1},Y_{\ell+1})},
		\end{align*}
		and thus we have the standard likelihood identity
		\begin{align*}
			\prod_{t=n_1}^{n-1}q_t(X_{t+1},Y_{t+1})=\sum_{r=1}^S\pi_r\prod_{t=n_1}^{n-1}q_{r,t}(X_{t+1},Y_{t+1}).
		\end{align*}
		Thus, for a fixed $s\in[S]$, using a similar argument as in the proof of Theorem 1 in \citep{yang2001adaptive}, we have
		\begin{align}
			\sum_{t=n_1}^{n-1}\EE[D(p_t\|q_t)]
			&=\EE\left[\log\frac{\prod_{t=n_1}^{n-1}p_t(X_{t+1},Y_{t+1})}{\prod_{t=n_1}^{n-1}q_t(X_{t+1},Y_{t+1})}\right] \nonumber\\
			&\leq \EE\left[\log\frac{\prod_{t=n_1}^{n-1}p_t(X_{t+1},Y_{t+1})}{\pi_s\prod_{t=n_1}^{n-1}q_{s,t}(X_{t+1},Y_{t+1})}\right] \nonumber\\
			&=\log(1/\pi_s)+\sum_{t=n_1}^{n-1}\EE\left[D(p_t\|q_{s,t})\right]. \label{eqn:KLdist_mixture}
		\end{align}
		Here, the densities $p_t$, $q_t$, and $q_{s,t}$ are conditional one-step-ahead densities given the online history up to time $t$. Accordingly,
		$D(p_t\|q_t)$ is a conditional KL divergence and is itself a random variable through its dependence on the past data and the algorithmic randomness. Thus $\EE[D(p_t\|q_t)]$ denotes the outer expectation of this conditional KL divergence. The expectations in the logarithmic display are taken over the full data path and the algorithmic randomness.
		
		Now, a similar argument as in \citep{yang2001adaptive} shows that
		\begin{align}
			D(p_t\|q_{s,t})
			&= \int\left(\int\frac{1}{\sigma}h_{\varepsilon,0}\left(\frac{y-f_0(x)}{\sigma}\right)\log\left(\frac{\sigma^{-1}h_{\varepsilon,0}\left((y-f_0(x))/\sigma\right)}{\hat{\sigma}_{s,t}^{-1}h_{\varepsilon,0}\left((y-\mu_{s,t}(x))/\hat{\sigma}_{s,t}\right)}\right)dy\right)P_X(dx) \nonumber\\
			&= \int\left(h_{\varepsilon,0}(y)\log\left(\frac{h_{\varepsilon,0}(y)}{\sigma/\hat{\sigma}_{s,t}h_{\varepsilon,0}\left(\sigma(y-\sigma^{-1}(\mu_{s,t}(x)-f_0(x)))/\hat{\sigma}_{s,t}\right)}\right)dy\right)P_X(dx) \nonumber\\
			&\leq C\int\left(\left(1-\frac{\hat{\sigma}_{s,t}}{\sigma}\right)^2+\left(\frac{\mu_{s,t}(x)-f_0(x)}{\sigma}\right)^2\right)\,P_X(dx) \nonumber\\
			&\leq C\left(\|\mu_{s,t}-f_0\|_{L_2(P_X)}^2+(\hat{\sigma}_{s,t}-\sigma)^2\right). \label{eqn:KLdiv_RHS}
		\end{align}
		Here the third line is by Assumption \ref{main-ass:ARM-regularity}, where Assumption \ref{ass:bound_scale_est} ensures that $\hat{\sigma}_{s,t}/\sigma$ is bounded and Assumption \ref{ass:bound_forecast} together with $\underline{\sigma}\leq \sigma\leq \bar{\sigma}$ ensure that $(\mu_{s,t}(x)-f_0(x))/\sigma$ is uniformly bounded.
		
		Now, consider the squared Hellinger distance $H^2(p,q)$. We use the convention
		\begin{align*}
			H^2(p,q):=\int(\sqrt p-\sqrt q)^2\,d\mu.
		\end{align*}
		The standard Hellinger inequality gives
		\begin{align*}
			H^2(p,q)\leq D(p\|q).
		\end{align*}
		Moreover, for densities $p$ and $q$ with uniformly bounded means $m_p$, $m_q$ and variance $v_p$, $v_q$, an elementary moment-Hellinger bound
		\begin{align*}
			H^2(p,q)\geq \frac{(m_p-m_q)^2}{2(v_p+v_q)+(m_p-m_q)^2}.
		\end{align*}
		Indeed, let $\Delta=m_p-m_q$ and $a=(m_p+m_q)/2$, then using Cauchy Schwarz's inequality, we have
		\begin{align*}
			\Delta^2
			&= \left(\int x(p-q)\,d\mu\right)^2=\left(\int (x-a)(p-q)\,dx\right)^2\\
			&\le \left(\int(\sqrt p-\sqrt q)^2\,dx\right)\left(\int (x-a)^2(\sqrt p+\sqrt q)^2\,dx\right)\\
			&\leq 2H^2(p,q)\left(\int(x-a)^2(p+q)dx\right)\\
			&= 2H^2(p,q)\left(v_p+v_q+\frac{\Delta^2}{2}\right),
		\end{align*}
		and thus the desired inequality holds. Now, if $H^2(p,q)\leq1/2$, then $(m_p-m_q)^2\leq4(v_p+v_q)H^2(p,q)$.
		If $H^2(p,q)>1/2$, the uniform mean bound gives $(m_p-m_q)^2\leq C H^2(p,q)$ after enlarging $C$.
		Thus, in all cases,
		\begin{align*}
			(m_p-m_q)^2\leq CH^2(p,q)\leq CD(p\|q).
		\end{align*}
		
		We now apply this inequality conditionally on the online history and on the covariate value. For fixed $x$ and time $t$, let
		\begin{align*}
			p_t^x(y)=\sigma^{-1}h_{\varepsilon,0}\left(\frac{y-f_0(x)}{\sigma}\right), \qquad q_t^x(y)=\sum_{s=1}^S w_{s,t}\hat\sigma_{s,t}^{-1}h_{\varepsilon,0}\left(\frac{y-\mu_{s,t}(x)}{\hat\sigma_{s,t}}\right).
		\end{align*}
		The conditional mean of $p_t^x$ is $f_0(x)$, while the conditional mean of $q_t^x$ is
		\begin{align*}
			\sum_{s=1}^S w_{s,t}\mu_{s,t}(x)=\tilde\mu_t(x).
		\end{align*}
		Let $v_0$ denote the variance of $h_{\varepsilon,0}$. The conditional variance of $p_t^x$ is $\sigma^2v_0$. The conditional variance of $q_t^x$ is
		\begin{align*}
			\sum_{s=1}^S w_{s,t}\hat\sigma_{s,t}^2v_0+\sum_{s=1}^S w_{s,t}\{\mu_{s,t}(x)-\tilde\mu_t(x)\}^2.
		\end{align*}
		By Assumptions~\ref{ass:bound_forecast} and~\ref{ass:bound_scale_est}, this variance is bounded above by $\bar{\sigma}^2v_0+4A^2$. Moreover, both conditional means are uniformly bounded. Hence the preceding moment-Hellinger bound implies that, for a constant $C$ independent of $s,t$ and $x$,
		\begin{align*}
			(\tilde\mu_t(x)-f_0(x))^2\le C H^2(p_t^x,q_t^x)\le C D(p_t^x\|q_t^x).
		\end{align*}
		Integrating with respect to $P_X(dx)$ gives, conditionally on the online history,
		\begin{align*}
			\|\tilde\mu_t-f_0\|_{L_2(P_X)}^2\le C D(p_t\|q_t),
		\end{align*}
		and thus taking outer expectation yields
		\begin{align}
			\EE\left\|\tilde{\mu}_t-f_0\right\|_{L_2(P_X)}^2\leq C\EE D(p_t\|q_t). \label{eqn:KLdiv_LHS}
		\end{align}
		
		Combining \eqref{eqn:KLdiv_LHS} with \eqref{eqn:KLdist_mixture} and \eqref{eqn:KLdiv_RHS}, then taking infimum over all $s\in[S]$, one can get the desired result.
	\end{proof}
	
	We now turn to the specific scenario considered in Algorithm \ref{main-alg:general}, that is, the burn-in block in Algorithm \ref{alg:scale_online} are further separated into a burn-in block and a calibration block, and in the aggregation block, we have $\mu_{s,t}(x)=\check{f}^t(x;\theta^{(s)})$, $\pi_s\equiv 1/S_{\eta}$, and $\hat{\sigma}_{s,t}\equiv \hat{\sigma}_s$. Denote $\cF_t$ be the $\sigma$-field generated by $\{Z_\ell\}_{\ell=1}^{t}$, and the randomness in the black-box algorithm up to time $t$. We first state the following lemma that removes the scale estimation term $\cV_s$ in \eqref{eqn:riskbound_after}.
	
	\begin{lemma}[Calibration-scale error for predictable candidate forecasts]\label{lemma:remove_scale_estimation}
		Let $\{\mu_{s,t}:s\in[S],\ t=n_1,\ldots,n_1+n_2-1\}$ be a collection of predictable candidate forecasts such that $\mu_{s,t}$ is measurable with respect to $\sigma(Z_1,\dots,Z_t,U)$ as a random function and
		\begin{align*}
			\|f_0\|_\infty \vee \sup_{s\in[S]} \sup_{t=n_1,\ldots,n_1+n_2-1} \|\mu_{s,t}\|_\infty \le B
		\end{align*}
		almost surely for some constant $B<\infty$. Define
		\begin{align*}
			\widetilde\sigma_s^2 = \frac{1}{n_2}\sum_{t=n_1}^{n_1+n_2-1}\left\{Y_{t+1}-\mu_{s,t}(X_{t+1})\right\}^2,
		\end{align*}
		and let $\hat\sigma_s^2 = \Pi_{[\underline\sigma^2,\bar\sigma^2]}\left(\widetilde\sigma_s^2\right)$ with $\hat\sigma_s=(\hat\sigma_s^2)^{1/2}$. Suppose Assumption~\ref{main-ass:ARM-regularity} holds. Then, for every $s\in[S]$,
		\begin{align}
			\EE\left(\hat\sigma_s-\sigma\right)^2 \le C\left\{\frac{1}{n_2}+\frac{1}{n_2}\sum_{t=n_1}^{n_1+n_2-1}\EE\left\|\mu_{s,t}-f_0\right\|_{L_2(P_X)}^2\right\}, \label{eqn:scale_remove}
		\end{align}
		where $C$ depends only on $B$, $\underline\sigma$, $\bar\sigma$, and the fourth moment of the standardized error distribution.
	\end{lemma}
	
	\begin{proof}[Proof of Lemma~\ref{lemma:remove_scale_estimation}]
		Fix $s\in[S]$ and define $r_{s,t}(x) = f_0(x)-\mu_{s,t}(x), t=n_1,\ldots,n_1+n_2-1$. By predictability, $r_{s,t}$ is measurable with respect to $\sigma(Z_1,\dots,Z_t,U)$ as a random function. Moreover, $Y_{t+1}-\mu_{s,t}(X_{t+1}) = \varepsilon_{t+1} + r_{s,t}(X_{t+1})$. Hence
		\begin{align*}
			\widetilde\sigma_s^2-\sigma^2
			&= \frac{1}{n_2}\sum_{t=n_1}^{n_1+n_2-1}\left(\varepsilon_{t+1}^2-\sigma^2\right) + \frac{2}{n_2}\sum_{t=n_1}^{n_1+n_2-1}\varepsilon_{t+1}r_{s,t}(X_{t+1}) + \frac{1}{n_2}\sum_{t=n_1}^{n_1+n_2-1}r_{s,t}^2(X_{t+1}) \\
			&=: I+II+III.
		\end{align*}
		The finite fourth moment of the error gives
		\begin{align}
			\EE I^2 \le \frac{C}{n_2}. \label{eqn:control_I}
		\end{align}
		Since $\EE(\varepsilon_{t+1}|Z_1,\dots,Z_t,U,X_{t+1}) = 0$, the summands constituting $II$ form a martingale-difference sequence. Moreover, because $\varepsilon_{t+1}$ is independent of $(Z_1,\dots,Z_t,U,X_{t+1})$, $\EE(\varepsilon_{t+1}^2)=\sigma^2$. Then we have
		\begin{align}
			\EE II^2
			&= \frac{4}{n_2^2}\sum_{t=n_1}^{n_1+n_2-1}\EE\left[\varepsilon_{t+1}^2 r_{s,t}^2(X_{t+1})\right] \nonumber\\
			&= \frac{4\sigma^2}{n_2^2}\sum_{t=n_1}^{n_1+n_2-1}\EE\left\|r_{s,t}\right\|_{L_2(P_X)}^2 \nonumber\\
			&\le \frac{C}{n_2}\left\{\frac{1}{n_2}\sum_{t=n_1}^{n_1+n_2-1}\EE\left\|r_{s,t}\right\|_{L_2(P_X)}^2\right\}. \label{eqn:control_II}
		\end{align}
		By Cauchy's inequality, $III^2 \le \frac{1}{n_2}\sum_{t=n_1}^{n_1+n_2-1} r_{s,t}^4(X_{t+1})$. The uniform boundedness of $f_0$ and $\mu_{s,t}$ implies
		\begin{align}
			\EE III^2 \le C\left\{\frac{1}{n_2}\sum_{t=n_1}^{n_1+n_2-1}\EE\left\|r_{s,t}\right\|_{L_2(P_X)}^2\right\}. \label{eqn:control_III}
		\end{align}
		Combining \eqref{eqn:control_I}--\eqref{eqn:control_III} gives
		\begin{align}
			\EE\left(\widetilde\sigma_s^2-\sigma^2\right)^2 \le C\left\{\frac{1}{n_2}+\frac{1}{n_2}\sum_{t=n_1}^{n_1+n_2-1}\EE\left\|\mu_{s,t}-f_0\right\|_{L_2(P_X)}^2\right\}. \label{eqn:scale_square_remove}
		\end{align}
		Since $\sigma^2\in[\underline\sigma^2,\bar\sigma^2]$ and projection onto a closed interval is nonexpansive, $|\hat\sigma_s^2-\sigma^2| \le |\widetilde\sigma_s^2-\sigma^2|$. Furthermore,
		\begin{align*}
			\left|\hat\sigma_s-\sigma\right| = \frac{\left|\hat\sigma_s^2-\sigma^2\right|}{\hat\sigma_s+\sigma} \le \frac{1}{2\underline\sigma}\left|\hat\sigma_s^2-\sigma^2\right|.
		\end{align*}
		The result follows.
	\end{proof}
	
	\begin{lemma}[Evolving-to-population comparator reduction]\label{lem:supp-evolving-population-comparator}
		Let $\Theta$ be a parameter set and let $g(\cdot;\theta) = \left(g_1(\cdot;\theta),\ldots,g_{\MR}(\cdot;\theta)\right)^\top, \theta\in\Theta$ be any gating rule satisfying $g_m(x;\theta)\ge 0$ and $\sum_{m=1}^{\MR}g_m(x;\theta)=1$ for every $x\in\cX$ and $\theta\in\Theta$. Define
		\begin{align*}
			\check f_g^t(x;\theta) &= \sum_{\ell=1}^{\MS}\hat f_{S,\ell}^t(x) + \sum_{m=1}^{\MR}g_m(x;\theta)\hat f_{R,m}^t(x), \\
			f_g^*(x;\theta) &= \sum_{\ell=1}^{\MS}f_{S,\ell}^*(x) + \sum_{m=1}^{\MR}g_m(x;\theta)f_{R,m}^*(x).
		\end{align*}
		Under Assumption~\ref{main-ass:evolving-experts},
		\begin{align}
			\EE\sup_{\theta\in\Theta}\left\|\check f_g^t(\cdot;\theta)-f_g^*(\cdot;\theta)\right\|_{L_2(P_X)}^2 \le C\left(\MS^2+\MR\right)a_t^2. \label{eq:supp-evolving-population-distance}
		\end{align}
		Consequently, for every $\theta\in\Theta$,
		\begin{align}
			\EE\left\|\check f_g^t(\cdot;\theta)-f_0\right\|_{L_2(P_X)}^2 \le 2\left\|f_g^*(\cdot;\theta)-f_0\right\|_{L_2(P_X)}^2 + C\left(\MS^2+\MR\right)a_t^2. \label{eq:supp-evolving-population-risk}
		\end{align}
		In particular, when $\MS$ is fixed and $\MR\ge 2$, the last term is bounded by $C\MR a_t^2$.
	\end{lemma}
	
	\begin{proof}
		Write
		\begin{align*}
			U_t = \sum_{\ell=1}^{\MS}\left(\hat f_{S,\ell}^t-f_{S,\ell}^*\right), \qquad V_t(\theta) = \sum_{m=1}^{\MR}g_m(\cdot;\theta)\left(\hat f_{R,m}^t-f_{R,m}^*\right).
		\end{align*}
		Then $\check f_g^t(\cdot;\theta)-f_g^*(\cdot;\theta) = U_t+V_t(\theta)$. For the shared component,
		\begin{align*}
			\left\|U_t\right\|_{L_2(P_X)}^2 \le \MS\sum_{\ell=1}^{\MS}\left\|\hat f_{S,\ell}^t-f_{S,\ell}^*\right\|_{L_2(P_X)}^2.
		\end{align*}
		For the routed component, Jensen's inequality gives, pointwise in $x$,
		\begin{align*}
			\left|V_t(x;\theta)\right|^2 \le \sum_{m=1}^{\MR}g_m(x;\theta)\left|\hat f_{R,m}^t(x)-f_{R,m}^*(x)\right|^2.
		\end{align*}
		Therefore,
		\begin{align*}
			\sup_{\theta\in\Theta}\left\|V_t(\theta)\right\|_{L_2(P_X)}^2 \le \sum_{m=1}^{\MR}\left\|\hat f_{R,m}^t-f_{R,m}^*\right\|_{L_2(P_X)}^2.
		\end{align*}
		Assumption~\ref{main-ass:evolving-experts} now gives
		\begin{align*}
			\EE\left\|U_t\right\|_{L_2(P_X)}^2 \le C\MS^2a_t^2, \qquad \EE\sup_{\theta\in\Theta}\left\|V_t(\theta)\right\|_{L_2(P_X)}^2 \le C\MR a_t^2.
		\end{align*}
		Moreover,
		\begin{align*}
			\sup_{\theta\in\Theta}\left\|U_t+V_t(\theta)\right\|_{L_2(P_X)}^2 \le 2\left\|U_t\right\|_{L_2(P_X)}^2 + 2\sup_{\theta\in\Theta}\left\|V_t(\theta)\right\|_{L_2(P_X)}^2.
		\end{align*}
		Combining this inequality with the preceding bounds proves \eqref{eq:supp-evolving-population-distance}.
		
		The second conclusion follows from
		\begin{align*}
			\left\|\check f_g^t(\cdot;\theta)-f_0\right\|_{L_2(P_X)}^2 \le 2\left\|f_g^*(\cdot;\theta)-f_0\right\|_{L_2(P_X)}^2 + 2\left\|\check f_g^t(\cdot;\theta)-f_g^*(\cdot;\theta)\right\|_{L_2(P_X)}^2.
		\end{align*}
		Here we complete the proof.
	\end{proof}

	\begin{proof}[Proof of Theorem \ref{main-thm:softmax}]
		We first verify that under our scenario with $\mu_{s,t}(x)=\check{f}^t(x;\theta^{(s)})$, $\pi_s\equiv 1/S_{\eta}$, and $\hat{\sigma}_{s,t}\equiv \hat{\sigma}_s$, Assumptions \ref{ass:bound_forecast} and \ref{ass:bound_scale_est} required by the general AFTER algorithm hold. For any $\theta^{(s)} \in \Theta(\eta)$,
		\begin{align*}
			\check{f}^t(x;\theta^{(s)}) = \sum_{\ell=1}^{\MS}\hat{f}_{S,\ell}^t(x) + \sum_{m=1}^{\MR}g_m^{\soft}(x;\theta^{(s)})\hat{f}_{R,m}^t(x).
		\end{align*}
		Since $\sum_{m=1}^{\MR}g_m^{\soft}(x;\theta^{(s)})=1$, Assumption~\ref{main-ass:bounded-experts} gives
		\begin{align*}
			\|\check{f}^t(\cdot;\theta^{(s)})\|_\infty \le \sum_{\ell=1}^{\MS}\|\hat{f}_{S,\ell}^t\|_\infty + \sum_{m=1}^{\MR}g_m^{\soft}(\cdot;\theta^{(s)})\|\hat{f}_{R,m}^t\|_\infty \le A(\MS+1).
		\end{align*}
		Thus the candidate forecast predictors are uniformly bounded. Moreover, since $\hat{\sigma}_s$ are projected onto $[\underline{\sigma},\bar{\sigma}]$, Assumption \ref{ass:bound_scale_est} is satisfied automatically.
		
		Thus, applying Proposition~\ref{prop:risk_AFTER} with burn-in time $n_1+n_2$ and aggregation legnth $n_3$, with the uniform prior $\pi_s=1/S_\eta$, and then using Lemma~\ref{lemma:remove_scale_estimation} with $\mu_{s,t} = \check f^t(\cdot;\theta^{(s)})$, we obtain
		\begin{align}
			\frac{1}{n_3}\sum_{t=n_1+n_2}^{n-1}\EE\left\|f_0-\widetilde f^t\right\|_{L_2(P_X)}^2
			&\le C\inf_{1\le s\le S_\eta}\left\{\frac{1}{n_3}\sum_{t=n_1+n_2}^{n-1}\EE\left\|f_0-\check f^t(\cdot;\theta^{(s)})\right\|_{L_2(P_X)}^2 \right. \nonumber\\
			&\qquad \left. + \frac{1}{n_2}\sum_{t=n_1}^{n_1+n_2-1}\EE\left\|f_0-\check f^t(\cdot;\theta^{(s)})\right\|_{L_2(P_X)}^2 + \frac{1}{n_2} + \frac{\log S_\eta}{n_3}\right\}. \label{eq:3.1_1}
		\end{align}
		
		Define $f^*(x;\theta) = \sum_{\ell=1}^{\MS}f_{S,\ell}^*(x) + \sum_{m=1}^{\MR}g_m^{\soft}(x;\theta)f_{R,m}^*(x)$, and let $\theta^\circ \in \argmin_{\theta\in\Theta_{\MR}}\left\|f_0-f^*(\cdot;\theta)\right\|_{L_2(P_X)}^2$. Choose $\theta^{(s^\circ)}\in\Theta(\eta)$ such that $\left\|\theta^{(s^\circ)}-\theta^\circ\right\|_2 \le \eta$.
		By little abuse of notation, for $z\in\mathbb{R}^{\MR}$, we denote $g^{\soft}(z)$ as the softmax function on $z$, with $g^{\soft}_m(z)=\exp\{z_m\}/\sum_{j=1}^{\MR}\exp\{z_j\}$. For $z,z'\in\RR^{\MR}$, the softmax map satisfies
		\begin{align*}
			\left\|g^{\soft}(z)-g^{\soft}(z')\right\|_1 \le \left\|z-z'\right\|_\infty.
		\end{align*}
		In fact, we have for the softmax map,
		\begin{align*}
			\frac{\partial g_m^{\soft}}{\partial z_j} = g_m^{\soft}(\delta_{mj}-g_j^{\soft}),
		\end{align*}
		and thus the Jacobian can be written as $J(z)=\operatorname{diag}(g^{\soft})-g^{\soft}(g^{\soft})^{\top}$. We now compute the $(\infty,1)$ norm of the Jacobian. Recall that it can be computed as
		\begin{align*}
			\|J(z)\|_{\infty \to 1} = \max_{\|x\|_\infty \le 1}\|J(z)x\|_1.
		\end{align*}
		For any $\|x\|_{\infty}\leq 1$, we have
		\begin{align*}
			(J(z)x)_j = g^{\soft}_j x_j - g^{\soft}_j\sum_{j'=1}^{\MR} g^{\soft}_{j'}x_{j'} = g^{\soft}_j(x_j-c),
		\end{align*}
		where $c = \langle g^{\soft},x\rangle = \sum_{j'=1}^{\MR} g^{\soft}_{j'}x_{j'}$, which can be understood as the expectation of $x$ under probability distribution $g^{\soft}$. Then, $\|J(z)x\|_1 = \sum_{j=1}^{\MR} g^{\soft}_j |x_j-c|$.
		
		Now, let $S^+ = \{j : x_j \ge c\}$ and $S^- = \{j : x_j < c\}$, and denote $p=\sum_{j\in S^+}g^{\soft}_j$. Moreover, define $\bar{x}_+ = \sum_{j \in S^+}g^{\soft}_j x_j/p$ be the conditional mean of $x_j$ in the positive index set $S^+$, and $\bar{x}_- = \sum_{j \in S^-}g^{\soft}_j x_j/(1-p)$ similarly. We then have
		\begin{align*}
			\|J(z)x\|_1 = 2\sum_{i \in S^+}g^{\soft}_i(x_i-c) = 2\left(p\bar{x}_+ - p(p\bar{x}_+ + (1-p)\bar{x}_-)\right) = 2p(1-p)(\bar{x}_+-\bar{x}_-).
		\end{align*}
		Since $\|x\|_{\infty}\leq 1$, we have $\bar{x}_+-\bar{x}_-\leq 2$, and thus
		\begin{align*}
			\|J(z)x\|_1 \le 4p(1-p) \le 1 \quad\Longrightarrow\quad \|J(z)\|_{\infty\to1}\leq 1.
		\end{align*}
		Finally, by mean-value theorem and Minkowski's inequality, we have for any $z,z'\in\mathbb{R}^{\MR}$,
		\begin{align*}
			\|g^{\soft}(z)-g^{\soft}(z')\|_1
			&\leq \int_0^1\|J(tz+(1-t)z')(z-z')\|_1\,\mathrm{d}t\\
			&\leq \int_0^1\|J(tz+(1-t)z')\|_{\infty \to 1}\|z-z'\|_\infty\,\mathrm{d}t\\
			&\leq \int_0^1\|z-z'\|_\infty\,\mathrm{d}t = \|z-z'\|_\infty.
		\end{align*}
		Since $\|(x^\top,1)^\top\|_2\le C\sqrt d$ uniformly over $x\in\cX$, by Cauchy's inequality, we have
		\begin{align*}
			\max_{m\le\MR}\left|s_m^{\lin}(x;\theta_m^{(s^\circ)})-s_m^{\lin}(x;\theta_m^\circ)\right| \le C\sqrt d\,\eta.
		\end{align*}
		It follows from Assumption~\ref{main-ass:bounded-experts} that
		\begin{align*}
			\left\|f^*(\cdot;\theta^{(s^\circ)})-f^*(\cdot;\theta^\circ)\right\|_{L_2(P_X)}^2 \le Cd\eta^2.
		\end{align*}
		Consequently,
		\begin{align*}
			\left\|f^*(\cdot;\theta^{(s^\circ)})-f_0\right\|_{L_2(P_X)}^2 \le C\left\{\fA_{\soft}(\MS,\MR)+d\eta^2\right\}.
		\end{align*}
		
		Applying Lemma~\ref{lem:supp-evolving-population-comparator} with $g=g^{\soft}$ gives
		\begin{align*}
			&\frac{1}{n_2}\sum_{t=n_1}^{n_1+n_2-1}\EE\left\|f_0-\check f^t(\cdot;\theta^{(s^\circ)})\right\|_{L_2(P_X)}^2 + \frac{1}{n_3}\sum_{t=n_1+n_2}^{n-1}\EE\left\|f_0-\check f^t(\cdot;\theta^{(s^\circ)})\right\|_{L_2(P_X)}^2 \\
			&\qquad \le C\left\{\fA_{\soft}(\MS,\MR) + \MR\left(\bar a_{n_2}^2+\bar a_{n_3}^2\right) + d\eta^2\right\}.
		\end{align*}
		Substituting the comparator $\theta^{(s^\circ)}$ into \eqref{eq:3.1_1} yields
		{\footnotesize
			\begin{align}
				\frac{1}{n_3}\sum_{t=n_1+n_2}^{n-1}\EE\left\|f_0-\widetilde f^t\right\|_{L_2(P_X)}^2 \le C\left\{\fA_{\soft}(\MS,\MR) + \MR\left(\bar a_{n_2}^2+\bar a_{n_3}^2\right) + d\eta^2 + \frac{1}{n_2} + \frac{\log S_\eta}{n_3}\right\}. \label{eq:supp-dense-before-covering}
			\end{align}
		}
		
		We next bound $S_\eta$. Under the normalization $\theta_{\MR}=0$, the effective parameter dimension is $D = (\MR-1)(d+1)$. Assumption~\ref{main-ass:compact-theta} implies that $\Theta_{\MR}$ is contained in a $D$-dimensional box of radius $C_1\log\MR$. Hence a Euclidean $\eta$-net may be chosen such that
		\begin{align*}
			\log S_\eta \le C\MR d\left[\log(\MR d) + \log\left(\frac{1}{\eta}\right)\right].
		\end{align*}
		Choosing $\eta \asymp \frac{\sqrt{\MR}}{\sqrt{n_3}}$ gives $d\eta^2 \le \frac{Cd\MR}{n_3}$. Moreover,
		\begin{align*}
			\frac{\log S_\eta}{n_3} \le C\frac{\MR d}{n_3}\log(\MR dn_3).
		\end{align*}
		Since $n_1\asymp n_2\asymp n_3$, $\MR\ge2$, and $d\ge1$, we have
		\begin{align*}
			\frac{1}{n_3}\sum_{t=n_1+n_2}^{n-1}\EE\left\|f_0-\widetilde f^t\right\|_{L_2(P_X)}^2 \le C\left\{\fA_{\soft}(\MS,\MR) + \MR\left(\bar a_{n_2}^2+\bar a_{n_3}^2\right) + \frac{\MR d}{n}\log(\MR d n)\right\}.
		\end{align*}
		
		It remains to pass from the time-averaged $\tilde{f}_0$ to the projected estimator $\hat{f}_0$ in Step (f). By Jensen's inequality,
		\begin{align}
			\EE\|f_0-\tilde{f}_0\|_{L_2(P_X)}^2
			&\le \frac{1}{n_3}\sum_{t=n_1+n_2}^{n-1}\EE\|f_0-\tilde{f}^t\|_{L_2(P_X)}^2 \nonumber\\
			&\le C\left[\fA_{\soft}(\MS,\MR) + \MR(\bar a_{n_2}^2+\bar a_{n_3}^2) + \frac{\MR d}{n}\log(\MR d n)\right]. \label{eqn:tildef_bound}
		\end{align}
		By definition,
		\begin{align*}
			\hat{f}_0 = \argmin_{\theta^{(s)} \in \Theta(\eta)}\|\tilde{f}_0-\bar{f}^{n_3}(\cdot;\theta^{(s)})\|_{L_2(P_X)}^2.
		\end{align*}
		Therefore, for every $\theta^{(s)} \in \Theta(\eta)$,
		\begin{align*}
			\|\tilde{f}_0-\hat{f}_0\|_{L_2(P_X)} \le \|\tilde{f}_0-\bar{f}^{n_3}(\cdot;\theta^{(s)})\|_{L_2(P_X)}.
		\end{align*}
		Using the triangle inequality,
		\begin{align*}
			\|f_0-\hat{f}_0\|_{L_2(P_X)}^2 \le 2\|f_0-\tilde{f}_0\|_{L_2(P_X)}^2 + 2\|\tilde{f}_0-\hat{f}_0\|_{L_2(P_X)}^2.
		\end{align*}
		Hence, for every $\theta^{(s)} \in \Theta(\eta)$,
		\begin{align*}
			\|f_0-\hat{f}_0\|_{L_2(P_X)}^2 \le 6\|f_0-\tilde{f}_0\|_{L_2(P_X)}^2 + 4\|f_0-\bar{f}^{n_3}(\cdot;\theta^{(s)})\|_{L_2(P_X)}^2,
		\end{align*}
		where we use the inequality $(a+b)^2 \le 2a^2 + 2b^2$.
		Recall that
		\begin{align*}
			\bar f^{n_3}(x;\theta^{(s)}) = \frac{1}{n_3}\sum_{t=n_1+n_2}^{n-1}\check f^t(x;\theta^{(s)}).
		\end{align*}
		Taking expectations and then the infimum over $\theta^{(s)}\in\Theta(\eta)$, Jensen's inequality gives
		\begin{align*}
			\EE\left\|f_0-\hat f_0\right\|_{L_2(P_X)}^2 \le 6\EE\left\|f_0-\widetilde f_0\right\|_{L_2(P_X)}^2 + \frac{4}{n_3}\inf_{\theta^{(s)}\in\Theta(\eta)}\sum_{t=n_1+n_2}^{n-1}\EE\left\|f_0-\check f^t(\cdot;\theta^{(s)})\right\|_{L_2(P_X)}^2.
		\end{align*}
		Choosing the grid comparator $\theta^{(s^\circ)}$ constructed above and applying the preceding comparator bound yield
		\begin{align*}
			\EE\left\|f_0-\hat f_0\right\|_{L_2(P_X)}^2 \le C\left\{\fA_{\soft}(\MS,\MR) + \MR\left(\bar a_{n_2}^2+\bar a_{n_3}^2\right) + \frac{\MR d}{n}\log(\MR d n)\right\}.
		\end{align*}
		This proves the desired oracle inequality.
	\end{proof}
	
	\subsection{Proof of Example \ref{main-exa:role-shared-routed-expert}}
	\label{subsec:supp-proof-ex-2}
	
	\begin{proof}[Proof of Example \ref{main-exa:role-shared-routed-expert}]
		Let $I_m = [(m-1)/\MR,m/\MR]$ be a partition into $\MR$ regions, and $x_m = (2m-1)/(2\MR)$ be the center of $I_m$. Denote $h = 1/\MR$. For sufficiently large $\MR$, set $\tau = c_{\Theta}/\lfloor \sqrt{\log \MR} \rfloor$, where $c_{\Theta} > 0$ is chosen large enough so that the normalized softmax parameters constructed below belong to the logarithmic parameter box in Assumption \ref{main-ass:compact-theta}. Finitely many smaller values of $\MR$ can be absorbed into the constants. Let
		\begin{align*}
			\theta_m = \left(\frac{x_m}{\tau^2},-\frac{x_m^2}{2\tau^2}\right), \qquad s_m^{\lin}(x;\theta_m) = (x,1)^\top \theta_m, \qquad g_m^{\soft}(x;\theta) = \frac{\exp\{s_m^{\lin}(x;\theta_m)\}}{\sum_{j=1}^{\MR}\exp\{s_j^{\lin}(x;\theta_j)\}}
		\end{align*}
		for $m \in [\MR]$. The displayed parameters are unnormalized. By the shift invariance of the softmax logits, replacing $\theta_m$ by $\theta_m - \theta_{\MR}$ imposes the reference normalization without changing the softmax weights. Since the coordinates of the normalized parameters are of order $\tau^{-2} = O(\log \MR)$, the constructed parameter belongs to $\Theta_{\MR}$ by the choice of $c_{\Theta}$.
		
		We can rewrite the gating function as
		\begin{align*}
			g_m^{\soft}(x;\theta)
			&= \frac{\exp\{-(x-x_m)^2/(2\tau^2)\}\exp\{x^2/(2\tau^2)\}}{\sum_{j=1}^{\MR}\exp\{-(x-x_j)^2/(2\tau^2)\}\exp\{x^2/(2\tau^2)\}} \\
			&= \frac{\exp\{-(x-x_m)^2/(2\tau^2)\}}{\sum_{j=1}^{\MR}\exp\{-(x-x_j)^2/(2\tau^2)\}}.
		\end{align*}
		We now prove that this gating mechanism yields the desired upper bounds.
		
		We first consider the case where the shared expert is included. Consistent with the main text, set $\MS = 1$ and take $f_{S,1}^*(x) = A \sin(2\pi Lx)$. The routed experts only need to approximate the residual $f_0 - f_{S,1}^* = x^2$. Take
		\begin{align*}
			f_{R,m}^*(x) = a_m x + b_m, \qquad a_m = 2x_m = \frac{2m-1}{\MR}, \qquad b_m = -x_m^2 = -\left(\frac{2m-1}{2\MR}\right)^2,
		\end{align*}
		for $m \in [\MR]$. Since $X \sim \Unif[0,1]$, the approximation error of this construction satisfies
		\begin{align*}
			\int_0^1 \left(x^2 - \sum_{m=1}^{\MR} g_m^{\soft}(x;\theta)(a_m x + b_m)\right)^2 dx = \int_0^1 \left(\sum_{m=1}^{\MR} g_m^{\soft}(x;\theta)(x - x_m)^2\right)^2 dx.
		\end{align*}
		
		We next bound the local second moment of the softmax weights. Set $u_m = (x - x_m)/h$, and $\varpi = h^2/\tau^2$. Then $\exp\{-(x - x_m)^2/(2\tau^2)\} = \exp\{-\varpi u_m^2/2\}$.
		
		For $\MR$ large enough, $\tau \ge 2h$. Uniformly over $x \in [0,1]$, including boundary points, there are at least a constant multiple of $\tau/h$ grid centers satisfying $|x - x_m| \le \tau/2$. Hence
		\begin{align*}
			\sum_{m=1}^{\MR} \exp\{-\varpi u_m^2/2\} \ge c \frac{\tau}{h},
		\end{align*}
		for some universal constant $c > 0$.
		
		Moreover, for each $k \ge 1$, there are at most two indices $m$ such that $k - 1/2 < |u_m| \le k + 1/2$. For those indices,
		\begin{align*}
			u_m^2 \exp\{-\varpi u_m^2/2\} \le \left(k + \frac{1}{2}\right)^2 \exp\{-\varpi (k - 1/2)^2/2\}.
		\end{align*}
		
		Therefore,
		\begin{align*}
			\sum_{m=1}^{\MR} \exp\{-\varpi u_m^2/2\} u_m^2
			&\le \frac{1}{4} + 2\sum_{k=1}^{\infty} \left(k + \frac{1}{2}\right)^2 \exp\{-\varpi (k - 1/2)^2/2\} \\
			&\lesssim \int_0^{\infty} u^2 \exp\{-\varpi u^2/4\} \, du \\
			&\asymp \varpi^{-3/2} = \frac{\tau^3}{h^3}.
		\end{align*}
		
		Combining the denominator lower bound and the numerator bound gives
		\begin{align*}
			\sum_{m=1}^{\MR} g_m^{\soft}(x;\theta) (x - x_m)^2 = h^2 \frac{\sum_{m=1}^{\MR} \exp\{-\varpi u_m^2/2\} u_m^2}{\sum_{m=1}^{\MR}\exp\{-\varpi u_m^2/2\}} \lesssim h^2\frac{\tau^3/h^3}{\tau/h} = \tau^2.
		\end{align*}
		Thus,
		\begin{align*}
			\left|x^2 - \sum_{m=1}^{\MR} g_m^{\soft}(x;\theta)(a_m x + b_m)\right| \lesssim \tau^2, \qquad x \in [0,1].
		\end{align*}
		Consequently,
		\begin{align*}
			\fA_{\soft}(1,\MR) = \inf_{\theta \in \Theta_{\MR}} \left\|f_0 - f_{S,1}^* - \sum_{m=1}^{\MR}g_m^{\soft}(\cdot;\theta)f_{R,m}^* \right\|_{L_2(P_X)}^2 = O(\tau^4) = O\left(\frac{1}{\log^2 \MR}\right).
		\end{align*}
		
		We now consider the case without shared experts. Set $\MS = 0$, so that the routed experts must approximate the full target function $f_0(x) = A \sin(2\pi Lx) + x^2$. For each $m \in [\MR]$, define the local affine approximation at $x_m$ by
		\begin{align*}
			f_{R,m}^*(x) = a_m x + b_m := f_0(x_m) + f_0'(x_m)(x - x_m),
		\end{align*}
		where
		\begin{align*}
			a_m = f_0'(x_m) = 2\pi A L \cos(2\pi Lx_m) + 2x_m, \qquad b_m = f_0(x_m) - a_m x_m.
		\end{align*}
		We assume, as in the example statement, that the routed expert class contains these local affine approximations. The constants below retain their dependence on the curvature of $f_0$. By Taylor's theorem, for every $x \in [0,1]$ and each $m \in [\MR]$, there exists $\xi_{m,x}$ between $x$ and $x_m$ such that
		\begin{align*}
			f_0(x) - f_{R,m}^*(x) = \frac{1}{2}f_0''(\xi_{m,x})(x - x_m)^2.
		\end{align*}
		Since
		\begin{align*}
			f_0''(x) = -(2\pi L)^2 A \sin(2\pi Lx) + 2,
		\end{align*}
		we have
		\begin{align*}
			\sup_{x \in [0,1]} |f_0''(x)| \le 4\pi^2 |A| L^2 + 2 =: C_L.
		\end{align*}
		Therefore,
		\begin{align*}
			\left|f_0(x) - \sum_{m=1}^{\MR}g_m^{\soft}(x;\theta)f_{R,m}^*(x)\right|
			&= \left|\sum_{m=1}^{\MR} g_m^{\soft}(x;\theta) \left(f_0(x) - f_{R,m}^*(x)\right)\right| \\
			&\le \frac{C_L}{2}\sum_{m=1}^{\MR}g_m^{\soft}(x;\theta)(x - x_m)^2.
		\end{align*}
		Using the same softmax second-moment bound as above,
		\begin{align*}
			\sum_{m=1}^{\MR}g_m^{\soft}(x;\theta)(x - x_m)^2 \le C \tau^2,
		\end{align*}
		uniformly in $x \in [0,1]$. Hence,
		\begin{align*}
			\left|f_0(x)- \sum_{m=1}^{\MR}g_m^{\soft}(x;\theta)f_{R,m}^*(x)\right| \le C C_L \tau^2.
		\end{align*}
		Integrating over $x \in [0,1]$ yields
		\begin{align*}
			\left\|f_0 - \sum_{m=1}^{\MR} g_m^{\soft}(\cdot;\theta) f_{R,m}^*\right\|_{L_2(P_X)}^2 \le C C_L^2 \tau^4.
		\end{align*}
		Since $C_L^2 = O(L^4 + 1)$, we conclude that
		\begin{align*}
			\inf_{\theta \in \Theta_{\MR}} \left\|f_0 - \sum_{m=1}^{\MR} g_m^{\soft}(\cdot;\theta) f_{R,m}^*\right\|_{L_2(P_X)}^2 = O\left(\frac{L^4 + 1}{\log^2 \MR}\right).
		\end{align*}
		In particular, when $L$ is large, this is of order $O(L^4/\log^2 \MR)$. This completes the proof.
	\end{proof}
	
	\section{Proof of Section \ref{main-sec:oracle_topK}}
	\label{sec:supp-proof-section-4}
	
	\begin{lemma}[Regularity from bounded relative derivative]\label{lem:supp-relative-derivative}
		Suppose that $\phi:\mathbb R\to(0,\infty)$ is continuously differentiable,
		strictly increasing, and satisfies
		\begin{align*}
			B_\phi:=\sup_{t\in\mathbb R}\frac{\phi'(t)}{\phi(t)}<\infty.
		\end{align*}
		Then $\phi$ is regular in the sense of Definition~\ref{main-ass:regular gating} with
		$l_{\phi}=0$. More precisely, for every $J\subset[\MR]$ with $|J|=K$ and every
		$z,z'\in\mathbb R^{\MR}$,
		\begin{align*}
			\left\|G_J^\phi(z)-G_J^\phi(z')\right\|_2 \le \frac{B_\phi}{2}\|z-z'\|_2.
		\end{align*}
	\end{lemma}
	
	\begin{proof}
		Fix $J\subset[\MR]$ with $|J|=K$ and write
		\begin{align*}
			p_m(z)=G_{J,m}^\phi(z)=\frac{\phi(z_m)}{\sum_{r\in J}\phi(z_r)},\qquad m\in J.
		\end{align*}
		Coordinates outside $J$ are identically zero. For $m,j\in J$, the quotient rule gives
		\begin{align*}
			\frac{\partial p_m(z)}{\partial z_j} = p_m(z)\{\delta_{mj}-p_j(z)\}\frac{\phi'(z_j)}{\phi(z_j)}.
		\end{align*}
		Hence the only nonzero Jacobian block is
		\begin{align*}
			\nabla_J G_J^\phi(z) = \{\operatorname{diag}(p_J(z))-p_J(z)p_J(z)^\top\}\operatorname{diag}\left(\frac{\phi'(z_j)}{\phi(z_j)}\right)_{j\in J}.
		\end{align*}
		For any $p\in\Delta^{K-1}$, the symmetric matrix
		$A(p)=\operatorname{diag}(p)-pp^\top$ satisfies
		\begin{align*}
			\|A(p)\|_{\mathrm{op}} \le \max_i\sum_{j=1}^K |A(p)_{ij}| = \max_i 2p_i(1-p_i) \le \frac{1}{2}.
		\end{align*}
		By the definition of $B_\phi$, the diagonal factor has operator norm at most
		$B_\phi$. Therefore $\|\nabla G_J^\phi(z)\|_{\mathrm{op}}\le B_\phi/2$
		uniformly in $z$, $J$, $K$, and $\MR$.
		
		For $z_t=z'+t(z-z')$, the fundamental theorem of calculus gives
		\begin{align*}
			G_J^\phi(z)-G_J^\phi(z') = \int_0^1\nabla G_J^\phi(z_t)(z-z')\,dt,
		\end{align*}
		and hence
		\begin{align*}
			\left\|G_J^\phi(z)-G_J^\phi(z')\right\|_2 \le \frac{B_\phi}{2}\|z-z'\|_2.
		\end{align*}
		This is Definition~\ref{main-ass:regular gating} with $L_\phi=B_\phi/2$ and
		$l_\phi=0$.
	\end{proof}
	
	\subsection{Proof of Proposition \ref{main-pro:continuity} and related remarks}
	\label{subsec:supp-proof-prop-4-3}
	
	\begin{proof}[Proof of Proposition \ref{main-pro:continuity}]
		Write $\delta:=\|\theta-\theta'\|_2$ and $\bar{x}:=(x^\top,1)^\top$. We prove the result for $\theta \in \widetilde{\Theta}_{\MR}(r_{\MR}), \theta' \in \Theta_{\MR}$, and $\delta \le 1$. By the standing covariate bound $\|X\|_2\le C_X\sqrt d$, we may take
		$R_X:=(C_X^2d+1)^{1/2}$, so that $\|\bar X\|_2\le R_X\le C\sqrt d$ almost surely. In what follows, $C$ may depend on $K$, $L_\phi$, and $C_4$ in Assumption~\ref{main-ass:distribution}, but not on $d,\MR,r_{\MR},\theta$, or $\theta'$.
		
		For each pair $(\theta,\theta')$, decompose $\cX=G_{\theta,\theta'}\cup B_{\theta,\theta'}$, where
		\begin{align*}
			G_{\theta,\theta'}=\{x:\cT_K(x;\theta)=\cT_K(x;\theta')\}, \qquad B_{\theta,\theta'}=\cX\setminus G_{\theta,\theta'}.
		\end{align*}
		Then
		\begin{align*}
			\left\|g^{(K,\phi)}(\cdot;\theta)-g^{(K,\phi)}(\cdot;\theta')\right\|_{L_2(P_X)}^2
			&= \EE\left[\left\|g^{(K,\phi)}(X;\theta)-g^{(K,\phi)}(X;\theta')\right\|_2^2 \one_{\{X\in G_{\theta,\theta'}\}}\right] \\
			&\quad + \EE\left[\left\|g^{(K,\phi)}(X;\theta)-g^{(K,\phi)}(X;\theta')\right\|_2^2 \one_{\{X\in B_{\theta,\theta'}\}}\right].
		\end{align*}
		On $G_{\theta,\theta'}$, the selected Top-$K$ set is fixed. Hence, by the regularity of $\phi$,
		\begin{align*}
			\left\|g^{(K,\phi)}(x;\theta)-g^{(K,\phi)}(x;\theta')\right\|_2
			&\le L_{\phi}\MR^{l_{\phi}}\left(\sum_{m\in\cT_K(x;\theta)}|s_m(x;\theta_m)-s_m(x;\theta'_m)|^2\right)^{1/2} \\
			&= L_{\phi}\MR^{l_{\phi}}\left(\sum_{m\in\cT_K(x;\theta)}|\bar{x}^\top(\theta_m-\theta'_m)|^2\right)^{1/2} \\
			&\le L_\phi R_X\MR^{l_{\phi}}\delta.
		\end{align*}
		Since both gating vectors lie in the simplex, their squared Euclidean distance is at most $2$. Therefore,
		\begin{align*}
			\EE\left[\left\|g^{(K,\phi)}(X;\theta)-g^{(K,\phi)}(X;\theta')\right\|_2^2 \one_{\{X\in G_{\theta,\theta'}\}}\right]
			&\le \min\left\{L_\phi^2R_X^2\MR^{2l_\phi}\delta^2,2\right\} \\
			&\le C R_X\MR^{l_\phi}\delta.
		\end{align*}
		
		Next consider $B_{\theta,\theta'}$. If $x \in B_{\theta,\theta'}$, then the two Top-$K$ active set differs. Hence there exist indices $i \in \cT_K(x;\theta) \setminus \cT_K(x;\theta')$ and $j \in \cT_K(x;\theta') \setminus \cT_K(x;\theta)$. By the definition of the Top-$K$ set, these indices satisfy
		\begin{align*}
			s_i(x;\theta_i)\ge s_j(x;\theta_j), \qquad s_i(x;\theta'_i)\le s_j(x;\theta'_j).
		\end{align*}
		Thus
		\begin{align*}
			|s_i(x;\theta_i)-s_j(x;\theta_j)| \le |s_i(x;\theta_i)-s_i(x;\theta'_i)| + |s_j(x;\theta_j)-s_j(x;\theta'_j)| \le C R_X\delta.
		\end{align*}
		Since $s_i(x;\theta_i)-s_j(x;\theta_j) = x^\top(\beta_i-\beta_j)+(\alpha_i-\alpha_j)$, we have
		\begin{align*}
			B_{\theta,\theta'} \subset \bigcup_{i\neq j}\left\{x:\left|x^\top(\beta_i-\beta_j)+(\alpha_i-\alpha_j)\right| \le C R_X\delta\right\}.
		\end{align*}
		By the definition of $\widetilde{\Theta}_{\MR}(r_{\MR})$, $\|\beta_i-\beta_j\|_2 \ge r_{\MR}, i \neq j$. Applying Assumption~\ref{main-ass:distribution} with
		\begin{align*}
			a=\frac{\beta_i-\beta_j}{\|\beta_i-\beta_j\|_2}, \qquad b=\frac{\alpha_i-\alpha_j}{\|\beta_i-\beta_j\|_2},
		\end{align*}
		gives
		\begin{align*}
			\PP\left(\left|X^\top(\beta_i-\beta_j)+(\alpha_i-\alpha_j)\right| \le C R_X\delta\right) \le C R_X\frac{\delta}{r_{\MR}}.
		\end{align*}
		A union bound over at most $\MR^2$ pairs yields
		\begin{align*}
			\PP(X\in B_{\theta,\theta'}) \le C R_X\MR^2\frac{\delta}{r_{\MR}}.
		\end{align*}
		Since both gating vectors lie in the simplex,
		\begin{align*}
			\left\|g^{(K,\phi)}(X;\theta)-g^{(K,\phi)}(X;\theta')\right\|_2^2\le 2,
		\end{align*}
		and therefore
		\begin{align*}
			\EE\left[\left\|g^{(K,\phi)}(X;\theta)-g^{(K,\phi)}(X;\theta')\right\|_2^2 \one_{\{X\in B_{\theta,\theta'}\}}\right] \le C R_X\MR^2\frac{\delta}{r_{\MR}}.
		\end{align*}
		Combining the two bounds gives
		\begin{align*}
			\left\|g^{(K,\phi)}(\cdot;\theta)-g^{(K,\phi)}(\cdot;\theta')\right\|_{L_2(P_X)}^2 \le C R_X\MR^{l_{\phi}}\delta + C R_X\MR^2\frac{\delta}{r_{\MR}}.
		\end{align*}
		Since $r_{\MR}\in(0,1)$, $\MR\ge2$, and $l_\phi\ge0$, both terms are bounded by a constant multiple of $R_X\MR^{2+2l_\phi}\delta/r_{\MR}$. Using $R_X\le C\sqrt d$, we obtain
		\begin{align*}
			\left\|g^{(K,\phi)}(\cdot;\theta)-g^{(K,\phi)}(\cdot;\theta')\right\|_{L_2(P_X)}^2 \le C R_X\MR^{2+2l_{\phi}}\frac{\delta}{r_{\MR}} \le C\sqrt d\,\MR^{2+2l_{\phi}}\frac{\delta}{r_{\MR}}.
		\end{align*}
		Taking square roots and recalling $\delta=\|\theta-\theta'\|_2$ gives
		\begin{align*}
			\left\|g^{(K,\phi)}(\cdot;\theta)-g^{(K,\phi)}(\cdot;\theta')\right\|_{L_2(P_X)} \le C d^{1/4}\MR^{1+l_{\phi}}\left(\frac{\|\theta-\theta'\|_2}{r_{\MR}}\right)^{1/2}.
		\end{align*}
		This completes the proof.
	\end{proof}

	\begin{proof}[Proof of remark~\ref{main-remark:topk-l2-covering}]
		By Proposition~\ref{main-pro:continuity}, for any $\theta \in \widetilde{\Theta}_{\MR}(r_{\MR})$ and $\theta' \in \Theta_{\MR}$ satisfying $\|\theta-\theta'\|_2\le 1$,
		\begin{align*}
			\left\|g^{(K,\phi)}(\cdot;\theta)-g^{(K,\phi)}(\cdot;\theta')\right\|_{L_2(P_X)} \le C d^{1/4}\MR^{1+l_{\phi}}\left(\frac{\|\theta-\theta'\|_2}{r_{\MR}}\right)^{1/2}.
		\end{align*}
		Fix $0<\varepsilon\le 1$ and set
		\begin{align*}
			u_\varepsilon:=\frac{c\,\varepsilon^2 r_{\MR}}{\sqrt d\,\MR^{2+2l_{\phi}}},
		\end{align*}
		with $c>0$ sufficiently small. We may assume $u_\varepsilon\le 1$ after decreasing $c$.
		
		Take a $u_\varepsilon$-cover of the larger parameter space $(\Theta_{\MR},\|\cdot\|_2)$:
		\begin{align*}
			\Theta_{\MR} \subset \bigcup_{j=1}^N B_2(\theta^j,u_\varepsilon), \qquad \theta^j\in\Theta_{\MR}.
		\end{align*}
		For every $\theta \in \widetilde{\Theta}_{\MR}(r_{\MR})$, choose $j$ such that $\|\theta-\theta^j\|_2 \le u_\varepsilon$. By the choice of $u_\varepsilon$ and Proposition~\ref{main-pro:continuity},
		\begin{align*}
			\left\|g^{(K,\phi)}(\cdot;\theta)-g^{(K,\phi)}(\cdot;\theta^j)\right\|_{L_2(P_X)} \le C d^{1/4}\MR^{1+l_{\phi}}\left(\frac{u_\varepsilon}{r_{\MR}}\right)^{1/2} \le \varepsilon.
		\end{align*}
		Thus the ambient grid induces an $\varepsilon$-cover of $\cG_{K,\phi}$.
		
		It remains to bound $N$. Since the reference normalization fixes $\theta_{\MR}=0$, only $\MR-1$ parameter vectors vary. Since each $\theta_m=(\beta_m^\top,\alpha_m)^\top \in \RR^{d+1}$, the effective dimension is $D=(\MR-1)(d+1)$. By Assumption~\ref{main-ass:compact-theta}, $\Theta_{\MR}$ is contained in a $D$-dimensional cube of radius $C'\log \MR$. Hence
		\begin{align*}
			N\left(u_\varepsilon,\Theta_{\MR},\|\cdot\|_2\right) \le \left(\frac{C'\log \MR\sqrt{(\MR-1)(d+1)}}{u_\varepsilon}\right)^{(\MR-1)(d+1)}.
		\end{align*}
		Substituting the value of $u_\varepsilon$ gives
		\begin{align*}
			N\left(\varepsilon,\cG_{K,\phi},L_2(P_X)\right) \le \left(\frac{C'\log \MR\sqrt{(\MR-1)d(d+1)}\,\MR^{2+2l_{\phi}}}{\varepsilon^2 r_{\MR}}\right)^{(\MR-1)(d+1)}.
		\end{align*}
		Since $\sqrt{(\MR-1)(d+1)} \le C\MR^{1/2}\sqrt{d}$, we obtain
		\begin{align*}
			N\left(\varepsilon,\cG_{K,\phi},L_2(P_X)\right) \le \left(\frac{C'\MR^{5/2+2l_{\phi}}d\log \MR}{\varepsilon^2 r_{\MR}}\right)^{(\MR-1)(d+1)}.
		\end{align*}
		Here we complete the proof.
	\end{proof}
	
	\subsection{Proof of Theorem \ref{main-thm:top-k}}
	\label{subsec:supp-proof-thm-4-4}
	
	\begin{proof}[Proof of Theorem \ref{main-thm:top-k}]
		Let $\Theta(\eta)=\{\theta^{(1)},\ldots,\theta^{(S_\eta)}\}$ be the Euclidean grid used by Algorithm~\ref{main-alg:general}. For $\theta\in\Theta_{\MR}$ define
		\begin{align*}
			R_{\mathrm{ag}}^{K,\phi}(\theta) &= \frac{1}{n_3}\sum_{t=n_1+n_2}^{n-1}\EE\left\|\check f^{K,\phi,t}(\cdot;\theta)-f_0\right\|_{L_2(P_X)}^2,\\
			R_{\mathrm{cal}}^{K,\phi}(\theta) &= \frac{1}{n_2}\sum_{t=n_1}^{n_1+n_2-1}\EE\left\|\check f^{K,\phi,t}(\cdot;\theta)-f_0\right\|_{L_2(P_X)}^2.
		\end{align*}
		Applying Proposition~\ref{prop:risk_AFTER} with burn-in time $n_1$, calibration time $n_2$, aggregation time $n_3$, and with
		\begin{align*}
			\mu_{s,t}(x) = \check f^{K,\phi,t}(x;\theta^{(s)}), \qquad \pi_s = S_\eta^{-1},
		\end{align*}
		gives the corresponding oracle inequality. Applying Lemma~\ref{lemma:remove_scale_estimation} to the same candidate forecasts then yields
		\begin{align}
			\frac{1}{n_3}\sum_{t=n_1+n_2}^{n-1}\EE\left\|\widetilde f^{K,\phi,t}-f_0\right\|_{L_2(P_X)}^2 \le C\inf_{1\le s\le S_\eta}\left\{R_{\mathrm{ag}}^{K,\phi}(\theta^{(s)}) + R_{\mathrm{cal}}^{K,\phi}(\theta^{(s)}) + \frac{\log S_\eta}{n_3} + \frac{1}{n_2}\right\}. \label{eq:supp-topk-finite-grid}
		\end{align}
		The boundedness conditions required by Proposition~\ref{prop:risk_AFTER} hold because the Top-$K$ gate lies in the simplex and the experts are uniformly bounded.
		
		Define
		\begin{align*}
			F^{*,K,\phi}(x;\theta) = \sum_{\ell=1}^{\MS}f_{S,\ell}^*(x) + \sum_{m=1}^{\MR}g_m^{(K,\phi)}(x;\theta)f_{R,m}^*(x),
		\end{align*}
		and let
		\begin{align*}
			\theta^\circ \in \argmin_{\theta\in\widetilde\Theta_{\MR}(r_{\MR})}\left\|f_0-F^{*,K,\phi}(\cdot;\theta)\right\|_{L_2(P_X)}^2.
		\end{align*}
		Choose $\theta^{(s^\circ)}\in\Theta(\eta)$ such that $\left\|\theta^{(s^\circ)}-\theta^\circ\right\|_2 \le \eta$. By Proposition~\ref{main-pro:continuity},
		\begin{align*}
			\left\|g^{(K,\phi)}(\cdot;\theta^{(s^\circ)})-g^{(K,\phi)}(\cdot;\theta^\circ)\right\|_{L_2(P_X)}^2 \le C_K\sqrt d\,\MR^{2+2l_\phi}\frac{\eta}{r_{\MR}}.
		\end{align*}
		For every $x$, the union of the supports of
		$g^{(K,\phi)}(x;\theta^{(s^\circ)})$ and
		$g^{(K,\phi)}(x;\theta^\circ)$ contains at most $2K$ indices. Hence,
		\begin{align*}
			\left|F^{*,K,\phi}(x;\theta^{(s^\circ)}) - F^{*,K,\phi}(x;\theta^\circ)\right|^2 \le 2KA^2\left\|g^{(K,\phi)}(x;\theta^{(s^\circ)}) - g^{(K,\phi)}(x;\theta^\circ)\right\|_2^2.
		\end{align*}
		Therefore,
		\begin{align*}
			\left\|F^{*,K,\phi}(\cdot;\theta^{(s^\circ)}) - F^{*,K,\phi}(\cdot;\theta^\circ)\right\|_{L_2(P_X)}^2 \le C_K\sqrt d\,\MR^{2+2l_\phi}\frac{\eta}{r_{\MR}}.
		\end{align*}
		It follows that
		\begin{align*}
			\left\|F^{*,K,\phi}(\cdot;\theta^{(s^\circ)})-f_0\right\|_{L_2(P_X)}^2 \le C_K\left\{\fA_{K,\phi}(\MS,\MR) + \sqrt d\,\MR^{2+2l_\phi}\frac{\eta}{r_{\MR}}\right\}.
		\end{align*}
		
		Applying Lemma~\ref{lem:supp-evolving-population-comparator} with $g=g^{(K,\phi)}$ gives
		\begin{align*}
			R_{\mathrm{ag}}^{K,\phi}(\theta^{(s^\circ)}) + R_{\mathrm{cal}}^{K,\phi}(\theta^{(s^\circ)}) \le C_K\left\{\fA_{K,\phi}(\MS,\MR) + \MR\left(\bar a_{n_2}^2+\bar a_{n_3}^2\right) + \sqrt d\,\MR^{2+2l_\phi}\frac{\eta}{r_{\MR}}\right\}.
		\end{align*}
		Substituting this comparator into \eqref{eq:supp-topk-finite-grid} yields
		\begin{align}
			\frac{1}{n_3}\sum_{t=n_1+n_2}^{n-1}\EE\left\|\widetilde f^{K,\phi,t}-f_0\right\|_{L_2(P_X)}^2
			&\le C_K\left\{\fA_{K,\phi}(\MS,\MR) + \MR\left(\bar a_{n_2}^2+\bar a_{n_3}^2\right)\right. \notag\\*
			&\qquad\left.{}+ \sqrt d\,\MR^{2+2l_\phi}\frac{\eta}{r_{\MR}} + \frac{1}{n_2} + \frac{\log S_\eta}{n_3}\right\}. \label{eq:supp-topk-before-covering}
		\end{align}
		We now choose the mesh size. Since $\Theta_{\MR}$ is contained in a box of radius $C_1\log\MR$ and has effective dimension $D=(\MR-1)(d+1)$, a Euclidean grid with mesh $\eta$ satisfies
		\begin{align*}
			\log S_\eta \le C\MR d\log\left(\frac{C\sqrt{\MR d}\log\MR}{\eta}\right).
		\end{align*}
		
		With $\eta \asymp \frac{\sqrt d\,r_{\MR}}{n_3\MR^{2+2l_\phi}}$, truncated at a sufficiently small absolute constant if necessary, the discretization term satisfies $\sqrt d\,\MR^{2+2l_\phi}\frac{\eta}{r_{\MR}}$ $\le C\frac{d}{n_3}$. Moreover,
		\begin{align*}
			\frac{\log S_\eta}{n_3} \le C\frac{\MR d}{n_3}\log\left(\frac{C\MR^{5/2+2l_\phi}n_3\log\MR}{r_{\MR}}\right).
		\end{align*}
		Since $l_\phi$ is fixed and $r_{\MR}\in(0,1)$,
		\begin{align*}
			\frac{\log S_\eta}{n_3} \le C\frac{\MR d}{n_3}\left\{\log(\MR d n_3) + \log\left(\frac{1}{r_{\MR}}\right)\right\}.
		\end{align*}
		Since $n_2\asymp n_3\asymp n$, both the discretization term and the calibration term $1/n_2$ are absorbed into
		\begin{align*}
			C\frac{\MR d}{n}\left\{\log(\MR d n) + \log\left(\frac{1}{r_{\MR}}\right)\right\}.
		\end{align*}
		Combining these bounds with the preceding online inequality proves \eqref{main-eqn:Topk_oracle}.
		
		For the online-to-batch average
		\begin{align*}
			\tilde f_0^{K,\phi}:=\frac{1}{n_3}\sum_{t=n_1+n_2}^{n-1}\tilde f^{K,\phi,t},
		\end{align*}
		Jensen's inequality gives
		\begin{align*}
			\EE\|\tilde f_0^{K,\phi}-f_0\|_{L_2(P_X)}^2 \le \frac{1}{n_3}\sum_{t=n_1+n_2}^{n-1}\EE\|\tilde f^{K,\phi,t}-f_0\|_{L_2(P_X)}^2,
		\end{align*}
		so the online-to-batch average satisfies the same bound. Finally, let
		\begin{align*}
			\bar f^{K,\phi,n_3}(\cdot;\theta^{(s)}) = \frac{1}{n_3}\sum_{t=n_1+n_2}^{n-1}\check f^{K,\phi,t}(\cdot;\theta^{(s)})
		\end{align*}
		be the time-averaged candidate used in the population projection step. By definition of the projection,
		\begin{align*}
			\hat f_0=\bar f^{K,\phi,n_3}(\cdot;\hat\theta), \qquad \hat\theta\in\argmin_{\theta^{(s)}\in\Theta(\eta)}\|\tilde f_0^{K,\phi}-\bar f^{K,\phi,n_3}(\cdot;\theta^{(s)})\|_{L_2(P_X)}^2.
		\end{align*}
		Thus, for every $\theta^{(s)}\in\Theta(\eta)$,
		\begin{align*}
			\|\tilde f_0^{K,\phi}-\hat f_0\|_{L_2(P_X)} \le \|\tilde f_0^{K,\phi}-\bar f^{K,\phi,n_3}(\cdot;\theta^{(s)})\|_{L_2(P_X)}.
		\end{align*}
		Using the triangle inequality and then $(a+b)^2\le2a^2+2b^2$, we obtain, for every $\theta^{(s)}$,
		\begin{align*}
			\|\hat f_0-f_0\|_{L_2(P_X)}^2
			&\le 2\|\tilde f_0^{K,\phi}-f_0\|_{L_2(P_X)}^2 + 2\|\tilde f_0^{K,\phi}-\hat f_0\|_{L_2(P_X)}^2\\
			&\le 6\|\tilde f_0^{K,\phi}-f_0\|_{L_2(P_X)}^2 + 4\|\bar f^{K,\phi,n_3}(\cdot;\theta^{(s)})-f_0\|_{L_2(P_X)}^2.
		\end{align*}
		Taking expectations, then infimizing over the grid, gives
		\begin{align*}
			\EE\|\hat f_0-f_0\|_{L_2(P_X)}^2 \le 6\EE\|\tilde f_0^{K,\phi}-f_0\|_{L_2(P_X)}^2 + 4\inf_{1\le s\le S_\eta}\EE\|\bar f^{K,\phi,n_3}(\cdot;\theta^{(s)})-f_0\|_{L_2(P_X)}^2.
		\end{align*}
		For the second term, Jensen's inequality over the aggregation block gives
		\begin{align*}
			\EE\|\bar f^{K,\phi,n_3}(\cdot;\theta^{(s)})-f_0\|_{L_2(P_X)}^2 \le \frac{1}{n_3}\sum_{t=n_1+n_2}^{n-1}\EE\|\check f^{K,\phi,t}(\cdot;\theta^{(s)})-f_0\|_{L_2(P_X)}^2.
		\end{align*}
		Taking $s=s^\circ$, where $\theta^{(s^\circ)}$ is the oracle grid comparator constructed above, and using the preceding population-comparison and evolving-expert bounds, we obtain
		\begin{align*}
			\inf_{1\le s\le S_\eta}\EE\left\|\bar f^{K,\phi,n_3}(\cdot;\theta^{(s)})-f_0\right\|_{L_2(P_X)}^2 \le C_K\left\{\fA_{K,\phi}(\MS,\MR) + \MR\bar a_{n_3}^2 + \sqrt d\,\MR^{2+2l_\phi}\frac{\eta}{r_{\MR}}\right\}.
		\end{align*}
		Together with the online-to-batch bound, this proves that the projected estimator satisfies the same oracle rate. This completes the proof.
	\end{proof}
	
	\subsection{Proof of Corollary \ref{main-cor:Top1-single}}
	\label{subsec:supp-proof-cor-4-5}
	
	\begin{proof}[Proof of Corollary \ref{main-cor:Top1-single}]
		We apply Theorem~\ref{main-thm:top-k} with $K=1$ and $\MS=0$. It remains to bound the Top-1 oracle approximation term
		\begin{align*}
			\inf_{\theta\in\widetilde{\Theta}_{\MR}(r_{\MR})}\left\|f_0-\sum_{m=1}^{\MR}g_m^{(1,\phi)}(\cdot;\theta)f_m^*\right\|_{L_2(P_X)}^2.
		\end{align*}
		For Top-1 routing, the selected expert has weight one, independently of $\phi$. We therefore only need to exhibit a separated linear score system that selects the region-wise specialist. By relabeling the routed experts, use experts $1,\ldots,\Mzero-1,\MR$ as the specialists for regions $1,\ldots,\Mzero$, respectively, with the normalization $\theta_{\MR}=0$. For a fixed small $\lambda>0$, set
		\begin{align*}
			\theta_r=(\beta_r^\top,\alpha_r)^\top=\lambda(v_r-v_{\Mzero}),\qquad r=1,\ldots,\Mzero-1.
		\end{align*}
		Choose $\lambda\le (1/2)\wedge(c_0/4)$. Since $\max_{r\le\Mzero}\|v_r\|_2=1$, this gives
		\begin{align*}
			\|\theta_r\|_2\le 2\lambda\le1,\qquad \|\theta_r\|_\infty\le 2\lambda\le c_0/2,\qquad r=1,\ldots,\Mzero-1.
		\end{align*}
		Thus the active parameters satisfy both the Euclidean normalization and the centered-box constraint. Moreover, their slopes are separated from each other and from the reference slope by at least $\lambda r_0$.
		
		It remains only to fill the unused routed experts without changing the Top-1 assignment. Let $B_X=1\vee\sup_{x\in\cX}\|x\|_2$, and define
		\begin{align*}
			b=\frac{1}{2}(c_0\wedge1),\qquad \delta=\frac{b}{2B_X},\qquad \mathcal B_{\rm act}=\{\beta_1,\ldots,\beta_{\Mzero-1},0\}.
		\end{align*}
		Set $a=c_1\delta\MR^{-1/d}$, where $c_1>0$ will be chosen sufficiently small. By the volumetric packing bound, there exists an $a$-separated set $\mathcal P\subset B_d(0,\delta)$ with
		\begin{align*}
			|\mathcal P|\ge c_d\left(\frac{\delta}{a}\right)^d = c_d c_1^{-d}\MR,
		\end{align*}
		where $c_d>0$ depends only on the dimension. Define
		\begin{align*}
			\mathcal P_0 = \left\{p\in\mathcal P:\min_{\beta\in\mathcal B_{\rm act}}\|p-\beta\|_2\ge a\right\}.
		\end{align*}
		For each fixed $\beta\in\mathcal B_{\rm act}$, the balls $\{B_d(p,a/2):p\in\mathcal P\cap B_d(\beta,a)\}$ are disjoint and are contained in $B_d(\beta,3a/2)$. Hence
		\begin{align*}
			|\mathcal P\cap B_d(\beta,a)|\le 3^d.
		\end{align*}
		Since $|\mathcal B_{\rm act}|=\Mzero\le\MR$, at most $3^d\MR$ points are removed from $\mathcal P$. Choose $c_1$ small enough so that $c_d c_1^{-d}-3^d\ge1$. Then $|\mathcal P_0|\ge\MR$, and we may choose distinct points $\beta_j^{\rm dum}\in\mathcal P_0$, $j=\Mzero,\ldots,\MR-1$. These dummy slopes satisfy
		\begin{align*}
			\|\beta_j^{\rm dum}\|_2\le\delta,\qquad \|\beta_j^{\rm dum}-\beta_k^{\rm dum}\|_2\ge a\quad(j\ne k),\qquad \min_{\beta\in\mathcal B_{\rm act}}\|\beta_j^{\rm dum}-\beta\|_2\ge a.
		\end{align*}
		Set
		\begin{align*}
			\theta_j=((\beta_j^{\rm dum})^\top,-b)^\top.
		\end{align*}
		Then
		\begin{align*}
			\|\theta_j\|_2^2 \le \delta^2+b^2 \le \frac{5}{4}b^2 \le 1,\qquad \|\theta_j\|_\infty\le c_0,
		\end{align*}
		so the dummy parameters also satisfy the required feasibility constraints. Moreover, for all $x\in\cX$,
		\begin{align*}
			s_j(x;\theta_j)\le B_X\delta-b=-b/2<0=s_{\MR}(x).
		\end{align*}
		Thus no dummy expert can beat the reference expert.
		
		For the active specialists, Assumption~\ref{main-ass:linearly-separable} implies that, up to a $P_X$-null boundary set, subtracting $\bar x^\top v_{\Mzero}$ and multiplying by $\lambda$ preserves all region-wise score comparisons. Hence the constructed parameter realizes
		\begin{align*}
			\cT_1(x;\theta) = \begin{cases}
				\{r\}, & x\in\cX_r,\quad r=1,\ldots,\Mzero-1,\\
				\{\MR\}, & x\in\cX_{\Mzero}.
			\end{cases}
		\end{align*}
		Moreover, the slope construction gives membership in $\widetilde{\Theta}_{\MR}(r_{\MR})$ for every $r_{\MR}\le c(r_0\wedge\MR^{-1/d})$, after reducing the constant $c>0$ if necessary.
		
		Consequently,
		\begin{align*}
			\inf_{\theta'\in\widetilde{\Theta}_{\MR}(r_{\MR})}\left\|f_0-\sum_{m=1}^{\MR}g_m^{(1,\phi)}(\cdot;\theta')f_m^*\right\|_{L_2(P_X)}^2
			&\le \left\|f_0-\sum_{m=1}^{\MR}g_m^{(1,\phi)}(\cdot;\theta)f_m^*\right\|_{L_2(P_X)}^2 \\
			&= \sum_{r=1}^{\Mzero}\|f_0-f_r^*\|_{\cX_r}^2.
		\end{align*}
		Applying Theorem~\ref{main-thm:top-k} gives
		\begin{align*}
			\EE\|\hat f_0-f_0\|_{L_2(P_X)}^2 \le C\left\{\sum_{r=1}^{\Mzero}\|f_0-f_r^*\|_{\cX_r}^2+\fL_{\sparse}(\MR,n,r_{\MR})\right\}.
		\end{align*}
		The same bound holds for the online average by Jensen's inequality applied to the online sequence bound in Theorem~\ref{main-thm:top-k}.
		This proves the result.
	\end{proof}
	
	\subsection{Proof of Proposition \ref{main-prop: Top-1 mixing}}
	\label{subsec:supp-proof-prop-4-6}
	
	\begin{proof}[Proof of Proposition \ref{main-prop: Top-1 mixing}]
		The oracle lower bound only uses the measurable partition fixed in the main text and the mixture advantage in Assumption~\ref{main-ass:two-expert-combination}.
		For any $\theta\in\widetilde{\Theta}_{\MR}(r_{\MR})$, the Top-1 gating mechanism assigns each $x\in\cX$ to exactly one expert. Denote by $\cX_m'(\theta)$ the region consisting of inputs assigned to $f_m^*$. Then $\{\cX_m'(\theta)\}_{m\in[\MR]}$ is a partition of $\cX$, and, because the active set has size one,
		\begin{align*}
			\sum_{m=1}^{\MR}g_m^{(1,\phi)}(x;\theta)f_m^*(x) = \sum_{m=1}^{\MR}\one_{\cX_m'(\theta)}(x)f_m^*(x).
		\end{align*}
		Therefore, for every $\theta\in\widetilde{\Theta}_{\MR}(r_{\MR})$,
		\begin{align*}
			\left\|f_0-\sum_{m=1}^{\MR}g_m^{(1,\phi)}(\cdot;\theta)f_m^*\right\|_{L_2(P_X)}^2
			&= \sum_{r=1}^{\Mzero}\int_{\cX_r}\left(f_0(x)-\sum_{m=1}^{\MR}\one_{\cX_m'(\theta)}(x)f_m^*(x)\right)^2dP_X(x) \\
			&\ge \sum_{r=1}^{\Mzero}\int_{\cX_r}\min_{k\in[\MR]}(f_0(x)-f_k^*(x))^2dP_X(x) \\
			&\ge \sum_{r=1}^{\Mzero}\int_{\cX_r}\left(f_0(x)-w_{r1}f_{r,1}^*(x)-w_{r2}f_{r,2}^*(x)\right)^2dP_X(x) + \sum_{r=1}^{\Mzero}\delta_r.
		\end{align*}
		The last inequality follows from Assumption \ref{main-ass:two-expert-combination}. Hence
		\begin{align*}
			\left\|f_0-\sum_{m=1}^{\MR}g_m^{(1,\phi)}(\cdot;\theta)f_m^*\right\|_{L_2(P_X)}^2 \ge \sum_{r=1}^{\Mzero}\|f_0-w_{r1}f_{r,1}^*-w_{r2}f_{r,2}^*\|_{\cX_r}^2 + \sum_{r=1}^{\Mzero}\delta_r.
		\end{align*}
		Taking the infimum over $\theta\in\widetilde{\Theta}_{\MR}(r_{\MR})$ gives
		\begin{align*}
			\fA_1(\MR) \ge \sum_{r=1}^{\Mzero}\|f_0-w_{r1}f_{r,1}^*-w_{r2}f_{r,2}^*\|_{\cX_r}^2 + \sum_{r=1}^{\Mzero}\delta_r.
		\end{align*}
		
		We next prove the second claim. By the projection step of Algorithm~\ref{main-alg:general}, the final Top-1 estimator has the form
		\begin{align*}
			\hat f_0(x)=\sum_{m=1}^{\MR}g_m^{(1,\phi)}(x;\hat\theta)\bar f_{R,m}^{n_3}(x),
		\end{align*}
		for a random grid parameter $\hat\theta\in\widetilde{\Theta}_{\MR}(r_{\MR})$, where the projected Top-1 class is the same separated class used in the definition of $\fA_1(\MR)$. Here $\bar f_{R,m}^{n_3}$ denotes the aggregation-block average of the $m$-th routed expert. Define the population-expert counterpart
		\begin{align*}
			F^*(x;\hat\theta)=\sum_{m=1}^{\MR}g_m^{(1,\phi)}(x;\hat\theta)f_m^*(x).
		\end{align*}
		For every realization of $\hat\theta$,
		\begin{align*}
			\|F^*(\cdot;\hat\theta)-f_0\|_{L_2(P_X)}^2\ge \fA_1(\MR).
		\end{align*}
		Moreover, for every realization of the data and every $x\in\cX$,
		\begin{align*}
			\hat f_0(x)-F^*(x;\hat\theta)=\sum_{m=1}^{\MR}g_m^{(1,\phi)}(x;\hat\theta)\{\bar f_{R,m}^{n_3}(x)-f_m^*(x)\}.
		\end{align*}
		Since $g^{(1,\phi)}(x;\hat\theta)\in\Delta^{\MR-1}$, Jensen's inequality with respect to the simplex weights gives
		\begin{align*}
			|\hat f_0(x)-F^*(x;\hat\theta)|^2 \le \sum_{m=1}^{\MR}g_m^{(1,\phi)}(x;\hat\theta)|\bar f_{R,m}^{n_3}(x)-f_m^*(x)|^2.
		\end{align*}
		Integrating with respect to $P_X$ and using $0\le g_m^{(1,\phi)}\le1$ yields
		\begin{align*}
			\EE\|\hat f_0-F^*(\cdot;\hat\theta)\|_{L_2(P_X)}^2
			&\le \sum_{m=1}^{\MR}\EE\|\bar f_{R,m}^{n_3}-f_m^*\|_{L_2(P_X)}^2\\
			&\le \MR\max_{m\le\MR}\EE\|\bar f_{R,m}^{n_3}-f_m^*\|_{L_2(P_X)}^2\\
			&\le C\MR\bar a_{n_3}^2.
		\end{align*}
		The last inequality follows from Jensen's inequality over the aggregation block and Assumption~\ref{main-ass:evolving-experts}:
		\begin{align*}
			\EE\|\bar f_{R,m}^{n_3}-f_m^*\|_{L_2(P_X)}^2 \le \frac{1}{n_3}\sum_{t=n_1+n_2}^{n-1}\EE\|\hat f_{R,m}^t-f_m^*\|_{L_2(P_X)}^2 \le C\bar a_{n_3}^2.
		\end{align*}
		The reverse triangle inequality in $L_2(\Omega\times P_X)$ gives
		\begin{align*}
			\left(\EE\|\hat f_0-f_0\|_{L_2(P_X)}^2\right)^{1/2} \ge \sqrt{\fA_1(\MR)}-C\sqrt{\MR}\,\bar a_{n_3}.
		\end{align*}
		Squaring the positive part proves the claimed lower bound.
		This proves the claim.
	\end{proof}

	\section{Proof of Section \ref{main-sec:adaptation}}
	\label{sec:supp-proof-section-5}
	
	\subsection{Proof of the oracle decomposition for gating approximation}
	\label{subsec:supp-proof-oracle-decomposition-gating-approximation}
	
	\begin{proof}
		Fix any $g\in\cG$. For $x\in\cX_r$, we have $g^*(x)=e_r$ and hence
		\begin{align*}
			f_0(x)-\sum_{m=1}^{\MR}g_m(x)f_m^*(x)
			=\bigl(f_0(x)-f_r^*(x)\bigr)
			+\sum_{m=1}^{\MR}\bigl(g_m^*(x)-g_m(x)\bigr)f_m^*(x).
		\end{align*}
		Because $g(x)\in\simplex{\MR}$,
		\begin{align*}
			\|g(x)-g^*(x)\|_1
			=2\{1-g_r(x)\}
			\le 2\|g(x)-g^*(x)\|_2.
		\end{align*}
		Therefore, Assumption~\ref{main-ass:bounded-experts} gives
		\begin{align*}
			\left|\sum_{m=1}^{\MR}\bigl(g_m^*(x)-g_m(x)\bigr)f_m^*(x)\right|
			\le A\|g(x)-g^*(x)\|_1
			\le 2A\|g(x)-g^*(x)\|_2.
		\end{align*}
		Applying $(a+b)^2\le 2a^2+2b^2$, integrating over each $\cX_r$, and using that $\{\cX_r\}_{r=1}^{\MR}$ partitions $\cX$, we obtain
		\begin{align*}
			\left\|f_0-\sum_{m=1}^{\MR}g_m f_m^*\right\|_{L_2(P_X)}^2
			&\le 2\sum_{r=1}^{\MR}\|f_0-f_r^*\|_{\cX_r}^2
			+8A^2\|g-g^*\|_{L_2(P_X)}^2 \\
			&\le C\left\{\sum_{r=1}^{\MR}\|f_0-f_r^*\|_{\cX_r}^2
			+\|g-g^*\|_{L_2(P_X)}^2\right\},
		\end{align*}
		where $C=\max\{2,8A^2\}$ is independent of $n$ and $\MR$. Taking the infimum over $g\in\cG$ proves the decomposition stated in the main text.
	\end{proof}
	
	\subsection{Proof of Example \ref{main-ex:L1-L2}}
	\label{subsec:supp-proof-ex-5-2}
	
	\begin{proof}[Proof of Example \ref{main-ex:L1-L2}]
		Recall that $\cX_1=[-1,0]$, $\cX_2=(0,1]$ and $X\sim\Unif[-1,1]$. The oracle assignment is $e_1$ on $\cX_1$ and $e_2$ on $\cX_2$. Thus for any gating rule $g(x)=(g_1(x),\dots,g_{\MR}(x))^\top$, the global and interior losses are
		\begin{align*}
			\cL^{(1)}(g) &= \frac{1}{2}\int_{-1}^0 \left((1-g_1(x))^2+\sum_{j\ne 1}g_j(x)^2\right) dx + \frac{1}{2}\int_0^1 \left((1-g_2(x))^2+\sum_{j\ne 2}g_j(x)^2\right) dx, \\
			\cL^{(2)}(g;\rho) &= \frac{1}{2}\int_{-1}^{-\rho} \left((1-g_1(x))^2+\sum_{j\ne 1}g_j(x)^2\right) dx + \frac{1}{2}\int_{\rho}^1 \left((1-g_2(x))^2+\sum_{j\ne 2}g_j(x)^2\right) dx.
		\end{align*}
		
		\textbf{Now we prove the lower bound for $\cL^{(1)}$ under softmax gating.} Let
		\begin{align*}
			s_r^{\lin}(x;\theta_r)=\beta_r x+\alpha_r, \qquad g_r(x)=\frac{\exp\{s_r^{\lin}(x;\theta_r)\}}{\sum_{j=1}^{\MR}\exp\{s_j^{\lin}(x;\theta_j)\}}.
		\end{align*}
		Since
		\begin{align*}
			1-g_1(x) \ge \frac{\exp\{s_2^{\lin}(x;\theta_2)\}}{\exp\{s_1^{\lin}(x;\theta_1)\}+\exp\{s_2^{\lin}(x;\theta_2)\}}, \qquad 1-g_2(x) \ge \frac{\exp\{s_1^{\lin}(x;\theta_1)\}}{\exp\{s_1^{\lin}(x;\theta_1)\}+\exp\{s_2^{\lin}(x;\theta_2)\}},
		\end{align*}
		we have
		\begin{align*}
			\cL^{(1)}(g) \ge \frac{1}{2}\int_{-1}^0 \left(\frac{1}{1+\exp\{s_1^{\lin}(x;\theta_1)-s_2^{\lin}(x;\theta_2)\}}\right)^2 dx + \frac{1}{2}\int_0^1 \left(\frac{1}{1+\exp\{s_2^{\lin}(x;\theta_2)-s_1^{\lin}(x;\theta_1)\}}\right)^2 dx.
		\end{align*}
		
		Set $\Delta(x) := s_2^{\lin}(x;\theta_2)-s_1^{\lin}(x;\theta_1) = (\beta_2-\beta_1)x+(\alpha_2-\alpha_1)$. By the parameter bound,
		\begin{align*}
			|\beta_2-\beta_1|\le 2C\log \MR, \qquad |\alpha_1-\alpha_2|\le 2C\log \MR.
		\end{align*}
		Let $L:=2C\log \MR$ and $k:=s_1^{\lin}(0;\theta_1)-s_2^{\lin}(0;\theta_2)$. We first consider the case $k \ge 0$. Then $0 \le k \le L$, and for $x\in[0,1]$,
		\begin{align*}
			\Delta(x)=-k+(\beta_2-\beta_1)x\le -k+Lx.
		\end{align*}
		Since $t\mapsto (1+\exp\{t\})^{-2}$ is decreasing,
		\begin{align*}
			\cL^{(1)}(g) \ge \frac{1}{2}\int_0^1 \left(\frac{1}{1+\exp\{\Delta(x)\}}\right)^2 dx \ge \frac{1}{2}\int_0^1 \left(\frac{1}{1+\exp\{-k+Lx\}}\right)^2 dx.
		\end{align*}
		Now split the integral at $x=k/L$, we obtain
		\begin{align*}
			\cL^{(1)}(g)
			&\ge \frac{1}{2}\int_{0}^{k/L}\left(\frac{1}{1+\exp\{-k+Lx\}}\right)^2 dx + \frac{1}{2}\int_{k/L}^{1}\left(\frac{1}{1+\exp\{-k+Lx\}}\right)^2 dx \\
			&\ge \frac{1}{2}\cdot\frac{k}{4L} + \frac{1}{2}\int_{k/L}^{1}\frac{1}{4\exp\{-2k+2Lx\}} dx \\
			&= \frac{k}{8L}+\frac{1}{16L}\left(1-\exp\{2k-2L\}\right) \\
			&= \frac{1}{16L}\left(2k+1-\exp\{2k-2L\}\right).
		\end{align*}
		Since $0 \le k \le L$, the last display implies $\cL^{(1)}(g) \ge \frac{1-\exp\{-2L\}}{16L}$. If $k < 0$, the same argument is applied on the interval $[-1,0]$, with the two sided reversed. This gives the same lower bound. Therefore, for every softmax gate $g$,
		\begin{align*}
			\cL^{(1)}(g) \ge \frac{1-\exp\{-2L\}}{16L}.
		\end{align*}
		Since $L=2C\log \MR$,
		\begin{align*}
			\cL^{(1)}(g) \ge \frac{1-\MR^{-4C}}{32C\log \MR} \gtrsim \frac{1}{\log \MR}.
		\end{align*}
		As $g\in\cG_{\soft}$ was arbitrary,
		\begin{align*}
			\cL^{(1)}(\cG_{\soft}) \gtrsim \frac{1}{\log \MR}.
		\end{align*}
		
		\textbf{Next we prove the upper bound for $\cL^{(1)}$ under softmax gating.} Choose
		\begin{align*}
			s_1^{\lin}(x;\theta_1)=B-ax, \qquad s_2^{\lin}(x;\theta_2)=B+ax, \qquad s_j^{\lin}(x;\theta_j)=-B,\qquad j\ge 3,
		\end{align*}
		with $a=B=C\log \MR$. Set $s:=(\MR-2)\exp\{-2B\}=(\MR-2)\MR^{-2C}$. Since $C>1/2$, we have $0\le s\le 1$ for all sufficiently large $\MR$; the finitely many smaller values of $\MR$ are absorbed into the constant. For $x\in[0,1]$,
		\begin{align*}
			1-g_2(x) = \frac{\exp\{-ax\}+s}{\exp\{-ax\}+\exp\{ax\}+s} \le \exp\{-2ax\}+s\exp\{-ax\},
		\end{align*}
		Similarly, for $x\in[-1,0]$,
		\begin{align*}
			1-g_1(x)\le \exp\{-2a|x|\}+s\exp\{-a|x|\}.
		\end{align*}
		Also,
		\begin{align*}
			(1-g_1(x))^2+\sum_{j\ne 1}g_j(x)^2\le 2(1-g_1(x))^2, \qquad (1-g_2(x))^2+\sum_{j\ne 2}g_j(x)^2\le 2(1-g_2(x))^2.
		\end{align*}
		Hence, by symmetry,
		\begin{align*}
			\cL^{(1)}(g) \le 2\int_0^1 \left(\exp\{-2ax\}+s\exp\{-ax\}\right)^2 dx.
		\end{align*}
		Using $(u+v)^2 \le 2(u^2+v^2)$ and $s\le 1$,
		\begin{align*}
			\cL^{(1)}(g) \le 4\int_0^1 \exp\{-4ax\} dx+4\int_0^1 \exp\{-2ax\} dx = \frac{1-\exp\{-4a\}}{a}+\frac{2(1-\exp\{-2a\})}{a} \le \frac{3}{a}.
		\end{align*}
		Therefore
		\begin{align*}
			\cL^{(1)}(g)\le \frac{3}{C\log \MR},
		\end{align*}
		and hence
		\begin{align*}
			\cL^{(1)}(\cG_{\soft}) \le \frac{3}{C\log \MR}.
		\end{align*}
		
		\textbf{Now we prove the upper bound for $\cL^{(2)}$ under softmax gating.} For the same choice of scores, by symmetry,
		\begin{align*}
			\cL^{(2)}(g;\rho) \le 2\int_{\rho}^1 \left(\exp\{-2ax\}+s\exp\{-ax\}\right)^2 dx.
		\end{align*}
		Using again $(u+v)^2\le 2(u^2+v^2)$ and $s\le 1$,
		\begin{align*}
			\cL^{(2)}(g;\rho) \le 4\int_{\rho}^1 \exp\{-4ax\} dx+4\int_{\rho}^1 \exp\{-2ax\} dx \le \frac{\exp\{-4a\rho\}}{a}+\frac{2\exp\{-2a\rho\}}{a} \le \frac{3\exp\{-2a\rho\}}{a}.
		\end{align*}
		Since $a=C\log \MR$,
		\begin{align*}
			\cL^{(2)}(g;\rho) \le \frac{3}{C\log \MR}\MR^{-2C\rho} \le \frac{3}{C\log 2}\MR^{-2C\rho}.
		\end{align*}
		Therefore
		\begin{align*}
			\cL^{(2)}(\cG_{\soft};\rho) \le \frac{3}{C\log 2}\MR^{-2C\rho}.
		\end{align*}
		
		\textbf{Finally we prove the results for Top-1 gating.} Choose scores
		\begin{align*}
			\left(s_1^{\lin}(x;\theta_1),\ldots,s_{\MR}^{\lin}(x;\theta_{\MR})\right)=(-x,x,-2,\dots,-2).
		\end{align*}
		Then the Top-1 rule assigns expert $1$ on $[-1,0]$ and expert $2$ on $(0,1]$, up to the boundary point $x=0$, which has $P_X$-measure zero. Therefore the oracle partition is realized exactly, and
		\begin{align*}
			\cL^{(1)}(\cG_{\mathrm{Top\text{-}1}})=0, \qquad \cL^{(2)}(\cG_{\mathrm{Top\text{-}1}};\rho)=0.
		\end{align*}
		This complete the proof.
	\end{proof}
	
	\subsection{Proof of Example \ref{main-exa:linear-quadratic-l1-l2}}
	\label{subsec:supp-proof-ex-5-3}
	
	\begin{proof}[Proof of Example \ref{main-exa:linear-quadratic-l1-l2}]
		This proof supplies the formal disk example underlying Example 5.3. Let $\cX=\{x\in\RR^2:\|x\|_2\le R_2\}$ with $X\sim\Unif(\cX)$, and let the oracle regions be
		\begin{align*}
			\cX_1=\{x\in\cX:\|x\|_2\le R_1\}, \qquad \cX_2=\{x\in\cX:R_1<\|x\|_2\le R_2\},
		\end{align*}
		where $0<R_1<R_2$. We prove that, for every fixed $0<\rho<\min\{R_1,R_2-R_1\}$, there exist constants $C_L,C_R,\alpha>0$ independent of $\MR$, such that
		\begin{align*}
			\cL^{(2)}(\cG_{\lin};\rho)\ge C_L, \qquad \cL^{(2)}(\cG_{\qdr};\rho)\le C_R\MR^{-\alpha}
		\end{align*}
		
		Define
		\begin{align*}
			\cX_1^{(-\rho)}:=\{x\in\RR^2:\|x\|_2\le R_1-\rho\}, \qquad \cX_2^{(-\rho)}:=\{x\in\RR^2:R_1+\rho\le \|x\|_2\le R_2\}.
		\end{align*}
		Since $X\sim\Unif(\cX)$, we have
		\begin{align*}
			\cL^{(2)}(g;\rho) = \frac{1}{\pi R_2^2}\int_{\cX_1^{(-\rho)}}\left((1-g_1)^2+\sum_{j\ne 1}g_j^2\right) dx + \frac{1}{\pi R_2^2}\int_{\cX_2^{(-\rho)}}\left((1-g_2)^2+\sum_{j\ne 2}g_j^2\right) dx.
		\end{align*}
		
		\textbf{We first prove the lower bound for $\cL^{(2)}$ under linear-softmax gating.} For $g\in\cG_{\lin}$, write $\Delta(x):=s_2^{\lin}(x;\theta_2)-s_1^{\lin}(x;\theta_1)=c+b^\top x$. First suppose $c\ge 0$. Since $\Delta(x)+\Delta(-x)=2c\ge 0$, symmetry implies that $\Delta(x) \ge 0$ on a subset of $\cX_1^{(-\rho)}$ with Lebesgue measure at least $\mathrm{Vol}(\cX_1^{(-\rho)})/2$. On this subset, $g_1(x) \le 1/2$. Thus
		\begin{align*}
			(1-g_1)^2+\sum_{j\ne 1}g_j^2\ge \frac{1}{4},
		\end{align*}
		on this subset, and therefore
		\begin{align*}
			\cL^{(2)}(g;\rho) \ge \frac{\mathrm{Vol}(\cX_1^{(-\rho)})}{8\pi R_2^2} = \frac{(R_1-\rho)^2}{8R_2^2}.
		\end{align*}
		If $c\le 0$, then $\Delta(x)+\Delta(-x)=2c \le 0$, so by the same symmetry argument $\Delta(x) \le 0$ on at least half of $\cX_2^{(-\rho)}$. On that subset, $g_2(x) \le 1/2$ and hence
		\begin{align*}
			\cL^{(2)}(g;\rho) \ge \frac{\mathrm{Vol}(\cX_2^{(-\rho)})}{8\pi R_2^2} = \frac{R_2^2-(R_1+\rho)^2}{8R_2^2}.
		\end{align*}
		Consequently,
		\begin{align*}
			\cL^{(2)}(\cG_{\lin};\rho) \ge C_L := \frac{1}{8R_2^2}\min\left\{(R_1-\rho)^2,R_2^2-(R_1+\rho)^2\right\} > 0.
		\end{align*}
		
		\textbf{We next prove the upper bound for $\cL^{(2)}$ under quadratic-softmax gating.} Since $C>1/2$, choose
		\begin{align*}
			0<c_0<\min\left\{\frac{2C-1}{R_1^2},\frac{C}{R_1^2},C\right\}.
		\end{align*}
		Then the interval $(\max\{0,1-C\},C-c_0R_1^2)$ is nonempty. Choose $b$ in this interval. Define
		\begin{align*}
			s_1^{\qdr}(x;\theta_1) &:= b\log \MR+aR_1^2-a\|x\|_2^2,\\
			s_2^{\qdr}(x;\theta_2) &:= b\log \MR-aR_1^2+a\|x\|_2^2,\\
			s_j^{\qdr}(x;\theta_j) &:= -B,\qquad j\ge 3,
		\end{align*}
		with $a:=c_0\log \MR$ and $B:=C\log \MR$. The choice of $c_0$ and $b$ ensures that all quadratic coefficients and intercepts in the first two scores are bounded by $C\log \MR$ in absolute value. The dummy scores also lie in the same box. Set
		\begin{align*}
			\eta_\rho:=R_1^2-(R_1-\rho)^2=2R_1\rho-\rho^2>0, \qquad s:=(\MR-2)\exp\{-(B+b\log \MR)\}.
		\end{align*}
		Since $C+b>1$, we have $s \le 1$ for al sufficiently large $\MR$; smaller values of $\MR$ can be absorbed into the constant. On $\cX_1^{(-\rho)}$,
		\begin{align*}
			s_1^{\qdr}(x;\theta_1)\ge b\log \MR+a\eta_\rho, \qquad s_2^{\qdr}(x;\theta_2)\le b\log \MR-a\eta_\rho.
		\end{align*}
		Therefore,
		\begin{align*}
			1-g_1 \le \frac{\exp\{s_2^{\qdr}(x;\theta_2)\}+(\MR-2)\exp\{-B\}}{\exp\{s_1^{\qdr}(x;\theta_1)\}} \le \exp\{-2a\eta_\rho\}+\exp\{-a\eta_\rho\} \le 2\exp\{-a\eta_\rho\}.
		\end{align*}
		Similarly, on $\cX_2^{(-\rho)}$, $1-g_2\le 2\exp\{-a\eta_\rho\}$. Thus on both interior regions,
		\begin{align*}
			(1-g_i)^2+\sum_{j\ne i}g_j^2\le 2(1-g_i)^2 \le 8\exp\{-2a\eta_\rho\}, \qquad i = 1, 2.
		\end{align*}
		It follows that
		\begin{align*}
			\cL^{(2)}(g;\rho) \le \frac{8\exp\{-2a\eta_\rho\}}{\pi R_2^2}\mathrm{Vol}(\cX) = 8\exp\{-2a\eta_\rho\} = 8\MR^{-2c_0\eta_\rho},
		\end{align*}
		where the last equality follows from $a = c_0 \log \MR$. Therefore
		\begin{align*}
			\cL^{(2)}(\cG_{\qdr};\rho) \le 8\MR^{-\alpha}, \qquad \alpha:=2c_0(2R_1\rho-\rho^2)>0.
		\end{align*}
		Combining the lower bound for linear-softmax gating and the upper bound for quadratic-softmax gating proves the claim.
	\end{proof}

	\subsection{Proof of Theorem \ref{main-thm:kernel-global-main}}
	\label{subsec:supp-proof-thm-5-5}
	
	We first recall the grid notation used in the kernel construction. For a mesh width $h>0$, define
	\begin{align*}
		\cC:=\{0,h,2h,\ldots,\lfloor h^{-1}\rfloor h\}^{d}\cap[0,1]^d
	\end{align*}
	as the set of grid centers. The Gaussian kernel bandwidth is set to $\sigma=\kappa h$ where $\kappa\in[1/2,1]$. The following lemma bound the total Gaussian tail mass contributed by grid centers in $\cC$ that lie at Euclidean distance at least $r$ from a fixed point $x$.
	
	\begin{lemma}[Gaussian tail sum over a uniform grid]\label{lem:supp-grid-tail}
		Let $\sigma=\kappa h$ for some fixed $\kappa\in[1/2,1]$. Then there exists a constant $A_{d,\kappa}>0$, depending only on $(d,\kappa)$, such that for every $x\in[0,1]^d$ and every $r\ge 2\sigma$,
		\begin{align*}
			\sum_{\substack{c\in\cC\\ \|x-c\|\ge r}} \exp\left\{-\frac{\|x-c\|^2}{2\sigma^2}\right\} \le A_{d,\kappa}\exp\left\{-\frac{r^2}{4\sigma^2}\right\}.
		\end{align*}
	\end{lemma}
	
	\begin{proof}[Proof of Lemma \ref{lem:supp-grid-tail}]
		Let $\Lambda:=h\mathbb{Z}^d$. Since $\cC\subset\Lambda$ and all terms are nonnegative,
		\begin{align*}
			\sum_{\substack{c\in\cC\\ \|x-c\|\ge r}} \exp\left\{-\frac{\|x-c\|^2}{2\sigma^2}\right\} \le \sum_{\substack{\lambda\in\Lambda\\ \|x-\lambda\|\ge r}} \exp\left\{-\frac{\|x-\lambda\|^2}{2\sigma^2}\right\} =: S.
		\end{align*}
		It suffices to bound $S$. For $R\ge 0$, let $N(R):=\#\{\lambda\in\Lambda:\|x-\lambda\|\le R\}$. Since the balls $B_d(\lambda,h/2)$, $\lambda\in\Lambda$, are pairwise disjoint up to boundaries of measure zero,
		\begin{align*}
			N(R)\mathrm{Vol}(B_d(0,h/2))\le \mathrm{Vol}(B_d(0,R+h/2)),
		\end{align*}
		hence
		\begin{align*}
			N(R)\le \left(1+\frac{2R}{h}\right)^d \le C_d\left(1+\frac{R}{h}\right)^d.
		\end{align*}
		
		Now decompose the lattice points into dyadic shells
		\begin{align*}
			A_m:=\{\lambda\in\Lambda:2^m r\le \|x-\lambda\|<2^{m+1}r\}, \qquad m\ge 0.
		\end{align*}
		Then
		\begin{align*}
			S = \sum_{m=0}^\infty \sum_{\lambda\in A_m} \exp\left\{-\frac{\|x-\lambda\|^2}{2\sigma^2}\right\} \le \sum_{m=0}^\infty \#A_m \exp\left\{-\frac{(2^m r)^2}{2\sigma^2}\right\}.
		\end{align*}
		Using $\#A_m\le N(2^{m+1}r)$, we get
		\begin{align*}
			S\le C_d\sum_{m=0}^\infty \left(1+\frac{2^{m+1}r}{h}\right)^d \exp\left\{-\frac{2^{2m}r^2}{2\sigma^2}\right\}.
		\end{align*}
		Set $t:=r/\sigma$. Since $r\ge 2\sigma$, $t=\frac{r}{\sigma} \ge 2$. Also $r/h=\kappa t$. Therefore
		\begin{align*}
			S\le C_d\sum_{m=0}^\infty (1+2^{m+1}\kappa t)^d \exp\{-2^{2m-1}t^2\}.
		\end{align*}
		Because $t\ge 2$ and $\kappa\ge 1/2$, we have $2^{m+1}\kappa t\ge 2$, hence
		\begin{align*}
			(1+2^{m+1}\kappa t)^d \le C_{d,\kappa}(2^m t)^d.
		\end{align*}
		Thus
		\begin{align*}
			S\le C_{d,\kappa}\sum_{m=0}^\infty (2^m t)^d \exp\{-2^{2m-1}t^2\}.
		\end{align*}
		Now use the fact that $y\mapsto y^d \exp\{-y^2/8\}$ is bounded on $[0,\infty)$. Hence there exists $C_d'>0$ such that
		\begin{align*}
			y^d\le C_d' \exp\{y^2/8\}, \qquad y\ge 0.
		\end{align*}
		Applying this with $y=2^m t$, we obtain
		\begin{align*}
			(2^m t)^d \exp\{-2^{2m-1}t^2\} \le C_d' \exp\{-2^{2m-1}t^2+2^{2m}t^2/8\} = C_d' \exp\{-(3/8)2^{2m}t^2\} \le C_d' \exp\{-2^{2m-2}t^2\}.
		\end{align*}
		Therefore
		\begin{align*}
			S\le C_{d,\kappa} C_d'\sum_{m=0}^\infty \exp\{-2^{2m-2}t^2\}.
		\end{align*}
		Since $4^m\ge m+1$,
		\begin{align*}
			\sum_{m=0}^\infty \exp\{-2^{2m-2}t^2\} = \sum_{m=0}^\infty \exp\{-(t^2/4)4^m\} \le \sum_{m=0}^\infty \exp\{-(t^2/4)(m+1)\}.
		\end{align*}
		For $t\ge2$, $\sum_{m=0}^\infty \exp\{-2^{2m-2}t^2\} \le \frac{e}{e-1}\exp\{-t^2/4\}$. Combining the preceding displays, and defining $A_{d,\kappa}=eC_{d,\kappa}C_d'/(e-1)$, we obtain
		\begin{align*}
			S \le A_{d,\kappa}\exp\{-t^2/4\} = A_{d,\kappa}\exp\left\{-\frac{r^2}{4\sigma^2}\right\}.
		\end{align*}
		This proves the lemma.
	\end{proof}
	
	\paragraph{Extension to general weights.}
	We use the general region-wise target-weight notation introduced in Section \ref{main-sec:adaptation} of the main text. Thus each cell $\cX_i$ is associated with a target vector $w_i^*=(w_{i1}^*,\dots,w_{i\MR}^*)\in\simplex{\MR}$. For the uniform grid $\cC\subset[0,1]^d$, assign each grid center $c \in \cC$ to a unique region by a fixed deterministic tie-breaking rule and let
	\begin{align*}
		C_i:=\{c \in \cC: c \text{ is assigned to } \cX_i\}, \qquad i=1,\dots,\MR.
	\end{align*}
	Then $\{C_i\}_{i=1}^{\MR}$ forms a disjoint partition of $\cC$. We consider the choice
	\begin{align*}
		w_{r,c}:=\sum_{i=1}^{\MR} w_{ir}^*\one_{\{c\in C_i\}}, \qquad 1\le r\le \MR,\quad c\in\cC.
	\end{align*}
	Thus, whenever $c\in C_i$, one has $w_{r,c}=w_{ir}^*$. Define
	\begin{align*}
		k_r(x):=\sum_{i=1}^{\MR} w_{ir}^*\sum_{c\in C_i}\varphi(x,c;\sigma), \qquad g_r(x):=\frac{k_r(x)}{\sum_{s=1}^{\MR}k_s(x)} = \frac{k_r(x)}{\sum_{c\in\cC}\varphi(x,c;\sigma)}.
	\end{align*}
	The last equality follows because the vectors $w_i^*$ lie in the simplex and the sets $C_i$ form a partition of $\cC$. The loss function is now defined as
	\begin{align*}
		\cL^{(1)}(g) := \sum_{i=1}^{\MR} \left\|\left(w_i^*-g(x)\right)\one_{\cX_i}(x)\right\|_{L_2(P_X)}^2, \qquad \cL^{(1)}(\cG_{\mathrm{ker}}) := \inf_{g\in\cG_{\mathrm{ker}}}\cL^{(1)}(g).
	\end{align*}
	
	\begin{proof}[Proof of Theorem \ref{main-thm:kernel-global-main}]
		Set
		\begin{align*}
			\rho_h:=2\kappa h\sqrt{\left(\frac{d}{2}+1\right)\log(1/h)}.
		\end{align*}
		The boundary-layer regularity condition is only assumed for radii at most $\rho_0$. Since $\rho_h\to0$ as $h\downarrow0$, there exists $h_0>0$, depending only on $(d,\kappa,\rho_0)$, such that $\rho_h\le\rho_0$ whenever $0<h\le h_0$. If $h>h_0$, then, because both $g(x)$ and every $w_i^*$ belong to $\simplex{\MR}$,
		\begin{align*}
			\cL^{(1)}(\cG_{\mathrm{ker}})\le 2
			\le 2h_0^{-(d+2)}\MR h^{d+2},
		\end{align*}
		which already has the asserted form. It therefore remains to consider $0<h\le h_0$, and in this case $\rho_h\le\rho_0$.
		Fix $i \in [\MR]$.
		We use the notation $\cX_i^{-\rho_h}$ and $B_i(\rho_h)$ from this definition. The assumed upper bound on $h$ implies $\rho_h\ge 2\sigma$, so Lemma~\ref{lem:supp-grid-tail} is applicable with $r=\rho_h$.
		
		We first estimate the error on $\cX_i^{-\rho_h}$. If $x\in\cX_i^{-\rho_h}$, then every $c\notin C_i$ satisfies $\|x-c\|\ge \rho_h$. Hence Lemma~\ref{lem:supp-grid-tail} gives
		\begin{align*}
			\sum_{c\notin C_i}\varphi(x,c;\sigma) \le A_{d,\kappa}\exp\left\{-\frac{\rho_h^2}{4\sigma^2}\right\}.
		\end{align*}
		On the other hand, since $\cC$ is a uniform grid with mesh width $h$, and taking $h^{-1} \in \mathbb{N}$ for notational simplicity, there exists a grid point $c_x\in\cC$ such that
		\begin{align*}
			\|x-c_x\|\le \frac{\sqrt{d}}{2}h.
		\end{align*}
		The case of a general mesh is identical after changing the constant $\sqrt{d}/2$ to a comparable mesh constant. The assumed upper bound on $h$ ensures $\frac{\sqrt{d}}{2}h<\rho_h$. Since $x\in\cX_i^{-\rho_h}$, this implies $c_x\in\cX_i$. Moreover, $c_x$ is strictly inside $\cX_i$, so the deterministic tie-breaking rule assigns $c_x$ to $C_i$. Therefore
		\begin{align*}
			\sum_{c\in C_i}\varphi(x,c;\sigma) \ge \varphi(x,c_x;\sigma) \ge \exp\left\{-\frac{d}{8\kappa^2}\right\}.
		\end{align*}
		Using the definition of $g$ and $k_r$, for every $r$,
		\begin{align*}
			|g_r(x)-w_{ir}^*| = \left|\frac{\sum_{k\ne i}(w_{kr}^*-w_{ir}^*)\sum_{c\in C_k}\varphi(x,c;\sigma)}{\sum_{k=1}^{\MR}\sum_{c\in C_k}\varphi(x,c;\sigma)}\right| \le \frac{\sum_{c\notin C_i}\varphi(x,c;\sigma)}{\sum_{c\in C_i}\varphi(x,c;\sigma)}.
		\end{align*}
		Hence
		\begin{align*}
			|g_r(x)-w_{ir}^*| \le A_{d,\kappa}\exp\left\{\frac{d}{8\kappa^2}\right\}\exp\left\{-\frac{\rho_h^2}{4\sigma^2}\right\}.
		\end{align*}
		Since $\frac{\rho_h^2}{4\sigma^2} = (d/2+1)\log(1/h)$, we have, for all $x\in\cX_i^{-\rho_h}$,
		\begin{align*}
			|g_r(x)-w_{ir}^*| \le C_{d,\kappa}h^{d/2+1}.
		\end{align*}
		Therefore,
		\begin{align*}
			\int_{\cX_i^{-\rho_h}} \sum_{r=1}^{\MR}(g_r(x)-w_{ir}^*)^2 dP_X(x) \le P_X(\cX_i)\MR C_{d,\kappa}h^{d+2}.
		\end{align*}
		We next control the boundary layer $B_i(\rho_h)$. Since $g(x)\in\simplex{\MR}$ and $w_i^*\in\simplex{\MR}$, $\sum_{r=1}^{\MR}(g_r(x)-w_{ir}^*)^2\le 2$. Thus
		\begin{align*}
			\int_{B_i(\rho_h)} \sum_{r=1}^{\MR}(g_r(x)-w_{ir}^*)^2 dP_X(x) \le 2P_X(B_i(\rho_h)).
		\end{align*}
		Because $p_X$ is continuous on the compact set $[0,1]^d$, we have $\|p_X\|_\infty<\infty$. By the boundary-layer regularity of the partition, $\mathrm{Vol}\left(B_i(\rho_h)\right) = \int_{B_i(\rho_h)} 1 dx \le \gamma_i\rho_h$, and therefore
		\begin{align*}
			P_X(B_i(\rho_h)) \le \|p_X\|_\infty\mathrm{Vol}\left(B_i(\rho_h)\right) \le \|p_X\|_\infty\gamma_i\rho_h.
		\end{align*}
		Hence
		\begin{align*}
			\int_{B_i(\rho_h)} \sum_{r=1}^{\MR}(g_r(x)-w_{ir}^*)^2 dP_X(x) \le 2\|p_X\|_\infty\gamma_i\rho_h.
		\end{align*}
		Summing over $i=1,\dots,\MR$, and using $\sum_i P_X(\cX_i) = 1$, we conclude that
		\begin{align*}
			\cL^{(1)}(\cG_{\mathrm{ker}}) \le C_{d,\kappa}\MR h^{d+2}+2\|p_X\|_\infty\Gamma_{\MR}\rho_h.
		\end{align*}
		Since $\rho_h=2\kappa h\sqrt{(d/2+1)\log(1/h)}$, the preceding display implies
		\begin{align*}
			\cL^{(1)}(\cG_{\mathrm{ker}}) \le C\left(\MR h^{d+2} + \Gamma_{\MR}h\sqrt{\log(1/h)}\right),
		\end{align*}
		for a constant $C$ independent of $h$, $\Gamma_{\MR}$, and $\MR$.
	\end{proof}
	
	\subsection{Proof of Corollary \ref{main-cor:kernel-interior-main}}
	\label{subsec:supp-proof-cor-5-7}
	
	\begin{proof}[Proof of Corollary \ref{main-cor:kernel-interior-main}]
		Fix $i \in [\MR]$. The condition $h<2\rho/\sqrt{d}$ gives $\frac{\sqrt{d}}{2}h<\rho$. Moreover, since $\sigma=\kappa h$, the condition $h<\rho/(2\kappa)$ implies $\rho\ge 2\sigma$. Thus Lemma~\ref{lem:supp-grid-tail} applies with $r = \rho$. The same interior-grid-point argument used in the proof of Theorem \ref{main-thm:kernel-global-main}, with $\rho$ in place of $\rho_h$, gives that for every $x\in\cX_i^{-\rho}$ and every $r=1,\dots,\MR$,
		\begin{align*}
			|g_r(x)-w_{ir}^*| \le C_{d,\kappa}\exp\left\{-\frac{\rho^2}{4\sigma^2}\right\}, \qquad \sigma=\kappa h.
		\end{align*}
		Therefore,
		\begin{align*}
			\|g(x)-w_i^*\|_2^2 = \sum_{r=1}^{\MR}(g_r(x)-w_{ir}^*)^2 \le \MR C_{d,\kappa}\exp\left\{-\frac{\rho^2}{2\sigma^2}\right\}.
		\end{align*}
		Integrating over $\cX_i^{-\rho}$ and summing over $i$, we obtain
		\begin{align*}
			\sum_{i=1}^{\MR}\left\|(g(x)-w_i^*)\one_{\cX_i^{-\rho}}\right\|_{L_2(P_X)}^2
			&= \sum_{i=1}^{\MR}\int_{\cX_i^{-\rho}}\|g(x)-w_i^*\|_2^2 dP_X(x) \\
			&\le \MR C_{d,\kappa}\exp\left\{-\frac{\rho^2}{2\sigma^2}\right\}\sum_{i=1}^{\MR}P_X(\cX_i^{-\rho}) \\
			&\le \MR C_{d,\kappa}\exp\left\{-\frac{\rho^2}{2\sigma^2}\right\},
		\end{align*}
		since the sets $\cX_i^{-\rho}$ are pairwise disjoint. Using $\sigma=\kappa h$, we get
		\begin{align*}
			\cL^{(2)}(\cG_{\mathrm{ker}};\rho) \le C\MR\exp\left\{-\frac{\rho^2}{2\kappa^2h^2}\right\},
		\end{align*}
		for a constant $C$ independent of $h$ and $\MR$. This completes the proof.
	\end{proof}
	
		\subsection{Proof of Theorem \ref{main-thm:selection(global)}}
	\label{subsec:supp-proof-thm-5-9}
	
	We prove the guarantee for the aggregation-projection procedure in Algorithm \ref{main-alg:selection} of the main text. Recall that each candidate estimator satisfies the class-specific oracle inequality \eqref{main-eq:2} in the main text.
	
	\begin{proof}[Proof of Theorem \ref{main-thm:selection(global)}]
		Define
		\begin{align*}
			R_\ell := \fA_\ell+\fL_\ell,\qquad \ell \in [L].
		\end{align*}
		Throughout the proof, $C$ denotes a positive constant that may change from line to line. It is independent of $n$, $L$, and $\ell$, but may depend on the fixed boundedness constants, the common envelope bound $A_0$, the error-distribution constants, and the uniform bound for the constants $C_\ell$ in \eqref{main-eq:2}.
		By the sample splitting in Algorithm~\ref{main-alg:selection}, after conditioning on the fitted candidate list $\{\check f_\ell\}_{\ell=1}^L$, we may apply Proposition~\ref{prop:risk_AFTER} to the calibration block $Z^{(2)}$ and the aggregation block $Z^{(3)}$. In the notation of Proposition~\ref{prop:risk_AFTER}, take
		\begin{align*}
			S=L,\qquad \mu_{\ell,t}(x)=\check f_\ell(x),\qquad \pi_\ell=1/L,\qquad \hat\sigma_{\ell,t}=\hat\sigma_\ell,
		\end{align*}
		for $t=n_1+n_2,\ldots,n-1$. Assumption~\ref{main-ass:bounded-experts} bounds $f_0$, and the common envelope bound on $\{\check f_\ell\}_{\ell=1}^L$ before \eqref{main-eq:2} gives Assumption~\ref{ass:bound_forecast}. The truncation of $\hat\sigma_\ell$ in Algorithm~\ref{main-alg:selection} gives Assumption~\ref{ass:bound_scale_est}. Hence
		\begin{align*}
			\frac{1}{n_3}\sum_{t=n_1+n_2}^{n-1}\EE_{\mathrm{sel}}\|f_0-\tilde f^{t}\|_{L_2(P_X)}^2
			\le C\inf_{\ell\in[L]}\left\{\|f_0-\check f_\ell\|_{L_2(P_X)}^2+\frac{\log L}{n_3}+\EE_{\mathrm{sel}}(\hat\sigma_\ell-\sigma)^2\right\}.
		\end{align*}
		Moreover, the same calculation as in Lemma~\ref{lemma:remove_scale_estimation}, now with the fixed residual $r_\ell(x)=f_0(x)-\check f_\ell(x)$ on the calibration block, gives
		\begin{align*}
			\EE_{\mathrm{sel}}(\hat\sigma_\ell-\sigma)^2 \le C\left\{\frac{1}{n_2}+\|f_0-\check f_\ell\|_{L_2(P_X)}^2\right\}.
		\end{align*}
		Finally, since $\tilde f_n=n_3^{-1}\sum_{t=n_1+n_2}^{n-1}\tilde f^{t}$, Jensen's inequality implies
		\begin{align*}
			\|f_0-\tilde f_n\|_{L_2(P_X)}^2 \le \frac{1}{n_3}\sum_{t=n_1+n_2}^{n-1}\|f_0-\tilde f^{t}\|_{L_2(P_X)}^2.
		\end{align*}
		Combining these displays yields
		\begin{align*}
			\EE_{\mathrm{sel}}\|f_0-\tilde{f}_n\|_{L_2(P_X)}^2 \le C\left(\inf_{\ell\in[L]}\|f_0-\check{f}_\ell\|_{L_2(P_X)}^2+\frac{\log L}{n_3}+\frac{1}{n_2}\right),
		\end{align*}
		where $\EE_{\mathrm{sel}}$ denotes conditional expectation over $Z^{(2)}$ and $Z^{(3)}$, given the fitted candidate list. Taking expectation over $Z^{(1)}$ and using \eqref{main-eq:2}, we obtain
		\begin{align}
			\EE\|f_0-\tilde{f}_n\|_{L_2(P_X)}^2
			&\le C\left(\inf_{\ell\in[L]}\EE\|f_0-\check{f}_\ell\|_{L_2(P_X)}^2+\frac{\log L}{n_3}+\frac{1}{n_2}\right) \nonumber\\
			&\le C\left(\inf_{\ell\in[L]}R_\ell+\frac{\log L}{n_3\wedge n_2}\right). \label{eq:5.9.2}
		\end{align}
		By the population $L_2(P_X)$-projection definition of $\ell^*$, for every sample realization,
		\begin{align*}
			\|\tilde{f}_n - \check{f}_{\ell^*}\|_{L_2(P_X)}^2 = \inf_{\ell \in [L]}\|\tilde{f}_n - \check{f}_{\ell}\|_{L_2(P_X)}^2.
		\end{align*}
		Therefore,
		\begin{align*}
			\EE\|\tilde{f}_n - \check{f}_{\ell^*}\|_{L_2(P_X)}^2 \le \inf_{\ell \in [L]}\EE\|\tilde{f}_n - \check{f}_{\ell}\|_{L_2(P_X)}^2.
		\end{align*}
		Using the inequality $\|a-b\|^2\le 2\|a-c\|^2+2\|c-b\|^2$, we have
		\begin{align*}
			\EE\|f_0-\check{f}_{\ell^*}\|_{L_2(P_X)}^2
			&\le 2\left(\EE\|f_0-\tilde{f}_n\|_{L_2(P_X)}^2 + \EE\|\tilde{f}_n-\check{f}_{\ell^*}\|_{L_2(P_X)}^2\right) \\
			&\le 2\left(\EE\|f_0-\tilde{f}_n\|_{L_2(P_X)}^2 + \inf_{\ell\in[L]}\EE\|\tilde{f}_n-\check{f}_\ell\|_{L_2(P_X)}^2\right) \\
			&\le 2\left(\EE\|f_0-\tilde{f}_n\|_{L_2(P_X)}^2 + 2\inf_{\ell\in[L]}\left(\EE\|f_0-\tilde{f}_n\|_{L_2(P_X)}^2 + \EE\|f_0-\check{f}_\ell\|_{L_2(P_X)}^2\right)\right) \\
			&\le C\left(\inf_{\ell\in[L]}R_\ell + \frac{\log L}{n_3\wedge n_2}\right),
		\end{align*}
		where the last inequality follows from \eqref{eq:5.9.2} and \eqref{main-eq:2}. Finally, if the split proportions are fixed so that $n_1\asymp n_2\asymp n_3\asymp n$, then the selection-stage cost is of order $\log L/n$. Hence
		\begin{align*}
			\EE\|f_0-\check{f}_{\ell^*}\|_{L_2(P_X)}^2 \le C\left(\inf_{\ell\in[L]}\left\{\fA_\ell+\fL_\ell\right\} + \frac{\log L}{n}\right),
		\end{align*}
		where $C$ does not depend on $n$ or $L$. This proves the theorem.
	\end{proof}
	
	\subsection{Proof of Corollary \ref{main-cor:selection(region-wise)}}
	\label{subsec:supp-proof-cor-5-10}
	
	\begin{proof}[Proof of Corollary \ref{main-cor:selection(region-wise)}]
		Throughout the proof, $C$ denotes a positive constant that may change from line to line and is independent of $n$ and $L$. By Theorem~\ref{main-thm:selection(global)},
		\begin{align*}
			\EE\|f_0-\check{f}_{\ell^*}\|_{L_2(P_X)}^2 \le C\left(\inf_{\ell\in[L]}\left(\fA_\ell+\fL_\ell\right) + \frac{\log L}{n}\right).
		\end{align*}
		It remains to upper bound the approximation term $\inf_{f\in\cF_\ell^*}\|f_0-f\|_{L_2(P_X)}^2$ for each fixed $\ell\in[L]$.
		
		Fix $\ell\in[L]$, and let $f\in\cF_\ell^*$ be induced by some gating function $g\in\cG_\ell$. By definition,
		\begin{align*}
			f(x)=\sum_{r=1}^{\MR}g_r(x)f_r^*(x).
		\end{align*}
		Since the oracle region assignment is $e_r$ on $\cX_r$, we may write
		\begin{align*}
			f_0(x)-f(x) = \sum_{r=1}^{\MR}(f_0(x)-f_r^*(x))\one_{\cX_r}(x) + \sum_{r=1}^{\MR}f_r^*(x)\left(\one_{\cX_r}(x)-g_r(x)\right).
		\end{align*}
		Therefore, using $(a+b)^2\le C(a^2+b^2)$,
		\begin{align*}
			\|f_0-f\|_{L_2(P_X)}^2 \le C(T_1+T_2),
		\end{align*}
		where
		\begin{align*}
			T_1 := \left\|\sum_{r=1}^{\MR}(f_0-f_r^*)\one_{\cX_r}\right\|_{L_2(P_X)}^2, \qquad T_2 := \left\|\sum_{r=1}^{\MR}f_r^*(\one_{\cX_r}-g_r)\right\|_{L_2(P_X)}^2.
		\end{align*}
		Since $\{\cX_r\}_{r=1}^{\MR}$ form a partition of $\cX$,
		\begin{align*}
			T_1=\sum_{r=1}^{\MR}\|f_0-f_r^*\|_{\cX_r}^2.
		\end{align*}
		Next, for $x\in\cX_r$, by $|f_j^*(x)| \le A$ in Assumption \ref{main-ass:bounded-experts} and $\sum_{j=1}^{\MR}g_j(x)=1$,
		\begin{align*}
			\left|\sum_{j=1}^{\MR}f_j^*(x)\left(\one_{\cX_j}(x)-g_j(x)\right)\right| \le A(1-g_r(x)) + A\sum_{j \ne r}g_j(x) = 2A(1-g_r(x)).
		\end{align*}
		Hence
		\begin{align*}
			T_2 = \sum_{r=1}^{\MR}\left\|\sum_{j=1}^{\MR}f_j^*(\one_{\cX_j}-g_j)\right\|_{\cX_r}^2 \le C\sum_{r=1}^{\MR}\|1-g_r\|_{\cX_r}^2 \le C\sum_{r=1}^{\MR}\|(e_r-g(x))\one_{\cX_r}(x)\|_{L_2(P_X)}^2.
		\end{align*}
		Combining the two bounds gives
		\begin{align*}
			\|f_0-f\|_{L_2(P_X)}^2 \le C\left\{\sum_{r=1}^{\MR}\|f_0-f_r^*\|_{\cX_r}^2 + \sum_{r=1}^{\MR}\|(e_r-g(x))\one_{\cX_r}(x)\|_{L_2(P_X)}^2\right\}.
		\end{align*}
		Taking the infimum over all $f\in\cF_\ell^*$, equivalently over the associated $g \in \cG_\ell$, yields
		\begin{align*}
			\inf_{f\in\cF_\ell^*}\|f_0-f\|_{L_2(P_X)}^2 \le C\left\{\sum_{r=1}^{\MR}\|f_0-f_r^*\|_{\cX_r}^2 + \inf_{g\in\cG_\ell}\sum_{r=1}^{\MR}\|(e_r-g(x))\one_{\cX_r}(x)\|_{L_2(P_X)}^2\right\}.
		\end{align*}
		Substituting this bound into Theorem \ref{main-thm:selection(global)}, and absorbing constants, gives
		\begin{align*}
			\EE\|f_0-\check{f}_{\ell^*}\|_{L_2(P_X)}^2 \le C\left[\sum_{r=1}^{\MR}\|f_0-f_r^*\|^2_{\cX_r} + \inf_{\ell \in [L]}\left\{\inf_{g\in\cG_\ell}\sum_{r=1}^{\MR}\|(e_r-g(x))\one_{\cX_r}(x)\|_{L_2(P_X)}^2 + \fL_{\ell}\right\} + \frac{\log L}{n}\right].
		\end{align*}
		This proves the corollary.
	\end{proof}
	
	\section{Proof of Section \ref{main-sec:shared}}
	\label{sec:supp-proof-section-6}
	
	\subsection{Proof of Proposition \ref{main-prop:procedure-shared}}
	\label{subsec:supp-proof-prop-6-1}
	
	\begin{proof}[Proof of Proposition \ref{main-prop:procedure-shared}]
		The upper bound follows directly from Assumption~\ref{main-ass:learner-separation} in the main text. The final comparison statement is also immediate once the lower bound for the pure-routed mechanism is established. We therefore focus on this lower bound.
		
		Let $N_j=\sum_{i=1}^n\one_{\{X_i\in\cX_j\}}$ be the number of observations in the $j$-th region. Conditional on $(N_1,N_2)=(m_1,m_2)$, the observations falling in $\cX_j$, after relabelling, are i.i.d. from the conditional distribution $P_X(\cdot\mid X\in\cX_j)$, with the same conditional noise distribution as in the definition of $\cR_j^{\cM}(h;m,\sigma)$. Thus, conditional on the regional counts, each regional learning rule is trained with a fixed sample size.
		
		For a given $f_0\in\cA_{\mathrm{sh}}$, write $f_0=f_S+\one_{\cX_1}f_{R,1}+\one_{\cX_2}f_{R,2}$. Recall that $H_j=f_S+f_{R,j}$ on $\cX_j$. By the definition of the pure-routed training rule and the disjointness of the two regions,
		\begin{align*}
			\left\|\hat{f}_{0,\mathrm{rt}}-f_0\right\|_{L_2(P_X)}^2 = \sum_{j=1}^2p_j\left\|\cM_j(\cD_{m_j,j}^{H_j})-H_j\right\|_j^2.
		\end{align*}
		Taking conditional expectation given $(N_1,N_2)=(m_1,m_2)$, we obtain
		\begin{align*}
			\EE_{f_0,\sigma}\left[\left\|\hat{f}_{0,\mathrm{rt}}-f_0\right\|_{L_2(P_X)}^2 \mid N_1=m_1,N_2=m_2\right]
			&= \sum_{j=1}^2p_j \EE_{f_0,\sigma}\left[\left\|\cM_j(\cD_{m_j,j}^{H_j})-H_j\right\|_j^2 \mid N_1=m_1,N_2=m_2\right] \\
			&= \sum_{j=1}^2p_j\cR_j^{\cM}(H_j;m_j,\sigma) \\
			&\ge \max_{j=1,2}p_j\cR_j^{\cM}(H_j;m_j,\sigma).
		\end{align*}
		Without loss of generality, suppose
		\begin{align*}
			p_1\tilde{\varrho}_1(np_1)=p_1\tilde{\varrho}_1(np_1)\vee p_2\tilde{\varrho}_2(np_2).
		\end{align*}
		Choose a branch function $H_1^\dagger \in \cB_1$ satisfying the branchwise lower bound $\cR_1^{\cM}(H_1^\dagger;m_1,\sigma) \ge c \tilde{\varrho}_1(m)$ for all $m$ in the relevant range $m \asymp np_1$. By our setting, $H_1^\dagger$ has the form $f_S^\dagger|_{\cX_1}+f_{R,1}^\dagger$. take $f_{R,2}^\dagger$ be an arbitrary function in $\cF_2$ and let $f_0^\dagger=f_S^\dagger+\one_{\cX_1}f_{R,1}^\dagger+\one_{\cX_2}f_{R,2}^\dagger$. Therefore, for every fixed $(m_1,m_2)$ with $m_1 \asymp np_1$,
		\begin{align}
			\EE_{f_0^\dagger,\sigma}\left[\left\|\hat{f}_{0,\mathrm{rt}}-f_0^\dagger\right\|_{L_2(P_X)}^2 \mid N_1=m_1,N_2=m_2\right] \ge p_1\cR_1^{\cM}(H_1^\dagger;m_1,\sigma) \ge cp_1\tilde{\varrho}_1(m_1). \label{eq:pure_lower_epsilon}
		\end{align}
		Let $A_n:=\{np_j/2\le N_j\le 2np_j,\ j=1,2\}$, and denote $p_{\min}:=\min\{p_1,p_2\}$. By a Chernoff bound, $\PP(A_n^c) \le 2\exp\{-c'np_{\min}\}$ for some constant $c'>0$. On $A_n$, the regularity condition in Assumption~\ref{main-ass:learner-separation} gives $\tilde{\varrho}_1(N_1)\asymp\tilde{\varrho}_1(np_1)$. Thus, for all sufficiently large $n$ such that $\PP(A_n)$ is bounded below by an absolute positive constant,
		\begin{align*}
			\sup_{f_0\in\cA_{\mathrm{sh}}}\EE_{f_0,\sigma}\left\|\hat{f}_{0,\mathrm{rt}}-f_0\right\|_{L_2(P_X)}^2 \ge \EE_{f_0^\dagger,\sigma}\left\|\hat{f}_{0,\mathrm{rt}}-f_0^\dagger\right\|_{L_2(P_X)}^2 \ge cp_1\EE\left[\tilde{\varrho}_1(N_1)\one_{A_n}\right] \ge cp_1\tilde{\varrho}_1(np_1).
		\end{align*}
		Since $p_1\tilde{\varrho}_1(np_1)$ is the larger of the two branch lower bounds by construction, this proves
		\begin{align*}
			\sup_{f_0 \in \cA_{\mathrm{sh}}}\EE_{f_0,\sigma}\left\|\hat{f}_{0,\mathrm{rt}}-f_0\right\|_{L_2(P_X)}^2 \ge c\{p_1\tilde\varrho_1(np_1) \vee p_2\tilde\varrho_2(np_2)\}.
		\end{align*}
		If the second branch gives the maximum, the same argument is applied with $j=2$. This completes the proof.
	\end{proof}
	
	\subsection{A Sieve Justification for the Shared Upper Bound}
	\label{subsec:supp-sieve-justification-shared-upper-bound}
	
	Recall that Assumption~\ref{main-ass:learner-separation} imposes the joint-estimation upper bound
	\begin{align*}
		\sup_{f_0\in\cA_{\mathrm{sh}}}\EE_{f_0,\sigma}\left\|\hat{f}_{0,\mathrm{sh}}-f_0\right\|^2_{L_2(P_X)} \le C\left[\varrho_{S}(n)+p_1\varrho_{1}(np_1)+p_2\varrho_2(np_2)\right].
	\end{align*}
	This subsection gives a sufficient condition under standard sieve least squares.
	
	Assume that, for the shared class $\cF_S$ and the residual classes $\cF_j$, there exist sieves $\cF_S(k_S)$ and $\cF_j(k_j)$ such that
	\begin{align*}
		A_S(k_S):=\sup_{s\in\cF_S}\inf_{u\in\cF_S(k_S)}\|s-u\|_{L_2(P_X)}^2,
	\end{align*}
	and
	\begin{align*}
		A_j(k_j):=\sup_{r_j\in\cF_j}\inf_{v_j\in\cF_j(k_j)}\|r_j-v_j\|_j^2,\qquad j=1,2.
	\end{align*}
	For the component classes, define the corresponding sieve oracle rates by
	\begin{align*}
		\varrho_S(m):=\inf_{k_S}\left\{A_S(k_S)+\frac{\sigma^2 k_S}{m}\right\},
	\end{align*}
	and
	\begin{align*}
		\varrho_j(m):=\inf_{k_j}\left\{A_j(k_j)+\frac{\sigma^2 k_j}{m}\right\},\qquad j=1,2.
	\end{align*}
	These are the usual approximation--estimation tradeoffs for sieve least-squares estimators: the first term is the approximation error of the sieve, and the second term is the variance cost of estimating $k$ coefficients from $m$ observations. Any additional logarithmic factors incurred by a particular sieve procedure can be absorbed into the definitions of $\varrho_S$ and $\varrho_j$.
	
	For the shared representation, we combine the three sieve learning rules through least squares over the additive sieve class
	\begin{align*}
		\cS_{\mathrm{sh}}(k_S,k_1,k_2) = \left\{s+\one_{\cX_1}r_1+\one_{\cX_2}r_2: s\in\cF_S(k_S),\ r_j\in\cF_j(k_j)\right\}.
	\end{align*}
	We impose the following regularity condition on least squares over this additive sieve.
	
	\begin{assumption}[Compatible additive complexity]\label{ass:supp-compatible-complexity}
		Assume that the least-squares estimator $\hat{h}_{k_S,k_1,k_2}$ over $\cS_{\mathrm{sh}}(k_S,k_1,k_2)$ satisfies the oracle inequality
		\begin{align*}
			\EE_{f_0,\sigma}\|\hat{h}_{k_S,k_1,k_2}-f_0\|_{L_2(P_X)}^2 \le C\left[\inf_{h\in\cS_{\mathrm{sh}}(k_S,k_1,k_2)}\|h-f_0\|_{L_2(P_X)}^2 + \frac{\sigma^2(k_S+k_1+k_2)}{n}\right].
		\end{align*}
	\end{assumption}
	The following proposition shows that this standard sieve oracle inequality is sufficient for the shared upper bound used in Assumption \ref{main-ass:learner-separation}.
	
	\begin{proposition}[Joint shared upper bound]\label{prop:supp-joint-shared-upper}
		Assume the known-partition model
		\begin{align*}
			f_0(x)=f_S(x)+\one_{\cX_1}(x)f_{R,1}(x)+\one_{\cX_2}(x)f_{R,2}(x),
		\end{align*}
		where $f_S\in\cF_S$ and $f_{R,j}\in\cF_j$. Suppose Assumption~\ref{ass:supp-compatible-complexity} holds. Then there exists a joint shared estimator
		\begin{align*}
			\hat{f}_{0,\mathrm{sh}}(x)=\hat{f}_S(x)+\one_{\cX_1}(x)\hat{f}_{R,1}(x)+\one_{\cX_2}(x)\hat{f}_{R,2}(x)
		\end{align*}
		obtained by least squares over a suitable additive sieve, such that
		\begin{align*}
			\sup_{f_0\in\cA_{\mathrm{sh}}}\EE_{f_0,\sigma}\|\hat{f}_{0,\mathrm{sh}}-f_0\|_{L_2(P_X)}^2 \le C\left\{\varrho_S(n)+p_1\varrho_1(np_1)+p_2\varrho_2(np_2)\right\}.
		\end{align*}
	\end{proposition}
	
	\begin{proof}[Proof of Proposition~\ref{prop:supp-joint-shared-upper}]
		Fix $k_S,k_1,k_2$. For any $f_S\in\cF_S$ and $f_{R,j}\in\cF_j$, choose $s_{k_S}\in\cF_S(k_S)$ and $r_{j,k_j}\in\cF_j(k_j)$ such that
		\begin{align*}
			\|s_{k_S}-f_S\|_{L_2(P_X)}^2\le 2A_S(k_S),
		\end{align*}
		and
		\begin{align*}
			\|r_{j,k_j}-f_{R,j}\|_j^2\le 2A_j(k_j).
		\end{align*}
		Define
		\begin{align*}
			h_k(x)=s_{k_S}(x)+\one_{\cX_1}(x)r_{1,k_1}(x)+\one_{\cX_2}(x)r_{2,k_2}(x).
		\end{align*}
		Then $h_k \in \cS_{\mathrm{sh}}(k_S,k_1,k_2)$. Using $(a+b+c)^2\le 3(a^2+b^2+c^2)$, we obtain
		\begin{align*}
			\|h_k-f_0\|_{L_2(P_X)}^2
			&\le C\|s_{k_S}-f_S\|_{L_2(P_X)}^2 + Cp_1\|r_{1,k_1}-f_{R,1}\|_1^2 + Cp_2\|r_{2,k_2}-f_{R,2}\|_2^2 \\
			&\le C\left\{A_S(k_S)+p_1A_1(k_1)+p_2A_2(k_2)\right\}.
		\end{align*}
		By the oracle inequality in Assumption~\ref{ass:supp-compatible-complexity},
		\begin{align*}
			\EE_{f_0,\sigma}\|\hat{h}_{k_S,k_1,k_2}-f_0\|_{L_2(P_X)}^2 \le C\left\{A_S(k_S)+p_1A_1(k_1)+p_2A_2(k_2) + \frac{\sigma^2(k_S+k_1+k_2)}{n}\right\}.
		\end{align*}
		Because $k_j/n=p_jk_j/(np_j)$, thus
		\begin{align*}
			\frac{k_S+k_1+k_2}{n} = \frac{k_S}{n} + p_1\frac{k_1}{np_1} + p_2\frac{k_2}{np_2}.
		\end{align*}
		Taking the infimum over $k_S,k_1,k_2$, we get
		\begin{align*}
			\EE_{f_0,\sigma}\|\hat{f}_{0,\mathrm{sh}}-f_0\|_{L_2(P_X)}^2 \le C\left[\inf_{k_S}\left\{A_S(k_S)+\frac{\sigma^2k_S}{n}\right\} + \sum_{j=1}^2p_j\inf_{k_j}\left\{A_j(k_j)+\frac{\sigma^2k_j}{np_j}\right\}\right].
		\end{align*}
		Choosing sieve dimension $(k_S^*,k_1^*,k_2^*)$ attaining the preceding infima, or within a fixed constant factor if the infima are not attained, and set
		\begin{align*}
			\hat{f}_{0,\mathrm{sh}} := \hat{h}_{k_S^*,k_1^*,k_2^*}.
		\end{align*}
		By the definition of $\varrho_S$ and $\varrho_j$, this gives
		\begin{align*}
			\EE_{f_0,\sigma}\|\hat{f}_{0,\mathrm{sh}}-f_0\|_{L_2(P_X)}^2 \le C\left\{\varrho_S(n)+p_1\varrho_1(np_1)+p_2\varrho_2(np_2)\right\}.
		\end{align*}
		Since the bound is uniform in $f_0 \in \cA_{\mathrm{sh}}$, taking the supremum over $f_0$ yields
		\begin{align*}
			\sup_{f_0\in\cA_{\mathrm{sh}}}\EE_{f_0,\sigma}\|\hat{f}_{0,\mathrm{sh}}-f_0\|_{L_2(P_X)}^2 \le C\left\{\varrho_S(n)+p_1\varrho_1(np_1)+p_2\varrho_2(np_2)\right\}.
		\end{align*}
		This proves the result.
	\end{proof}
	
	\subsection{Proof of Example \ref{main-ex:fourier-x-sin-cos}}
	\label{subsec:supp-proof-ex-5}
	
	\begin{lemma}[Risk lower bound for the Fourier-sieve learner]\label{lem:supp-fourier-sieve-lower}
		Let $X\sim\Unif[0,1]$, and suppose
		\begin{align*}
			Y=h(X)+\xi, \qquad \EE(\xi\mid X)=0, \qquad \operatorname{Var}(\xi\mid X)\ge c_\xi\sigma^2.
		\end{align*}
		Let $\cS_K=\operatorname{span}\{1,\sin(2\pi kt),\cos(2\pi kt):1\le k\le K\}$, and let $\hat{h}_K$ be the Fourier orthogonal-series estimator over $\cS_K$. Then
		\begin{align*}
			\EE\|\hat{h}_K-h\|_{L_2[0,1]}^2 \ge c\left\{\|h-\Pi_Kh\|_{L_2[0,1]}^2 + \frac{\sigma^2K}{m}\right\},
		\end{align*}
		where $\Pi_K$ is the $L_2[0,1]$-projection onto $\cS_K$.
	\end{lemma}
	
	\begin{proof}[Proof of Lemma~\ref{lem:supp-fourier-sieve-lower}]
		Write $\hat{h}_K-h=(\hat{h}_K-\Pi_Kh)+(\Pi_Kh-h)$. Since $\hat{h}_K-\Pi_Kh\in\cS_K$ and $\Pi_Kh-h\perp\cS_K$, the Pythagorean identity gives
		\begin{align*}
			\|\hat{h}_K-h\|_{L_2[0,1]}^2 = \|\hat{h}_K-\Pi_Kh\|_{L_2[0,1]}^2 + \|h-\Pi_Kh\|_{L_2[0,1]}^2.
		\end{align*}
		It remains to lower-bound the first term. Let $\{\phi_\ell\}_{\ell=1}^{d_K}$ be an orthonormal basis of $\cS_K$, where $d_K=2K+1$. Then
		\begin{align*}
			\Pi_K h(t)=\sum_{\ell=1}^{d_K}\theta_\ell\phi_\ell(t), \qquad \theta_\ell=\EE\{Y\phi_\ell(X)\},
		\end{align*}
		where the identity follows from $\EE(\xi \mid X) = 0$. The orthogonal-series estimator is
		\begin{align*}
			\hat{h}_K(t)=\sum_{\ell=1}^{d_K}\widehat{\theta}_\ell\phi_\ell(t), \qquad \widehat{\theta}_\ell=\frac{1}{m}\sum_{i=1}^mY_i\phi_\ell(X_i).
		\end{align*}
		Therefore,
		\begin{align*}
			\|\hat{h}_K-\Pi_Kh\|^2 = \sum_{\ell=1}^{d_K}(\widehat{\theta}_{\ell}-\theta_{\ell})^2.
		\end{align*}
		For each $\ell$,
		\begin{align*}
			\EE(\widehat{\theta}_\ell-\theta_\ell)^2 = \frac{1}{m}\operatorname{Var}\{Y\phi_\ell(X)\} \ge \frac{1}{m}\EE\{\operatorname{Var}(Y\phi_\ell(X)\mid X)\}.
		\end{align*}
		Since $Y=h(X)+\xi$ and $\operatorname{Var}(\xi\mid X)\ge c_\xi\sigma^2$,
		\begin{align*}
			\operatorname{Var}(Y\phi_\ell(X)\mid X) = \phi_\ell(X)^2\operatorname{Var}(\xi\mid X) \ge c_\xi\sigma^2\phi_\ell(X)^2.
		\end{align*}
		Thus, using $X \sim \Unif[0,1]$ and the orthonormality of $\phi_\ell$,
		\begin{align*}
			\EE(\widehat{\theta}_\ell-\theta_\ell)^2 \ge \frac{c_\xi\sigma^2}{m}\EE\phi_\ell(X)^2 = \frac{c_\xi\sigma^2}{m}.
		\end{align*}
		Summing over $\ell=1,\ldots,d_K$ gives
		\begin{align*}
			\EE\|\hat{h}_K-\Pi_Kh\|_{L_2[0,1]}^2 = \sum_{\ell=1}^{d_K}\EE(\widehat{\theta}_\ell-\theta_\ell)^2 \ge c\frac{\sigma^2K}{m},
		\end{align*}
		Because $d_K = 2K + 1$. Combining this variance lower bound with the approximation term proves the claim.
	\end{proof}
	
	\begin{proof}[Proof of Example \ref{main-ex:fourier-x-sin-cos}]
		We first verify the shared upper bound. Consider the finite-dimensional additive class
		\begin{align*}
			\cA_{\mathrm{sh},K_0} = \left\{\alpha x+\one_{\cX_1}(x)r_1(t_1(x))+\one_{\cX_2}(x)r_2(t_2(x)): \alpha\in\RR,\ r_j\in\cS_{K_0}\right\}.
		\end{align*}
		By construction, the true regression function belongs to $\cA_{\mathrm{sh},K_0}$. The class $\cA_{\mathrm{sh},K_0}$ is finite-dimensional with fixed dimension $d_{\mathrm{sh}}$. Moreover, the basis functions defining $\cA_{\mathrm{sh},K_0}$ are linearly independent under the uniform design. Indeed, if
		\begin{align*}
			\alpha x+\one_{\cX_1}(x)r_1(t_1(x))+\one_{\cX_2}(x)r_2(t_2(x))=0 \quad \text{a.e.},
		\end{align*}
		with $r_j\in\cS_{K_0}$, then on each region $r_j$ must coincide with an affine function of the local coordinate $t_j$. A finite trigonometric polynomial on $[0,1]$ is periodic at the endpoints, and therefore cannot coincide with a nonconstant affine function. Hence $\alpha=0$, and then $r_1=r_2=0$. Thus the population Gram matrix is positive definite. Since its dimension is fixed, its smallest eigenvalue is bounded away from zero by a constant. Consequently, ordinary least squares over $\cA_{\mathrm{sh},K_0}$ satisfies
		\begin{align}
			\EE\|\hat{f}_{0,\mathrm{sh}}-f_0\|_{L_2(P_X)}^2 \lesssim \frac{\sigma^2 d_{\mathrm{sh}}}{n} \lesssim \frac{\sigma^2}{n}. \label{eq:6.5.1}
		\end{align}
		We next derive the Fourier-sieve risk for the pure-routed branch. It is enough to work on one region in the local coordinate $t\in[0,1]$. Since the noise is Gaussian, the variance condition in Lemma~\ref{lem:supp-fourier-sieve-lower} holds with $c_{\xi}=1$. Thus, by Lemma~\ref{lem:supp-fourier-sieve-lower}, it remains to control the approximation error term.
		
		In the local coordinate, the branch target has the form
		\begin{align*}
			H_j(t) = \gamma\left(c_j+\frac{t}{2}\right) + \sum_{k=1}^{K_0}\left(a_{j,k}\sin(2k\pi t)+b_{j,k}\cos(2k\pi t)\right),
		\end{align*}
		where $c_1=0$ and $c_2=1/2$. For $K\ge K_0$, the constant term and the first sine-cosine component belong to $\cS_K$. Hence the approximation error is determined by the ramp component $t$.
		
		The Fourier expansion of $t$ on $[0,1]$ is
		\begin{align*}
			t=\frac{1}{2}-\frac{1}{\pi}\sum_{k=1}^{\infty}\frac{\sin(2\pi kt)}{k} \quad \text{in } L_2[0,1].
		\end{align*}
		Using the orthonormal sine basis $\{\sqrt{2}\sin(2\pi kt)\}_{k\ge 1}$, the corresponding Fourier coefficient is
		\begin{align*}
			\left\langle t,\sqrt{2}\sin(2\pi kt)\right\rangle = -\frac{1}{\sqrt{2}\pi k}.
		\end{align*}
		Therefore,
		\begin{align*}
			\inf_{g\in\cS_K}\|t-g\|_{L_2[0,1]}^2 = \frac{1}{2\pi^2}\sum_{k>K}\frac{1}{k^2} \asymp K^{-1}.
		\end{align*}
		Since $|\gamma|\ge\gamma_0$, the ramp contribution in $H_j$ yields, for all $K \ge K_0$,
		\begin{align*}
			\inf_{g\in\cS_K}\|H_j-g\|_{L_2[0,1]}^2 \ge c\gamma_0^2K^{-1}.
		\end{align*}
		For the finitely many choices $K < K_0$, the same lower bound is harmless after changing the constant: the approximation error is bounded below by a positive constant depending only on $K_0$ and $\gamma_0$, which dominates $K^{-1}$ over this finite range. Hence, up to a universal constant depending on $K_0$,
		\begin{align*}
			\inf_{g\in\cS_K}\|H_j-g\|_{L_2[0,1]}^2 \ge c\gamma_0^2K^{-1}, \qquad K \ge 1.
		\end{align*}
		Thus, by Lemma~\ref{lem:supp-fourier-sieve-lower}, the oracle risk of the Fourier-orthogonal-series learner on $H_j$ is bounded below by
		\begin{align*}
			\inf_{K\ge 1}\left\{c\gamma_0^2K^{-1}+c'\frac{\sigma^2K}{m}\right\}.
		\end{align*}
		The minimum is achieved at $K \asymp \gamma_0\sqrt{m}/\sigma$, and the corresponding value is of order $\gamma_0\sigma/\sqrt{m}$. Therefore,
		\begin{align*}
			\cR_j^{\cM}(H_j;m,\sigma) \gtrsim \frac{\gamma_0\sigma}{\sqrt{m}}.
		\end{align*}
		Finally, the pure-routed estimator has disjoint regional loss decomposition:
		\begin{align*}
			\|\hat{f}_{0,\mathrm{rt}}-f_0\|_{L_2(P_X)}^2 = p_1\|\hat{H}_1-H_1\|_1^2+p_2\|\hat{H}_2-H_2\|_2^2.
		\end{align*}
		Let $N_j=\sum_{i=1}^n\one\{X_i\in\cX_j\}$ be the number of training observations in region $\cX_j$. Taking expectations first conditional on $(N_1,N_2)=(m_1,m_2)$ and using the same concentration-event argument as in the proof of Proposition \ref{main-prop:procedure-shared}, with $p_1=p_2=1/2$, gives
		\begin{align*}
			\EE\|\hat{f}_{0,\mathrm{rt}}-f_0\|_{L_2(P_X)}^2 \gtrsim \sum_{j=1}^2p_j\frac{\gamma_0\sigma}{\sqrt{np_j}} \asymp \frac{\gamma_0\sigma}{\sqrt{n}}.
		\end{align*}
		Combining this lower bound with the shared upper bound \eqref{eq:6.5.1} proves the stated procedure-specific rate gap.
	\end{proof}
	
	\section{Additional Materials}
	\label{sec:supp-additional-materials}
	
	\subsection{Order of Empirical Risk Minimization under Sub-Gaussian Noise}
	\label{subsec:supp-erm-subgaussian-noise}
	
	We now examine what the traditional empirical risk minimization (ERM) theory would yield for the gating-estimation problem considered in the main text. The purpose of this discussion is not to advocate ERM computationally, since the corresponding optimization problem is generally non-convex. Rather, it is to show that even if a global empirical minimizer were available, the standard uniform-concentration argument leads to a square-root complexity term.
	
	Recall that in Algorithm \ref{main-alg:general}, we use the first sample $Z^{(1)}$ to burn-in, the second sample $Z^{(2)}$ to calibrate the scale, and the last sample to do online aggregation and estimation for $\theta$. For $t=n_1+n_2,\ldots,n-1$, the one-step-ahead dense-softmax predictor is denoted as
	\begin{align}
		\check f^{\,t}(x;\theta)
		=
		\sum_{\ell=1}^{\MS}\hat f^{\,t}_{S,\ell}(x)
		+
		\sum_{m=1}^{\MR}g_m^{\soft}(x;\theta)\hat f^{\,t}_{R,m}(x),
		\qquad \theta\in\Theta_{\MR}. \label{eq:online-check-f}
	\end{align}
	Consequently, for $\MR\ge 2$, define the fixed-expert dense softmax class
	\begin{align*}
		\cF_{\MR}^{\mathrm{on}}:=\left\{\{\check f^{\,t}(\cdot;\theta)\}_{t=n_1+n_2}^{n-1}:\theta\in\Theta_{\MR}\right\}.
	\end{align*}
	For comparison with the oracle terms in Theorem \ref{main-thm:softmax} of the main text, also define the corresponding limit-expert predictor
	\begin{align*}
		f^*(x;\theta) = \sum_{\ell=1}^{\MS}f_{S,\ell}^*(x) + \sum_{m=1}^{\MR}g_m^{\soft}(x;\theta)f_{R,m}^*(x).
	\end{align*}
	
	The direct online ERM benchmark minimizes the one-step-ahead squared loss over the aggregation block:
	\begin{align}
		R_{n_3}(\theta)
		:=
		\frac{1}{n_3}\sum_{t=n_1+n_2}^{n-1}
		\{Y_{t+1}-\check f^{\,t}(X_{t+1};\theta)\}^2. \label{eq:online-erm-empirical-risk}
	\end{align}
	Let
	\begin{align*}
		\hat\theta_{\rm ERM}
		\in
		\argmin_{\theta\in\Theta_{\MR}}R_{n_3}(\theta).
	\end{align*}
	This is a theoretical oracle benchmark, because solving this requires global optimization of a non-convex softmax objective. The resulting online sequence and its time-averaged version are
	\begin{align*}
		\check f^{\,t}_{\rm ERM}(x)=\check f^{\,t}(x;\hat\theta_{\rm ERM}), \qquad
		\bar f_{\rm ERM}(x)=\frac{1}{n_3}\sum_{t=n_1+n_2}^{n-1}\check f^{\,t}(x;\hat\theta_{\rm ERM}).
	\end{align*}
	Equivalently,
	\begin{align*}
		\bar f_{\rm ERM}(x)
		=
		\sum_{\ell=1}^{\MS}\bar f^{n_3}_{S,\ell}(x)
		+
		\sum_{m=1}^{\MR}g_m^{\soft}(x;\hat\theta_{\rm ERM})\bar f^{n_3}_{R,m}(x),
	\end{align*}
	where
	\begin{align*}
		\bar f^{n_3}_{S,\ell}=\frac{1}{n_3}\sum_{t=n_1+n_2}^{n-1}\hat f^{\,t}_{S,\ell}, \qquad
		\bar f^{n_3}_{R,m}=\frac{1}{n_3}\sum_{t=n_1+n_2}^{n-1}\hat f^{\,t}_{R,m}.
	\end{align*}
	Moreover, define
	\begin{align}
		\mathcal R_{n_3}(\theta)
		:=
		\frac{1}{n_3}\sum_{t=n_1+n_2}^{n-1}
		\left\|\check f^{\,t}(\cdot;\theta)-f_0\right\|_{L_2(P_X)}^2. \label{eq:online-erm-prequential-risk}
	\end{align}
	This is the random prequential prediction risk of the fixed gate $\theta$ along the online expert trajectory. The next proposition gives the ERM benchmark.
	
	\begin{proposition}[ERM under sub-Gaussian noise]\label{prop:supp-erm-subgaussian}
		Suppose that $Y_i=f_0(X_i)+\varepsilon_i$ and $\EE(\varepsilon_i\mid X_i)=0$, where $\varepsilon_i$ is conditionally sub-Gaussian with parameter $\sigma_{\varepsilon}$, and the observations are independent. Assume Assumptions~\ref{main-ass:compact-theta} and \ref{main-ass:bounded-experts} in the main text hold. Then there exist constants $C,c>0$, depending only on the constants in Assumptions~\ref{main-ass:compact-theta} and \ref{main-ass:bounded-experts}, $\sigma_{\varepsilon}$, and the fixed number of shared experts $\MS$, such that, conditional on the pretrained experts, with probability at least $1-n_3^{-c}$,
		\begin{align}
			\cR_{n_3}(\hat\theta_{\rm ERM})
			-
			\inf_{\theta\in\Theta_{\MR}}\cR_{n_3}(\theta)
			\le
			C\sqrt{\frac{\MR d \log(\MR d n_3)}{n_3}}. \label{eq:online-erm-excess}
		\end{align}
		Moreover, if Assumption~\ref{main-ass:evolving-experts} in the main text also holds, then
		\begin{align}
			\EE\cR_{n_3}(\hat\theta_{\rm ERM})
			\le
			C\left\{
			\fA_{\soft}(\MS,\MR)
			+\MR\bar a_{n_3}^2
			+\sqrt{\frac{\MR d \log(\MR d n_3)}{n_3}}
			\right\}, \label{eq:online-erm-pop-oracle}
		\end{align}
		where we recall that
		\begin{align*}
			\fA_{\soft}(\MS,\MR)
			:=
			\inf_{\theta\in\Theta_{\MR}}
			\left\|f_0-f^*(\cdot;\theta)\right\|_{L_2(P_X)}^2,\qquad
			\bar{a}_{n_3}^2=\frac{1}{n_3}\sum_{t=n_1+n_2}^{n-1}a_t^2.
		\end{align*}
		Moreover, the same bound holds for the time-averaged output $\bar{f}_{\rm ERM}$ as
		\begin{align}
			\EE\|\bar f_{\rm ERM}-f_0\|_{L_2(P_X)}^2
			\le
			C\left\{
			\fA_{\soft}(\MS,\MR)
			+\MR\bar a_{n_3}^2
			+\sqrt{\frac{\MR d \log(\MR d n_3)}{n_3}}
			\right\}. \label{eq:online-erm-avg-output}
		\end{align}
	\end{proposition}
	
	\begin{proof}[Proof of Proposition~\ref{prop:supp-erm-subgaussian}]
		Let $\mathcal F_t$ be the natural online filtration generated by $Z_1,\ldots,Z_t$ and by the algorithmic randomness used up to time $t$. By the online update order, $\check f^{\,t}(\cdot;\theta)$ is $\mathcal F_t$-measurable for each $\theta$. Moreover, by Assumption~\ref{main-ass:bounded-experts} and the identity $\sum_{m=1}^{\MR}g_m^{\soft}(x;\theta)=1$, there exists a constant $B_0 > 0$, depending on the uniform bounds and on the fixed number $\MS$, such that
		\begin{align}
			|f_0(x)| \vee |\check{f}^t(x;\theta)| \le B_0, \qquad x\in\cX, \quad \theta\in\Theta_{\MR}\quad t\geq n_1+n_2. \label{eqn:uniform_bound}
		\end{align}
		For fixed $\theta$, define
		\begin{align*}
			V_{t+1}(\theta)
			:=
			Y_{t+1}-\check f^{\,t}(X_{t+1};\theta)
			=
			\varepsilon_{t+1}+f_0(X_{t+1})-\check f^{\,t}(X_{t+1};\theta).
		\end{align*}
		By \eqref{eqn:uniform_bound}, the conditional mean component $f_0(X_{t+1})-\check f^{\,t}(X_{t+1};\theta)$ is uniformly bounded. Therefore, conditional on $\mathcal F_t$, the centered variable
		\begin{align*}
			V_{t+1}(\theta)-\EE[V_{t+1}(\theta)\mid\mathcal F_t]
		\end{align*}
		is sub-Gaussian with a parameter depending only on $A$, $\MS$, and $\sigma_\varepsilon$. By Lemma~\ref{lem:supp-concentration}, for every fixed $\theta$ and every $0<u\le1$,
		\begin{align}
			\PP\left(\left|R_{n_3}(\theta)-\bar R_{n_3}(\theta)\right|\ge u\right) \le 2\exp\{-c_1n_3u^2\}, \label{eqn:fixed-theta-loss-conc}
		\end{align}
		where
		\begin{align*}
			\bar R_{n_3}(\theta)
			:=
			\frac{1}{n_3}\sum_{t=n_1+n_2}^{n-1}
			\EE\left[\{Y_{t+1}-\check f^{\,t}(X_{t+1};\theta)\}^2\mid\mathcal F_t\right].
		\end{align*}
		Since $\EE(\varepsilon_{t+1}\mid X_{t+1})=0$ and $\varepsilon_{t+1}$ is independent of $\mathcal F_t$, we have
		\begin{align}
			\bar R_{n_3}(\theta) = \cR_{n_3}(\theta)+\EE\varepsilon_1^2. \label{eqn:noise-cancel-online-erm}
		\end{align}
		
		We next pass from a fixed $\theta$ to the full parameter space. Let $\Theta(\eta)=\{\theta^{(1)},\ldots,\theta^{(S_\eta)}\}$ be an $\eta$-net of $\Theta_{\MR}$. By Assumption~\ref{main-ass:compact-theta}, the parameter space $\Theta_{\MR}$ is contained in a Euclidean box of dimension at most $\MR(d+1)$ and side length of order $\log \MR$. Hence a standard grid construction gives, for $0 < \eta \le 1$,
		\begin{align*}
			\log S_\eta \le C\MR d \log\left(\frac{C\MR d\log \MR}{\eta}\right).
		\end{align*}
		Therefore, by the union bound applying to \eqref{eqn:fixed-theta-loss-conc}, we have
		\begin{align}
			\PP\left(\max_{1\le s\le S_\eta}|R_{n_3}(\theta^{(s)})-\bar R_{n_3}(\theta^{(s)})|\ge u\right)
			\le
			2\exp\left\{C\MR d\log\left(\frac{C\MR d\log\MR}{\eta}\right)-c_1n_3u^2\right\}. \label{eqn:net-union-bound-online-erm}
		\end{align}
		
		It remains to control the discretization error.
		Using the Lipschitz property of the softmax map and \eqref{eqn:uniform_bound}, we have
		\begin{align*}
			|\check f^{\,t}(x;\theta)-\check f^{\,t}(x;\theta')|
			\le
			L_{\MR}\|\theta-\theta'\|_2,
			\qquad
			L_{\MR}\le C\MR^{3/2}\sqrt d,
		\end{align*}
		for all $t$, $x$, and $\theta,\theta'\in\Theta_{\MR}$. For each $\theta\in\Theta_{\MR}$, choose $\pi(\theta)\in\Theta(\eta)$ such that $\|\theta-\pi(\theta)\|_2 \le \eta$. Then by boundness, \eqref{eqn:noise-cancel-online-erm} and the fact that $a^2-b^2\leq (a-b)(a+b)$, we have
		\begin{align*}
			|\bar R_{n_3}(\theta)-\bar R_{n_3}(\pi(\theta))|
			\le
			CL_{\MR}\eta.
		\end{align*}
		Similarly,
		\begin{align*}
			|R_{n_3}(\theta)-R_{n_3}(\pi(\theta))|
			\le
			L_{\MR}\eta\cdot
			\frac{1}{n_3}\sum_{t=n_1+n_2}^{n-1}
			\{4B_0+2|\varepsilon_{t+1}|\}.
		\end{align*}
		Since $\varepsilon_i$ is sub-Gaussian, $|\varepsilon_i|$ is sub-exponential. Hence Bernstein's inequality implies that there exist constants $c_\varepsilon,C_\varepsilon>0$ such that
		\begin{align*}
			\frac{1}{n_3}\sum_{t=n_1+n_2}^{n-1}|\varepsilon_{t+1}| \le C_\varepsilon
		\end{align*}
		with probability at least $1-\exp\{-c_\varepsilon n_3\}$. On this event,
		\begin{align*}
			\sup_{\theta\in\Theta_{\MR}}|R_{n_3}(\theta)-\bar{R}_{n_3}(\theta)| \le \max_{1\le s\le S_\eta}|R_{n_3}(\theta^{(s)})-\bar{R}_{n_3}(\theta^{(s)})| + CL_{\MR}\eta.
		\end{align*}
		Set
		\begin{align*}
			u = K_0\sqrt{\frac{\MR d \log(\MR d n_3)}{n_3}},\qquad \eta = \min\left\{1,\frac{u}{L_{\MR}}\right\},
		\end{align*}
		into \eqref{eqn:net-union-bound-online-erm}, where $K_0>0$ is sufficiently large. If $u>1$, the desired bound is trivial after enlarging the constant, since both $f_0$ and $\check{f}^{t}(\cdot;\theta)$ are uniformly bounded. Thus assume $u\le 1$. Since $L_{\MR}$ is polynomial in $\MR$ up to logarithmic factors,
		\begin{align*}
			\log\left(\frac{C\MR d\log \MR}{\eta}\right) \lesssim \log(\MR d n_3).
		\end{align*}
		Choosing $K_0$ sufficiently large and combining the preceding bounds gives
		\begin{align*}
			\sup_{\theta\in\Theta_{\MR}}|R_{n_3}(\theta)-\bar{R}_{n_3}(\theta)| \le Cu
		\end{align*}
		with probability at least $1-n_3^{-c}$.
		
		Finally, let $\theta_{n_3}^*\in\argmin_{\theta\in\Theta_{\MR}}\bar R_{n_3}(\theta)$. By empirical optimality, $R_{n_3}(\hat{\theta}_{\mathrm{ERM}}) \le R_{n_3}(\theta_{n_3}^*)$. Therefore,
		\begin{align*}
			\bar R_{n_3}(\hat\theta_{\rm ERM})-\bar R_{n_3}(\theta_{n_3}^*)
			\le
			2\sup_{\theta\in\Theta_{\MR}}|R_{n_3}(\theta)-\bar R_{n_3}(\theta)|
			\le
			C\sqrt{\frac{\MR d \log(\MR d n_3)}{n_3}}.
		\end{align*}
		By \eqref{eqn:noise-cancel-online-erm},
		this is exactly \eqref{eq:online-erm-empirical-risk}.
		
		We now translate the random evolving-expert oracle into the population-expert oracle. Fix any $\theta\in\Theta_{\MR}$. By Assumption~\ref{main-ass:evolving-experts} and the same triangle-inequality calculation as in the proof of Theorem \ref{main-thm:softmax},
		\begin{align*}
			\EE\sup_{\theta\in\Theta_{\MR}}\left\|\check{f}^t(\cdot;\theta)-f^*(\cdot;\theta)\right\|_{L_2(P_X)}^2 \le C\MR a_t^2.
		\end{align*}
		Averaging over $t$ gives
		\begin{align*}
			\EE\sup_{\theta\in\Theta_{\MR}}\frac{1}{n_3}\sum_{t=n_1+n_2}^{n-1}\|\check f^{\,t}(\cdot;\theta)-f^*(\cdot;\theta)\|_{L_2(P_X)}^2 \le C\MR\bar a_{n_3}^2.
		\end{align*}
		Consequently,
		\begin{align*}
			\EE\inf_{\theta\in\Theta_{\MR}}\cR_{n_3}(\theta) \le C\left\{\fA_{\soft}(\MS,\MR)+\MR\bar a_{n_3}^2\right\}.
		\end{align*}
		Combining this bound with \eqref{eq:online-erm-excess}, and integrating the high-probability statement over the failure event using the uniform boundedness of the risks, yields \eqref{eq:online-erm-pop-oracle}.
		
		Finally, by convexity of the squared norm,
		\begin{align*}
			\left\|\bar f_{\rm ERM}-f_0\right\|_{L_2(P_X)}^2
			&=
			\left\|\frac{1}{n_3}\sum_{t=n_1+n_2}^{n-1}(\check f^{\,t}(\cdot;\hat\theta_{\rm ERM})-f_0)\right\|_{L_2(P_X)}^2 \\
			&\le
			\frac{1}{n_3}\sum_{t=n_1+n_2}^{n-1}
			\left\|\check f^{\,t}(\cdot;\hat\theta_{\rm ERM})-f_0\right\|_{L_2(P_X)}^2.
		\end{align*}
		Taking expectation and applying \eqref{eq:online-erm-pop-oracle} proves \eqref{eq:online-erm-avg-output}.
	\end{proof}
	
	\begin{lemma}[Concentration for predictable squared sub-Gaussian losses]\label{lem:supp-concentration}
		Let $\{\mathcal G_i\}_{i=0}^{T}$ be a filtration. Suppose $V_i$, $i=1,\ldots,T$, is adapted to $\mathcal G_i$ and, conditional on $\mathcal G_{i-1}$, the centered variable $V_i-\EE(V_i\mid\mathcal G_{i-1})$ is sub-Gaussian with parameter $\sigma_V$, while $|\EE(V_i\mid\mathcal G_{i-1})|\le m_0$ almost surely. Then there exists $c>0$, depending only on $\sigma_V$ and $m_0$, such that for every $u>0$,
		\begin{align*}
			\PP\left(\left|\frac{1}{T}\sum_{i=1}^{T}\{V_i^2-\EE(V_i^2\mid\mathcal G_{i-1})\}\right|\ge u\right) \le 2\exp\{-cT\min(u^2,u)\}.
		\end{align*}
		In particular, for $0<u\le1$, the right-hand side is bounded by $2\exp\{-cTu^2\}$.
	\end{lemma}
	
	\begin{proof}[Proof of Lemma~\ref{lem:supp-concentration}]
		Write $\mu_i:=\EE(V_i\mid\mathcal G_{i-1})$ and $W_i:=V_i-\mu_i$. Then $\{W_i\}$ is a martingale-difference sequence and $W_i$ is conditionally sub-Gaussian with parameter $\sigma_V$. Conditional sub-Gaussianity implies that $W_i^2-\EE(W_i^2\mid\mathcal G_{i-1})$ is conditionally sub-exponential: there exist constants $\nu,\alpha>0$, depending only on $\sigma_V$, such that
		\begin{align*}
			\EE\left[\exp\{\lambda(W_i^2-\EE(W_i^2\mid\mathcal G_{i-1}))\}\mid\mathcal G_{i-1}\right] \le \exp\left\{\frac{\nu^2\lambda^2}{2}\right\}, \qquad |\lambda|\le\frac{1}{\alpha}.
		\end{align*}
		By the Bernstein inequality for martingale differences with conditional sub-exponential increments,
		\begin{align}
			\PP\left(\left|\frac{1}{n_3}\sum_{i=1}^{T}\{W_i^2-\EE(W_i^2\mid\mathcal G_{i-1})\}\right|\ge s\right) \le 2\exp\{-c_1T\min(s^2,s)\}. \label{eq:subexp-md-conc}
		\end{align}
		Next,
		\begin{align*}
			V_i^2-\EE(V_i^2\mid\mathcal G_{i-1}) = W_i^2-\EE(W_i^2\mid\mathcal G_{i-1})+2\mu_iW_i.
		\end{align*}
		Since $|\mu_i|\le m_0$ and $W_i$ is a conditionally sub-Gaussian martingale difference, Azuma--Hoeffding's inequality for sub-Gaussian martingale differences gives
		\begin{align}
			\PP\left(\left|\frac{1}{n_3}\sum_{i=1}^{T}\mu_iW_i\right|\ge r\right) \le 2\exp\{-c_2Tr^2\}, \label{eq:linear-md-conc}
		\end{align}
		where $c_2$ depends only on $\sigma_V$ and $m_0$. Combining \eqref{eq:subexp-md-conc} and \eqref{eq:linear-md-conc} with $s=u/2$ and $r=u/4$ gives
		\begin{align*}
			\PP\left(\left|\frac{1}{n_3}\sum_{i=1}^{T}\{V_i^2-\EE(V_i^2\mid\mathcal G_{i-1})\}\right|\ge u\right) \le 2\exp\{-cT\min(u^2,u)\}
		\end{align*}
		after decreasing $c$ if necessary. The final statement follows because $\min(u^2,u)=u^2$ for $0<u\le1$.
	\end{proof}

	\subsection{Details for Empirical Parameter Estimation Using Cross-Entropy}
	\label{subsec:supp-cross-entropy-fixed-mr}
	
	This subsection gives the details for the fixed-$\MR$ implementation described in the main text. Throughout this subsection, $\MR$ and $\MS$ are treated as fixed, and constants denoted by $C_{\MR}$ may depend on $\MR$, $\MS$, the fixed-dimensional softmax class, the compact parameter set, and the uniform boundedness constants, but not on the sample sizes.
	
	The population projection in Algorithm \ref{main-alg:general} selects a representative element from a time-averaged discretized MoE class. We now replace this population $L_2(P_X)$ projection by an empirical cross-entropy minimization. Notice that the online aggregate generally cannot be written as a MoE predictor with one single averaged gate, because the aggregation weights and the experts both vary over time. We therefore fit a single softmax gate to the whole sequence of online aggregated gates.
	
	Specifically, we choose the same $\eta$-net $\Theta(\eta)$ of $\Theta_{\MR}$ as in step (a) of Algorithm \ref{main-alg:general}. Then we randomly split the data into four subsamples, $Z^{(1)} = \{(X_i,Y_i)\}_{i=1}^{n_1}$, $Z^{(2)} = \{(X_i,Y_i)\}_{i=n_1+1}^{n_1+n_2}$, $Z^{(3)} = \{(X_i,Y_i)\}_{i=n_1+n_2+1}^{n_1+n_2+n_3}$ and $Z^{(4)} = \{(X_i,Y_i)\}_{i=n_1+n_2+n_3+1}^{n}$, and we denote $n_4=n-n_1-n_2-n_3$. For notational simplicity, denote $n_{\rm on}=n_1+n_2+n_3$. We use steps (c), (d) and (e) in Algorithm \ref{main-alg:general} on $Z^{(1)}$, $Z^{(2)}$ and $Z^{(3)}$ respectively, to obtain weights $\{\hat w_{s,t}\}_{s=1}^{S_\eta}$ over a discretized softmax net $\Theta(\eta)=\{\theta^{(1)},\ldots,\theta^{(S_\eta)}\}$ for $n_1+n_2\leq t\leq n_{\rm on}-1$, and online aggregated estimators $\{\tilde{f}^t(x)\}_{t=n_1+n_2}^{n_{\rm on}-1}$. Define the time-dependent aggregated gate
	\begin{align*}
		\tilde g_{m,t}(x)
		:=\sum_{s=1}^{S_\eta}\hat w_{s,t}g_m^{\soft}(x;\theta^{(s)}),
		\qquad m\in[\MR],\quad n_1+n_2\leq t\leq n_{\rm on}-1.
	\end{align*}
	For any $\theta\in\Theta_{\MR}$, recall the corresponding one-step-ahead candidate predictor
	\begin{align*}
		\check f^{\,t}(x;\theta)
		=\sum_{\ell=1}^{\MS}\hat f_{S,\ell}^{\,t}(x)
		+\sum_{m=1}^{\MR}g_m^{\soft}(x;\theta)\hat f_{R,m}^{\,t}(x).
	\end{align*}
	The time-averaged online aggregate and the time-averaged candidate predictor are
	\begin{align}
		\tilde f_0(x):=\frac{1}{n_3}\sum_{t=n_1+n_2}^{n_{\rm on}-1}\tilde f^{\,t}(x),
		\qquad
		\bar f^{n_3}(x;\theta):=\frac{1}{n_3}\sum_{t=n_1+n_2}^{n_{\rm on}-1}\check f^{\,t}(x;\theta). \label{eq:online-time-averages}
	\end{align}
	Equivalently,
	\begin{align}
		\bar f^{n_3}(x;\theta)
		=\sum_{\ell=1}^{\MS}\bar f^{n_3}_{S,\ell}(x)
		+\sum_{m=1}^{\MR}g_m^{\soft}(x;\theta)\bar f^{n_3}_{R,m}(x), \label{eq:bar-f-theta-averaged-experts}
	\end{align}
	where
	\begin{align*}
		\bar f^{n_3}_{S,\ell}:=\frac{1}{n_3}\sum_{t=n_1+n_2}^{n_{\rm on}-1}\hat f^{\,t}_{S,\ell},
		\qquad
		\bar f^{n_3}_{R,m}:=\frac{1}{n_3}\sum_{t=n_1+n_2}^{n_{\rm on}-1}\hat f^{\,t}_{R,m}.
	\end{align*}
	The online aggregate $\tilde f_0$, however, does not generally admit the same representation with a single averaged gate, because
	\begin{align*}
		\frac{1}{n_3}\sum_{t=n_1+n_2}^{n_{\rm on}-1}\tilde g_{m,t}(x)\hat f^{\,t}_{R,m}(x)
		\ne
		\left\{\frac{1}{n_3}\sum_{t=n_1+n_2}^{n_{\rm on}-1}\tilde g_{m,t}(x)\right\}
		\left\{\frac{1}{n_3}\sum_{t=n_1+n_2}^{n_{\rm on}-1}\hat f^{\,t}_{R,m}(x)\right\}
	\end{align*}
	in general.
	
	Define the empirical cross-entropy objective
	\begin{align}
		\cL_{\rm soft}^{\rm on}(\theta;Z^{(4)})
		:=-\frac{1}{n_4n_3}\sum_{i=n_{\rm on}+1}^{n}\sum_{t=n_1+n_2}^{n_{\rm on}-1}\sum_{m=1}^{\MR}
		\tilde g_{m,t}(X_i)\log g_m^{\soft}(X_i;\theta). \label{eq:time-expanded-ce}
	\end{align}
	Thus cross-entropy implementation estimates the softmax weights by solving
	\begin{align}
		\hat\theta_{\rm ce}\in\argmin_{\theta\in\Theta_{\MR}}\cL_{\rm soft}^{\rm on}(\theta;Z^{(4)}). \label{eq:theta-ce-online}
	\end{align}
	One can verify that with softmax gating, the above objective is convex in $\theta$, thus this optimization problem is numerically feasible. The final estimator is
	\begin{align}
		\hat f_{0,\rm ce}(x):=\bar f^{n_3}(x;\hat\theta_{\rm ce})
		=\frac{1}{n_3}\sum_{t=n_1+n_2}^{n_{\rm on}-1}\check f^{\,t}(x;\hat\theta_{\rm ce}). \label{eq:final-ce-online}
	\end{align}
	The full algorithm is summarized in Algorithm \ref{alg:cross-entropy}.
	
	\begin{algorithm}[htbp]
		\caption{Discritized Aggregation with Empirical Parameter Estimation using Cross-Entropy}\label{alg:cross-entropy}
		\textbf{Input:} Pre-trained shared experts $\{\hat{f}_{S,\ell}^N\}_{\ell=1}^{\MS}$ and routed experts $\{\hat{f}_{R,m}^N\}_{m=1}^{\MR}$; softmax-router family $g^{\soft}(\cdot;\theta)$ with $\theta \in \Theta_{\MR}$; data $\{(X_i,Y_i)\}_{i=1}^n$.\\
		\textbf{Output:} Estimated gating parameter $\hat{\theta}$ and final estimator $\hat{f}_0(x)$.
		\vspace{5pt}
		
		\begin{enumerate}[label=(\alph*)]
			\item Choose an $\eta$-net of the normalized bounded gating parameter space $\Theta_{\MR}$, denoted by $\Theta(\eta)$, similarly as in Algorithm \ref{main-alg:general}. Write $\Theta(\eta) = \{\theta^{(1)}, \ldots, \theta^{(S_\eta)}\}$ where $S_\eta = |\Theta(\eta)|$, which induce the discretized routers $g(\cdot;\theta^{(1)}), \ldots, g(\cdot;\theta^{(S_\eta)})$.
			\item Randomly partition the data into three subsamples, $Z^{(1)} = \{(X_i,Y_i)\}_{i=1}^{n_1}$, $Z^{(2)} = \{(X_i,Y_i)\}_{i=n_1+1}^{n_1+n_2}$, $Z^{(3)} = \{(X_i,Y_i)\}_{i=n_1+n_2+1}^{n_{\rm on}}$ and $Z^{(4)} = \{(X_i,Y_i)\}_{i=n_{\rm on}+1}^{n_{\rm on}+n_4}$ where $n = n_{\rm on}+n_4$. The first three subsamples are used for burn-in, calibration and aggregation, and the last subsample is used for empirical parameter estimation.
			\item
			Use steps (c), (d) and (e) in Algorithm \ref{main-alg:general} on $Z^{(1)}$, $Z^{(2)}$ and $Z^{(3)}$ respectively, to obtain weights $\{\hat{w}_{s,t}\}_{s=1}^{S_{\eta}}$ for $n_1+n_2\leq t\leq n_{\rm on}-1$. Denote
			\begin{align*}
				\tilde g_{m,t}(x)
				:=\sum_{s=1}^{S_\eta}\hat w_{s,t}g_m^{\soft}(x;\theta^{(s)}),
				\qquad m\in[\MR],\quad n_1+n_2\leq t\leq n_{\rm on}-1.
			\end{align*}
			\item Solve the convex optimization $\hat{\theta}_{\mathrm{ce}}= \argmin_{\theta\in\Theta_{\MR}}\cL^{\rm on}_{\soft}(\theta;Z^{(4)})$, where
			\begin{align*}
				\cL_{\rm soft}^{\rm on}(\theta;Z^{(4)})
				:=-\frac{1}{n_4n_3}\sum_{i=n_{\rm on}+1}^{n}\sum_{t=n_1+n_2}^{n_{\rm on}-1}\sum_{m=1}^{\MR}
				\tilde g_{m,t}(X_i)\log g_m^{\soft}(X_i;\theta)..
			\end{align*}
			Obtain the resulting the final output
			$\hat{f}_{0,\rm ce}(x)=\bar f^{n_3}(x;\hat\theta_{\rm ce})$ as in \eqref{eq:final-ce-online}.
		\end{enumerate}
	\end{algorithm}
	
	For the cross-entropy objective, its population counterpart is
	\begin{align*}
		\cL_{\rm soft}^{\rm on}(\theta)
		:=-\frac{1}{n_3}\sum_{t=n_1+n_2}^{n_{\rm on}-1}\int_{\cX}\sum_{m=1}^{\MR}
		\tilde g_{m,t}(x)\log g_m^{\soft}(x;\theta)dP_X(x).
	\end{align*}
	The compactness, boundedness, and error-distribution assumptions used in the main online oracle inequality are kept in force here and are not repeated. We impose one additional fixed-$\MR$ compatibility condition to convert time-averaged weight approximation into time-averaged function approximation.
	
	\begin{assumption}[Additional fixed-$\MR$ cross-entropy regularity]\label{ass:supp-ce-fixed-mr}
		Let
		\begin{align*}
			\cG_{\MR}^{\soft} := \left\{x\mapsto g^{\soft}(x;\theta) = \left(g_1^{\soft}(x;\theta),\ldots,g_{\MR}^{\soft}(x;\theta)\right)^\top: \theta\in\Theta_{\MR}\right\}
		\end{align*}
		be the softmax gating class. Define the online softmax-difference class $\Delta_{\MR}^{\rm on}$ as the collection of all time-indexed vector-valued functions $\delta=(\delta_t)_{t=n_1+n_2}^{n_{\rm on}-1}$, where
		\begin{align*}
			\delta_t(x)=\sum_{s=1}^{S_\eta}p_{s,t}g^{\soft}(x;\theta^{(s)})-g^{\soft}(x;\theta),
			\qquad p_{s,t}\in[0,1/2],
			\qquad \sum_{s=1}^{S_\eta}p_{s,t}=1,
			\qquad \theta\in\Theta_{\MR}.
		\end{align*}
		The routed experts satisfy the online compatibility condition that, for some $c_{\MR}>0$ and all $\delta\in\Delta_{\MR}^{\rm on}$,
		\begin{align}
			\frac{1}{n_3}\sum_{t=n_1+n_2}^{n_{\rm on}-1}
			\left\|
			\sum_{m=1}^{\MR}\delta_{m,t}\hat f^{\,t}_{R,m}
			\right\|_{L_2(P_X)}^2
			\ge
			c_{\MR}
			\frac{1}{n_3}\sum_{t=n_1+n_2}^{n_{\rm on}-1}
			\sum_{m=1}^{\MR}\|\delta_{m,t}\|_{L_2(P_X)}^2. \label{eq:online-compat}
		\end{align}
	\end{assumption}
	
	\begin{remark}
		Assumption~\ref{ass:supp-ce-fixed-mr} can be viewed as a positive minimum singular-value condition for a sequence of bounded linear operator. For each time $t$, let
		\begin{align*}
			T_{t}:\{L_2(P_X)\}^{\MR}\to L_2(P_X), \qquad
			T_t(\delta_t)(x):=\sum_{m=1}^{\MR}\delta_{m,t}(x)\hat f^{\,t}_{R,m}(x).
		\end{align*}
		Since the routed experts are uniformly bounded, $T_{t}$ is bounded. The condition can be viewed as requiring the average operator $(T_t)_{t=n_1+n_2}^{n_{\rm on}-1}$ to have a positive minimum ``singular value'' on the class of softmax differences generated by the online aggregated gates and a single candidate softmax gate. It prevents different gate discrepancies from being invisible after multiplication by the evolving routed experts.
		
		As a concrete intuition, suppose $\hat{f}_{R,m}^t(x)=\hat{f}_{R,m}(x):=\max\{a_m^\top x+b_m,0\}$ are nondegenerate ReLU experts and does not vary with respect to $t$. Then all the above defined operators $T_t$ are identical to a specific $T_{\hat f}$. The kink hyperplanes of these ReLU functions make $T_{\hat f}$ injective on analytic coefficient functions under suitable nondegeneracy conditions. If one restricts to a truncated class $\Delta_{\MR,N}^{\rm on}$ with $L\le N$, then compactness of the parameter space and continuity of the softmax map imply a positive constant $c_{\MR,N}$, provided that the normalized class is separated from the kernel of $T_{\hat{f}}$. More generally, the same conclusion holds whenever the normalized softmax-difference class is contained in a compact subset of analytic functions that does not intersect the kernel.
	\end{remark}
	
	We first record the elementary comparison between cross-entropy and squared weight distance that will be used in both empirical and population forms.
	
	\begin{lemma}[KL--$L_2$ comparison for softmax weights]\label{lem:supp-kl-l2-weight-comparison}
		Let $p=(p_1,\ldots,p_{\MR})$ and $q=(q_1,\ldots,q_{\MR})$ be probability vectors. If $q_m \ge \underline{s}_{\MR} > 0$ for all $m \in [\MR]$, then
		\begin{align*}
			\frac{1}{2}\|p-q\|_2^2 \le \sum_{m=1}^{\MR}p_m\log\frac{p_m}{q_m} \le \frac{1}{\underline{s}_{\MR}}\|p-q\|_2^2.
		\end{align*}
		Consequently, the same inequalities hold after averaging over any empirical distribution or integrating with respect to $P_X$.
	\end{lemma}
	
	\begin{proof}[Proof of Lemma~\ref{lem:supp-kl-l2-weight-comparison}]
		The lower bound follows from Pinsker's inequality:
		\begin{align*}
			\sum_{m=1}^{\MR}p_m\log\frac{p_m}{q_m} \ge \frac{1}{2}\|p-q\|_1^2 \ge \frac{1}{2}\|p-q\|_2^2.
		\end{align*}
		For the upper bound, use the standard inequality
		\begin{align*}
			\sum_{m=1}^{\MR}p_m\log\frac{p_m}{q_m} \le \sum_{m=1}^{\MR}\frac{(p_m-q_m)^2}{q_m} \le \frac{1}{\underline{s}_{\MR}}\|p-q\|_2^2.
		\end{align*}
		Averaging or integrating these pointwise inequalities gives the empirical and population versions.
	\end{proof}
	
	Define the empirical and population time-expanded squared weight errors by
	\begin{align*}
		R_{n_4,w}^{\rm on}(\theta)
		&:=\frac{1}{n_4n_3}\sum_{i=n_{\rm on}+1}^{n}\sum_{t=n_1+n_2}^{n_{\rm on}-1}\sum_{m=1}^{\MR}
		\left\{\tilde g_{m,t}(X_i)-g_m^{\soft}(X_i;\theta)\right\}^2, \\
		R_w^{\rm on}(\theta)
		&:=\frac{1}{n_3}\sum_{t=n_1+n_2}^{n_{\rm on}-1}\sum_{m=1}^{\MR}
		\|\tilde g_{m,t}-g_m^{\soft}(\cdot;\theta)\|_{L_2(P_X)}^2.
	\end{align*}
	
	\begin{lemma}[Time-averaged weight estimation via cross-entropy]\label{lem:supp-ce-weight-estimator}
		Suppose the assumptions of the main softmax oracle inequality hold and assume $n_1 \asymp n_2 \asymp n_3\asymp n_4 \asymp n$. Let $\mathfrak{F}_{\rm on}$ denote the $\sigma$-field generated by by the data and algorithmic randomness used before the cross-entropy step. Then, conditional on $\mathfrak{F}_{\rm on}$, for fixed $\MR$,
		\begin{align}
			\EE\left[R_w^{\rm on}(\hat\theta_{\rm ce})\mid\mathfrak{F}_{\rm on}\right]
			\le
			C_{\MR}
			\left\{
			\inf_{\theta\in\Theta_{\MR}}R_w^{\rm on}(\theta)
			+\frac{\log n}{n}
			\right\}. \label{eq:online-weight-ce-bound}
		\end{align}
	\end{lemma}
	
	\begin{proof}[Proof of Lemma~\ref{lem:supp-ce-weight-estimator}]
		Throughout the proof, condition on $\mathfrak{F}_{\rm on}$. Under this conditioning, $\tilde g_{m,t}$, $\hat f^{\,t}_{S,\ell}$, and $\hat f^{\,t}_{R,m}$ are fixed functions, while the covariates in the third split sample remain i.i.d. from $P_X$.
		Constants denoted by $C_{\MR}$ may change from line to line and may absorb additional fixed constants depending only on the fixed-$\MR$ softmax class, the compact parameter set, and the sample-splitting proportions.
		
		Since the compactness assumption is in force and $\MR$ is fixed, the softmax probabilities are uniformly bounded away from zero on $\Theta_{\MR}$: there exists $\underline{s}_{\MR}>0$ such that
		\begin{align*}
			g_m^{\soft}(x;\theta)\ge \underline{s}_{\MR}, \qquad x\in\cX, \qquad \theta\in\Theta_{\MR}, \qquad m\in[\MR].
		\end{align*}
		Define the empirical time-expanded KL loss
		\begin{align*}
			D_{n_4}^{\rm on}(\tilde g,g_\theta)
			:=\frac{1}{n_4n_3}\sum_{i=n_{\rm on}+1}^{n}\sum_{t=n_1+n_2}^{n_{\rm on}-1}\sum_{m=1}^{\MR}
			\tilde g_{m,t}(X_i)
			\log\frac{\tilde g_{m,t}(X_i)}{g_m^{\soft}(X_i;\theta)},
		\end{align*}
		with the convention $0 \log 0 = 0$. The empirical cross-entropy objective in \eqref{eq:time-expanded-ce} differs from $D_{n_4}^{\rm on}(\tilde g,g_\theta)$ by a term independent of $\theta$. Therefore, by the definition of $\hat{\theta}_{\mathrm{ce}}$,
		\begin{align*}
			D_{n_4}^{\rm on}(\tilde g,g_{\hat\theta_{\rm ce}})
			\le
			D_{n_4}^{\rm on}(\tilde g,g_\theta).
		\end{align*}
		Applying Lemma~\ref{lem:supp-kl-l2-weight-comparison} to both sides gives the empirical $L_2$ oracle inequality
		\begin{align*}
			R_{n_4,w}^{\rm on}(\hat\theta_{\rm ce})
			\le
			C_{\MR}R_{n_4,w}^{\rm on}(\theta),
			\qquad \theta\in\Theta_{\MR}.
		\end{align*}
		Taking the infimum over $\theta$ yields
		\begin{align*}
			R_{n_4,w}^{\rm on}(\hat\theta_{\rm ce})
			\le
			C_{\MR}\inf_{\theta\in\Theta_{\MR}}R_{n_4,w}^{\rm on}(\theta).
		\end{align*}
		It remains to transfer this empirical inequality to its population version. For $\theta\in\Theta_{\MR}$, write
		\begin{align*}
			\ell_\theta(x):=\frac{1}{n_3}\sum_{t=n_1+n_2}^{n_{\rm on}-1}\sum_{m=1}^{\MR}
			\left(\tilde g_{m,t}(x)-g_m^{\soft}(x;\theta)\right)^2.
		\end{align*}
		Then
		\begin{align*}
			R_w^{\rm on}(\theta)=\EE_{P_X}\ell_\theta(X),
			\qquad
			R_{n_4,w}^{\rm on}(\theta)=\frac{1}{n_4}\sum_{i=n_{\rm on}+1}^{n}\ell_\theta(X_i).
		\end{align*}
		Since both $\tilde g_t(x)$ and $g^{\soft}(x;\theta)$ are probability vectors, $0\le \ell_\theta(x)\le 4$, and hence
		\begin{align*}
			\mathrm{Var}\{\ell_\theta(X)\}
			\le
			\EE\ell_\theta^2(X)
			\le
			4R_w^{\rm on}(\theta).
		\end{align*}
		Bernstein's inequality then gives, for every fixed $\theta$ and every $t>0$,
		\begin{align*}
			\PP\left(|R_{n_4,w}^{\rm on}(\theta)-R_w^{\rm on}(\theta)| \ge \frac{1}{2}R_w^{\rm on}(\theta)+u \;\middle|\; \mathfrak{F}_{\rm on}\right) \le 2\exp\{-c n_4 u\},
		\end{align*}
		where $c>0$ is an absolute constant.
		
		Next, since $\MR$ is fixed, $\Theta_{\MR}$ is finite-dimensional and compact. Moreover, the map $\theta\mapsto \ell_\theta(x)$ is uniformly Lipschitz on $\Theta_{\MR}$; that is, for some constant $L_{\MR} < \infty$,
		\begin{align*}
			\sup_{x\in\cX}|\ell_\theta(x)-\ell_{\theta'}(x)| \le L_{\MR}\|\theta-\theta'\|_2.
		\end{align*}
		Let $\Theta(\eta)=\{\theta^{(1)},\ldots,\theta^{(S_\eta)}\}$ be an $\eta$-net of $\Theta_{\MR}$ satisfying
		\begin{align*}
			\log S_\eta \le C_{\MR}\log\frac{C_{\MR}}{\eta}.
		\end{align*}
		Therefore, by the union bound,
		\begin{align*}
			\PP\left(\max_{1\le j\le S_\eta}\left\{|R_{n_4,w}^{\rm on}(\theta^{(j)})-R_w^{\rm on}(\theta^{(j)})|-\frac{1}{2}R_w^{\rm on}(\theta^{(j)})\right\} \ge u \;\middle|\; \mathfrak{F}_{\rm on}\right)
			\le
			2\exp\left\{C_{\MR}\log\frac{C_{\MR}}{\eta}-c n_4u\right\}.
		\end{align*}
		Take $u=C_{\MR}\log n/n$ with $C_{\MR}$ sufficiently large, and set $\eta=u/(4L_{\MR})$. Then
		\begin{align*}
			\log\frac{C_{\MR}}{\eta} = \log\frac{4C_{\MR}L_{\MR}}{u} \le C_{\MR}\log n.
		\end{align*}
		Since $n_4 \asymp n$, the preceding probability is bounded by $C_{\MR}n^{-2}$ after enlarging $C_{\MR}$ if necessary. Hence, with conditional probability at least $1-C_{\MR}n^{-2}$,
		\begin{align*}
			|R_{n_4,w}^{\rm on}(\theta^{(j)})-R_w^{\rm on}(\theta^{(j)})| \le \frac{1}{2}R_w^{\rm on}(\theta^{(j)})+C_{\MR}\frac{\log n}{n}, \qquad 1\le j\le S_\eta.
		\end{align*}
		For an arbitrary $\theta\in\Theta_{\MR}$, choose $\theta^{(j)}\in\Theta(\eta)$ such that $\|\theta-\theta^{(j)}\|_2\le \eta$. The Lipschitz bound implies
		\begin{align*}
			|R_w^{\rm on}(\theta)-R_w^{\rm on}(\theta^{(j)})| + |R_{n_4,w}^{\rm on}(\theta)-R_{n_4,w}^{\rm on}(\theta^{(j)})| \le 2L_{\MR}\eta \le \frac{1}{2}C_{\MR}\frac{\log n}{n}.
		\end{align*}
		Also $R_w^{\rm on}(\theta^{(j)})\le R_w^{\rm on}(\theta)+L_{\MR}\eta$. Combining these bounds gives the uniform comparison
		\begin{align*}
			|R_{n_4,w}^{\rm on}(\theta)-R_w^{\rm on}(\theta)| \le \frac{1}{2}R_w^{\rm on}(\theta)+C_{\MR}\frac{\log n}{n}, \qquad \forall\theta\in\Theta_{\MR}.
		\end{align*}
		Consequently, on this event,
		\begin{align*}
			R_w^{\rm on}(\theta) \le 2R_{n_4,w}^{\rm on}(\theta)+C_{\MR}\frac{\log n}{n}, \qquad R_{n_4,w}^{\rm on}(\theta) \le \frac{3}{2}R_w^{\rm on}(\theta)+C_{\MR}\frac{\log n}{n},
		\end{align*}
		uniformly over $\theta\in\Theta_{\MR}$. Using these two inequalities together with the empirical oracle inequality, we obtain
		\begin{align*}
			R_w^{\rm on}(\hat\theta_{\rm ce})
			&\le
			2R_{n_4,w}^{\rm on}(\hat\theta_{\rm ce})+C_{\MR}\frac{\log n}{n}\\
			&\le
			C_{\MR}\inf_{\theta\in\Theta_{\MR}}R_{n_4,w}^{\rm on}(\theta)
			+C_{\MR}\frac{\log n}{n}\\
			&\le
			C_{\MR}\inf_{\theta\in\Theta_{\MR}}R_w^{\rm on}(\theta)
			+C_{\MR}\frac{\log n}{n}.
		\end{align*}
		Since both $R_w^{\rm on}$ and $R_{n_4,w}^{\rm on}$ are uniformly bounded, the complement of the high-probability event contributes only $O(n_4^{-3})$ to the conditional expectation. Therefore,
		\begin{align*}
			\EE\left[R_w^{\rm on}(\hat\theta_{\rm ce})\mid\mathfrak{F}_{\rm on}\right]
			\le
			C_{\MR}\left\{
			\inf_{\theta\in\Theta_{\MR}}R_w^{\rm on}(\theta)
			+\frac{\log n}{n}
			\right\}.
		\end{align*}
		This proves the lemma.
	\end{proof}
	
	\begin{proposition}[Function estimation via time-averaged cross-entropy weights]\label{prop:supp-fixed-mr-ce}
		Suppose the assumptions of the main softmax oracle inequality and Assumption~\ref{ass:supp-ce-fixed-mr} hold. Assume that the sample is split so that $n_1\asymp n_2\asymp n_3\asymp n_4\asymp n$. Let $\hat\theta_{\rm ce}$ be obtained from the time-expanded cross-entropy criterion \eqref{eq:theta-ce-online}, and define $\hat f_{0,\rm ce}(x)=\bar f^{n_3}(x;\hat\theta_{\rm ce})$ as in \eqref{eq:final-ce-online}. Then, for fixed $\MR$, with expectation taken over all randomness, including data and algorithmic randomness, we have
		\begin{align}
			\EE \left\|\hat f_{0,\mathrm{ce}}-f_0\right\|_{L_2(P_X)}^2 \le C_{\MR} \left\{ \fA_{\soft}(\MS,\MR) + \bar a_{n_2}^2+\bar a_{n_3}^2 + \frac{\log n}{n} \right\}. \label{eqn:softmax_practical_oracle}
		\end{align}
	\end{proposition}
	
	\begin{proof}[Proof of Proposition~\ref{prop:supp-fixed-mr-ce}]
		Recall the $\sigma$-field $\mathfrak{F}_{\rm on}$ defined in Lemma \ref{lem:supp-ce-weight-estimator}. We first work conditionally on $\mathfrak{F}_{\rm on}$. For each $n_1+n_2\leq t\leq n_{\rm on}-1$, the shared part cancels in the difference between $\tilde f^{\,t}$ and $\check f^{\,t}(\cdot;\hat\theta_{\rm ce})$. Thus
		\begin{align*}
			\tilde f^{\,t}(x)-\check f^{\,t}(x;\hat\theta_{\rm ce}) = \sum_{m=1}^{\MR} \left\{\tilde g_{m,t}(x)-g_m^{\soft}(x;\hat\theta_{\rm ce})\right\} \hat f^{\,t}_{R,m}(x).
		\end{align*}
		By uniform boundedness of the routed experts and Cauchy's inequality,
		\begin{align*}
			|\tilde f^{\,t}(x)-\check f^{\,t}(x;\hat\theta_{\rm ce})|^2 \le C_{\MR} \sum_{m=1}^{\MR} \left|\tilde g_{m,t}(x)-g_m^{\soft}(x;\hat\theta_{\rm ce})\right|^2.
		\end{align*}
		Integrating with respect to $P_X$, averaging over $n_1+n_2\leq t\leq n_{\rm on}-1$, and applying Lemma~\ref{lem:supp-ce-weight-estimator} gives
		\begin{align*}
			\EE\left[\frac{1}{n_3}\sum_{t=n_1+n_2}^{n_{\rm on}-1} \|\tilde f^{\,t}-\check f^{\,t}(\cdot;\hat\theta_{\rm ce})\|_{L_2(P_X)}^2 \middle|\mathfrak{F}_{\rm on}\right] \le C_{\MR}\left\{\inf_{\theta\in\Theta_{\MR}}R_w^{\rm on}(\theta)+\frac{\log n}{n}\right\}.
		\end{align*}
		It remains to relate the weight approximation error to the corresponding prediction approximation error. For any $\theta\in\Theta_{\MR}$, define
		\begin{align*}
			\delta^{\theta}_{m,t}(x):=\tilde g_{m,t}(x)-g_m^{\soft}(x;\theta), \qquad m\in[\MR],\,\, n_1+n_2\leq t\leq n_{\rm on}-1.
		\end{align*}
		Since both $\tilde g_{m,t}(x)$ and $g^{\soft}(x;\theta)$ are probability vectors, $\sum_{m=1}^{\MR}\delta^{\theta}_{m,t}(x)$. Moreover, the ARM construction writes $\tilde g_{m,t}(x)$ as a finite convex combination of softmax weight vectors. For a time $t$, if one aggregation coefficient is larger than $1/2$, we split it by keeping mass $1/2$ on the original softmax vector and assigning the excess mass to the softmax vector with the nearest parameter; by compactness of $\Theta_{\MR}$ and continuity of the softmax map, this replacement can be made arbitrarily close in $L_2(P_X)$ and is absorbed by the closure in the definition of $\Delta_{\MR}^{\rm on}$. Hence $\delta^\theta=(\delta_t^\theta)_{n_1+n_2\leq t\leq n_{\rm on}-1}$ belongs to the softmax-difference class in Assumption \ref{ass:supp-ce-fixed-mr}.
		Consequently,
		\begin{align*}
			c_{\MR}R_w^{\rm on}(\theta)
			&\le \frac{1}{n_3}\sum_{t=n_1+n_2}^{n_{\rm on}-1} \left\|\sum_{m=1}^{\MR} \left(\tilde g_{m,t}-g_m^{\soft}(\cdot;\theta)\right) \hat f^{\,t}_{R,m}\right\|_{L_2(P_X)}^2 \\
			&= \frac{1}{n_3}\sum_{t=n_1+n_2}^{n_{\rm on}-1} \|\tilde f^{\,t}-\check f^{\,t}(\cdot;\theta)\|_{L_2(P_X)}^2.
		\end{align*}
		Taking the infimum over $\theta \in \Theta_{\MR}$ and absorbing $c_{\MR}^{-1}$ into $C_{\MR}$, we obtain
		\begin{align*}
			\inf_{\theta\in\Theta_{\MR}}R_w^{\rm on}(\theta)\le C_{\MR}\inf_{\theta\in\Theta_{\MR}}\frac{1}{n_3}\sum_{t=n_1+n_2}^{n_{\rm on}-1} \|\tilde f^{\,t}-\check f^{\,t}(\cdot;\theta)\|_{L_2(P_X)}^2.
		\end{align*}
		By Jensen's inequality, we have
		\begin{align*}
			\|\tilde f_0-\hat f_{0,\rm ce}\|_{L_2(P_X)}^2
			&= \left\|\frac{1}{n_3}\sum_{t=n_1+n_2}^{n_{\rm on}-1} \left\{\tilde f^{\,t}-\check f^{\,t}(\cdot;\hat\theta_{\rm ce})\right\}\right\|_{L_2(P_X)}^2 \\
			&\le \frac{1}{n_3}\sum_{t=n_1+n_2}^{n_{\rm on}-1} \|\tilde f^{\,t}-\check f^{\,t}(\cdot;\hat\theta_{\rm ce})\|_{L_2(P_X)}^2.
		\end{align*}
		Taking expectations in the preceding conditional bound and using the above bounds, we obtain
		\begin{align*}
			\EE\|\hat{f}_{0,\mathrm{ce}}-\tilde{f}_0\|_{L_2(P_X)}^2 \le C_{\MR}\left( \EE\inf_{\theta\in\Theta_{\MR}}\frac{1}{n_3}\sum_{t=n_1+n_2}^{n_{\rm on}-1} \|\tilde f^{\,t}-\check f^{\,t}(\cdot;\theta)\|_{L_2(P_X)}^2 + \frac{\log n}{n} \right).
		\end{align*}
		To control the projection term on the right-hand side, notice that
		\begin{align*}
			\inf_{\theta\in\Theta_{\MR}} \frac{1}{n_3}\sum_{t=n_1+n_2}^{n_{\rm on}-1} \|\tilde f^{\,t}-\check f^{\,t}(\cdot;\theta)\|_{L_2(P_X)}^2
			&\le \frac{2}{n_3}\sum_{t=n_1+n_2}^{n_{\rm on}-1}\|\tilde f^{\,t}-f_0\|_{L_2(P_X)}^2 \\
			&\quad+\inf_{\theta\in\Theta_{\MR}} \frac{2}{n_3}\sum_{t=n_1+n_2}^{n_{\rm on}-1}\|\check f^{\,t}(\cdot;\theta)-f_0\|_{L_2(P_X)}^2.
		\end{align*}
		Taking expectations and applying the ARM oracle bound for the online aggregated sequence $\{\tilde{f}^{t}\}$, together with the expert approximation bounds used in the proof of Theorem \ref{main-thm:softmax}, we obtain
		\begin{align*}
			\EE\|\hat{f}_{0,\mathrm{ce}}-\tilde{f}_0\|_{L_2(P_X)}^2 \le C_{\MR}\left( \fA_{\soft}(\MS,\MR) + a_{n_2}^2+a_{n_3}^2+\frac{\log n}{n} \right).
		\end{align*}
		Finally, by the squared triangle inequality,
		\begin{align*}
			\|\hat{f}_{0,\mathrm{ce}}-f_0\|_{L_2(P_X)}^2 \le 2\|\hat{f}_{0,\mathrm{ce}}-\tilde{f}_0\|_{L_2(P_X)}^2 + 2\|\tilde{f}_0-f_0\|_{L_2(P_X)}^2,
		\end{align*}
		and by the ARM oracle bound for time-averaged $\tilde{f}_0$, we obtain the stated bound for $\hat{f}_{0,\mathrm{ce}}$.
		
		Here we still need to confirm that under the potential adjustment for $\hat{w}$, the ARM oracle bound for $\tilde{f}_0$ still holds. In fact, for a given time $t$, recall that $\tilde{g}_{m,t}(x)=\sum_{s=1}^{S_{\eta}}\hat{w}_{s,t}g_m^{\soft}(x;\theta^{(s)})$. Without loss of generalization, assume $\hat{w}_{1,t}>1/2$ and $\theta^{(2)}$ is the nearest point to $\theta^{(1)}$ in $\Theta(\eta)$. Then, $\|\theta^{(2)}-\theta^{(1)}\|_2\leq \eta$. We adjust the weight as $\hat{w}_{1,t}':=1/2$, $\hat{w}_{2,t}':=\hat{w}_{1,t}+\hat{w}_{2,t}-1/2$, $\hat{w}_{s,t}':=\hat{w}_{s,t}$ for $s\geq 3$,
		\begin{align*}
			\tilde{g}'_{m,t}(x):=\sum_{s=1}^{S_{\eta}}\hat{w}_{s,t}'g_m^{\soft}(x;\theta^{(s)}),
		\end{align*}
		and
		\begin{align*}
			\tilde{f}^{t'}(x):=\sum_{\ell=1}^{\MS}\hat{f}^{t}_{\ell,S}(x)+\sum_{m=1}^{\MR}\tilde{g}'_{m,t}(x)\hat{f}^{t}_{m,R}(x).
		\end{align*}
		Then, we have
		\begin{align*}
			\tilde{f}^{t'}(x)-\tilde{f}^t(x)=(\hat{w}_{1,t}-1/2)\left(\check{f}^t(x;\theta^{(1)})-\check{f}^t(x;\theta^{(2)})\right).
		\end{align*}
		By the proof of Theorem \ref{main-thm:softmax}, we know that under the same choice of $\eta$,
		\begin{align*}
			\left\|\check{f}^t(x;\theta^{(1)})-\check{f}^t(x;\theta^{(2)})\right\|_{L_2(P_X)}^2\leq \frac{C\MR d}{n}:=\frac{C_{\MR}}{n}.
		\end{align*}
		Thus this adjustment only brings an additional $1/n$ term, which can be absorbed into the original ARM oracle bound for the online aggregated sequence $\{\tilde{f}^t\}$.
	\end{proof}

	\subsection{Specific Upper Bounds for Dense Softmax Gating in Section \ref{main-subsec:regionwise}}
	\label{subsec:supp-dense-softmax-upper-bounds-section-4-2}
	
	\subsubsection{Dense Softmax Upper Bound after Corollary \ref{main-cor:Top1-single}}
	\label{subsubsec:supp-dense-softmax-upper-bound-after-cor-4-5}
	
	Recall that we have a partition of $\cX$ into $\cX=\cup_{r=1}^{\Mzero}\cX_r$, up to $P_X$-null boundary sets, suppose that there exist vectors $\{v_r\in\RR^{d+1}\}_{r=1}^{\Mzero}$, normalized so that $\max_{r\in[\Mzero]}\|v_r\|_2\le 1$, such that, for $\bar{x}=(x^\top,1)^\top$ and each $r\in[\Mzero]$,
	\begin{align*}
		\cX_r=\{x\in\cX:\bar{x}^\top v_r>\bar{x}^\top v_j,\ \forall j\ne r\}.
	\end{align*}
	For each region $\cX_r$, define the pointwise margin function
	\begin{align}
		m_r(x):=v_r^\top\bar{x}-\max_{j\ne r} v_j^\top\bar{x}\ge 0, \qquad x\in\cX_r.
	\end{align}
	We impose the following condition on the distribution of $X$.
	
	\begin{assumption}[Margin regularity for linearly separable regions]\label{ass:supp-margin-regularity}
		Let $P_X$ be the distribution of $X$ and, for each $r$, denote by $G_r$ the random variable $G_r:=m_r(X)$ conditional on $X\in\cX_r$. Suppose that each $G_r$ admits a density $f_{G_r}$ on $[0,\infty)$ that is uniformly bounded:
		\begin{align*}
			\sup_{r\le \Mzero}\sup_{u\ge 0} f_{G_r}(u)\le L<\infty.
		\end{align*}
		Moreover, assume that the normalized parameter space contains a logarithmic inner box: for some constant $C_{\rm in}>2$,
		\begin{align*}
			[-C_{\rm in}\log \MR,C_{\rm in}\log \MR]^{(\MR-1)(d+1)} \subseteq \Theta_{\MR}.
		\end{align*}
	\end{assumption}
	The density condition controls the probability mass near the interfaces between regions. In particular, it implies that the margin variable $m_r(X)$ has at most linear probability mass near zero. This condition is satisfied, for example, when $X$ has a bounded density on a bounded domain and the separated hyperplanes are non-degenerate.
	
	We next record a dense-softmax approximation bound for the hard Top-1 region assignment.
	
	\begin{proposition}\label{prop:supp-linear-no-margin}
		Under the linearly separable partition condition above and Assumption~\ref{ass:supp-margin-regularity}, for $\MR\ge \Mzero$,
		\begin{align*}
			\inf_{\theta\in\Theta_{\MR}} \sum_{r=1}^{\Mzero}\int_{\cX_r}\|g^{\soft}(x;\theta)-e_r\|_2^2 dP_X(x) = O\left(\frac{\Mzero}{\log \MR}\right).
		\end{align*}
		In particular, when $\Mzero$ is fixed, this bound is of order $O(1/\log \MR)$.
	\end{proposition}
	
	\begin{proof}[Proof of Proposition~\ref{prop:supp-linear-no-margin}]
		Since the softmax weights are invariant under a simultaneous relabeling of the routed experts and the corresponding assignment vectors, we may assume, without loss of generality, that the specialist for $\cX_{\Mzero}$ is used as the reference expert. We therefore impose the reference normalization $\theta_{\Mzero}=0$ in this proof. Define
		\begin{align*}
			B_{\cX}'=1\vee\max_{x\in\cX}\|\bar{x}\|_2, \qquad \bar{x}=(x^\top,1)^\top,
		\end{align*}
		and choose
		\begin{align*}
			t:=\frac{C_{\rm in}\log \MR}{4B_{\cX}'}.
		\end{align*}
		For $r=1,\ldots,\Mzero-1$, define $\tilde{\theta}_r:=t(v_r-v_{\Mzero})$ and set $\tilde{\theta}_{\Mzero}=0$. For the dummy experts $j=\Mzero+1,\ldots,\MR$, define $\tilde{\theta}_j=(0,\ldots,0,-C_{\rm in}\log \MR)^\top$. Because $\|v_r\|_2 \le 1$ and $B_{\cX}' \ge 1$, the above choice of $t$ ensures that the constructed parameters lie in the logarithmic inner box. Hence $\tilde{\theta} \in \Theta_{\MR}$.
		
		Fix $r\in[\Mzero]$ and $x\in\cX_r$, and let $g^{\soft}(x;\tilde{\theta})$ denote the corresponding softmax vector. Write
		\begin{align*}
			a_i(x)=v_i^\top\bar{x}, \quad i\in[\Mzero], \qquad a_{(-r)}(x):=\max_{j\ne r} v_j^\top\bar{x}.
		\end{align*}
		Then $m_r(x)=a_r(x)-a_{(-r)}(x)\ge 0$. For the constructed logits $\ell_j(x):=\tilde{\theta}_j^\top\bar{x}$, we have
		\begin{align*}
			\ell_j(x)=t(a_j(x)-a_{\Mzero}(x)),\quad j<\Mzero, \quad \ell_{\Mzero}(x)=0, \quad \ell_j(x)=-C_{\rm in}\log \MR, \quad j>\Mzero.
		\end{align*}
		Therefore
		\begin{align*}
			g_r^{\soft}(x;\tilde{\theta}) \ge \frac{1}{1+(\Mzero-1)\exp\{-tm_r(x)\}+\MR^{-(C_{\rm in}-2)/2}}.
		\end{align*}
		Indeed,
		\begin{align*}
			1+\sum_{j\ne r}\exp\{\ell_j(x)-\ell_r(x)\} \le 1 + \sum_{j\in[\Mzero],j\ne r}\exp\{-tm_r(x)\} + \sum_{j>\Mzero}\exp\{-C_{\rm in}\log \MR+t(a_{\Mzero}(x)-a_r(x))\}.
		\end{align*}
		The first sum is bounded by $(\Mzero-1)\exp\{-tm_r(x)\}$. For the dummy experts, we use
		\begin{align*}
			|a_{\Mzero}(x) - a_r(x)| \le \|v_{\Mzero} - v_r\|_2 \|\bar{x}\|_2 \le 2 B_{\cX}',
		\end{align*}
		and hence
		\begin{align*}
			t(a_{\Mzero}(x) - a_r(x)) \le 2 t B_{\cX}' = \frac{C_{\rm in}}{2} \log \MR.
		\end{align*}
		Thus,
		\begin{align*}
			\sum_{j>\Mzero}\exp\{-C_{\rm in}\log \MR+t(a_{\Mzero}(x)-a_r(x))\} \le (\MR-\Mzero)\exp\{C_{\rm in}\log \MR/2\} \le \MR^{-(C_{\rm in}-2)/2}.
		\end{align*}
		Combining the preceding bounds gives
		\begin{align*}
			1+\sum_{j\ne r}\exp\{\ell_j(x)-\ell_r(x)\} \le 1 + (\Mzero-1)\exp\{-tm_r(x)\} + \MR^{-(C_{\rm in}-2)/2},
		\end{align*}
		and therefore the displayed lower for $g_r^{\soft}(x;\tilde{\theta})$ follows.
		
		It follows that
		\begin{align*}
			1-g_r^{\soft}(x;\tilde{\theta}) \le (\Mzero-1)\exp\{-tm_r(x)\}+\MR^{-(C_{\rm in}-2)/2}.
		\end{align*}
		Since $0 \le g_j^{\soft} \le 1$ and $\sum_j g_j^{\soft} = 1$,
		\begin{align*}
			\|g^{\soft}(x;\tilde{\theta})-e_r\|_2^2 = (1-g_r^{\soft}(x;\tilde{\theta}))^2+\sum_{j\ne r}\{g_j^{\soft}(x;\tilde{\theta})\}^2 \le 2(1-g_r^{\soft}(x;\tilde{\theta})).
		\end{align*}
		Consequently,
		\begin{align*}
			\|g^{\soft}(x;\tilde{\theta})-e_r\|_2^2 \le 2(\Mzero-1)\exp\{-tm_r(x)\}+2\MR^{-(C_{\rm in}-2)/2}.
		\end{align*}
		Integrating over $\cX_r$ and summing over $r$ yields
		\begin{align*}
			\sum_{r=1}^{\Mzero}\int_{\cX_r}\|g^{\soft}(x;\tilde{\theta})-e_r\|_2^2 dP_X(x) \le 2(\Mzero-1)\sum_{r=1}^{\Mzero}\int_{\cX_r} \exp\{-tm_r(x)\} dP_X(x) + 2\Mzero\MR^{-(C_{\rm in}-2)/2}.
		\end{align*}
		Let $p_r:=\PP(X\in\cX_r)$. By the definition of $G_r$,
		\begin{align*}
			\int_{\cX_r} \exp\{-tm_r(x)\} dP_X(x) = p_r\int_0^\infty \exp\{-tu\}f_{G_r}(u)du.
		\end{align*}
		Using the uniform density bound,
		\begin{align*}
			\int_{\cX_r} \exp\{-tm_r(x)\} dP_X(x) \le p_rL\int_{0}^{\infty} \exp\{-tu\}du = p_r\frac{L}{t}.
		\end{align*}
		Summing over $r$ and using $\sum_r p_r=1$, we obtain
		\begin{align*}
			\sum_{r=1}^{\Mzero}\int_{\cX_r}\|g^{\soft}(x;\tilde{\theta})-e_r\|_2^2 dP_X(x) \le 2(\Mzero-1)\frac{L}{t}+2\Mzero\MR^{-(C_{\rm in}-2)/2}.
		\end{align*}
		Since $t=C_{\rm in} \log \MR /(4 B_{\cX}')$, the first term is of order $\Mzero/\log \MR$. The second term is polynomially small in $\MR$ because $C_{\rm in} > 2$. Hence
		\begin{align*}
			\sum_{r=1}^{\Mzero}\int_{\cX_r}\|g^{\soft}(x;\tilde{\theta})-e_r\|_2^2 dP_X(x) = O\left(\frac{\Mzero}{\log \MR}\right).
		\end{align*}
		Taking the infimum over $\theta \in \Theta_{\MR}$ completes the proof.
	\end{proof}
	Proposition~\ref{prop:supp-linear-no-margin} yields the following risk upper bound for dense softmax gating in the setting of Corollary \ref{main-cor:Top1-single}.
	
	\begin{corollary}\label{cor:supp-example-softmax-upper1}
		Assume the dense-softmax oracle assumptions in the main text and Assumption~\ref{ass:supp-margin-regularity} hold. Let $\hat{f}_0$ be the estimator produced by Algorithm 1 in the main text. Then
		\begin{align*}
			\EE\|\hat{f}_0-f_0\|_{L_2(P_X)}^2 \le C\left( \sum_{r=1}^{\Mzero}\|f_0-f_r^*\|_{\cX_r}^2 + \MR (\bar a_{n_2}^2+\bar a_{n_3}^2) + \frac{\MR d}{n}\log(\MR d n) + \frac{\Mzero}{\log \MR} \right),
		\end{align*}
		where $C$ is a constant independent of $n$, $\Mzero$ and $\MR$.
	\end{corollary}
	
	\begin{proof}[Proof of Corollary~\ref{cor:supp-example-softmax-upper1}]
		We first apply the dense-softmax oracle inequality in Theorem \ref{main-thm:softmax}. This gives
		\begin{align*}
			\EE\|\hat{f}_0-f_0\|_{L_2(P_X)}^2 \le C\left(\inf_{\theta\in\Theta_{\MR}} \left\| f_0- \sum_{m=1}^{\MR}g_m^{\soft}(x;\theta)f_{m}^* \right\|_{L_2(P_X)}^2 + \MR(\bar a_{n_2}^2+\bar a_{n_3}^2) + \frac{\MR d}{n}\log(\MR d n) \right).
		\end{align*}
		It remains to bound the approximation term. For any $\theta \in \Theta_{\MR}$,
		\begin{align*}
			\left\|f_0(x)-\sum_{m=1}^{\MR}g_m^{\soft}(x;\theta)f_{m}^*(x)\right\|_{L_2(P_X)}^2 = \sum_{r=1}^{\Mzero} \left\|\left(f_0(x)-\sum_{m=1}^{\MR}g_m^{\soft}(x;\theta)f_{m}^*(x)\right)\one_{\cX_r}\right\|_{L_2(P_X)}^2.
		\end{align*}
		Using the squared triangle inequality on each region, we obtain
		\begin{align*}
			\left\|f_0-\sum_{m=1}^{\MR}g_m^{\soft}(\cdot;\theta)f_{m}^*\right\|_{L_2(P_X)}^2 \le 2\sum_{r=1}^{\Mzero}\|(f_0-f_r^*)\one_{\cX_r}\|_{L_2(P_X)}^2 + 2\sum_{r=1}^{\Mzero} \left\|\left(f_r^*-\sum_{m=1}^{\MR}g_m^{\soft}(\cdot;\theta)f_{m}^*\right)\one_{\cX_r} \right\|_{L_2(P_X)}^2.
		\end{align*}
		Taking the infimum over $\theta \in \Theta_{\MR}$, we get
		\begin{align*}
			\inf_{\theta \in \Theta_{\MR}} \left\|f_0-\sum_{m=1}^{\MR}g_m^{\soft}(\cdot;\theta)f_{m}^*\right\|_{L_2(P_X)}^2
			&\le 2\sum_{r=1}^{\Mzero}\|(f_0-f_r^*)\one_{\cX_r}\|_{L_2(P_X)}^2 \\
			&\quad+ 2 \inf_{\theta \in \Theta_{\MR}} \sum_{r=1}^{\Mzero} \left\|\left(f_r^*-\sum_{m=1}^{\MR}g_m^{\soft}(\cdot;\theta)f_{m}^*\right)\one_{\cX_r} \right\|_{L_2(P_X)}^2.
		\end{align*}
		We now control the second term. Since the expert functions are uniformly bounded by $A$, for $x \in \cX_r$,
		\begin{align*}
			\left|f_r^*(x)-\sum_{m=1}^{\MR}g_m^{\soft}(x;\theta)f_{m}^*(x)\right|
			&= \left| (1-g_r^{\soft}(x;\theta))f_r^*(x) - \sum_{m\ne r}g_m^{\soft}(x;\theta)f_m^*(x) \right| \\
			&\le (1-g_r^{\soft}(x;\theta))|f_r^*(x)| + \sum_{m\ne r}g_m^{\soft}(x;\theta)|f_m^*(x)| \\
			&\le 2A(1-g_r^{\soft}(x;\theta)).
		\end{align*}
		Therefore,
		\begin{align*}
			\sum_{r=1}^{\Mzero} \left\|\left(f_r^*-\sum_{m=1}^{\MR}g_m^{\soft}(\cdot;\theta)f_{m}^*\right)\one_{\cX_r} \right\|_{L_2(P_X)}^2 \le 4A^2 \sum_{r=1}^{\Mzero} \|(e_r-g^{\soft}(\cdot;\theta))\one_{\cX_r}\|_{L_2(P_X)}^2,
		\end{align*}
		where we used $(1 - g_r^{\soft}(x;\theta))^2 \le \|e_r - g^{\soft}(x;\theta)\|_2^2$. Combining the preceding displays and applying Proposition \ref{prop:supp-linear-no-margin} gives
		\begin{align*}
			\inf_{\theta\in\Theta_{\MR}} \left\|f_0(x)-\sum_{m=1}^{\MR}g_m^{\soft}(x;\theta)f_{m}^*(x)\right\|_{L_2(P_X)}^2 \le C \left( \sum_{r=1}^{\Mzero}\|f_0-f_r^*\|_{\cX_r}^2 + \frac{\Mzero}{\log \MR} \right).
		\end{align*}
		Substituting this bound into the dense-softmax oracle inequality yields the desired result.
	\end{proof}

	\subsubsection{Dense Softmax Upper Bound after Proposition \ref{main-prop: Top-1 mixing}}
	\label{subsubsec:supp-dense-softmax-upper-bound-after-prop-4-6}
	
	After relabeling the routed experts, assume that for each $r\in[\Mzero]$, the best two-expert convex combination on region $\cX_r$ is $w_{r1}f_{2r-1}^*+w_{r2}f_{2r}^*$. This is only a notational simplification of the general two-expert pair in the main text. Denote $\mathbf{w}_r:=(0,\ldots,0,w_{r1},w_{r2},0,\ldots,0)$, where $w_{r1}$ is on the $(2r-1)$-th entry and $w_{r2}$ is on the $2r$-th entry. We next give the corresponding dense-softmax approximation bound analogue to Proposition~\ref{prop:supp-linear-no-margin}.
	
	\begin{proposition}\label{prop:supp-linear-top2-no-margin}
		Under the linearly separable partition condition above and Assumption~\ref{ass:supp-margin-regularity}, further assume that $w_{ri}$ are bounded away from $0$ in the sense that $\max_{r \in [\MR],i=1,2}|\log w_{ri}|<C_w$ for some constant $C_w<\infty$. Then, for $\MR\ge 2\Mzero$,
		\begin{align*}
			\inf_{\theta\in\Theta_{\MR}} \sum_{r=1}^{\Mzero}\int_{\cX_r}\|g^{\soft}(x;\theta)-\mathbf{w}_r\|_2^2 dP_X(x) = O\left(\frac{\Mzero}{\log \MR}\right).
		\end{align*}
		In particular, when $\Mzero$ is fixed, this bound is of order $O(1/\log \MR)$.
	\end{proposition}
	
	\begin{proof}[Proof of Proposition~\ref{prop:supp-linear-top2-no-margin}]
		As in the proof of Proposition \ref{prop:supp-linear-no-margin}, we may relabel the routed experts and use the specialist indexed by $2 \Mzero$ as the reference expert. Thus, in the proof we impose the reference normalization $\theta_{2\Mzero}=0$. Let
		\begin{align*}
			B_{\cX}'=1\vee\max_{x\in\cX}\|\bar{x}\|_2, \qquad \bar{x}=(x^\top,1)^\top,
		\end{align*}
		and choose
		\begin{align*}
			t=\frac{C_{\rm in}\log \MR}{4B_{\cX}'}.
		\end{align*}
		For $r = 1, \ldots, \Mzero$, define $c_r=\log w_{r1}-\log w_{r2}$. Then $|c_r| \le 2 C_w$. Define the candidate parameters by
		\begin{align*}
			\tilde{\theta}_{2r-1} = t(v_r-v_{\Mzero})+(0,\ldots,0,c_r)^\top, \qquad \tilde{\theta}_{2r} = t(v_r-v_{\Mzero}), \qquad r=1,\dots,\Mzero,
		\end{align*}
		and set $\tilde{\theta}_j=(0,\ldots,0,-C_{\rm in}\log \MR)^\top$ for $j>2\Mzero$. The boundedness of $c_r$, the normalization $\|v_r\| \le 1$, and the logarithmic inner-box condition imply that $\tilde{\theta} \in \Theta_{\MR}$ for all sufficiently large $\MR$. Finitely many smaller values of $\MR$ are absorbed into the constant in the final bound.
		
		Fix $r\in[\Mzero]$ and $x\in\cX_r$. Write
		\begin{align*}
			a_i(x)=v_i^\top\bar{x}, \quad i\in[\Mzero], \qquad a_{(-r)}(x)=\max_{j\ne r} v_j^\top\bar{x}.
		\end{align*}
		Then $m_r(x)=a_r(x)-a_{(-r)}(x)\ge 0$. Let $\ell_j(x):=\tilde{\theta}_j^\top\bar{x}$. For the two experts assigned to region $\cX_r$, $\ell_j(x)=-C_{\rm in}\log \MR, j>2\Mzero$. Define the relative contribution of all experts outside the pair $\{2r-1,2r\}$ by
		\begin{align*}
			D_r(x)=\sum_{j\notin\{2r-1,2r\}}\exp\{\ell_j(x)-\ell_{2r}(x)\}.
		\end{align*}
		Then
		\begin{align*}
			g_{2r-1}^{\soft}(x;\tilde{\theta}) = \frac{\exp\{c_r\}}{1+\exp\{c_r\}+D_r(x)}, \qquad g_{2r}^{\soft}(x;\tilde{\theta}) = \frac{1}{1+\exp\{c_r\}+D_r(x)}.
		\end{align*}
		We next bound $D_r(x)$. For $q \ne r$, the margin condition gives $a_q(x)-a_r(x) \le -m_r(x)$. Therefore, the contribution from the non-target region pairs satisfies
		\begin{align*}
			\sum_{q \ne r} \left\{\exp\{\ell_{2q-1}(x) - \ell_{2r}(x)\} + \exp\{\ell_{2q}(x) - \ell_{2r}(x)\}\right\} \le C(\Mzero-1)\exp\{-tm_r(x)\},
		\end{align*}
		where $C$ depends only on the weight lower bound through $C_w$. For the dummy experts, using $|a_{\Mzero}(x)-a_r(x)| \le \|v_{\Mzero} - v_r\|_2 \|\bar{x}\|_2 \le 2 B_{\cX}'$, we obtain
		\begin{align*}
			\sum_{j>2\Mzero}\exp\{-C_{\rm in}\log \MR+t(a_{\Mzero}(x)-a_r(x))\} \le \MR^{-(C_{\rm in}-2)/2}.
		\end{align*}
		Hence
		\begin{align*}
			D_r(x) \le C(\Mzero-1)\exp\{-tm_r(x)\}+\MR^{-(C_{\rm in}-2)/2}.
		\end{align*}
		Since $\exp\{c_r\}/(1+\exp\{c_r\})=w_{r1}$ and $1/(1+\exp\{c_r\})=w_{r2}$, the two target softmax weights have the same relative proportions as $(w_{r1}, w_{r2})$. Moreover, $g_{2r-1}^{\soft}(x;\tilde{\theta})+g_{2r}^{\soft}(x;\tilde{\theta}) = \frac{1+\exp\{c_r\}}{1 + \exp\{c_r\} + D_r(x)}$, so
		\begin{align*}
			1-g_{2r-1}^{\soft}(x;\tilde{\theta})-g_{2r}^{\soft}(x;\tilde{\theta}) = \frac{D_r(x)}{1+\exp\{c_r\}+D_r(x)} \le D_r(x).
		\end{align*}
		Consequently
		\begin{align*}
			\|g^{\soft}(x;\tilde{\theta})-\mathbf{w}_r\|_2^2 = (w_{r1}-g_{2r-1}^{\soft}(x;\tilde{\theta}))^2 + (w_{r2}-g_{2r}^{\soft}(x;\tilde{\theta}))^2 + \sum_{j\notin\{2r-1,2r\}} g_j^{\soft}(x;\tilde{\theta})^2.
		\end{align*}
		Because the first two softmax weights are proportional to $(w_{r1}, w_{r2})$, their deviation from $(w_{r1}, w_{r2})$ is controlled by the missing mass assigned outside the pair. Thus
		\begin{align*}
			\|g^{\soft}(x;\tilde{\theta})-\mathbf{w}_r\|_2^2 \le 2(1-g_{2r-1}^{\soft}(x;\tilde{\theta})-g_{2r}^{\soft}(x;\tilde{\theta})) \le 2D_r(x).
		\end{align*}
		Combining this with the preceding bound on $D_r(x)$, we obtain the pointwise inequality
		\begin{align*}
			\|g^{\soft}(x;\tilde{\theta})-\mathbf{w}_r\|_2^2 \le C(\Mzero-1)\exp\{-tm_r(x)\}+2\MR^{-(C_{\rm in}-2)/2}.
		\end{align*}
		Integrating over $\cX_r$ and summing over $r$ gives
		\begin{align*}
			\sum_{r=1}^{\Mzero}\int_{\cX_r}\|g^{\soft}(x;\tilde{\theta})-\mathbf{w}_r\|_2^2 dP_X(x) \le C(\Mzero-1)\sum_{r=1}^{\Mzero}\int_{\cX_r} \exp\{-tm_r(x)\} dP_X(x) + 2\Mzero\MR^{-(C_{\rm in}-2)/2}.
		\end{align*}
		The same margin-distribution calculation as in the proof of Proposition~\ref{prop:supp-linear-no-margin} gives
		\begin{align*}
			\sum_{r=1}^{\Mzero}\int_{\cX_r}\|g^{\soft}(x;\tilde{\theta})-\mathbf{w}_r\|_2^2 dP_X(x) \le C(\Mzero-1)\frac{L}{t}+2\Mzero\MR^{-(C_{\rm in}-2)/2}.
		\end{align*}
		Since $t \asymp \log \MR$ and $C_{\rm in} > 2$, the right-hand side is of order $O(\Mzero/\log \MR)$. Taking the infimum over $\theta \in \Theta_{\MR}$ completing the proof.
	\end{proof}
	
	This yields the following risk upper bound for dense softmax gating in the setting of Corollary \ref{main-cor:kernel-interior-main}.
	
	\begin{corollary}\label{cor:supp-example-softmax-upper2}
		Assume the dense-softmax oracle assumptions in the main text and Assumption~\ref{ass:supp-margin-regularity} hold. Assume further that $w_{r1} + w_{r2} = 1, w_{ri} > 0$, and the weights are uniformly bounded away from zero in the sense that $\max_{r,i}|\log w_{ri}|<C_w$ for some constant $C_w<\infty$. Let $\hat{f}_0$ denote the estimator produced by Algorithm 1 in the main text. Then
		\begin{align*}
			\EE\|\hat{f}_0-f_0\|_{L_2(P_X)}^2 \le C\left(\sum_{r=1}^{\Mzero} \|f_0-w_{r1}f_{r,1}^*-w_{r2}f_{r,2}^*\|_{\cX_r}^2 + \MR (\bar a_{n_2}^2+\bar a_{n_3}^2) + \frac{\MR d}{n}\log(\MR d n) + \frac{\Mzero}{\log \MR} \right),
		\end{align*}
		where $C$ is a constant independent of $n$, $\Mzero$ and $\MR$.
	\end{corollary}
	
	\begin{proof}[Proof of Corollary~\ref{cor:supp-example-softmax-upper2}]
		We first apply the dense-softmax oracle inequality in Theorem \ref{main-thm:softmax}. This gives
		\begin{align*}
			\EE\|\hat{f}_0-f_0\|_{L_2(P_X)}^2 \le C\left( \inf_{\theta\in\Theta_{\MR}} \left\| f_0- \sum_{m=1}^{\MR}g_m^{\soft}(x;\theta)f_m^* \right\|_{L_2(P_X)}^2 + \MR (\bar a_{n_2}^2+\bar a_{n_3}^2) + \frac{\MR d}{n}\log(\MR d n) \right).
		\end{align*}
		It remains to bound the approximation term. For any $\theta \in \Theta_{\MR}$,
		\begin{align*}
			\left\|f_0-\sum_{m=1}^{\MR}g_m^{\soft}(\cdot;\theta)f_m^*\right\|_{L_2(P_X)}^2 = \sum_{r=1}^{\Mzero}\left\|\left(f_0-\sum_{m=1}^{\MR}g_m^{\soft}(\cdot;\theta)f_m^*\right)\one_{\cX_r}\right\|_{L_2(P_X)}^2.
		\end{align*}
		Using the squared inequality on each region, we obtain
		\begin{align*}
			\left\|f_0-\sum_{m=1}^{\MR}g_m^{\soft}(\cdot;\theta)f_m^*\right\|_{L_2(P_X)}^2
			&\le C\sum_{r=1}^{\Mzero}\|(f_0-w_{r1}f_{r,1}^*-w_{r2}f_{r,2}^*)\one_{\cX_r}\|_{L_2(P_X)}^2 \\
			&\quad + C\sum_{r=1}^{\Mzero}\left\|\left(w_{r1}f_{r,1}^* + w_{r2}f_{r,2}^* - \sum_{m=1}^{\MR}g_m^{\soft}(x;\theta)f_m^*\right)\one_{\cX_r} \right\|_{L_2(P_X)}^2.
		\end{align*}
		Taking the infimum over $\theta \in \Theta_{\MR}$, it remains to control the second term. Since the expert functions are uniformly bounded by $A$, for $x \in \cX_r$,
		\begin{align*}
			\left| w_{r1}f_{r,1}^*(x) + w_{r2}f_{r,2}^*(x) - \sum_{m=1}^{\MR}g_m^{\soft}(x;\theta)f_m^*(x)\right|
			&= \left| (w_{r1}-g_{2r-1}^{\soft}(x;\theta))f_{r,1}^*(x) + (w_{r2}-g_{2r}^{\soft}(x;\theta))f_{r,2}^*(x) - \sum_{m\notin\{2r-1,2r\}}g_m^{\soft}(x;\theta)f_m^*(x) \right| \\
			&\le A\left( |w_{r1}-g_{2r-1}^{\soft}(x;\theta)| + |w_{r2}-g_{2r}^{\soft}(x;\theta)| + 1-g_{2r-1}^{\soft}(x;\theta)-g_{2r}^{\soft}(x;\theta) \right) \\
			&\le 2A\left( |w_{r1}-g_{2r-1}^{\soft}(x;\theta)| + |w_{r2}-g_{2r}^{\soft}(x;\theta)| \right),
		\end{align*}
		where the last inequality holds because $w_{r1} + w_{r2} = 1$. Therefore, using $(a+b)^2 \le 2(a^2 + b^2)$,
		\begin{align*}
			\sum_{r=1}^{\Mzero} \left\| \left(w_{r1}f_{r,1}^* + w_{r2}f_{r,2}^* - \sum_{m=1}^{\MR}g_m^{\soft}(\cdot;\theta)f_m^* \right)\one_{\cX_r} \right\|_{L_2(P_X)}^2 \le 8A^2\sum_{r=1}^{\Mzero} \|(\mathbf{w}_r-g^{\soft}(\cdot;\theta))\one_{\cX_r}\|_{L_2(P_X)}^2.
		\end{align*}
		Combining the preceding displays and apply Proposition \ref{prop:supp-linear-top2-no-margin} gives
		\begin{align*}
			\inf_{\theta\in\Theta_{\MR}} \left\|f_0(x)-\sum_{m=1}^{\MR}g_m^{\soft}(\cdot;\theta)f_{m}^*\right\|_{L_2(P_X)}^2 \le C \left( \sum_{r=1}^{\Mzero}\|f_0-w_{r1}f_{r,1}^* - w_{r2}f_{r,2}^*\|_{\cX_r}^2 + \frac{\Mzero}{\log \MR} \right).
		\end{align*}
		Substituting this bound into the dense-softmax oracle inequality yields the desired result.
	\end{proof}
	
	\section{Simulation Details}
	\label{sec:supp-simulation-details}
	
	\subsection{Comparison of different gating classes}
	\label{subsec:supp-comparison-different-gating-class}
	
	We conduct Monte Carlo simulations to illustrate how the performance of different gating classes depends on the geometry of the underlying routing regions. The goal is not to provide a comprehensive benchmark comparison, but rather to examine whether a gating class performs best when its geometric structure is aligned with that of the true partition.
	
	We consider three data-generating processes (DGPs) corresponding to region partitions with linear, quadratic and highly nonlinear boundaries. These designs form a hierarchy of geometric complexity and allow us to assess how different gating classes adapt to different partition structures. Throughout, the expert functions are fixed and treated as known at the estimation stage, so that differences across methods primarily reflect the ability of the gating mechanism to approximate the true routing structure.
	
	For each design, we compare three candidate gating specifications: linear gating, quadratic gating, and Gaussian-kernel-based gating, as defined in the main text. Observations are generated according to
	\begin{align}
		y_i = \sum_{m=1}^3 g_m^*(x_i) f_m(x_i) + \varepsilon_i, \qquad i=1,\dots,n, \label{eq:supp-linear-dgp}
	\end{align}
	where $\varepsilon_i\sim N(0,\sigma^2)$ and $x_i=(x_{i1},x_{i2})^\top$. In all designs, the covariates $x_{i1}$ and $x_{i2}$ are generated independently from the uniform distribution on $[-1,1]$. The gating function follows the softmax form
	\begin{align}
		g_m^*(x) = \frac{\exp\{\eta_m^*(x)\}}{\sum_{\ell=1}^3 \exp\{\eta_\ell^*(x)\}}, \qquad m=1,2,3, \label{eq:supp-true-quadratic-gating}
	\end{align}
	where the latent scores $\eta_m^*(x)$ vary across the three designs and induce expert-dominance partitions of differing geometric complexity.
	
	\textbf{Linear DGP.}
	We first consider a design in which the true routing regions are induced by linear scores. The latent scores are given by
	\begin{align*}
		\eta_1^*(x)=a^*x_1, \qquad \eta_2^*(x)=a^*x_2, \qquad \eta_3^*(x)=0,
	\end{align*}
	where $a^*=2$. Under this construction, the third expert serves as a baseline component, and the resulting expert-dominance regions are separated by linear decision boundaries. The expert functions are fixed as
	\begin{align*}
		f_1(x)=2+2x_1-x_2, \qquad f_2(x)=-1-x_1+2x_2, \qquad f_3(x)=1+\sin(\pi x_1)-1.5x_2^2.
	\end{align*}
	We set $n\in\{200,400\}$ and $\sigma=0.5$.
	
	\textbf{Quadratic DGP.}
	The second design considers routing regions generated by quadratic scores:
	\begin{align*}
		\eta_1^*(x)=2x_1-1.2x_1^2-0.8x_2^2, \qquad
		\eta_2^*(x)=2x_2-0.8x_1^2-1.2x_2^2, \qquad
		\eta_3^*(x)=0.
	\end{align*}
	The third expert again serves as the baseline component. Compared with the linear DGP, the quadratic terms generate curved decision boundaries, so that the resulting dominance regions are no longer linearly separable. This setting therefore allows us to examine the performance loss from using a geometrically misspecified linear gate.
	
	The expert functions are the same as those in the linear DGP:
	\begin{align*}
		f_1(x)=2+2x_1-x_2, \qquad f_2(x)=-1-x_1+2x_2, \qquad f_3(x)=1+\sin(\pi x_1)-1.5x_2^2.
	\end{align*}
	We again consider sample sizes $n\in\{200,400\}$ and set $\sigma=0.5$.
	
	\textbf{Nonlinear DGP.}
	The third design considers routing regions with highly nonlinear boundaries. Unlike the previous two settings, the true gating structure is generated through a nonlinear deformation of angular regions. This construction cannot be represented exactly by either linear or quadratic gating specifications and therefore provides a setting in which more flexible gates may be beneficial. Specifically, let $\theta=\operatorname{atan2}(x_2,x_1)$ and $r=(x_1^2+x_2^2)^{1/2}$ denote the polar angle and radial coordinate, and define the deformed angular coordinate $\widetilde{\theta}=\theta+0.85\sin(4.5r+1.8\theta)+0.35\sin(4\theta)$. We then set $\eta_m^*(x)=\tau\cos(\widetilde{\theta}-\phi_m)$, with $\tau=7$ and angular centers $\phi_1=\pi$, $\phi_2=\pi/3$, and $\phi_3=-\pi/3$. Under this construction, the three experts dominate different regions of the covariate space, while the boundaries between regions become highly irregular due to the angular deformation.
	
	The expert functions are specified as $f_1(x)=3+3x_1-2x_2$, $f_2(x)=-2-2x_1+3x_2$, and $f_3(x)=2.5\sin(\pi x_1x_2)-3x_1+2x_2$. We set $n\in\{500,1000\}$ and $\sigma=0.3$.
	
	\textbf{Implementation of gating estimators.}
	The implementation details differ slightly across the three designs. In the linear DGP, the linear gating parameters are searched over the common grid $\{0,0.5,1,\ldots,3\}$, and the quadratic gating parameters over $\{0,0.6,1.2,\ldots,3\}$. The Gaussian-kernel-based gating uses a regular $3\times3$ grid of kernel centers on $[-1,1]^2$ with bandwidth $\sigma_K=\kappa h$, where $\kappa=0.75$ and $h=1$ is the spacing between neighboring grid points.
	
	In the quadratic DGP, the linear and quadratic grids are centered around the true quadratic specification. For linear gating, the four coefficients are searched over $\{-2.5,-1.25,0,1.25,2.5\}$, $\{-1,0,1\}$, $\{-1,0,1\}$, and $\{-2.5,-1.25,0,1.25,2.5\}$, respectively. For quadratic gating, the linear coefficients are searched over the same corresponding grids, while each quadratic coefficient is searched over $\{-1.6,-1.2,-0.8,0\}$. The Gaussian-kernel-based gating again uses a $3\times3$ grid with $\kappa=0.75$.
	
	In the nonlinear DGP, the linear gating parameters are searched over $\{-3,-2,\ldots,3\}$, and the quadratic gating parameters over $\{-2,-1,0,1,2\}$. The Gaussian-kernel-based gating uses a denser $4\times4$ grid of kernel centers on $[-1,1]^2$ with $\kappa=0.5$.
	
	\textbf{Monte Carlo evaluation.}
	We conduct 100 Monte Carlo replications. For each replication, let $\hat g(x_i)=(\hat g_1(x_i),\ldots,\hat g_3(x_i))^\top$ and $g^*(x_i)=(g_1^*(x_i),\ldots,g_3^*(x_i))^\top$ denote the estimated and true gating-weight vectors at observation $i$, respectively. We evaluate gating-weight estimation using the mean $\ell_1$ error and mean squared $\ell_2$ error, defined as
	\begin{align*}
		\frac{1}{n}\sum_{i=1}^n\|\hat g(x_i)-g^*(x_i)\|_1 \qquad\text{and}\qquad \frac{1}{n}\sum_{i=1}^n\|\hat g(x_i)-g^*(x_i)\|_2^2,
	\end{align*}
	respectively. For each gating class, we report the Monte Carlo mean and standard deviation of both error measures.
	
	In addition to numerical summaries, Figure~\ref{fig:all-gating-results} compares the true and estimated expert-dominance regions, defined by $\arg\max_m g_m^*(x)$ and $\arg\max_m\hat g_m(x)$, respectively, using the first Monte Carlo replication. These visualizations illustrate how different gating classes adapt to the geometry of the underlying partition.
	
	\begin{figure}[!ht]
		\centering

		\begin{subfigure}{0.32\textwidth}
			\centering
			\includegraphics[width=\textwidth]{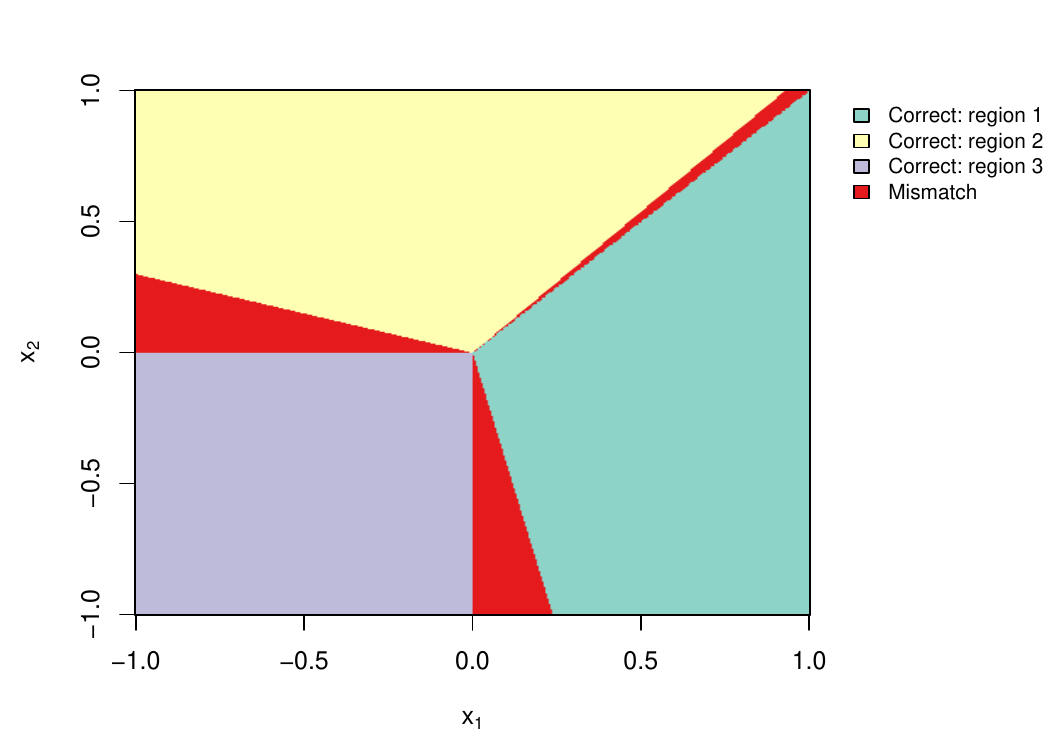}
			\caption{Linear gating}
		\end{subfigure}
		\hfill
		\begin{subfigure}{0.32\textwidth}
			\centering
			\includegraphics[width=\textwidth]{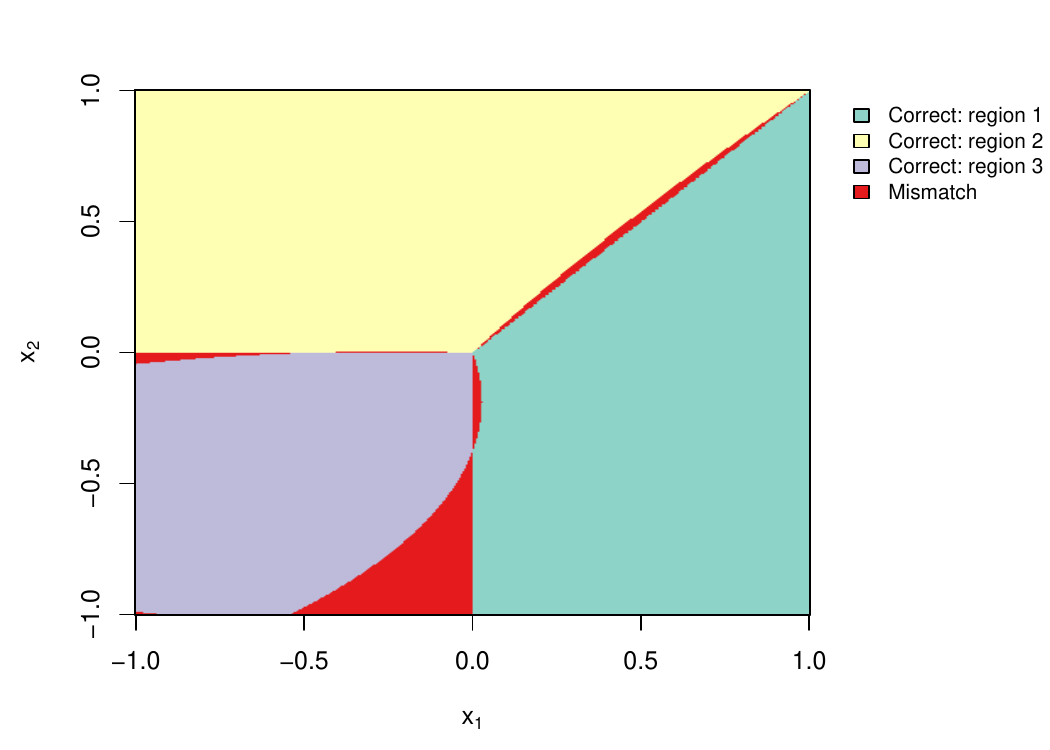}
			\caption{Quadratic gating}
		\end{subfigure}
		\hfill
		\begin{subfigure}{0.32\textwidth}
			\centering
			\includegraphics[width=\textwidth]{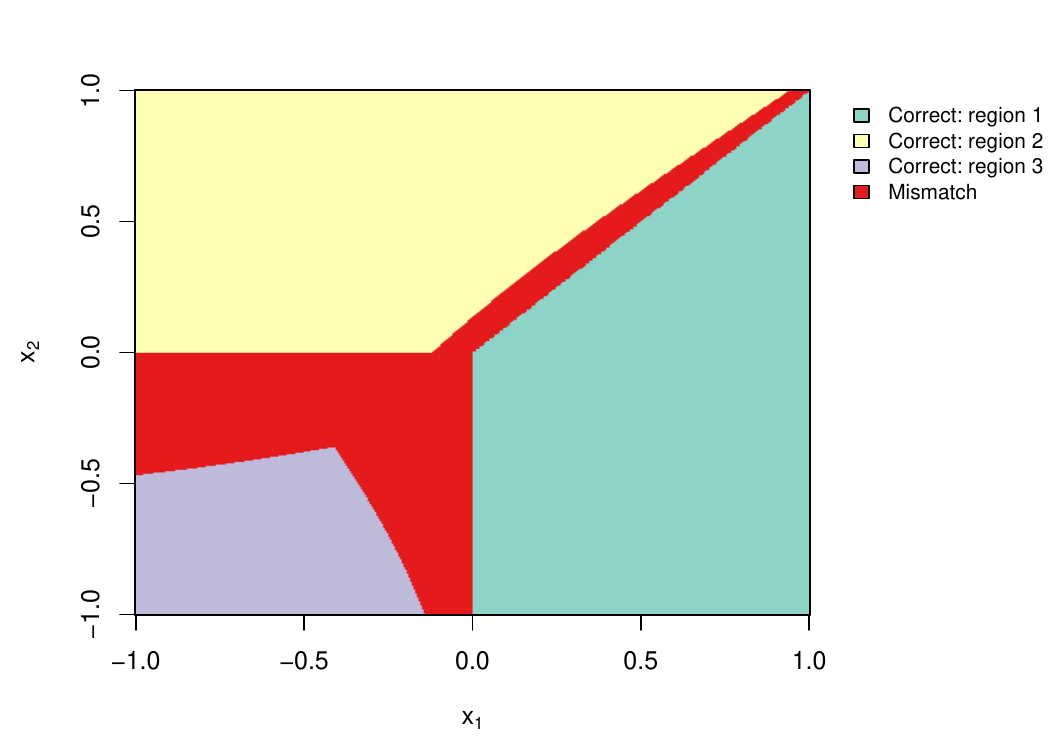}
			\caption{GK-based gating}
		\end{subfigure}
		
		\vspace{0.5cm}

		\begin{subfigure}{0.32\textwidth}
			\centering
			\includegraphics[width=\textwidth]{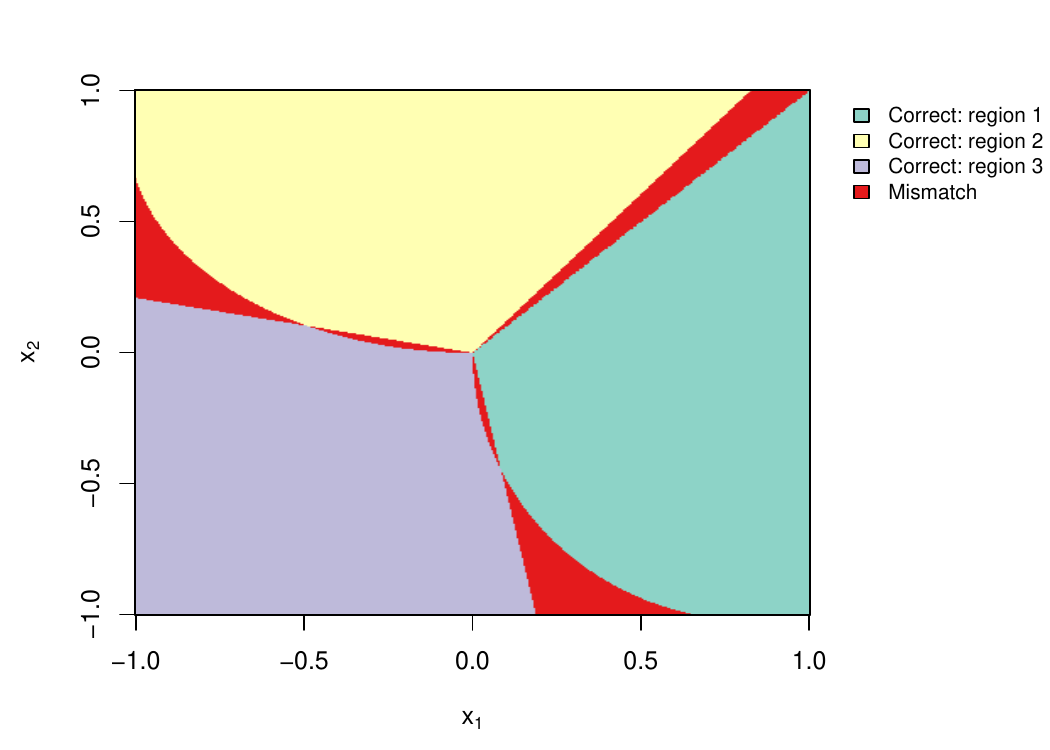}
			\caption{Linear gating}
		\end{subfigure}
		\hfill
		\begin{subfigure}{0.32\textwidth}
			\centering
			\includegraphics[width=\textwidth]{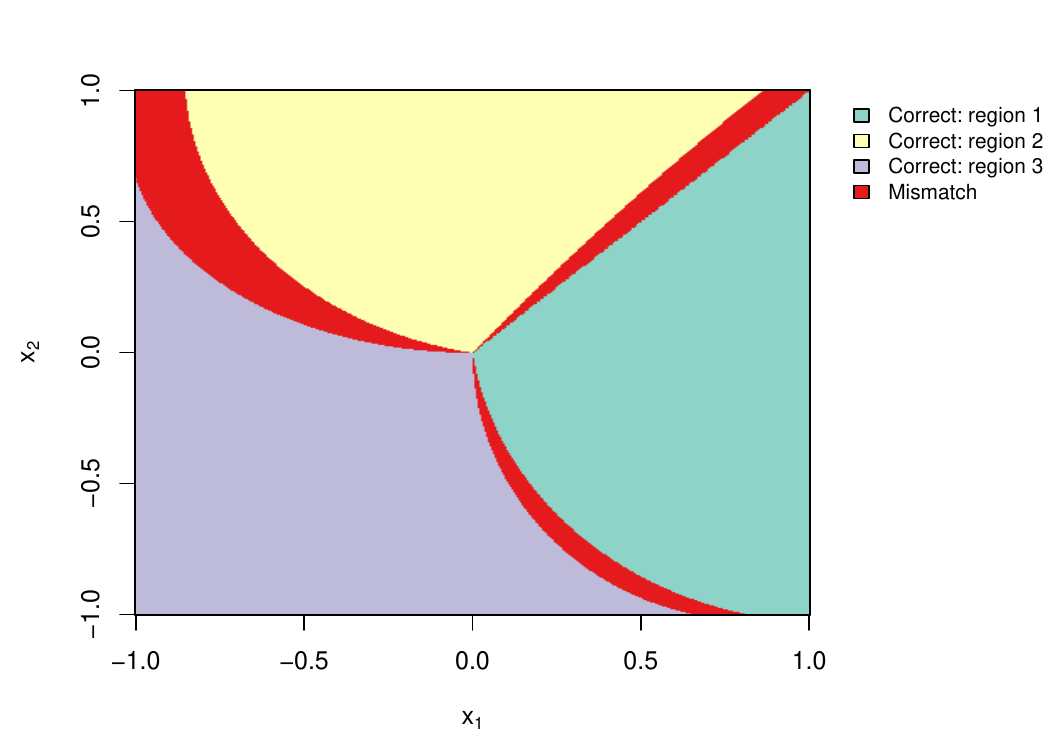}
			\caption{Quadratic gating}
		\end{subfigure}
		\hfill
		\begin{subfigure}{0.32\textwidth}
			\centering
			\includegraphics[width=\textwidth]{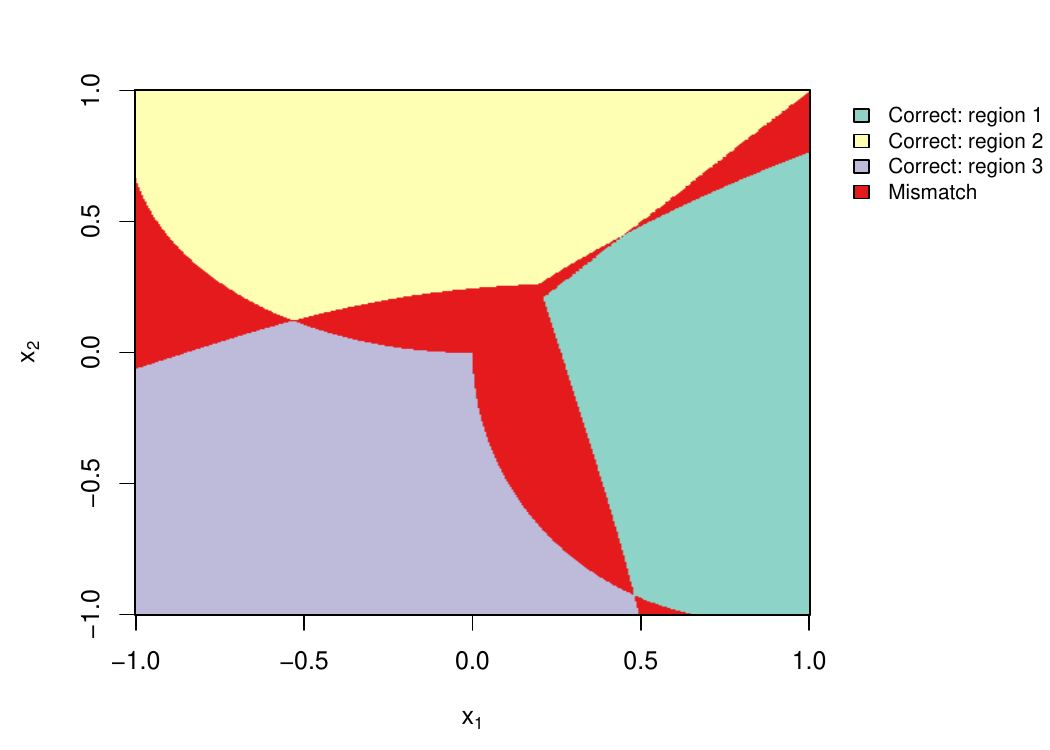}
			\caption{GK-based gating}
		\end{subfigure}
		
		\vspace{0.5cm}

		\begin{subfigure}{0.32\textwidth}
			\centering
			\includegraphics[width=\textwidth]{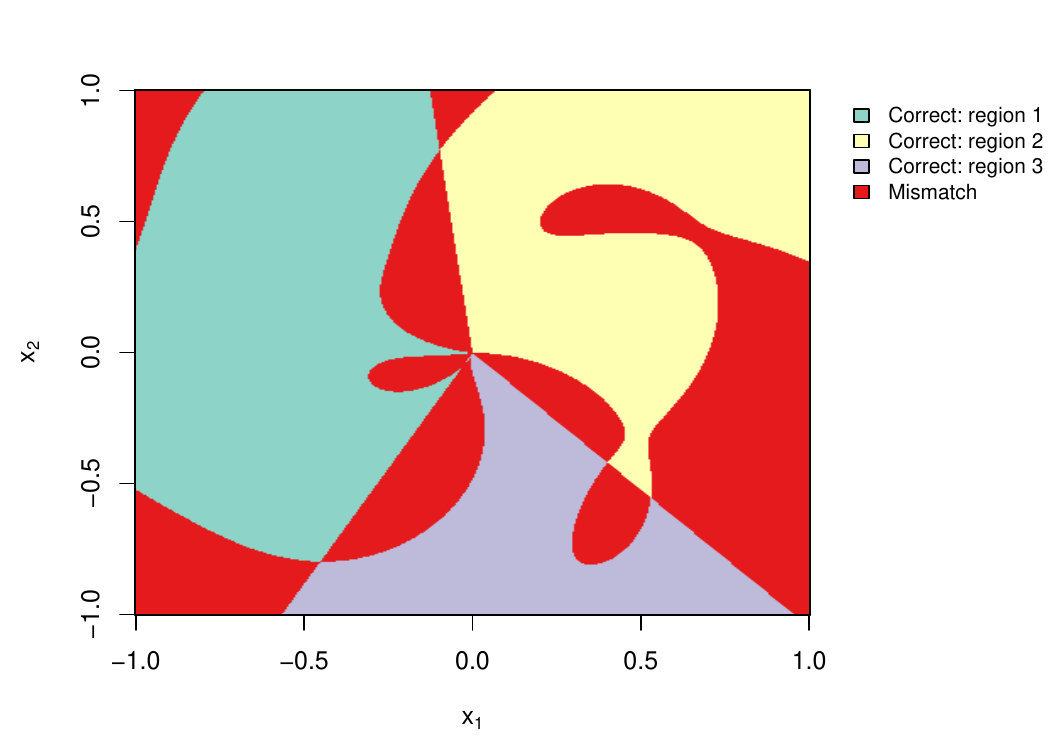}
			\caption{Linear gating}
		\end{subfigure}
		\hfill
		\begin{subfigure}{0.32\textwidth}
			\centering
			\includegraphics[width=\textwidth]{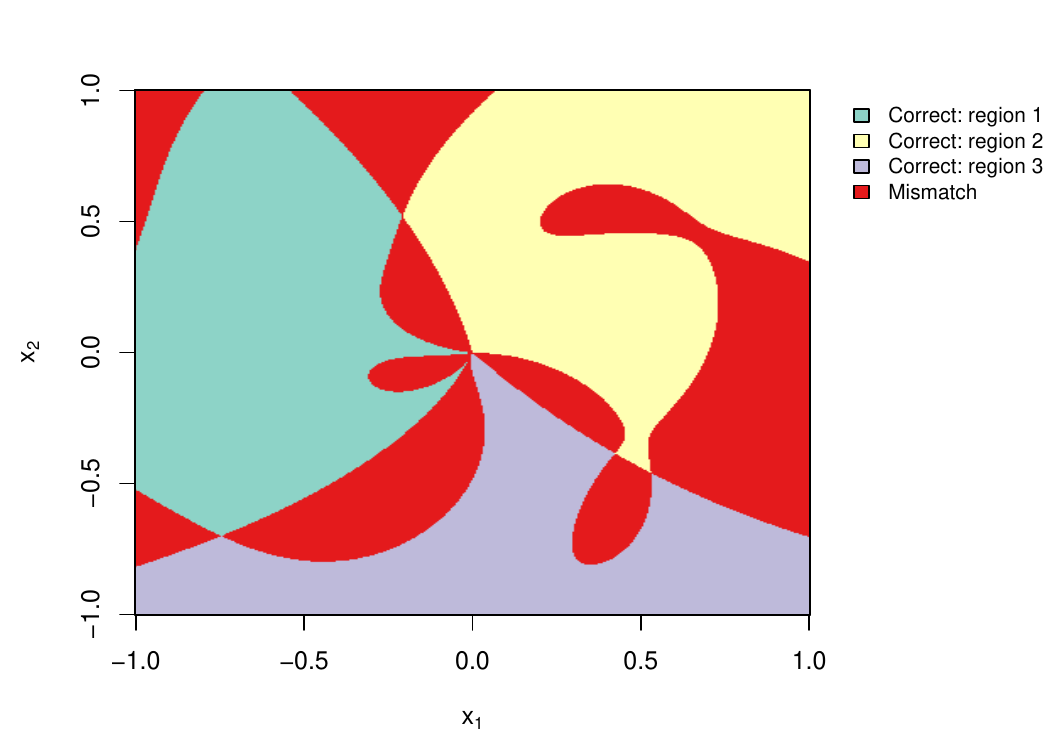}
			\caption{Quadratic gating}
		\end{subfigure}
		\hfill
		\begin{subfigure}{0.32\textwidth}
			\centering
			\includegraphics[width=\textwidth]{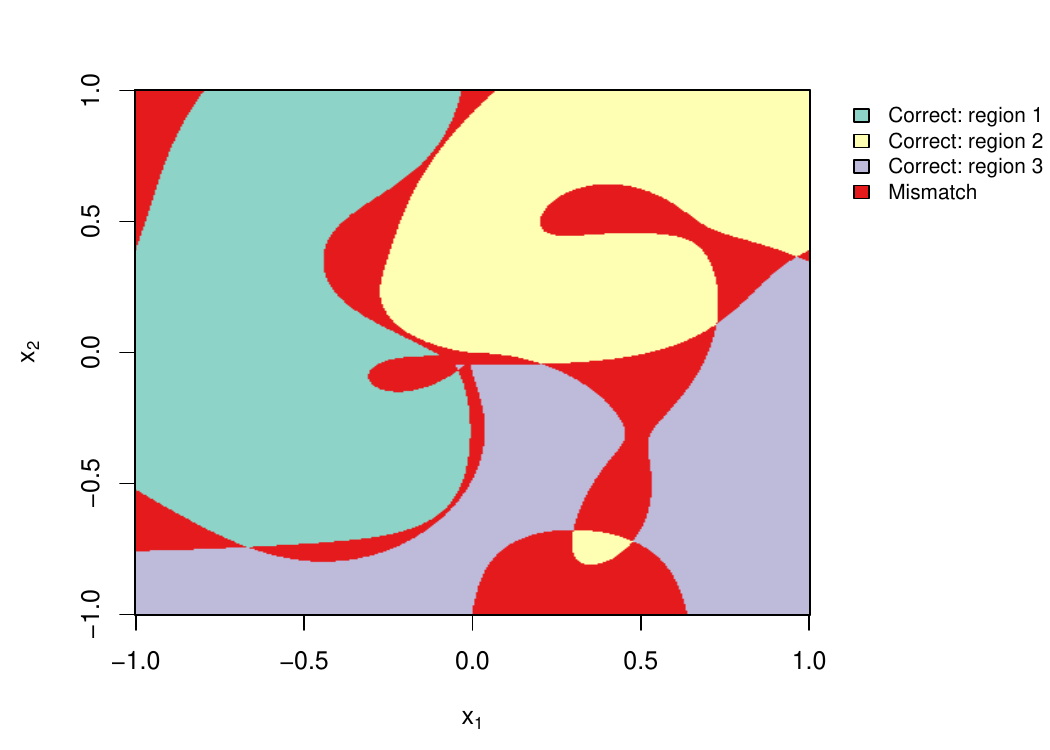}
			\caption{GK-based gating}
		\end{subfigure}
		\caption{
			Comparison of estimated routing regions under different expert-induced partition geometries and gating classes. Rows correspond to the linear, quadratic, and nonlinear DGPs, while columns correspond to the linear, quadratic, and Gaussian-kernel-based gating. The results illustrate the importance of matching the geometric flexibility of the gating class to the complexity of the expert-induced partition.
		}
		\label{fig:all-gating-results}
	\end{figure}
	
	\subsection{MSE Comparison With and Without a Shared Expert}
	\label{subsec:supp-mse-comparison-shared-expert}
	
	We conduct two simulation studies to examine whether a shared-routed representation improves prediction accuracy relative to a pure-routed representation when the routing threshold is unknown. Both studies are based on the Fourier-sieve illustration in Example~\ref{main-ex:fourier-x-sin-cos}, but emphasize two complementary regimes: a simple shared component with complex routed components, and a complex shared component with simple routed components.
	
	In both designs, we generate $X\sim\Unif[0,1]$ and $Y=f_0(X)+\varepsilon$, where $\varepsilon\sim N(0,\sigma^2)$, $\sigma=0.5$, and the true threshold is $c_0=1/2$.
	
	\textbf{Simple shared, complex routed.}
	The first design sets
	\begin{align*}
		f_0(x)=\theta x+\one\{x\le c_0\}f_{R,1}(2x)+\one\{x>c_0\}f_{R,2}(2x-1),
	\end{align*}
	with $\theta=1$. The routed components are Fourier series of order $K_0=3$ without intercepts,
	\begin{align*}
		f_{R,j}(t)=\sum_{k=1}^{3}\left\{a_{j,k}\sin(2\pi kt)+b_{j,k}\cos(2\pi kt)\right\}, \qquad j=1,2,
	\end{align*}
	where $a_1=(1.0,0.6,0.4)$, $b_1=(0.7,-0.5,0.3)$, $a_2=(-0.8,0.5,-0.3)$, and $b_2=(0.6,0.4,-0.5)$.
	
	\textbf{Complex shared, simple routed.}
	The second design sets
	\begin{align*}
		f_0(x)=f_S(x)+\one\{x\le c_0\}f_{R,1}(2x)+\one\{x>c_0\}f_{R,2}(2x-1),
	\end{align*}
	where $f_S(x)=A\sin(2\pi qx)$ with $A=1$ and $q=3$. The routed components are simple linear functions with opposite slopes, $f_{R,1}(t)=\delta(t-1/2)$ and $f_{R,2}(t)=-\delta(t-1/2)$ for $t\in[0,1]$, where $\delta=0.6$.
	
	For each design, we compare a shared-routed estimator with a pure-routed estimator. In the first design, the shared-routed estimator includes a linear shared component and routed Fourier sieves of order $K=3$ with intercepts; the pure-routed estimator uses the same routed Fourier sieves but omits the shared component. In the second design, the shared-routed estimator includes a Fourier shared component of order $Q=3$ and routed linear experts in $\operatorname{span}\{1,t\}$; the pure-routed estimator uses the same routed linear experts but omits the shared component.
	
	For both estimators and both designs, the threshold $c$ is unknown and is estimated by profile least squares over a 200-point grid on $[0.1,0.9]$. Conditional on each candidate threshold, all linear coefficients are estimated by ordinary least squares, and the threshold minimizing the residual sum of squares is selected.
	
	For each $n\in\{200,400,800,1600\}$, the experiment is repeated 200 times. Prediction accuracy is evaluated by the integrated squared error
	\begin{align*}
		\MSE=\int_0^1\{\widehat f(x)-f_0(x)\}^2 dx,
	\end{align*}
	which is approximated on a dense grid of 5000 equally spaced points in $[0,1]$.
	
\end{document}